%% file: paper.tex
\documentclass[10pt]{report}
\usepackage[numbers]{natbib}
\pdfoutput=1

\input{header.tex}

\input{notation.tex}

\usepackage[left=1.5in, right=1in, top=1in, bottom=1.25in, includefoot]{geometry}
\usepackage{makeidx}

\makeindex

\begin{document}

\setlength{\parskip}{2mm}
\setlength{\parindent}{0pt} 
	


\parindent 0pt
\parskip 1ex
\renewcommand{\baselinestretch}{1.33}
\numberwithin{equation}{section}
\renewcommand{\bibname}{Bibliography}
\renewcommand{\contentsname}{Contents}
\pagenumbering{roman}
\bibliographystyle{unsrtnat}

\def\urlprefix{}
   \def\url#1{}

	\title{
		\Huge{\textbf{Learning\\ From An Optimization Viewpoint}} \\[1.2cm]
		\Large{	by} \\[0.4cm]
		\Large{\bf Karthik Sridharan} \\[0.4cm]
		\Large{Submitted to : }\\
		\Large{Toyota Technological Institute at Chicago}\\[-0.2cm]
		\Large{6045 S. Kenwood Ave, Chicago, IL, 60637}\\[0.4cm]
		\Large{For the degree of {Doctor of Philosophy in Computer Science}}\\[1cm]
		\Large{Thesis Committee : \\{\bf Nathan Srebro (Thesis Supervisor)}, \\{\bf David McAllester},\\ {\bf Arkadi Nemirovski}, \\ {\bf Alexander Razborov}}
	} 
	\author{} \date{}
	\maketitle


	\pdfbookmark[0]{Dedication}{dedication}
	\begin{center} 
~~~\\
~~~\\
~~~\\
~~~\\
~~~\\
\includegraphics[width=0.75\linewidth]{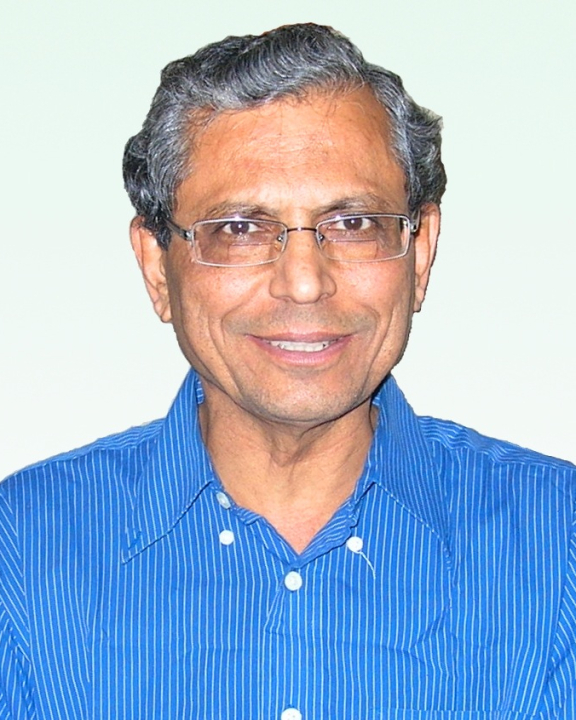}
~~~\\
~~~\\
{\Large In memory of my dear father, Raghavan Sridharan \dots}
 \end{center}

\chapter*{Acknowledgements}
\pdfbookmark[0]{Acknowledgements}{acknowledgements}
{\large 
Starting to write this acknowledgement I realized that I owed thanks to many people who have directly or indirectly influenced my research and have helped me complete this dissertation.

\vspace{0.1in} First and foremost, I would like to thank my advisor Nati Srebro for his encouragement, guidance and patience throughout my PhD. One too many times I have walked into Nati's office with vague questions and still unformed ideas and every time I walked out with crisp and formalized ideas and clear thoughts on how to approach problems. Nati was
the one who inspired me to look at learning alternatively as an optimization problem, which culminated into the central idea of this thesis. Besides research skills, I also learned (or at least started the process of learning) other skills including scientific writing, presenting my ideas succinctly and choosing the right problems to work on. Perhaps the most important
thing I learned from Nati is that asking the right questions and formalizing these questions precisely is as important as, if not more than, finding answers.

\vspace{0.1in} I would like to thank Sham Kakade with whom I worked during the first three years of my PhD on many exciting projects and learned many things from. I am greatly indebted to Shai Shalev-Shwartz; apart from being a wonderful collaborator, Shai has guided me in many ways and greatly influence my research and thinking. I express my heartfelt
gratitude to Alexander Rakhlin, Ohad Shamir and Ambuj Tewari. Besides being my close collaborators and friends, they gave me timely advice and help, and taught me many things regarding research  and other aspects.
I would like to also thank my other research collaborators who influenced and guided my research: Shai Ben-David, Kamalika Chaudhuri, Andrew Cotter, Ofer Dekel, Dean Foster, Rina Foygel, Claudio Gentile, Karen Livescue and David Loker.

\vspace{0.1in} I would like to thank my committee members David McAllester, Arkadi Nemirovski and Alexander Razborov for all their help in guiding and shaping this thesis work. I enjoyed the many insightful discussions with David; I will especially remember the ones we had
fairly early in the morning at TTI-Chicago when it used to be the start of DavidÕs day and most likely the end of mine. I am grateful to Arkadi Nemirovski for his valuable suggestions that helped resolve issues I was struggling with, and also make my thesis better . I am grateful
to Alexander Razborov for his valuable guidance and advice. Overall I feel privileged and honored to have a thesis committee that guided and moulded my research and the way I think about problems.

\vspace{0.1in} I greatly acknowledge the most wonderful staff at TTI-Chicago who made my stay at Chicago so memorable and pleasant. I would like to specially thank Carole Flemming, Gary Hamburg, Liv Leader and Christina Novak for helping me with numerous things throughout my PhD. I also thank Avleen Bijral, Wonseok Chae, Heejin Choi, Andrew Cotter, Patrick Donovan, Taehwan Kim, Jianzhu Ma, Arild Brandrud Nss, Jian Peng, Ankan Saha, Allie Shapiro, Hao Tang, Hoang Trinh, Zhiyong Wang, Xing Xu, Payman
Yadollahpour, Jian Yao and Feng Zhao for their friendship and making my time at TTI-Chicago fun.  Thanks also to the other students and faculty at TTI-Chicago with whom I have had insightful discussions and pleasant conversations.

\vspace{0.1in} Last, but not the least, I would like to express my deepest gratitude to my family and relatives, without whose support I would not have been able to complete my PhD. To Aditi, Pranav, Lavan, Sreeram and Amma, thanks for all the love and encouragement you have given me. I thank my parents for their unconditional support and confidence they have shown in me. I dedicate this dissertation to my late father, Raghavan Sridharan, whose greatest dream was to see me complete my PhD.

}

\chapter*{Abstract}
	\pdfbookmark[0]{Abstract}{abstract}

{  Optimization has always played a central role in machine learning and  advances in the field of optimization and mathematical programming have greatly influenced machine learning models. However the connection between optimization and learning is much deeper : one can phrase statistical and online learning problems directly as corresponding optimization problems. In this dissertation I take this viewpoint and analyze learning problems in both the statistical and online learning frameworks from an optimization perspective. In doing so we develop a deeper understanding of the connections between statistical and online learning and between learning and optimization.\\

The dissertation can roughly be divided into two parts. In the first part we consider the question of learnability and possible learning rates for general statistical and online learning problems without regard to tractability issues. In the second part we restrict ourselves to convex learning problems and address the issue of tractability for both online and statistical learning problems by considering the oracle complexity of these learning problems. 

\begin{enumerate}

\item[I.~]  We first consider the question of learnability and possible learning rates for statistical learning problems under the general learning setting. The notion of learnability was first introduced by Valiant (1984) for the problem of binary classification in the realizable case. Vapnik (1995) introduced the general learning setting as a unifying framework for the general problem of statistical learning from empirical data.  In this framework the learner is provided with a sample of instances drawn i.i.d. from some distribution unknown to the learner. The goal of the learner is to pick a hypothesis with low expected loss based on the sample received.  The question of learnability is well studied and fully characterized for binary classification using the Vapnik Chervonenkis (VC) theory and for real valued supervised learning problems using the theory of uniform convergence with tools like Rademacher complexity, covering numbers and fat-shattering dimension etc. However we show that for the general learning setting the traditional approach of using uniform convergence theory to characterize learnability fails. Specifically we phrase the learning problem as a stochastic optimization problem and construct an example of a convex problem where Stochastic Approximation (SA) approach provides successful learning guarantee but Empirical Risk Minimization (ERM) (or equivalently Sample Average Approximation (SAA) approach) fails to give any meaningful learning guarantee. This example establishes that for general learning problems the concept of uniform convergence fails to capture learnability and ERM/SAA approaches can fail to provided successful learning algorithms. To fill this void in the theory of statistical learnability in the general setting we instead turned to the concept of stability of learning algorithms to fully characterize learnability in the general setting. We specifically show that a problem is learnable if and only if there exists a stable approximate minimizer of average loss over the sample. Using this notion of stability, we also provide a universal learning procedure that guarantees success whenever the problem is learnable.\\

Next we consider the problem of online learning in the general setting. Online learning is a continual and sequential learning process where instances provided to the learner round by round and can be chosen adversarially (as opposed to stochastically). The goal of the learner for an online learning problem is to minimize regret with respect to the single best hypothesis that can be chosen in hindsight. Most of the work on online learning problems so far have been algorithmic and problem specific. The usual approach has been to build algorithm for the specific problem at hand and prove regret guarantees for this algorithm which in turn implies learnability with associated learning rates. Unlike the statistical learning framework there is a dearth of generic tools that can be used to establish learnability and rates for online learning problems in general. Only recently, Ben-David, Pal and Shalev-Shwartz (2009) showed that the Littlestone dimension (introduced by Littlestone (1988)) is an online analog to the VC dimension and fully characterizes learnability for online binary classification problems. However the question of characterizing online learnability for even the real valued supervised learning problems in the online framework was open. In this dissertation, analyzing the so called value of the online learning game, we provide online analogs to classical tools from statistical learning theory like Rademacher complexity, covering numbers, fat-shattering dimension etc. While these tools can be used to provide upper bounds for more general online learning problems, the results mirror uniform convergence theory for online learning problems. Hence analogous to the statistical learning case, we used these tools to fully characterize learnability and rates for real valued online supervised learning problems. We also provide a generic algorithm for the real valued online supervised learning problem. However unlike the statistical learning case, we don't yet have a full characterization of learnability for general online learning problems and leave this as open problem for future work.\\

\item[II.~]  In the first part of the dissertation we focused on the question of learnability and rates for statistical and online learning problems without paying attention to tractability or efficiency of the learning algorithms that could be used. Even the generic algorithms provided in the first part are in general not at all tractable. In the second part of the dissertation we  address the issue of building tractable or efficient learning algorithms by first focusing our attention specifically on convex learning problems. In the second part, for general classes of convex learning problems, we provide appropriate mirror descent updates that are guaranteed to be successful for online and statistical learning of these convex problems. Further, using and extending results from the geometry of Banach spaces, we show that the the mirror descent method (with appropriate prox-function and step-size) is near optimal for online convex learning problems and for most reasonable cases, is also near optimal for statistical convex learning problems. Further noting that when used for statistical convex learning, the mirror descent algorithm is a first-order (gradient based), $O(1)$ memory, single pass algorithm, we conclude that mirror descent method is also near optimal in terms of number of gradient accesses and for many commonly encountered problems, optimal also in terms of computational time. We next consider the problem of (offline) convex optimization and to capture the notion of efficiency of optimization procedures for these problems, we use the notion of oracle complexity of the problem introduced by Nemirovski and Yudin (1978). 
We show that for a general class of convex optimization problems, oracle complexity of the problem can be lower bounded by the so called fat-shattering dimension of the associated linear class. Thus we establish a strong connection between offline convex optimization problems and statistical convex learning problems. We further show that for a large class of infinite dimensional (or high dimensional)  optimization problems, mirror descent is in fact near optimal in terms of oracle efficiency even for these offline convex optimization problems. \\

%
\end{enumerate}
}

	{\large \tableofcontents }
	{\large \listoffigures}

	{\large \listoftables}
		\newpage	
	\pagenumbering{arabic}

	\newpage

\input{intro}

\part[\Large Statistical and Online Learning : Learnability and Rates]{Statistical and Online Learning :  Learnability and Rates} \label{part:one}

\input{prelims}

\input{stat}

\input{onln}

\part[\Large Convex Problems : Oracle Efficient Learning/Optimization]{Convex Problems : Oracle Efficient Learning/Optimization} \label{part:two}

\input{convex}
\input{md}
\input{opton}

\input{optstat}

\input{optoff}

\addtocontents{toc}{\bigskip}
\input{future}

	\phantomsection
	\addcontentsline{toc}{chapter}{\Large Bibliography}
\bibliography{../part1.bib,../bib.bib,../addbib.bib,../mdbib.bib}

\appendix
\input{appendix1}

\newpage	
{\footnotesize \singlespacing
\printindex}

\end{document}

%% file: header.tex
\usepackage{amsmath,amsfonts,amssymb,amsthm, color}
\usepackage{algorithm,algorithmic}
\usepackage{times}
\usepackage{fancyvrb}
\usepackage{enumerate}
\usepackage{tikz}
\usepackage{bm}
\usetikzlibrary{arrows,shapes}
\usetikzlibrary{patterns}

\usepackage{graphicx}
\usepackage{fullpage}
\usepackage{ltxtable}
\usepackage{nicefrac}
\usepackage{multirow}

\usepackage{setspace}

\usepackage[colorlinks=true, linkcolor=black, bookmarks=true]{hyperref}

\usepackage{natbib}
\renewcommand{\ref}{\hyperref}

\usepackage{multicol}
\makeatletter
  {\end{multicols}\if@restonecol\onecolumn\else\clearpage\fi}
\makeatother

\newcommand{\secref}[1]{Section~\ref{#1}}
\newcommand{\thmref}[1]{Theorem~\ref{#1}}
\newcommand{\lemref}[1]{Lemma~\ref{#1}}

\newcommand{\defref}[1]{Definition~\ref{#1}}

\renewcommand{\eqref}[1]{Eq.~(\ref{#1})}
\newcommand{\subsecref}[1]{Subsection~\ref{#1}}
\newcommand{\exref}[1]{Example~\ref{#1}}

\newtheorem{theorem}{Theorem}
\newtheorem{lemma}[theorem]{Lemma}
\newtheorem{claim}[theorem]{Claim}
\newtheorem{corollary}[theorem]{Corollary}

\newtheorem{proposition}[theorem]{Proposition}
\newtheorem{remark}[theorem]{Remark}
\theoremstyle{definition}
\newtheorem{definition}{Definition}
\newtheorem{example}{Example}
\newtheorem{question}{Question}

\newtheorem{ulemma}[theorem]{Utility Lemma}

%% file: notation.tex
\newcommand{\Dcal}{\mathcal{D}}

\newcommand{\A}{\mathbf{A}}

\newcommand{\z}{z}

\def\mathbi#1{{\bm #1}}
\newcommand{\tz}{\mathbi{z}}
\newcommand{\tx}{\mathbi{x}}

\newcommand{\ts}{\mathbi{s}}
\newcommand{\tv}{\mathbi{v}}
\newcommand{\tu}{\mathbi{u}}
\newcommand{\ty}{\mathbi{y}}
\newcommand{\ta}{\mathbi{a}}

\newcommand{\popt}{p^\star}
\newcommand{\qopt}{q^\star}
\newcommand{\w}{\ensuremath{\mathbf{w}}}
\newcommand{\h}{\ensuremath{\mathbf{h}}}
\renewcommand{\v}{\ensuremath{\mathbf{v}}}

\newcommand{\x}{\ensuremath{\mathbf{x}}}

\newcommand{\empF}{\hat{F}}
\newcommand{\empL}{\hat{L}}
\newcommand{\moff}{m_{\mrm{off}}}

\newcommand{\mstat}{m_{\mrm{stat}}}

\newcommand{\eps}[1]{\epsilon_{\mathrm{#1}}}

\newcommand{\epscon}{\eps{cons}}
\newcommand{\epsempp}{\eps{emp}}
\newcommand{\epsapprox}{\eps{erm}}
\newcommand{\epsstable}{\eps{stable}}
\newcommand{\epsgen}{\eps{gen}}
\newcommand{\epspgen}{\eps{oag}}

\newcommand{\hopt}{\h^{\star}}

\newcommand{\hemp}{\hat{\h}}

\newcommand{\F}{{L}}

\newcommand{\G}{\mathcal{G}}

\newcommand{\Lipf}{L}

\newcommand{\X}{\mathcal{X}}
\newcommand{\Y}{\mathcal{Y}}

\newcommand{\e}{\mathbf{e}}

\newcommand{\y}{\mathbf{y}}

\newcommand{\fcx}{\ell_{\eqref{eq:fcx}}}
\newcommand{\Fcx}{L_{\eqref{eq:fcx}}}
\newcommand{\empFcx}{\empL_{\eqref{eq:fcx}}}
\newcommand{\fcxx}{\ell_{\eqref{eq:fcxx}}}
\newcommand{\Fcxx}{\L_{\eqref{eq:fcxx}}}
\newcommand{\empFcxx}{\empL_{\eqref{eq:fcxx}}}
\newcommand{\fcs}{\ell_{\eqref{eq:fcs}}}
\newcommand{\Fcs}{L_{\eqref{eq:fcs}}}
\newcommand{\empFcs}{\empL_{\eqref{eq:fcs}}}

\newcommand{\breg}[3]{\Delta_{#1}\left( #2\middle| #3 \right)}
\newcommand{\mc}[1]{\mathcal{#1}}
\newcommand{\mbb}[1]{\mathbb{#1}}
\newcommand{\mbf}[1]{\mathbf{#1}}

\newcommand{\Rad}{\mc{R}}
\newcommand{\Z}{\mc{Z}}

\newcommand{\fat}{\mathrm{fat}}

\newcommand{\Es}[2]{\mbb{E}_{#1}\left[ #2 \right]}
\newcommand{\B}{\mc{B}}
\newcommand{\Bd}{\mc{B}^{\star}}

\renewcommand{\H}{\mc{H}}
\newcommand{\Hd}{\mc{H}^\star}
\renewcommand{\X}{\mc{X}}
\newcommand{\Val}{\mc{V}}
\newcommand{\Valstat}{\mc{V}^{\mrm{iid}}}
\newcommand{\Reg}{\mbf{R}}

\newcommand{\Xd}{\mc{X}^\star}
\newcommand{\En}{\mbb{E}}

\newcommand{\Algo}{\mc{A}}

\newcommand{\V}{V}

\newcommand{\Algomd}{\mc{A}_{\mathrm{MD}}}
\newcommand{\Algosmd}{\overline{\mc{A}_{\mathrm{MD}}}}
\newcommand{\MD}{\mathrm{MD}}

\newcommand{\prightarrow}{\stackrel{\scriptscriptstyle P}{\rightarrow}}
\newcommand{\mrm}[1]{\mathrm{#1}}
\newcommand{\halgo}{\tilde{\h}} 
\newcommand{\abs}[1]{\left|#1\right|}

\newcommand{\half}{\frac{1}{2}}

\newcommand{\balpha}{\mbf{\alpha}}
\newcommand{\phifunc}[1]{\phi\left(#1\right)}
\newcommand{\trm}[1]{\textrm{#1}}
\newcommand{\norm}[1]{\left\|#1\right\|}
\newcommand{\ip}[2]{\left<#1,#2\right>}

\newcommand{\argmin}[1]{\underset{#1}{\mrm{argmin}} \ }
\newcommand{\argmax}[1]{\underset{#1}{\mrm{argmax}} \ }
\newcommand{\reals}{\mathbb{R}}
\newcommand{\E}[1]{\mathbb{E}\left[ #1 \right]} 
\newcommand{\Ebr}[1]{\mathbb{E}\left\{ #1 \right\}} 

\newcommand{\Ps}[2]{\mathbb{P}_{#1}\left[ #2 \right]}
\newcommand{\conv}{\operatorname{conv}}
\newcommand{\inner}[1]{\left\langle #1 \right\rangle}

\newcommand{\Phifunc}[1]{\Phi\left(#1\right)}
\newcommand{\ind}[1]{{\bf 1}_{\left\{#1\right\}}}

\newcommand\s{\mathbf{s}}

\renewcommand\v{\mathbf{v}}

\newcommand\cD{\mathcal{D}}
\renewcommand\Z{\mathcal{Z}}
\renewcommand\F{\mathcal{F}}
\newcommand\cH{\mathcal{H}}
\newcommand\bcH{\bar{\mathcal{H}}}

\newcommand\M{\mathcal{M}}
\newcommand\Nhat{\mathcal{\widehat{N}}}

\newcommand\ldim{\mathrm{Ldim}}
\newcommand\Img{\mbox{Img}}

\newcommand\D{\mc{D}}

\newcommand{\Prob}[1]{\mathbb{P}\left[ #1 \right]}

\newcommand{\Learner}{\emph{Learner} }
\renewcommand{\L}{L}

\newcommand{\bnsp}{\mc{B}}
\renewcommand{\Algo}{\mbf{A}}

\newcommand\Ans{\mrm{I}}

\newcommand\Info{\mc{I}}
\newcommand\Oracle{\mc{O}}

\newcommand{\fatstat}{\fat^{\mrm{iid}}}
\newcommand{\faton}{\mathfrak{fat}^{\mrm{seq}}}
\newcommand{\Radstat}{\Rad^{\mrm{iid}}}
\newcommand{\Radon}{\mathfrak{R}^{\mrm{seq}}}
\newcommand{\Nstat}{\mathcal{N}^{\mrm{iid}}}
\newcommand{\Nonl}{\mathfrak{N}^{\mrm{seq}}}

\newcommand\Dudleyon{\mathfrak{D}^{\mrm{seq}}}

%% file: intro.tex
\chapter[\Large Introduction]{Introduction} \label{chp:intro}

There are two main paradigms under which machine learning problems have been commonly studied; these are the statistical learning framework and the online learning framework. In the statistical learning framework, the data source is said to be iid, that is data (or instances) are assumed to be drawn independently from some fixed but unknown (to the learner) source distribution. The goal of the learner in this setting is to pick a hypothesis that minimizes expected error on future examples drawn from the source distribution.  The statistical learning framework is well suited for applications like object recognition, natural language processing and other applications where the data generation process is unchanging with time. For instance when we consider a problem of image classification between cat and dog images, it is not unreasonable to assume the data source does not vary over time and each time we sample repeatedly form the same distribution over input images. Roughly speaking cats and dogs would look the same even 10 years from now and so we can think of the distribution from which we draw the image label pair to be fixed. 

The online learning framework on the other hand is a learning scenario where the learner is faced with a data source that changes over time or even reacts to the learners choices. The online learning framework is often studied as a multiple round game where the data source could even be adversarial.  The instances produced by the adversary at any round could depend on the past hypotheses chosen by the learner. The online learning framework is better suited for more interactive learning tasks like spam mail detection and stock market prediction where decisions of the learner could affect future instances the learner receives. For instance in a spam mail filter application, while the company providing the spam filter service works on filtering out spam mails, the spammers themselves try to outwit the system by adapting to how the spam detection software works. In a sense there is a continuous game played between the company providing spam filtering service and the spammers where each one tries to outwit the other. Such a scenario is naturally captured by the online learning framework.

In this dissertation we consider learning in both online and statistical learning frameworks. We view learning problems in both the frameworks from an optimization viewpoint and establish connections between learning and optimization. Specifically, in the first part of the dissertation we consider very general learning problems in both the statistical and online learning frameworks and focus on the question of when a given learning problem is learnable and at what rates. Of course at that generality we restrict our focus to only questions of learnability and rates for the learning problems and do not pay attention to tractability/efficiency of possible learning algorithms that can be used for these problems.

In the second half of the thesis, we restrict our focus to convex learning problems and beyond just characterizing optimal rates for these learning problems, we also provide simple first order methods based on mirror descent and show that they are near optimal both in terms of rate of convergence and in terms of their efficiency.  Specifically we establish that for a large family of general convex learning problems, mirror descent algorithm is universal and near optimal for the online learning framework. We also show that for most reasonable cases, mirror descent is near optimal for these convex learning problems even in the statistical learning framework. Further, the fact that mirror descent is an order $O(1)$ memory, simple first order method makes it near optimal even in terms of efficiency for many of these learning problems. We also establish that for several high dimensional problems, mirror descent could be near optimal (in terms of oracle efficiency) for (offline) convex optimization problems. In the second half of the thesis we establish some interesting connections between convex learning problems and convex optimization. We establish that for most reasonable cases online convex learning and statistical convex learning are equally hard (or easy). The main contention of the sec on part of the thesis is that for a large family of convex learning problems, whether it is online or statistical learning framework, the simple one pass, first order method of mirror descent is optimal both in terms of rates and efficiency.

While the work in this dissertation is of a theoretical in nature, we believe that the results and viewpoints provided will drive the choice of learning algorithms considered for various problems and models chosen for learning. Further the second part of the thesis has more direct practical implications and advocates the use of simple first order online learning methods for convex learning problems even when our end goal is statistical learning.

\section{Learning and Optimization}
Optimization has always played a central role in shaping machine learning algorithms and models. The typical approach taken to tackle learning problems has been to pick a suitable set of models or hypotheses or predictors, pick appropriate empirical cost over the training sample and finally solve the optimization problem of picking the hypothesis from the set of hypotheses that minimize the empirical cost over training sample. The learning algorithm itself in this case corresponds to the optimization algorithm that solves the minimization problem of picking that hypothesis from the set of hypotheses that minimizes empirical cost over training sample. Theoretical guarantees on learning rates are typically provided by decomposing error term into two terms. The first term, the estimation error that accounts for error we would incur if we picked the hypothesis that minimizes empirical cost over training sample chosen. The second term is the optimization error term that accounts for the sub-optimality of the optimization algorithm in minimizing the empirical cost chosen. 

Overall, we see that optimization plays a critical role in the design and analysis of algorithms for learning problems. However the connection between optimization and learning is much deeper. Learning problems under both statistical and online learning frameworks can be directly framed as optimization problems and analysis of the sub-optimality of algorithms the . 
In this dissertation I will take this viewpoint that learning problems can be viewed as optimization problems and using this point of view establish connections between learning in statistical and online frameworks and optimization.  Specifically in the second half of the dissertation, focusing on convex learning problems, we show that mirror descent method originally introduced for convex optimization problems by Nemirovski and Yudin \cite{NemirovskiYu78}, is always near optimal for online learning problems, near optimal for most reasonable statistical learning problems and even near optimal for several high dimensional offline convex optimization problems. We use concepts from the geometry of Banach spaces to establish these optimality results for mirror descent. We also establish some interesting connections between convex learning and convex optimization.


\section{Overview of the Thesis}
The dissertation can be split into two parts; the first part
focuses on the question of learnability and learning rates for fairly general class of learning problems in both the statistical and online learning frameworks. In the first part, we do not concern ourselves with questions of efficiency or tractability of the learning methods/algorithms for solving these problems. In the second part of this dissertation, we restrict ourselves to convex learning problems and show that for most reasonable problems, the first order method of mirror descent is near optimal for both statistical and online convex learning problems. We also establish connections between convex learning and (offline) convex optimization and establish that for several high dimensional (offline) convex optimization problems, mirror descent is optimal in terms of efficiency for (offline) convex optimization problem. The notion of efficiency we use in this case is the notion of oracle complexity introduced by Nemirovski and Yudin (1978).

\subsection{Part I : Statistical and Online Learning : Learnability and Rates}
The first part of this dissertation contains two chapters. The first one, Chapter \ref{chp:stat}, is dedicated to the question of learnability in the statistical learning framework. In this chapter, through an example that is an instance of a stochastic convex optimization problem, we show that the concept of uniform convergence (and hence tools like Rademacher complexity, fat-shattering dimension) that are commonly used to analyze learning rates in the supervised learning settings fail in the general. Thus we establish that the folklore of uniform convergence is necessary and sufficient for learnability is not true in general.  This of course opens up the question : ``How can one characterize learnability of general learning problems, in the statistical framework ?". We provide an answer to this question in this chapter by turning to the concept of stability of learning algorithms. We show that learnability in the general setting can be characterized by existence of learning algorithm that is stable and an approximate empirical minimizer. We further proceed to show that if we are allowed to consider randomized learning rules, then we can provide a ``Universal Learning Algorithm" which has non-trivial learning rate whenever the problem is learnable. 

Chapter \ref{chp:online}, is the second chapter in the first part of the dissertation. It deals with the question of learnability and rates for online learning problems.  While the question of learnability is well studied in the statistical framework, the question of learnability in the online setting has  relatively less explored. Most of the work on online learning problems so far have been algorithmic and problem specific. The usual approach has been to build algorithm for the specific problem at hand and prove regret guarantees for this algorithm which in turn implies learnability with associated learning rates. Unlike the statistical learning framework there is a dearth of generic tools that can be used to establish learnability and rates for online learning problems in general. Only recently, Ben-David, Pal and Shalev-Shwartz (2009) showed that the Littlestone dimension (introduced by Littlestone (1988)) is an online analog to the VC dimension and fully characterizes learnability for online binary classification problems. In general there have been no generic tools like Rademacher complexity, covering numbers and fat-shattering dimension that are present for analyzing statistical learning problems. In this chapter we explore the the question of learnability and optimal rates for online learning by first formalizing them as value of the online learning game. We then build complexity measures analogous to those in the statistical framework like Rademacher complexity, covering numbers and fat-shattering dimension and show that these tools can be used to bound learning rates for online learning problems. These tools can be seen as tools for studying uniform convergence for general stochastic processes (non iid). We go ahead and show that these tools can even characterize learnability and learning rates of real valued supervised learning problems in the online learning framework. Based on these complexity measures, we also provide a generic algorithm for the supervised learning problem in the online learning framework that has diminishing regret whenever the problem is online learnable.

\subsection{Part II : Convex Problems : Oracle Efficient Learning/Optimization}

In the first part of this thesis while we even provide generic learning algorithms for fairly general class of problems in both statistical and online learning frameworks, we never concerned ourselves with any form of tractability of these learning algorithms. However, in the second part of this dissertation, focusing on convex learning problems, we aim at providing optimal and  efficient learning algorithms for both the statistical and online learning frameworks. In doing so we also explore connections between learning and convex optimization. There are five chapters in this part. The first chapter of this part, Chapter \ref{chp:cvx}, introduces the basic set up of the convex learning and optimization problems we consider in this part, describes the online and statistical learning protocols/frameworks for these convex problems and introduces the oracle model for offline convex optimization. The next chapter, Chapter \ref{chp:md} introduces the mirror descent methods (see \cite{NemirovskiYu78}) for the statistical and online convex learning problems described in the previous chapter and provides upper bounds on learning rate for them.  Upper bounds on rate of optimization of the mirror descent method for offline convex optimization problems are also provided. Chapter \ref{cup:option} deals with online convex learning problems and in the chapter we show that mirror descent algorithm is universal near optimal for online convex learning problems. In Chapter \ref{chp:optstat} statistical convex learning problems are considered and lower bounds on learning rates for these problems are provided. Using these lower bounds we further show that for most reasonable cases, the mirror descent method for statistical learning (stochastic mirror descent method) is near optimal even for statistical learning. In the final chapter of the second part of this dissertation, Chapter \ref{chp:optoff}, (offline) convex optimization problems are considered. As mentioned earlier, we use oracle complexity of the learning problem (ie. minimum number of calls to any local oracle needed by any method to achieve desired accuracy) as a measure of efficiency of the optimization procedure. We show that the oracle complexity of convex optimization problems can be lower bounded by fat-shattering dimension of the associated linear function class (a classic concept form statistical learning theory). Using this lower bound and results from previous chapters we show that for certain classes of convex optimization problems (high dimensional), mirror descent method is near optimal even for offline convex optimization problem. These results are also further used to argue that for certain statistical convex learning problems, mirror descent method is near optimal even when the learner has access to a parallel computation oracles (i.e. oracle query on entire sample is considered as one oracle call in the model). This in turn implies that for these convex learning problems, parallelization does not help.

\section{Main Contributions}

\begin{enumerate}
\item Part I : 
\begin{enumerate}
\item Statistical learning
\begin{itemize}
\item Illustrate limitations of uniform convergence and ERM/SAA
\item Provide characterization of statistical learning in general through stability
\item Provide universal randomized learning rule
\end{itemize}
\item Online learning
\begin{itemize}
\item Introduce analogs of various complexity measures of hypothesis class for online learning, provide tools martingale uniform convergence theory 
\item Characterize Online learnability for supervised learning problems
\item Generic algorithm for online supervised learning
\end{itemize}
\end{enumerate}
\item Part II : 
\begin{enumerate}
\item Online convex learning problems
\begin{itemize}
\item Characterize learnability and rates for various convex learning problems through notion of martingale type
\item Show universality and near optimality of online mirror descent for online convex learning problems
\end{itemize}
\item Statistical convex learning problems
\begin{itemize}
\item Establish lower bounds for oracle complexity and learning rates for various convex learning problems in statistical framework using Rademacher complexity of linear class
\item Show that for most reasonable statistical convex learning problems, mirror descent algorithm is near optimal both in terms of learning rates and number of oracle (gradient) access
\end{itemize}
\item Offline convex optimization problems
\begin{itemize}
\item Generic lower bound on convex optimization problem by fat-shattering dimension of associated linear class
\item Show that for a large class of large dimensional convex optimization problems, mirror descent is near optimal even for offline convex optimization 
\end{itemize}
\end{enumerate}
\end{enumerate}

\section{Bibliographic Notes}
Results in Chapter \ref{chp:stat} are from joint work with Shai Shalev-Shwartz, Ohad Shamir and Nati Srebro. Early versions of the results can be found in \cite{ShalevShSrSr09,ShalevShSrSr09a}. See \cite{ShalevShSrSr10} for later version which is closer to the one presented in the chapter. The results in Chapter \ref{chp:online} are from joint work with Alexander Rakhlin and Ambuj Tewari and can be found in \cite{RakSriTew10}. In the second part of the dissertation, few of the result from Chapter \ref{chp:md} can be found in \cite{SreSriTew11}. The results in Chapter \ref{chp:opton} is from joint work with Nathan Srebro and Ambuj Tewari. Relating basic concept of martingale type and certain online convex learning problems was first done in \cite{SriTew10}. In \cite{SreSriTew11} the result of universality and near optimality of mirror descent method from Chapter \ref{chp:opton} is provided. Chapters \ref{chp:optstat} and \ref{chp:optoff} are joint work with Nati Srebro.

%% file: prelims.tex
\chapter[Preliminary Setup and Notations]{Preliminary Setup and Notations} \label{chp:prelims}

In this chapter we provide the basic setup, some preliminary definitions and notations used throughout this dissertation. \\

\section{General Learning Problem Setup}
In all the learning problems we consider instances provided to the learner are chosen from the instance set $\Z$. The learner in picks hypotheses from set $\bcH$. The instantaneous loss incurred by learner on instance $\z \in \Z$ for picking hypothesis $\h \in \bcH$ is given by $\ell(\h,\z)$ where $\ell : \bcH \times \Z \mapsto \reals$ is the ``loss" or cost function. The exact setup of the learning problem and the goal of the learner depends on whether we consider the statistical learning framework or the online learning framework and will formally be specified later on in the corresponding chapters or sections. However, irrespective of the exact framework, the rough goal of the learner for the learning problems we consider, is to pick hypotheses that are competitive with respect to the best hypothesis from a set of hypotheses $\cH \subseteq \bcH$. Notice that usually, in most of the learning problems considered in literature, the setting considered is one where learner also picks hypothesis from set $\cH$. In contrast we shall consider the so called ``Improper learning setting" here and allow learner to pick hypothesis from a set $\bcH$ while the goal is to compete with the best hypothesis from a smaller set $\cH$. Whenever $\bcH = \cH$ we call this setting as the proper learning setting. This framework is sufficiently general to include a large portion of the learning and optimization problems we are aware of, such as:
\begin{itemize}
    \item \textbf{Binary Classification:} Let
    $\Z=\X\times\{0,1\}$, let $\bcH$ be a set
    of functions $\h:\X\mapsto \{0,1\}$, and let
    $\ell(\h;(\x,y))=\ind{\h(\x)\neq y}$. Here, loss function is simply the $0-1$ loss, measuring whether the binary hypothesis $\h$ misclassified the example $(\x,y)$.
    \item \textbf{Regression:} Let
    $\Z=\X\times \Y$ where $\X$ and $\Y$ are bounded subsets of $\reals^d$ and $\reals$ respectively. Let $\bcH$ be a set of
    bounded functions $\h:\X \mapsto \reals$, and let
    $\ell(\h;(\x,y))=(\h(\x)-y)^2$. Here, the loss function is simply the squared loss.
    \item \textbf{Large Margin Classification in a Reproducing Kernel
    Hilbert Space (RKHS):} Let $\Z =\X\times \{0,1\}$, where $\X$
    is a bounded subset of an RKHS, let $\bcH$ be another bounded subset
    of the same RKHS, and let $\ell(\h;(\x,y))=\max\{0,1-y\inner{\x,\h}\}$. Here, the loss function is the well known hinge loss function, and our goal is to perform margin-based linear classification in the RKHS.
    \item \textbf{K-Means Clustering in Euclidean Space:} Let
    $\Z =\reals^n$, let $\bcH$ be all subsets of $\reals^n$ of size $k$,
    and let $\ell(\h;\z)=\min_{\mbf{c} \in \h}\norm{\mbf{c}-\z}^2$.  Here, each $\h$
    represents a set of $k$ centroids, and the loss $\ell$ measures the
    Euclidean distance squared between an instance $\z$ and its nearest
    centroid, according to the hypothesis $\h$.
    \item \textbf{Density Estimation:} Let $\Z$ be a subset of
    $\reals^n$, let $\bcH$ be a set of bounded probability densities on
    $\Z$, and let $\ell(\h;\z)=-\log(\h(\z))$. Here, loss function $\ell$ is simply
    the negative log-likelihood of an instance $\z$ according to the
    hypothesis density $\h$. 
    \item \textbf{Convex Learning Problems:} Let
      $\Z$ be an arbitrary measurable set, let $\bcH$ be a closed, convex subset of a vector space, and for each $z \in \Z$, let the function $\ell(\h;\z)$ be convex w.r.t. its first argument. 
\end{itemize}

We shall use the letter $S$ to denote a sample of instances, that is a sample $S \in \bigcup_{n \in \mbb{N}} \Z^n$ is a sequence of instances. For instance $S = (\z_1,\ldots,\z_n)$ is a sample of size $n$, note that the order and and multiplicity of instances may be important. Given a sample $S$ we shall denote the empirical average loss over this sample $S$ as :
$$
\L_S(\h) = \frac{1}{|S|} \sum_{i = 1}^{|S|} \ell(\h,\z_i)
$$
We shall also use the notation $\empL(\h)$ to refer $\L_S(\h)$ whenever the sample $S$ is well understood under the context. Further, given a distribution $\D$ on instance space $\Z$, we shall use the notation 
$$
\L_\D(\h) = \Es{\z \sim \D}{\ell(\h,z)}
$$
Further we shall use the notation $\L(\h)$ to refer to $\L_\D(\h)$ whenever the distribution is understood under the context. 


\section{More Definitions and Notations}

An important object that we will encounter especially while analyzing online learning are trees. Unless specified, all trees considered in this paper are either \emph{rooted complete binary} trees. While it is useful to have the tree picture in mind when reading the paper, it is also necessary to precisely define trees as mathematical objects. We opt for the following definition.

\begin{definition}[Trees]\index{tree}
Given some set ${\mathcal Z}$, a \emph{${\mathcal Z}$-valued tree} $\tz$, of depth $n$ is a sequence $(\tz_1,\ldots,\tz_n)$ of $n$ mappings  $\tz_i : \{\pm 1\}^{i-1} \mapsto \mathcal{Z}$. The \emph{root} of the tree $\tz$ is the constant function $\tz_1 \in {\mathcal Z}$. 
\end{definition}
A tree of infinite depth is defined exactly as above as an infinite sequence $(\tz_n)_{n \in \mathbb{N}}$.

Armed with this definition, we can talk about various operations on trees. For a function $f:{\mathcal Z} \mapsto {\mathcal U}$, $f(\tz)$ denotes the ${\mathcal U}$-valued tree defined by the mappings $(f \circ \tz_1, \ldots, f \circ \tz_n)$. Analogously, for $f:{\mathcal Z}\times {\mathcal Z} \mapsto {\mathcal U}$, the ${\mathcal U}$-valued tree $f(\tz,\tz')$ is defined as mappings $(f(\tz_1,\tz'_1),\ldots, f(\tz_n,\tz'_n))$. In particular, this defines the usual binary arithmetic operations on real-valued trees. Furthermore, for a class of functions $\F$ and a tree~$\tz$, the projection of $\F$ onto $\tz$ is $\F(\tz) = \{f(\tz): f\in \F\}$.

\begin{definition}[Path]\index{path}
	A \emph{path} of length $n$ is a sequence $\epsilon = (\epsilon_1,\ldots,\epsilon_{n-1}) \in \{\pm1\}^{n-1}$. 
\end{definition}

We shall abuse notation by referring to $\tz_i(\epsilon_{1},\ldots,\epsilon_{i-1})$ by  $\tz_i(\epsilon)$. Clearly $\tz_i$ only depends on the first $i-1$ elements of $\epsilon$. We will also refer to $\epsilon = (\epsilon_1,\ldots,\epsilon_n) \in \{\pm1\}^n$ as a path in a tree of depth $T$ even though the value of $\epsilon_T$ is inconsequential. Next we define the notion of subtrees.

\begin{definition}[Subtrees]\index{subtree}
	The \emph{left subtree} $\tz^{\ell}$ of $\tz$ at the root is defined as $n-1$ mappings $(\tz^{\ell}_1, \ldots, \tz^{\ell}_{n-1})$ with $\tz^{\ell}_i(\epsilon) = \tz_{i+1} (\{-1\}\times \epsilon)$ for $\epsilon\in\{\pm 1\}^{n-1}$. The right subtree $\tz^{r}$ is defined analogously by conditioning on the first coordinate of $\tz_{i+1}$ to be $+1$. 
\end{definition}	
Given two subtrees $\tz$, $\tv$ of the same depth $n-1$ and a constant mapping $\z_1$, we can {\em join} the two subtrees to obtain a new set of mappings $(\tx_1,\ldots,\tx_{n})$ as follows. The root is the constant mapping $\z_1$. For $i\in\{2,\ldots, n\}$ and $\epsilon\in\{\pm1\}^{n}$, $\tx_i(\epsilon) = \tz_{i-1} (\epsilon)$ if $\epsilon_1 = -1$ and $\w_i(\epsilon) = \tv_{i-1} (\epsilon)$ if $\epsilon_1 = +1$. 

We will also need to talk about the values given by the tree $\tz$ over all the paths. Formally, let $\Img(\tz) = \tz\left(\{\pm1\}^n\right) = \{\x_t(\epsilon): t\in[n],  \epsilon\in\{\pm1\}^n \}$ be the image of the mappings of $\tz$.



Table \ref{tab:notations} contains a list of the basic notations used.

\begin{table}[h]
\begin{center}
\begin{tabular}{|l l|} \hline
$\mbb{N}$ & \hspace{0.5in} The set of natural numbers\\ \hline
$\mbb{R}$ & \hspace{0.5in} The set of real numbers\\ \hline
$\mbb{R}_+$ & \hspace{0.5in} The set of non-negative real numbers\\ \hline
$[n]$ & \hspace{0.5in} The set $\{1,2,\ldots,n\}$\\ \hline
$\ind{A}$ & \hspace{0.5in}  $1$ if predicate $A$ holds and $0$ otherwise\\ \hline
$[a]_+$ & \hspace{0.5in}  $\max\{0,a\}$\\ \hline
$x,w$ & \hspace{0.5in} Scalars\\ \hline
$\x,\w$ & \hspace{0.5in} Vectors\\ \hline
$\mathbf{X},\mathbf{W}$ & \hspace{0.5in} Matrices\\ \hline
$\x[i]$ & \hspace{0.5in} $i^{\mathrm{th}}$ element of vector $\x$\\ \hline
$\mathbf{X}[i,j]$ & \hspace{0.5in} $i\times j$th entry of matrix $\mbf{X}$\\ \hline
$\B$ & \hspace{0.5in} Vector space \\ \hline
$\Bd$ & \hspace{0.5in} Dual of vector space $\B$ \\ \hline
$\inner{\x, \w}$ & \hspace{0.5in}  Linear functional $\x \in \Bd$ applied to $\w \in \B$ \\ \hline
$f,g$ & \hspace{0.5in}  Functions \\ \hline
$\Delta(\X)$ & \hspace{0.5in} Set of Borel probability measures on set $\X$ \\ \hline
$f^\star$ & \hspace{0.5in}  The Fenchel conjugate of function $f$ \\ \hline
$\nabla f(\w)$ & \hspace{0.5in}  A sub-gradient of $f$ at $\w$\\ \hline
$\norm{\cdot}_*$ & \hspace{0.5in}  Dual norm of the norm $\norm{\cdot}$ \\ \hline
$\E{Z}$ & \hspace{0.5in}  Expectation of random variable $Z$ \\ \hline
$\Prob{A}$ & \hspace{0.5in}  Probability that event $A$ occurs \\ \hline
\end{tabular}
\caption{Summary of Notations}\label{tab:notations}
\end{center}
\end{table}

%% file: stat.tex
\chapter{Statistical Learning/Optimization}\label{chp:stat}

In this chapter we consider the problem of statistical learning in the general learning problems introduced by \cite{Vapnik95} where we would like to minimize a population risk functional (stochastic
objective)
\begin{equation}\label{eq:Fh}
\L(\h)=\Es{\z\sim\Dcal}{\ell(\h;\z)}
\end{equation}
based on i.i.d.~sample $\z_1,\ldots,\z_n$ drawn from $\Dcal$ over some target hypothesis class $\H$. The distribution $\Dcal$ over instance space $\Z$ is unknown to the learner. Notice that this is basically a stochastic optimization problem.  In this chapter we are mainly concerned with the question of statistical ``learnability''.  That is, when can \eqref{eq:Fh} be minimized to within arbitrary precision based only on a finite sample $\z_1,\ldots,\z_n$, as $n \rightarrow \infty$?  

For supervised classification and regression problems, it is well
known that a problem is learnable if and only if the empirical risks
\begin{equation}
  \L_S(\h) = \tfrac{1}{n} \sum_{i=1}^n \ell(\h,\z_i)
\end{equation}
for all $\h\in \H$ converge uniformly to the population risk (\cite{BlumerEhHaWa89,AlonBeCsHa97}).
If uniform convergence holds, then the empirical risk minimizer (ERM) is {\em
  consistent}, i.e.~the population risk of the ERM converges to the optimal
population risk, and the problem is learnable using the ERM.  We therefore
have:
\begin{itemize}
\item A necessary and sufficient condition for learnability, namely
  uniform convergence of the empirical risks.  Furthermore, this can
  be shown to be equivalent to a combinatorial condition: having
  finite VC-dimension in the case of classification, and having
  finite fat-shattering dimensions in the case of regression.
\item A complete understanding of {\em how} to learn: since learnability is
  equivalent to learnability by ERM, we can focus our attention solely on
  empirical risk minimizers.
\end{itemize}
The situation, for supervised classification and regression, can be
depicted as follows:
\begin{center}
\begin{tikzpicture}[node distance=3cm, auto,>=latex',
cond/.style={draw, thin, rounded corners, inner sep=1ex, text centered}]
\node[text width=1.8cm, style=cond] (finitedim) {\small Finite Dim.};
\node[text width=1.8cm, style=cond, right of=finitedim] (uGC)
{\small Uniform Convergence};
\node[text width =1.8cm, style=cond, right of=uGC] (ERM) {\small
Learnable with ERM};
\node[text width=1.8cm, style=cond, right of=ERM] (learnable) {\small Learnable};
\path (finitedim) edge[->, double distance=1pt] (uGC)
      (uGC) edge[->, double distance=1pt] (ERM)
      (ERM) edge[->, double distance=1pt] (learnable);
\path[->, draw, double distance=1pt,sloped] (learnable) -- +(0,-0.75) -| (finitedim);
\end{tikzpicture}
\end{center}

In this chapter we start by showing that the situation for general learning problems in the statistical learning framework is actually much more complex.  In particular, in \subsecref{sec:ugc} we show an example of a learning problem which is learnable (using Stochastic Approximation approch), but is \emph{not} learnable using empirical risk minimization and uniform convergence fails. We discuss how notion stability (through regularization) plays an important role in learnability of the problem.  Having shown that uniform convergence fails to characterize learnability for general learning problems we then approach the question of characterizing learnability in general for statistical learning problems.
To do so in Section \ref{sec:stability} we first introduce definitions of stability of learning rules we consider using for characterizing learnability. In section \ref{sec:main}, we show that any problem that is learnable, is always learnable with some learning rule which is an ``asymptotically ERM''. Moreover, such an AERM must be stable (under a suitable notion of stability). Namely, we provide following characterization of learnability for general statistical learning problems and this can be considered as the highlight of this chapter :
\begin{center}
\begin{tikzpicture}
[node distance=3cm, auto,>=latex', cond/.style={draw, thin, rounded corners,
inner sep=1ex, text centered}]

\node[text width=1.8cm, style=cond] at (0,0) (EAERM_s) {\small Exists Stable
AERM};

\node[text width=1.8cm, style=cond] at (3,0)(AERM_l) {\small Learnable with
AERM};

\node[text width=1.8cm, style=cond] at (6,0) (l){\small Learnable};

\path(EAERM_s) edge[<->, double distance=1pt] (AERM_l);

\path(AERM_l) edge[<->,double distance=1pt] (l);

\end{tikzpicture}
\end{center}
Note that this characterization holds even for learnable problems with no
uniform convergence. In this sense, stability emerges as a strictly more
powerful notion than uniform convergence for characterizing learnability.
Finally in Section \ref{sec:randomization} we show how we can get stronger results by considering randomized learning algorithms and we go on to provide a generic learning rule that is guaranteed to be successful whenever the problem is learnable. 

Section \ref{sec:proofs} contains the details of the proofs and technical results used in the chapter. Following that we conclude with Bibliographic notes and Discussion.

\section{The Statistical Learning Problem and Learnability}\index{stochastic optimization}
In the statistical learning problem, instances are drawn i.i.d. from some fixed distribution $\D$ unknown to the learner. Given a sample $S = \z_1,\ldots,\z_n$ of size $n$, the goal of the learner then is to pick a hypothesis $\h \in \bcH$ based only on the sample $S$ that has small expected loss $\L_\D(\h)$. This problem of learning can in turn be phrased as a stochastic optimization problem where our goal is the minimization problem 
$$
\min_{\h \in \bcH} \L_\D(\h)
$$
The term $\L_\D(\h)$ is often referred to as the risk of choosing hypothesis $\h$. 
However given a target hypothesis class $\cH$ our goal is only to do as well as the best hypothesis in this target class and so the sub-optimality (also refered to as ) of any hypothesis $\h \in \bcH$ is given by
$$\index{excess risk}
\L_D(\h) - \inf_{\h \in \cH} L_\D(\h)~.
$$
Of course notice that for proper learning case, the problem is exactly that of stochastic optimization. 

To formally refer to the strategy or algorithm used by the learner to pick hypotheses based on sample provided, we now define the notion of learning algorithm for statistical learning problems. 

\begin{definition}\index{learning rule!statistical}
A ``Statistical Learning Rule" $\Algo : \bigcup_{n \in \mathbb{N}} \Z^n \mapsto \bcH$ is a mapping from sequences of instances in $\Z$ to the set of Hypothesis $\bcH$.
\end{definition}
We shall refer to any ``Learning Rule" $\Algo$ that only outputs  hypothesis in the set $\cH$ instead of entire $\bcH$ as a ``Proper Learning Rule". \index{learning rule!statistical!proper}

%
%

Learnability deals with the question of when it is even possible (information theoretically) to drive the sub-optimality for a given problems to $0$ with increase in sample size. Note that since we are given a randomly drawn sample we shall requires all results in expectation over draw of sample. Later on we illustrate how such a result can be converted to a result that hold with high probability over the sample.  

\begin{definition}\index{learnability!statistical}
We say that a problem is ``Statistically Learnable" given a target hypothesis class $\cH$ with rate $\epscon(n)$ if there exists a Statistical Learning Rule, $\Algo$, such that :
$$
\sup_{\D}\, \Es{S \sim \D^n}{\L_\D(\Algo(S)) - \inf_{\h \in \cH} \L_\D(\h)} \le \epscon(n)
$$ 
Further as long as $\epscon(n) \rightarrow 0$ we simply say that the problem is statistically learnable/optimizable.
\end{definition}

Whenever a problem is learnable, we will refer to any learning rule $\Algo$ such that 
$$
\sup_{\D}\, \Es{S \sim \D^n}{\L_\D(\Algo(S)) - \inf_{\h \in \cH} \L_\D(\h)} \rightarrow 0$$ 
as a universally consistent learning rule.
This definition of learnability, requiring a uniform rate for all
distributions, is the relevant notion for studying learnability of a
hypothesis class.  It is a direct generalization of agnostic PAC-learnability
(\cite{KearnsScSe92}) to Vapnik''s General Setting of Learning as studied by \cite{Haussler92} and others.

A closely related notion to learnability is that of the sample complexity of a problem. We first define the sample complexity of a given learning rule and then proceed to define sample complexity of a learning problem.

\begin{definition}\index{sample complexity}
Given a learning rule $\Algo$, and $\epsilon > 0$, the sample complexity of the rule for $\epsilon$ is defined as
$$
n(\epsilon,\Algo,\Z) = \inf\left\{n \in \mathbb{N}\ \middle|\ \sup_{\D}  \Es{S \sim \D^n}{\L_\D(\Algo(S)) - \inf_{\h \in \cH} \L_\D(\h)} \le \epsilon \right\}
$$
Sample complexity of the learning problem is defined as $n(\epsilon,\Z) = \inf_{\Algo} n(\epsilon,\Algo,\Z)$.
\end{definition}
Obviously a problem is learnable if and only if for each $\epsilon > 0$, $n(\epsilon) < \infty$ and for any universally consistent learning rule $\Algo$ and any $\epsilon > 0$, $n(\epsilon,\Algo) < \infty$. 

The notion of sample complexity talked about above is the worst case one in the sense that we want learnability with uniform rates, irrespective of the distribution chosen. Sometimes one might have prior knowledge or restrictions on distributions that can occur and so one might only need rates to hold uniformly over a specific family of distributions. To capture this notion we define below sample complexity of a learning problem given a specific family of distributions $\mbf{\D}$ over instances $\Z$.  

\begin{definition}\index{sample complexity!distribution specific}
Given a family $\mbf{\D}$ of Borel distributions over instance space $\Z$ and an $\epsilon > 0$, the sample complexity of any rule $\Algo$ is defined as
$$
n^{\mbf{\D}}(\epsilon,\Algo) = \inf\left\{n \in \mathbb{N}\ \middle|\ \sup_{\D \in \mbf{\D}}  \Es{S \sim \D^n}{\L_\D(\Algo(S)) - \inf_{\h \in \cH} \L_\D(\h)} \le \epsilon \right\}
$$
Further the sample complexity of the learning problem over this family of distributions $\mbf{\D}$ is defined as $n^{\mbf{\D}}(\epsilon,\Z) = \inf_{\Algo} n^{\mbf{\D}}(\epsilon,\Algo,\Z)$.
\end{definition}

A concept closely related to statistical learning that plays an important role in characterizing learnability is the notion of generalization.  
\begin{definition}\index{generalization}
A learning rule $\A$ \emph{generalizes} with rate $\epsgen(n)$ under distribution $\Dcal$ if for all $n \in \mathbb{N}$,
\begin{equation}
\Es{S\sim\Dcal^n}{\abs{ \L(\A(S)) - \L_S(\A(S)) } } \le \epsgen(n) .
\end{equation}
The rule $\Algo$ is said to \emph{universally generalize} with rate $\epsgen(n)$ if it generalizes with rate $\epsgen(n)$ under all distributions $\Dcal$ over $\Z$.
\end{definition}

We note that other authors sometimes define ``consistent'', and thus also
``learnable'' as a combination of our notions of ``consistent'' and
``generalizing''.

\paragraph{\large Empirical Risk Minimization (Sample Average Approximation Approach) and Related Notions : \\}
~\\
\noindent Perhaps the most common approach used for learning problems is Empirical Risk Minimization (ERM) as it is referred to in machine learning terminology or Sample Average Approximation (SAA) approach as is widely referred to in the stochastic optimization terminology. The basic idea is for the learning algorithm to return hypothesis in target class $\cH$ that minimizes average loss over sample. Specifically a learning rule $\A_{\mrm{ERM}}$ is an \emph{ERM (Empirical Risk
Minimizer)} if it minimizes the average loss, that is :
\begin{equation}
\L_S(\A_{\mrm{ERM}}(S)) = \L_S(\hemp_S) = \inf_{\h \in \cH} \L_S(\h).
\end{equation}
where we use $\L_S(\hemp_{S}) = \inf_{\h \in \cH} \L_S(\h)$ to refer to the minimal empirical loss.  But since there might be several hypotheses minimizing the empirical risk, $\hemp_{S}$ does not refer to a specific hypotheses and there might be many rules which are all ERM's. While Empirical Risk Minimization (or equivalently the Sample Average Approximation approach) is a widely used learning rule, we now introduce a closely related concept of an Asymptotic Empirical Risk Minimizer which will play an important role in characterizing learnability in general problems.

\begin{definition}\index{AERM}
We say that a rule $\A$ is an \emph{AERM (Asymptotic Empirical Risk
Minimizer)} with rate $\epsapprox(n)$ under distribution $\Dcal$ if:
\begin{equation}\label{eq:AERM_def}
\Es{S\sim\Dcal^n}{\empL(\A(S)) - \empL(\hemp_{S}) } \le \epsapprox(n)
\end{equation}
Further, a learning rule is \emph{universally an AERM} with rate $\epsapprox(n)$, if it is an AERM with rate $\epsapprox(n)$ under all distributions $\Dcal$ over $\Z$. A 
\end{definition}

Yet another closely related notion of an approximate ERM is the following notion of alway AERM. 

\begin{definition}\index{AERM!always AERM}
A learning rule $\Algo$ is an \emph{always AERM} with rate $\epsapprox(n)$, if for \emph{any} sample $S$ of size $n$, it holds that 
\begin{equation}\label{eq:alwaysAERM_def}
\empL(\A(S)) -
\empL(\hemp_S) \le \epsapprox(n)
\end{equation}
\end{definition}

\section{Background}\label{sec:background}

\subsection{Learnability and Uniform Convergence}\index{uniform convergence!statistical}
\label{subsec:learnability_uniform_convergence}

As discussed in the introduction, a central notion for characterizing
learnability is uniform convergence. Formally, we say that uniform convergence
holds for a learning problem, if the empirical risks of hypotheses in the
hypothesis class converges to their population risk uniformly, with a
distribution-independent rate:
\begin{align}\label{eq:unifconv}
\sup_{\Dcal}~\Es{S\sim \Dcal^n}{\sup_{\h\in \cH}\left|\L(\h)-\L_S(\h)\right|} \stackrel{n\rightarrow \infty}{\longrightarrow} 0.
\end{align}
It is straightforward to show that if uniform convergence holds, then a problem can be learned with the ERM learning rule.

For binary classification problems (where $\Z=\X\times \{0,1\}$, each
hypothesis is a mapping from $\X$ to $\{0,1\}$, and
$\ell(\h;(\x,y))=\ind{\h(\x)\neq y}$), \cite{VapnikCh71} showed that the
finiteness of a simple combinatorial measure known as the VC-dimension implies
uniform convergence.  Furthermore, it can be shown that binary classification
problems with infinite VC-dimension are not learnable in a
distribution-independent sense. This establishes the condition of having finite VC-dimension, and thus also uniform convergence, as a necessary and sufficient
condition for learnability.

Such a characterization can also be extended to regression, such as regression
with squared loss, where $\h$ is now a real-valued function, and
$\ell(\h;(\x,y))=(\h(\x)-y)^2$.  The property of having finite fat-shattering
dimension at all finite scales now replaces the property of having finite
VC-dimension, but the basic equivalence still holds: a problem is learnable if
and only if uniform convergence holds (\cite{AlonBeCsHa97}, see also
\cite{AnthonyBa99}, Chapter 19). These results are usually based on clever
reductions to binary classification. However, the General Learning Setting that
we consider is much more general than classification and regression, and
includes setting where a reduction to binary classification is impossible.

To justify the necessity of uniform convergence even in the General
Learning Setting, Vapnik attempted to show that in this setting,
learnability with the ERM learning rule is equivalent to uniform
convergence (\cite{Vapnik98}). Vapnik noted that this result does not
hold, due to ``trivial'' situations. In particular, consider the case
where we take an arbitrary learning problem (with hypothesis class
$\cH$), and add to $\cH$ a single hypothesis $\tilde{\h}$ such
that $\ell(\tilde{\h},\z)<\inf_{\h\in \cH}\ell(\h,\z)$ for all $\z\in
\Z$ (see figure \ref{fig:strict_consistency} below). This learning
problem is now trivially learnable, with the ERM learning rule which
always picks $\tilde{\h}$. Note that no assumptions whatsoever are
made on $\cH$ - in particular, it can be arbitrarily complex, with
no uniform convergence or any other particular property. Note also
that such a phenomenon is not possible in the binary classification
setting, where $\ell(h;(\x,y))=\ind{\h(\x)\neq y}$, since on any $(x,y)$
we will have hypotheses with $\ell(h;(\x,y)) = \ell(\tilde{h};(\x,y))$ and
thus if $\cH$ is very complex (has infinite VC dimension) then on
every training set there will be many hypotheses with zero empirical
error.

\begin{figure}[t]
\begin{center}
\includegraphics[trim = 0mm 90mm 0mm 90mm, clip,
scale=0.6]{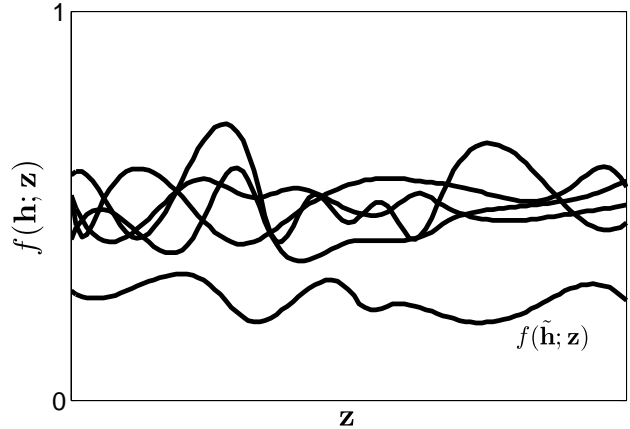}
\end{center}
\caption{An example of a ``trivial'' learning situation. Each line represents
some $\h\in \cH$, and shows the value of $\ell(\h,\z)$ for all $\z\in \Z$.
The hypothesis $\tilde{\h}$ dominates any other hypothesis (e.g.,
$\ell(\tilde{\h};\z)<\ell(\h;\z)$ uniformly for all $\z$), and thus the problem is
learnable without uniform convergence or any other property of $\cH$
.}\label{fig:strict_consistency}
\end{figure}

To exclude such ``trivial'' cases, Vapnik introduced a stronger notion of
consistency, termed as ``strict consistency'', which in our notation is defined
as
\[
\forall c\in \reals,~~~ \inf_{\h:\L(\h)\geq c} \L_S(\h) \stackrel{n\rightarrow
  \infty}{\longrightarrow} \inf_{\h:\L(\h)\geq c} \L(\h) ~,
\]
where the convergence is in probability. The intuition is that we require the
empirical risk of the ERM to converge to the lowest possible risk, even after
discarding all the ``good'' hypotheses whose risk is smaller than some
threshold.  Vapnik then showed that such strict consistency of the ERM is in
fact equivalent to (one-sided) uniform convergence, of the form
\begin{equation}
  \sup_{\h\in \cH}\left(\L(\h)-\L_S(\h)\right)
\stackrel{n\rightarrow \infty}{\longrightarrow} 0
\end{equation}
in probability. Note that this equivalence holds for every distribution
separately, and does not rely on universal consistency of the ERM.

These results seem to imply that up to ``trivial'' situations, a uniform
convergence property indeed characterizes learnability, at least using the ERM learning rule. However, as we will see later on, the situation is in fact not that simple.

\subsection{Various Complexity Measures and Uniform Convergence}

We begin by defining the class of functions $\F \subseteq \reals^{\Z}$ we will often refer to the loss class, as : 
\begin{align}\label{eq:lossclass}
\F(\H,\Z) := \left\{\z \mapsto \ell(\h,\z) :  \h \in \cH\right\}
\end{align}
We shall often drop the arguments $\H$ and $\Z$ and simply use $\F$ to mean $\F(\H,\Z)$. Notice that the function class is indexed by hypotheses in class $\cH$. Note that the notion uniform convergence in Equation \ref{eq:unifconv} can now be rewritten as :
$$
\sup_{\Dcal}~\Es{S\sim \Dcal^n}{\sup_{f\in \F}\left(\E{f}- \widehat{\En}_S[f]\right)} \stackrel{n\rightarrow \infty}{\longrightarrow} 0.
$$

Various complexity measures have been introduced in machine learning and empirical process theory literature to bound rates if uniform convergence and thus get upper bounds on learning rates for ERM/SA algorithms. We introduce and discuss about a few of these below. See Appendix \ref{App:B} for the relationships among these various complexity measures.

\begin{definition}[Statistical Rademacher Complexity]\index{Rademacher complexity!statistical}
The empirical Rademacher Complexity of a function class $\F \subset \reals^\Z$ given a sample $S = \{\z_1,\ldots, \z_{|S|}\}$ is defined as : 
$$
\widehat{\Rad}_S(\F) :=  \Es{\epsilon}{\sup_{f \in \F} \frac{1}{|S|} \sum_{i=1}^{|S|} \epsilon_t f(\z_i) }
$$
Further for any $n \in \mathbb{N}$, we define the worst-case statistical Rademacher complexity as :
\begin{align}\label{eq:radstat}
\Radstat_n(\F) = \sup_{\z_1,\ldots,\z_n \in \Z} \Es{\epsilon}{ \sup_{f \in \F} \frac{1}{n} \sum_{i=1}^n \epsilon_i f(\z_i)} 
\end{align}
\end{definition}

It can be easily shown that for any distribution $\D$,
$$
\Es{S\sim \Dcal^n}{\sup_{f\in \F}\left(\E{f}- \widehat{\En}_S[f]\right)} \le 2 \Es{S \sim \D^n}{ \widehat{\Rad}_S(\F)} \le 2 \Radstat_n(\F)
$$
and so the empirical Rademacher complexity and so also the worst case statistical Rademacher complexity provide upper bound on the rate of uniform convergence over any function class $\F$.

Another tool from empirical process theory that is often used to upper bound rates of uniform convergence is covering numbers of the function class. The statistical covering number of a function class is defined below.

\begin{definition}[Statistical Covering Number]	\index{covering number!statistical} \label{def:statcover}
A set $V \subset \reals^n$ is \emph{an $\alpha$-cover} (with respect to $\ell_p$-norm) of a function class $\F \subseteq \reals^\Z$ on a sample $S = \z_1,\ldots,\z_n$ if,
$$
\forall f \in \F,\ \forall \epsilon \in \{\pm1\}^n \ \exists v \in V \  \mrm{s.t.}  ~~~~ \left( \frac{1}{n} \sum_{t=1}^n |v_t - f(z_t)|^p \right)^{1/p} \le \alpha
$$
The \emph{statistical covering number} of a function class $\F$ on a given sample $S = \z_1,\ldots,\z_n $ is defined as 
$$
\Nstat_p(\alpha, \F, S) = \min\{|V| :  V \ \trm{is an }\alpha-\text{cover w.r.t. }\ell_p\trm{-norm of }\F \trm{ on } S\}.
$$
Further define $\Nstat_p(\alpha, \F , n) = \sup_{S \in \Z^n} \Nstat_p(\alpha, \F, S) $, the maximal $\ell_p$ covering number of $\F$ over samples of size $n$. 
\end{definition}

Pollard's bound \cite{Pollard84} and Dudley integral bound \cite{Dudley67} can be used to bound rates of uniform convergence in terms of covering numbers. See \ref{sec:radcov} in the appendix for a bound on statistical Rademacher complexity in terms of covering numbers through a refined Dudley bound. Yet another combinatorial tool that can be used to bound rates of uniform convergence is the so called fat-shattering dimension defined below.

\begin{definition}[Fat-Shattering Dimension] \index{fat-shattering!statistical}\label{def:statfat}
A sample $S = \z_1,\ldots,\z_d$ is said to be \emph{$\alpha$-shattered} by a function class $\F \subseteq \reals^\X$ , if there exists $s_1,\ldots,s_d \in \reals$ such that 
$$
\forall \epsilon \in \{\pm1\}^d , \ \exists f \in \F \ \ \ \trm{s.t. } \forall t \in [d], \  \epsilon_t (f(\z_t) - s_t) \ge \alpha/2
$$
The sequence $s_1,\ldots,s_d$ is called the \emph{witness to shattering}. The \emph{statistical fat-shattering dimension} $\fatstat_\alpha(\F, \Z)$ at scale $\alpha$ is the largest $d$ such that $\F$ $\alpha$-shatters some sample $S$ of size $d$. 
\end{definition}

One can upper bound covering numbers in terms of fat-shattering dimension. See Section \ref{sec:covfat} of Appendix \ref{app:complexity} for a bound on covering number in terms of fat shattering dimension of function class.

\subsection{Learnability and Stability}\label{subsec:learnability_stability}\index{stability}

Instead of focusing on the hypothesis class, and ensuring uniform
convergence of the empirical risks of hypothesis in this class, an
alternative approach is to directly control the variance of the
learning rule.  Here, it is not the complexity of the hypothesis class
which matters, but rather the way that the learning rule explores this
hypothesis class. This alternative approach leads to the notion of
stability in learning.  It is important to note that stability is a
property of a learning rule, not of the hypothesis class.

In the context of modern learning theory\footnote{In a more general
  mathematical context, stability has been around for much longer. The
  necessity of stability for so-called inverse problems to be well posed was
  first recognized by \cite{Hadamard02}. The idea of regularization (that is,
  introducing stability into ill-posed inverse problems) became widely known
  through the works of \cite{Tikhonov43} and \cite{Phillips62}. We return to
  the notion of regularization later on.}, the use of stability can be traced
back at least to the work of \cite{RogersWagner78}, which noted that the
sensitivity of a learning algorithm with regard to small changes in the sample
controls the variance of the leave-one-out estimate. The authors used this
observation to obtain generalization bounds (w.r.t. the leave-one-out estimate)
for the $k$-nearest neighbor algorithm. It is interesting to note that a
uniform convergence approach for analyzing this algorithm simply cannot work,
because the ``hypothesis class'' in this case has unbounded complexity. These
results were later extended to other ``local'' learning algorithms (see
\cite{DevroyeGyLu96} and references therein). In addition, practical methods
have been developed to introduce stability into learning algorithms, in
particular the Bagging technique introduced by \cite{Breiman96}.

Over the last decade, stability was studied as a generic condition for
learnability. \cite{KearnsRon99} showed that an algorithm operating on a
hypothesis class with finite VC dimension is also stable (under a certain
definition of stability).  \cite{BousquetEl02} introduced a strong notion of
stability (denoted as \emph{uniform stability}) and showed that it is a
sufficient condition for learnability, satisfied by popular learning algorithms
such as regularized linear classifiers and regressors in Hilbert spaces \index{Hilbert space}
(including several variants of SVM). \cite{KutinNi02} introduced several weaker variants of stability, and showed how they are sufficient to obtain
generalization bounds for algorithms stable in their sense.

The works cited above mainly considered stability as a \emph{sufficient} condition
for learnability. A more recent line of work
(\cite{RakhlinMuPo05},\cite{MukherjeeNPR06}) studied stability as a
\emph{necessary} condition for learnability. However, the line of argument is
specific to settings where uniform convergence holds and is necessary for
learning. With this assumption, it is possible to show that the ERM algorithm
is stable, and thus stability is also a necessary condition for learning.
However, as we will see later on in this chapter, uniform convergence is in fact
not necessary for learning in the General Learning Setting, and stability plays
there a key role which has nothing to do with uniform convergence.

Finally, it is important to note that the results cited above make use of many
different definitions of stability, which unfortunately are not always
comparable. All of them measure stability as the amount of change in the
algorithm's output as a function of small changes to the sample on which the
algorithm is run. However, ``amount of change to the output'' and
``small changes to the sample'' can be defined in many different ways.
``Amount of change to the output'' can mean change in risk, change in loss with
respect to particular examples, or supremum of change in loss over all
examples. ``Small changes to the sample'' usually mean either deleting one
example or replacing it with another one (and even here, one can talk about
removing/replacing one instance at random, or in some arbitrary manner).
Finally, this measure of change can be measured with respect to any arbitrary
sample, in expectation over samples drawn from the underlying distribution; or
in high probability over samples. .

\section{Failure of Uniform Convergence and ERM/SAA Approaches}\label{sec:ugc}
In this section, we study a special case of the General Learning Setting, where
there is a real gap between learnability and uniform convergence, in the sense
that there are non-trivial problems where no uniform convergence holds (not
even in a local sense), but they are still learnable. Moreover, some of these
problems are learnable with an ERM (again, without any uniform convergence),
and some are not learnable with an ERM, but rather with a different
mechanism. We also discuss why this peculiar behavior does not formally
contradict Vapnik's results on the equivalence of strict consistency of the ERM
and uniform convergence, as well as the important role that regularization
seems to play here, but in a different way than in standard theory.

\subsection{Learning without Uniform Convergence : Stochastic Convex Optimization}\label{subsec:ugc_fails}

A stochastic convex optimization problem is a special case of the General
Learning Setting discussed above, with the added constraints that the objective
function $\ell(\h;\z)$ is Lipschitz-continuous and convex in $\h$ for every $\z$,
and that $\cH$ is closed, convex and bounded. We will focus here on problems
where $\cH$ is a subset of a Hilbert space. A special case is the familiar
linear prediction setting, where $\z=(\x,y)$ is an instance-label pair, each
hypothesis $\h$ belongs to a subset $\cH$ of a Hilbert space, and
$\ell(\h;\x,y)=\ell(\inner{\h,\phi(\x)},y)$ for some feature mapping
$\phi$ and a loss function $\ell : \reals \times \Y \to \reals$, which is convex w.r.t. its first argument.

The situation in which the stochastic dependence on $\h$ is linear, as in the
preceding example, is fairly well understood.  When the domain $\cH$ and the
mapping $\phi$ are bounded, we have uniform convergence, in the sense that
$|L(\h)-\empL(\h)|$ is uniformly bounded over all $\h \in \cH$ (see
\cite{SridharanSrSh08}). This uniform convergence of $\empL(\h)$ to $L(\h)$
justifies choosing the empirical minimizer $\hemp_{S} = \arg\min_{\h}
\empL(\h)$, and guarantees that the expected value of $L(\hemp_{S})$ converges
to the optimal value $L^* = \inf_{\h} L(\h)$.

Even if the dependence on $\h$ is not linear, it is still possible to establish uniform convergence (using covering number arguments) provided that $\cH$ is finite dimensional. Unfortunately, when we turn to infinite dimensional hypothesis spaces, uniform convergence might not hold and the problem might not be learnable with empirical minimization. Surprisingly, it turns out that this does not imply that the problem is unlearnable. We will show that using a regularization mechanism, it is possible to devise a learning algorithm for any stochastic convex optimization problem, even when uniform convergence does not hold. This mechanism is fundamentally related to the idea of stability, and will be a good starting point for our more general treatment of stability and learnability in the next section of this chapter.

We now turn to discuss our first concrete example. Consider the convex
stochastic optimization problem given by
\begin{equation}
\fcx(\h;(\x,\balpha)) ~=~ \norm{ \balpha \ast (\h-\x) } ~=~ \sqrt{\sum_i
\balpha^2[i] (\h[i]-\x[i])^2} ~,\label{eq:fcx}
\end{equation}
where for now we let $\cH$ to be the $d$-dimensional unit sphere $\cH =
\left\{ \h \in \reals^d \,:\, \norm{\h} \leq 1 \right\}$, we let $\z =
(\x,\balpha)$ with $\balpha \in [0,1]^d$ and $\x \in \cH$, and we define $u
\ast v$ to be an element-wise product.  We will first consider a sequence of
problems, where $d=2^n$ for any sample size $n$, and establish that we cannot
expect a convergence rate which is independent of the dimensionality $d$.  We
then formalize this example in infinite dimensions.

One can think of the problem in \eqref{eq:fcx} as that of finding the ``center''
of an unknown distribution over $\x \in \reals^d$, where we also have stochastic
per-coordinate ``confidence'' measures $\balpha[i]$.  We will actually focus on
the case where some coordinates are missing, namely that $\balpha[i]=0$.

Consider the following distribution over $(\x,\balpha)$: $\x=0$ with
probability one, and $\balpha$ is uniform over $\{0,1\}^d$.  That is,
$\balpha[i]$ are i.i.d.~uniform Bernoulli.  For a random sample
$(\x_1,\balpha_1),\ldots,(\x_n,\balpha_n)$ if $d > 2^n$ then we have that with
probability greater than $1-e^{-1}>0.63$, there exists a coordinate $j \in
1\ldots d$ such that all confidence vectors $\balpha_i$ in the sample are zero
on the coordinate $j$, that is $\balpha_i[j]=0$ for all $i=1..n$.  Let $\e_j
\in \cH$ be the standard basis vector corresponding to this coordinate.  Then
\[
\empFcx(\e_j) ~=~\frac{1}{n}\sum_{i=1}^n \norm{ \balpha_i \ast (\e_j-0) }
~=~\frac{1}{n}\sum_{i=1}^n \abs{\balpha_i[j]} ~=~ 0,
\]
where $\empFcx(\cdot)$ denotes the empirical risk w.r.t. the function
$\fcx(\cdot)$. On the other hand, letting $\Fcx(\cdot)$ denote the actual risk
w.r.t. $\fcx(\cdot)$, we have
\[
\Fcx(\e_j) ~=~\Es{\x,\balpha}{\norm{\balpha \ast (\e_j-0)}} ~=~
\Es{\x,\balpha}{\abs{\balpha[j]}} ~=~ 1/2.
\]
Therefore, for any $n$, we can construct a convex Lipschitz-continuous
objective in a high enough dimension such that with probability at least $0.63$
over the sample, $\sup_{\h} \abs{ \Fcx(\h)-\empFcx(\h) } \geq 1/2$.
Furthermore, since $\ell(\cdot;\cdot)$ is non-negative, we have that $\e_j$ is an
empirical minimizer, but its expected value $\Fcx(\e_j)=1/2$ is far from the
optimal expected value $\min_{\h} \Fcx(\h) = \Fcx(0) = 0$.

To formalize the example in a sample-size independent way, take $\cH$ to be
the unit sphere of an infinite-dimensional Hilbert space with orthonormal basis
$\e_1,\e_2,\ldots$, where for $\mathbf{v} \in \cH$, we refer to its
coordinates $\mathbf{v}[j]=\ip{\mathbf{v}}{\e_j}$ w.r.t~this basis.  The
confidences $\balpha$ are now a mapping of each coordinate to $[0,1]$.  That
is, an infinite sequence of reals in $[0,1]$.  The element-wise product
operation $\balpha \ast \mathbf{v}$ is defined with respect to this basis and
the objective function $\fcx(\cdot)$ of \eqref{eq:fcx} is well defined
in this infinite-dimensional space.

We again take a distribution over $z=(\x,\balpha)$ where $\x=0$ and $\balpha$
is an infinite i.i.d.~sequence of uniform Bernoulli random variables
(that is, a Bernoulli process with each $\alpha_i$ uniform over
$\{0,1\}$ and independent of all other $\alpha_j$).  Now, for any finite sample there is almost surely a coordinate $j$ with $\balpha_i[j]=0$ for all $i$, and so we a.s.~have an empirical minimizer $\empFcx(\e_j)=0$ with $\Fcx(\e_j)=1/2>0=\Fcx(0)$.

As a result, we see that the empirical values $\empFcx(\h)$ do not converge
uniformly to their expectations, and empirical minimization is not guaranteed
to solve the problem. Moreover, it is possible to construct a sharper
counterexample, in which the {\em unique} empirical minimizer $\hemp_{S}$ is far
from having optimal expected value.  To do so, we augment $\fcx(\cdot)$ by a
small term which ensures its empirical minimizer is unique, and far from the
origin. Consider:
\begin{equation}\label{eq:fcxx}
\fcxx(\h;(\x,\balpha)) = \fcx(\h;(\x,\balpha)) + \epsilon\sum_i 2^{-i}
(\h[i]\!-\!1)^2
\end{equation}
where $\epsilon=0.01$.  The objective is still convex and
$(1+\epsilon)$-Lipschitz.  Furthermore, since the additional term is strictly
convex, we have that $\fcxx(\h;\z)$ is strictly convex w.r.t.~$\h$ and so the
empirical minimizer is unique.

Consider the same distribution over $z$: $\x=0$ while $\balpha[i]$ are
i.i.d.~uniform zero or one.  The empirical minimizer is the minimizer of
$\empFcxx(\h)$ subject to the constraints $\norm{\h}\leq 1$.  Identifying the
solution to this constrained optimization problem is tricky, but fortunately
not necessary.  It is enough to show that the optimum of the {\em
  unconstrained} optimization problem $h^*_{\textsc{uc}} = \arg\min
\empFcxx(\h)$ (without constraining $\h \in \cH$) has norm
$\norm{\h^*_{\textsc{uc}}} \geq 1$.  Notice that in the unconstrained problem,
whenever $\balpha_i[j]=0$ for all $i=1..n$, only the second term of $\fcxx$
depends on $\h[j]$ and we have $\h^*_{\textsc{uc}}[j]=1$. Since this happens
a.s.~for some coordinate $j$, we can conclude that the solution to the
constrained optimization problem lies on the boundary of $\cH$, that is
$\norm{\hemp_{S}}=1$.  But for such a solution we have $$\Fcxx(\hemp_{S}) \geq
\Es{\balpha}{\sqrt{\sum_i \balpha[i] \hemp_{S}^2[i]}} \geq \Es{\balpha}{\sum_i
  \balpha[i] \hemp_{S}^2[i]} = \sum_i
 \hemp_{S}^2[i] \Es{\balpha}{ \balpha[i]} = \half\norm{\hemp_{S}}^2 =
 \half ,$$ while
$L^* \leq L(0) = \epsilon$.

In conclusion, no matter how big the sample size is, the unique empirical
minimizer $\hemp_{S}$ of the stochastic convex optimization problem in
\eqref{eq:fcxx} is a.s.~much worse than the population optimum, $L(\hemp_{S})
\geq \half > \epsilon \geq L^*$, and certainly does not converge to it.

\subsubsection{Still Learnable Using Stochastic Approximation (SA) Approach :}
So far we established that the learning problem specified in Equation \ref{eq:fcxx} fails to satisfy uniform convergence and further ERM/SAA approach fails to provide any non-trivial guarantee. Specifically we saw that with the particular distribution over $\alpha$ as the uniform distribution over $\{\pm1\}^{\mathbb{N}}$, we have that $\L(\hemp_S) \geq \half$ but however $\L^* \le \epsilon = 0.01$. However we shall show that this problem is still learnable. Specifically we will see that using the so called Stochastic Approximation (SA) approach, we can get a guarantee over learning rate of order $1/\sqrt{n}$ for the problem. 

The first thing we already noticed was that the problem specified in Equation \ref{eq:fcxx} is $(1 + \epsilon)$-Lipschitz. Further note that the hypothesis set $\H$ is the unit ball in a Hilbert space and that the problem is convex in its first argument. In the second part of this dissertation we will formally define Stochastic Approximation (SA) approach for convex learning problems and in a more general way than usually described in literature. For now we will informally use the term Stochastic Approximation approach to refer to stochastic gradient descent (or online gradient descent followed by averaging). To this end consider the learning algorithm for the problem which given sample $S = \{(\x_1,\alpha_1), \ldots,(\x_n,\alpha_n)\}$ is described below : 

\begin{center}
\fbox{
\begin{minipage}[t]{0.52\textwidth} 
{\bf Stochastic Gradient Descent for problem \eqref{eq:fcxx}:   } \vspace{-0.15in}

{\bf Initialize} $\h_1 = \mbf{0}$ and $\eta = 1/\sqrt{|S|}$ \\
{\bf for} $t = 1$ {\bf to} $|S|$ \vspace{-0.12in}
\begin{itemize}
\item[] $\h'_{t+1} = \h_t - \eta \nabla \ell_{\eqref{eq:fcxx}}(\h_t;(\x_t,\alpha_t))$\vspace{-0.12in}
\item[] $\h_{t+1} = \left\{\begin{array}{cc}
\h'_{t+1}  & \textrm{if } \norm{\h'_{t+1}} \le 1\\
\frac{\h'_{t+1}}{\norm{\h'_{t+1}}}  & \textrm{otherwise}
\end{array}\right.
$\vspace{-0.12in}
\end{itemize}
{\bf end for}\\
{\bf Return} $\overline{\h}_S = \frac{1}{|S|} \sum_{t=1}^{|S|} \h_t$.
\end{minipage}
}
\end{center}
Since the problem is convex and $(1 + \epsilon)$-Lipschitz and since $\H$ is the unit ball in the Hilbert space it follows (for instance from \cite{Zinkevich03} + online to batch conversion) that irrespective of which distribution $\D$ over instances we use,
$$
\L_\D(\overline{h}_S) - \inf_{\h \in \H} \L_\D(\h) \le \sqrt{\frac{2 (1 + \epsilon)}{n}}
$$
Thus we can conclude that the problem is learnable and in fact enjoys a rate of order $\frac{1}{\sqrt{n}}$, yet as we already say both uniform convergence and ERM (SAA approach) fails. 

\subsection{Learnability via Stability : Role of Regularization}\label{subsec:learn_via_stab}

At this point, we have seen an example in the stochastic convex optimization framework where uniform convergence does not hold, and the ERM algorithm fails. Yet we saw that the problem was learnable using SA approach. We will now see an alternate explanation for why the problem is learnable that mainly uses stability to explain the success. We will specifically consider an algorithm that minimizes a regularized average loss and show that this algorithm can guarantee a similar learning rate as the stochastic gradient descent approach. We will show that the regularization induces stability to the learning problem and this stability in turn assures learnability.

Given a stochastic convex optimization problem with an objective function
$\ell(\h;\z)$, consider a \emph{regularized} version of it: instead of minimizing
the expected risk $\Es{\z}{\ell(\h;\z)}$ over $\h\in \cH$, we will try to
minimize
\[
\Es{\z}{\ell(\h;\z)+\frac{\lambda}{2} \norm{\h}^2}
\]
for some $\lambda>0$. Notice that this is simply a stochastic convex
optimization problem w.r.t. the objective function
$\ell(\h;\z)+\frac{\lambda}{2}\norm{\h}^2$. We will show that this regularized
problem \emph{is} learnable using the ERM algorithm (namely, by attempting to
minimize $\frac{1}{n}\sum_{i}\ell(\h;\z_i)+\frac{\lambda}{2} \norm{\h}^2$), by
showing that the ERM algorithm is \emph{stable}. By taking $\lambda \rightarrow
0$ at an appropriate rate as the sample size increases, we are able to solve
the original stochastic problem optimization problem, w.r.t. $\ell(\h;\z)$.

The key characteristic of the regularized objective function we need is that it
is $\lambda$-\emph{strongly convex}. Formally, we say that a real function
$g(\cdot)$ over a domain $\cH$ in a Hilbert space is $\lambda$-strongly
convex (where $\lambda\geq 0$), if the function
$g(\cdot)-\tfrac{\lambda}{2}\|\cdot\|^2$ is convex. In this case, it
is easy to verify that if
$\h$ minimizes $g$ then
\[
\forall \h',~ g(\h')-g(\h) \ge \tfrac{\lambda}{2} \|\h'-\h\|^2 ~.
\]
When $\lambda=0$, strong convexity corresponds to standard convexity. In particular, it is
immediate from the defintion that $\ell(\h;\z)+\frac{\lambda}{2}\norm{\h}^2$ is
$\lambda$-strongly convex w.r.t. $\h$ (assuming $\ell(\h;\z)$ is convex).

The arguments above are formalized in the following two theorems:
\begin{theorem} \label{thm:ERMHazan}
Consider a stochastic convex optimization problem such that $\ell(\h;\z)$ is $\lambda$-strongly convex and $\Lipf$-Lipschitz with respect to $\h \in \cH$. Let $\z_1,\ldots,\z_n$ be an i.i.d.~sample and let $\hemp_{S}$ be the empirical minimizer. Then, we have that :
\begin{equation}\label{eq:thm:ERMHazan}
\Es{S}{L(\hemp_{S}) - \inf_{\h \in \H} L(\h)} ~\leq~ \frac{4\Lipf^2}{\lambda\,n} ~.
\end{equation}
\end{theorem}

\begin{theorem} \label{thm:RERM}
Let $f : \cH \times \Z \to \reals$ be such that $\cH$ is bounded by $B$
and $\ell(\h,\z)$ is convex and $\Lipf$-Lipschitz with respect to $\h$.  Let $\z_1,\ldots,\z_n$ be an i.i.d. sample,  let $\lambda = \sqrt{\tfrac{16\Lipf^2}{B^2\,n}}$ and let $\hemp_{\lambda}$ be the
minimizer of 
\begin{equation} \label{eqn:RERM}
\hemp_{\lambda} = \min_{\h \in \cH} \left( \tfrac{1}{n} \sum_{i=1}^n
\ell(\h,\z_i)+\tfrac{\lambda}{2} \norm{\h}^2\right)
\end{equation}
Then, we have that
$$ 
\Es{S}{L(\hemp_{\lambda}) -  \inf_{\h \in \H} L(\h)} ~\leq~ 4 \sqrt{\frac{\Lipf^2 B^2}{n}} \left(1 + \frac{8}{n}\right) ~.
$$
\end{theorem}

From the above theorem, we see that regularization is essential for convex
stochastic optimization. It is important to note that even for the strongly
convex optimization problem in \thmref{thm:ERMHazan}, where the ERM algorithm
does work, it is not due to uniform convergence. To see this, consider
augmenting the objective function $\fcx(\cdot)$ from \eqref{eq:fcx} with a
strongly convex term:
\begin{equation} \label{eq:fcs}
\fcs(\h;\x,\balpha) = \fcx(\h;\x,\balpha) + \frac{\lambda}{2}\norm{\h}^2.
\end{equation}
The modified objective $\fcs(\cdot;\cdot)$ is $\lambda$-strongly convex and
$(1+\lambda)$-Lipschitz over $\cH = \left\{ \h : \norm{\h} \leq 1
\right\}$ and thus satisfies the conditions of Theorem \ref{thm:ERMHazan}. Now,
consider the same distribution over $z=(\x,\balpha)$ used earlier: $\x = 0$ and
$\balpha$ is an i.i.d.~sequence of uniform zero/one Bernoulli variables.
Recall that almost surely we have a coordinate $j$ that is never ``observed'',
namely such that $\forall_i \balpha_i[j]=0$.  Consider a vector $t \e_j$ of
magnitude $0<t\leq 1$ in the direction of this coordinate.  We have that
$\empFcs(t \e_j) = \tfrac{\lambda}{2}t^2$ (where $\empFcs(\cdot)$ is the
empirical risk w.r.t. $\fcs(\cdot)$) but $\Fcs(t\e_j) = \tfrac{1}{2}
t+\tfrac{\lambda}{2} t^2$.  Hence, letting $\Fcs(\cdot)$ denote the risk
w.r.t. $\fcs(\cdot)$, we have that $\Fcs(t\e_j) - \empFcs(t\e_j) = t/2$. In
particular, we can set $t=1$ and establish $\sup_{\h \in \cH} (\Fcs(\h) -
\empFcs(\h)) \geq \half$ regardless of the sample size.

We see then that the empirical averages $\empFcs(\h)$ do {\em not} converge
uniformly to their expectations. Moreover, the example above shows that there
is no uniform convergence even in a local sense, namely over all hypotheses
whose risk is close enough to $L^*$, or those close enough to the minimizer of
$\fcs(\h;\x,\balpha)$.

\subsubsection{Regularization Vs Constrained Minimization}

The technique of regularizing the objective function by adding a ``bias'' term
is old and well known. In particular, adding $\norm{\h}^2$ is the so-called
Tikhonov Regularization technique, which has been known for more than half a
century (see \cite{Tikhonov43}). However, the role of regularization in our
case is very different than in familiar settings such as $\ell_2$
regularization in SVMs and $\ell_1$ regularization in LASSO.  In those
settings regularization serves to constrain our domain to a low-complexity
domain (e.g., low-norm predictors), where we rely on uniform convergence.  In
fact, almost all learning guarantees that we are aware of can be expressed in
terms of some sort of uniform convergence.

In our case, constraining the norm of $\h$ does \emph{not} ensure uniform
convergence.  Consider the example $\fcx(\cdot)$ we have seen earlier.  Even
over a restricted domain $\cH_r = \left\{ \h : \norm{\h} \leq r \right\}$,
for arbitrarily small $r>0$, the empirical averages $\empL(\h)$ do \emph{not}
uniformly converge to $L(\h)$. Furthermore, consider replacing the
regularization term $\lambda \norm{\h}^2$ with a constraint on the norm of
$\norm{\h}$, namely, solving the problem
\begin{equation}
\halgo_r = \arg\min_{\norm{\h} \leq r} \empL(\h)
\end{equation}
We cannot solve the stochastic optimization problem by setting $r$ in a
distribution-independent way (i.e., without knowing the solution...).  To see
this, note that when $\x=0$ a.s.~we must have $r \rightarrow 0$ to ensure
$L(\halgo_r) \rightarrow L^*$.  However, if $\x=\e_1$ a.s., we must set $r
\rightarrow 1$.  No constraint will work for all distributions over
$\Z=(\X,\balpha)$!  This sharply contrasts with traditional uses of
regularization, where learning guarantees are typically stated in
terms of a constraint on the norm rather than in terms of a parameter such as
$\lambda$, and adding a regularization term of the form
$\tfrac{\lambda}{2}\norm{\h}^2$ is viewed as a proxy for bounding the norm
$\norm{\h}$.

\subsection{Contradiction to Vapnik?}

In \subsecref{subsec:learnability_uniform_convergence}, we discussed how
Vapnik showed that uniform convergence is in fact necessary for learnability
with the ERM.  At first glance, this might seem confusing in light of the
examples presented above, where we have problems learnable with the ERM without
uniform convergence whatsoever.

The solution for this apparent paradox is that our examples are not ``strictly consistent'' in Vapnik's sense. Recall that in order to exclude ``trivial'' cases,
Vapnik defined strict consistency of empirical minimization as (in our
notation):
\begin{equation}\label{eq:strictdef} \index{consistency!strict}
\forall c \in \reals,~~\inf_{\h:L(\h)\geq c} \L_S(\h) \longrightarrow
\inf_{\h:L(\h)\geq c} L(\h) ~,
\end{equation}
where the convergence is in probability.  This condition
indeed ensures that $L(\hemp_{S}) \prightarrow L^*$.  Vapnik's Key Theorem
on Learning Theory \cite[Theorem 3.1]{Vapnik98} then states that \emph{strict}
consistency of empirical minimization is equivalent to
one-sided \footnote{``One-sided'' meaning requiring only $\sup (L(\h) -
  \L_S(\h)) \longrightarrow 0$, rather then $\sup \abs{L(\h) -
    \L_S(\h)} \longrightarrow 0$.} uniform convergence. In the example presented above, even though Theorem \ref{thm:ERMHazan}
establishes $\Fcs(\hemp_{S}) \prightarrow L^*$, the consistency isn't
``strict'' by the definition above.  To see this, for any $c>0$, consider the
vector $t\e_j$ (where $\forall_i \balpha_i[j]=0$) with $t=2c$.  We have
$\Fcs(t\e_j) = \half t + \tfrac{\lambda}{2} t^2 > c$ but $\empFcs(t \e_j) =
\tfrac{\lambda}{2}t^2 = 2\lambda c^2$.  Focusing on $\lambda=\half$ we get:
\begin{equation}\label{eq:notstrict}
\inf_{\Fcs(\h) \geq c} \empFcs(\h) \leq c^2
\end{equation}
almost surely for any sample size $n$, violating the strict consistency
requirement \eqref{eq:strictdef}.

We emphasize that stochastic convex optimization is far from ``trivial'' in
that there is no dominating hypothesis that will always be selected.  Although
for convenience of analysis we took $\x=0$, one should think of situations in
which $\x$ is stochastic with an unknown distribution. This shows that uniform convergence is a sufficient, but not at all necessary, condition for consistency of empirical minimization in non-trivial settings.


\section{Stability of Learning Rules}\label{sec:stability}\index{stability|(}

In the previous section, we have shown that in the General Learning Setting, it
is possible for problems to be learnable without uniform convergence, in sharp
contrast to previously considered settings. The key underlying mechanism which
allowed us to learn is stability. In this section, we study the connection
between learnability and stability in greater depth, and show that stability
can in fact \emph{characterize} learnability. Also, we will see how various
``common knowledge facts'', which we usually take for granted and are based on a
``uniform convergence equivalent to learnability'' assumption, do not hold in the
General Learning Setting, and things can be much more delicate.

We will refer to settings where learnability is equivalent to uniform
convergence as ``supervised classification'' settings. While supervised
classification does not encompass all settings where this equivalence holds,
most equivalence results refer to it either explicitly or implicitly (by
reduction to a classification problem).


We start by giving the exact definition of the stability notions that we will
use. As discussed earlier, there are many possible stability measures, some of
which can be used to obtain results of a similar flavor to the ones below. The
definition we use seems to be the most convenient for the goal of
characterizing learnability in the General Learning Setting. In
subsection \ref{subsec:rovsloo}, we provide a few illustrating examples to the subtle differences that can arise from slight variations in the stability measure.

Our two stability notions are based on replacing one of the training sample
instances.  For a sample $S$ of size $n$, let $S^{(i)} =
\{\z_1,...,\z_{i-1},\z_i',\z_{i+1},...,\z_n\}$ be a sample obtained by
replacing the $i$-th observation of $S$ with some different instance $\z_i'$.
When not discussed explicitly, the nature of how $\z_i'$ is obtained should be
obvious from context.

\begin{definition}\label{def:alwaysstable}  \index{stability!(replace one) RO!uniform}
A rule $\A$ is {\bf uniform-RO stable}\footnote{RO is short for ``replace-one''.} with rate $\epsstable(n)$, if for all possible $S^{(i)}$ and any $\z'
\in \Z$,
\begin{equation*}
\frac{1}{n}\sum_{i=1}^n \abs{ \ell(\A(S^{(i)});\z') - \ell(\A(S);\z')} \le
\epsstable(n).
\end{equation*}
\end{definition}



\begin{definition}\label{def:avgSstable} \index{stability!(replace one) RO!average}
A rule $\A$ is {\bf average-RO stable} with rate $\epsstable(n)$ under
distributions $\Dcal$ if
\begin{equation*}
\abs{ \frac{1}{n}\sum_{i=1}^{n} \Es{S\sim\Dcal^n,
    (\z_1',...,\z_n')\sim\Dcal^n}{ \ell(\A(S^{(i)});\z'_i) - \ell(\A(S);\z'_i)} } \le
\epsstable(n).
\end{equation*}
\end{definition}
Note that this definition corresponds to assuming that the expected empirical risk of the learning rule converges to the expected risk - see \lemref{lem:pseudo}.

We say that a rule is {\em universally} stable with rate $\epsstable(n)$, if
the stability property holds with rate $\epsstable(n)$ for all distributions.

\begin{claim}
Uniform-RO stability with rate $\epsstable(n)$ implies average-RO stability with rate $\epsstable(n)$.
\end{claim}

The following subsection contains a brief literature survey on various definitions of stability and discusses and illustrates the differences between these various notions. The main results in this chapter are only based on the two definitions of stability we introduced so far and the following subsection is mainly for pedantic purposes and the reader may choose to skip it.

\subsection{Comparison with Existing Literature and Other Notions of Stability}

The existing literature on stability in learning, briefly surveyed in
\subsecref{subsec:learnability_stability}, utilizes many different stability
measures. All of them measure the amount of change in the algorithm's output as
a function of small changes to the sample on which the algorithm is run.
However, they differ in how ``output'', ``amount of change to the output'', and
``small changes to the sample'' are defined. In
\secref{sec:learnability_general}, we used three stability measures. Roughly
speaking, one measure (average-RO stability) is the expected change in the
objective value on a particular instance, after that instance is replaced with
a different instance. The second measure and third measure (uniform-RO
stability and strongly-uniform-RO stability respectively) basically deal with
the maximal possible change in the objective value with respect to a
particular instance, by replacing a single instance in the training set.
However, instead of measuring the objective value on a specific instance, we
could have measured the change in the risk of the returned hypothesis, or any
other distance measure between hypotheses. Instead of replacing an instance,
we could have talked about adding or removing one instance from the sample,
either in expectation or in some arbitrary manner. Such variations are common
in the literature.

To relate our stability definitions to the ones in the literature, we note that
our definitions of uniform-RO stability and strongly-uniform-RO stability are
somewhat similar to uniform stability (\cite{BousquetEl02}), which in our
notation is defined as $\sup_{S,\z} \max_i |\ell(\A(S;\z))-\ell(\A(S^{\setminus
  i});\z)|$, where $S^{\setminus i}$ is the training sample $S$ with instance
$\z_i$ removed. Compared to uniform-RO stability, here we measure maximal
change over any particular instance, rather than average change over all
instances in the training sample. Also, we deal with removing an instance
rather than replacing it. Strongly-uniform-RO stability is more similar, with
the only formal difference being removal vs. replacement of an instance.
However, the results for uniform stability mostly assume deterministic
learning rules, while in this work we have used strongly-uniform-RO stability
solely in the context of randomized learning rules. For deterministic learning
rules, the differences outlined above are sufficient to make uniform stability
a strictly stronger requirement than uniform-RO stability, since it is easy to
come up with learning problems and (non-symmetric) learning rules which are
uniform-RO stable but not uniformly stable. Moreover, we show 
that uniform-RO stable AERM's characterize learnability, while it is well
known that uniformly stable AERM's are not necessary for learnability (see
\cite{KutinNi02}). For the same reason, our notion of strongly-uniform-RO
stability is apparently too strong to characterize learnability when we deal
with deterministic learning rules, as opposed to randomized learning rules.

Our definition of average-RO stable is similar to ``average stability'' defined
in \cite{RakhlinMuPo05}, which in our notation is defined as $\Es{S\sim
  \Dcal^{n},\z'_1}{\ell(\A(S^{(i)});\z_1)-\ell(\A(S);\z_1)}$. Compared to average-RO
stability, the main difference is that the change in the objective value is
measured with respect to $\z_1$ rather than an average over $\z_i$ for all $i$, and stems from the assumption there that the learning algorithm is symmetric. Notice however that we do not make such an assumption.

For an elaborate study on other stability notions and their relationships, see \cite{KutinNi02}.

Unfortunately, many of the stability notions in the literature are
incomparable, and even slight changes in the definition radically affect their behavior. WE shall now investigate these differences in more detail.

\subsubsection{LOO Stability vs. RO Stability}\label{subsec:rovsloo}

The stability definitions we saw introduced up to now were all based on the idea of replacing one instance in the training sample by another instance (e.g., ``RO'' or ``replace-one'' stability). An alternative set of definitions can be
obtained based on \emph{removing} one instance in the training sample (e.g.,
``LOO'' or ``leave-one-out'' stability).  Despite seeming like a
small change, it turns out there is a considerable discrepancy in terms of the
obtainable results, compared to RO stability. In this subsection, we wish to
discuss these discrepancies, as well as show how small changes to the
stability definition can materially affect its strength.

Specifically, we consider the following four LOO stability measures, each
slightly weaker than the previous one. The first and last are similar to our
notion of uniform-RO stability and average-RO stability respectively. However,
we emphasize that RO stability and LOO stability are in general incomparable
notions, as we shall see later on. Also, we note that some of these definitions
appeared in previous literature. For instance, the notion of ``all-i-LOO'' below
has been studied by several authors under different names
\cite{BousquetEl02,MukherjeeNPR06,RakhlinMuPo05}. The notation $S^{\setminus
  i}$ below refer to a training sample $S$ with instance $\z_i$ removed.

\begin{definition}\label{def:loo_alwaysstable}  \index{stability!(leave one out) LOO!uniform}
A rule $\A$ is {\bf uniform-LOO stable} with rate $\epsstable(n)$ if for all
samples $S$ of $n$ points and for all $i$:
\begin{equation*}
\abs{ \ell(\A(S^{\setminus i});\z_i) - \ell(\A(S);\z_i)} \le \epsstable(n).
\end{equation*}
\end{definition}

\begin{definition}\label{def:loo_PHstable}\index{stability!(leave one out) LOO!all-i}
A rule $\A$ is {\bf all-i-LOO stable} with rate $\epsstable(n)$ under
distribution $\Dcal$ if for all $i$:
\begin{equation*}
\Es{S\sim\Dcal^n}{\abs{ \ell(\A(S^{\setminus
i});\z_i) - \ell(\A(S);\z_i)} } \le \epsstable(n).
\end{equation*}
\end{definition}

\begin{definition}\label{def:loo_avgistable}\index{stability!(leave one out) LOO}
A rule $\A$ is {\bf LOO stable} with rate $\epsstable(n)$ under distribution
$\Dcal$ if
\begin{equation*}
\frac{1}{n}\sum_{i=1}^{n} \Es{S\sim\Dcal^n}{  \abs{ \ell(\A(S^{\setminus
i});\z_i) - \ell(\A(S);\z_i)} } \le \epsstable(n).
\end{equation*}
\end{definition}

\begin{definition}\label{def:loo_avgSstable}\index{stability!(leave one out) LOO!on average}
A rule $\A$ is {\bf on-average-LOO stable} with rate $\epsstable(n)$ under
distribution $\Dcal$ if
\begin{equation*}
\abs{ \frac{1}{n}\sum_{i=1}^{n} \Es{S\sim\Dcal^n}{  \ell(\A(S^{\setminus
i});\z_i) - \ell(\A(S);\z_i)} } \le \epsstable(n).
\end{equation*}
\end{definition}

While some of the definitions above might look rather similar, we show below
that each one is strictly stronger than the other. Example
\ref{ex:average_not_loo} is interesting in its own right, since it presents a
learning problem and an AERM that is universally consistent, but not LOO
stable. While this is possible in the General Learning Setting, in supervised
classification every such AERM has to be LOO stable (this is essentially
proven in \cite{MukherjeeNPR06}).

\begin{example}\label{ex:loo_not_uniform}
There exists a learning problem with a universally consistent and
all-i-LOO stable learning rule, but there is no universally consistent and
uniform LOO stable learning rule.
\end{example}
\begin{proof}
This example is taken from \cite{KutinNi02}. Consider the hypothesis space
$\{0,1\}$, the instance space $\{0,1\}$, and the objective function $\ell(\h,z) =
|h-z|$.

It is straightforward to verify that an ERM is a universally consistent
learning rule. It is also universally all-i-LOO stable, because removing an
instance can change the hypothesis only if the original sample had an equal
number of $0$'s and $1$'s (plus or minus one), which happens with probability
at most $O(1/\sqrt{n})$ where $n$ is the sample size. However, it is not hard
to see that the only uniform LOO stable learning rule, at least for large
enough sample sizes, is a constant rule which always returns the same
hypothesis $h$ regardless of the sample. Such a learning rule is obviously not
universally consistent.
\end{proof}

\begin{example}\label{ex:not_symmetric}
There exists a learning problem with a universally consistent and LOO-stable
AERM, which is not symmetric and is not all-i-LOO stable.
\end{example}
\begin{proof}
Let the instance space be $[0,1]$, the hypothesis space $[0,1]\cup 2$, and the
objective function $\ell(h,z)=\ind{h=z}$. Consider the following learning rule $\A$:
given a sample, check if the value $z_1$ appears more than once in the
sample. If no, return $z_1$, otherwise return $2$.

Since $L_S(2)=0$, and $z_1$ returns only if this value constitutes $1/n$ of the
sample, the rule above is an AERM with rate $\epsapprox(n)=1/n$. To see
universal consistency, let $\Pr{z_1}=p$. With probability $(1-p)^{n-2}$,
$z_1\notin\{z_2,\ldots,z_n\}$, and the returned hypothesis is $z_1$, with
$L(z_1)=p$. Otherwise, the returned hypothesis is $2$, with $L(2)=0$. Hence
$\Es{S}{L(\A(S))}\leq p(1-p)^{n-2}$, which can be easily verified to be at most
$1/(n-1)$, so the learning rule is consistent with rate $\epscon(n)\leq
1/(n-1)$. To see LOO-stability, notice that our learning hypothesis can change
by deleting $z_i$, $i>1$, only if $z_i$ is the only instance in
$z_2,\ldots,z_n$ equal to $z_1$. So $\epsstable(n)\leq 2/n$ (in fact,
LOO-stability holds even without the expectation). However, this learning rule
is not all-i-LOO-stable. For instance, for any continuous distribution,
$|\ell(\A(S^{\setminus 1}),z_1)-\ell(\A(S),z_1)|=1$ with probability $1$, so it
obviously cannot be all-i-LOO-stable with respect to $i=1$.
\end{proof}

\begin{example}\label{ex:average_not_loo}
There exists a learning problem with a universally consistent (and
on-average-LOO stable) AERM, which is not LOO stable.
\end{example}
\begin{proof}
Let the instance space, hypothesis space and objective function be as in
\exref{ex:loo_not_uniform}. Consider the following learning rule, based on a
sample $S=(z_1,\ldots,z_n)$: if $\sum_{i}\ind{z_i=1}/n> 1/2+\sqrt{\log(4)/2n}$,
return $1$. If $\sum_{i}\ind{z_i=1}/n<1/2-\sqrt{\log(4)/2n}$, return $0$.
Otherwise, return $\text{Parity}(S)=(z_1+\ldots z_n) ~\text{mod}~ 2$.

This learning rule is an AERM, with $\epsapprox(n)=\sqrt{2\log(4)/n}$.
Since we have only two hypotheses, we have uniform convergence of $L_S(\cdot)$
to $L(\cdot)$ for any hypothesis. Therefore, our learning rule universally
generalizes (with rate $\epsgen(n)=\sqrt{\log(4/\delta)/2n}$), and by
\thmref{thm:ucons}, this implies that the learning rule is also universally
consistent and on-average-LOO stable.

However, the learning rule is not LOO stable. Consider the uniform distribution
on the instance space. By Hoeffding's inequality,
$|\sum_{i}\ind{z_i=1}/n-1/2|\leq \sqrt{\log(4)/2n}$ with probability at least
$1/2$ for any sample size $n$. In that case, the returned hypothesis is the
parity function (even when we remove an instance from the sample, assuming
$n\geq 3$). When this happens, it is not hard to see that for any $i$,
\[
\ell(\A(S),z_i)-\ell(\A(S^{\setminus i}),z_i) =\ind{z_i=1}(-1)^{\text{Parity(S)}}.
\]
This implies that
\begin{align}
&\E{\frac{1}{n}\sum_{i=1}^{n}  \left| \left(\ell(\A(S^{\setminus i});\z_i) -
\ell(\A(S);\z_i) \right) \right| }\label{eq:av_stab_proof}\\
&\geq \frac{1}{2}\E{\frac{1}{n}\sum_{i=1}^{n}\ind{z_i=1}\Bigg|
\sqrt{\frac{\log(4)}{2n}}\geq
\Big|\sum_{i=1}^{n}\frac{\ind{z_i=1}}{n}-\frac{1}{2}\Big|}\notag\\
&\geq
\frac{1}{2}\left(\frac{1}{2}-\sqrt{\frac{\log(4)}{2n}}\right)~~\longrightarrow
~~\frac{1}{4}~~,\notag
\end{align}
which does not converge to zero with the sample size $n$. Therefore, the
learning rule is not LOO stable.
\end{proof}

Note that the proof implies that on-average-LOO stability cannot be replaced
even by something between on-average-LOO stability and LOO stability. For
instance, a natural candidate would be
\begin{equation}\label{eq:av_stab_proof2}
\Es{S\sim\Dcal^n}{\left|\frac{1}{n}\sum_{i=1}^{n}  \left(\ell(\A(S^{\setminus
i});\z_i) - \ell(\A(S);\z_i) \right) \right| },
\end{equation}
where the absolute value is now over the entire sum, but inside the
expectation. In the example used in the proof, \eqref{eq:av_stab_proof2} is
still lower bounded by \eqref{eq:av_stab_proof}, which does not converge to
zero with the sample size.

After showing that the hierarchy of definitions above is indeed strict, we turn
to the question of what can be characterized in terms of LOO stability. In
\cite{ShalevShaSreSri09}, we show a version of \thmref{thm:main}, which asserts
that a problem is learnable if and only if there is an on-average-LOO stable
AERM. However, on-average-LOO stability is qualitatively much weaker than the
notion of uniform-RO stability used in \thmref{thm:main} (see
\defref{def:alwaysstable}). Rather, we would expect to prove a version of the
theorem with the notion of unform-LOO stability or at least LOO stability,
which are more analogous to uniform-RO stability. However, the proof of
\thmref{thm:main} does not work for these stability definitions (technically,
this is because the proof relies on the sample size remaining constant, which
is true for replacement stability, but not when we remove an instance as in LOO
stability). We do not know if one can prove a version of \thmref{thm:main} with
an LOO stability notion stronger than on-average-LOO stability.

On the plus side, LOO stability allows us to prove the following interesting
result, specific to ERM learning rules.

\begin{theorem}\label{thm:erm}
For an ERM the following are equivalent: \\
\indent {\bf $\bullet$} Universal LOO stability. \\
\indent {\bf $\bullet$} Universal consistency. \\
\indent {\bf $\bullet$} Universal generalization.
\end{theorem}

In particular, the theorem implies that LOO stability is a necessary property
for consistent ERM learning rules. This parallels \thmref{thm:ucons}, which
dealt with AERM's in general, and used RO stability. As before, we do not know
how to obtain something akin to \thmref{thm:ucons} with RO stability.

\index{stability|)}

\section{Characterizing Learnability : Main Results}\label{sec:main_results}

Our overall goal is to characterize learnable problems (namely, problems for
which there exists a universally consistent learning rule, as in
\eqref{eq:consistency}). That means finding some condition which is both
\emph{necessary} and \emph{sufficient} for learnability. In the uniform
convergence setting, such a condition is the stability of the ERM (under any of
several possible stability measures, including both variants of RO-stability
defined above). This is still sufficient for learnability in the General
Learning Setting, but far from being necessary, as we have seen in
\secref{sec:gaps}.

The most important result in this section is a condition which is necessary
and sufficient for learnability in the General Learning Setting:

\begin{theorem}\label{thm:main}
A learning problem is learnable if and only if there exists a
uniform-RO stable, universally AERM learning rule.

In particular, if there exists a $\epscon(n)$-universally consistent rule, then
there exists a rule that is $\epsstable(n)$-uniform-RO stable and
universally $\epsapprox(n)$-AERM where:
\begin{equation}\label{eq:mainrates}
\begin{aligned}
& \epsapprox(n) =  3 \epscon(n^{1/4})  + \tfrac{8 B}{\sqrt{n}} \ \ , \\
& \epsstable(n) =   \tfrac{2 B}{\sqrt{n}}.
\end{aligned}
\end{equation}

In the opposite direction, if a learning rule is $\epsstable(n)$-uniform-RO
stable and universally $\epsapprox(n)$-AERM, then it is universally consistent
with rate
\[
\epscon(n) \leq \epsstable(n) +  \epsapprox(n)
\]
\end{theorem}

Thus, while we have seen in \secref{sec:gaps} that the ERM rule might fail for
learning problems which are in fact learnable, there is always an AERM rule
which will work. In other words, when designing learning rules, we might need
to look beyond empirical risk minimization, but not beyond AERM learning
rules. On the downside, we must choose our AERM carefully, since not any AERM
will work. This contrasts with supervised classification, where any AERM will
work if the problem is learnable at all.

How do we go about proving this assertion? The easier part is showing
sufficiency. Namely, that a stable AERM must be consistent (and generalizing).
In fact, this holds both separately for any particular distribution $\Dcal$s,
and uniformly over all distributions:
\begin{theorem}\label{thm:cons}
If a rule is an AERM with rate $\epsapprox(n)$ and average-RO stable (or
uniform-RO stable) with rate $\epsstable(n)$ under $\Dcal$, then it is
consistent and generalizes under $\Dcal$ with rates
\begin{align*}
& \epscon(n) \leq \epsstable(n) +  \epsapprox(n)\\
& \epsgen(n) \leq \epsstable(n) + 2 \epsapprox(n)  +\tfrac{2B}{\sqrt{n}}
\end{align*}
\end{theorem}
The second part of \thmref{thm:main} follows as a direct corollary. We note
that close variants of \thmref{thm:cons} has already appeared in previous
literature (e.g., \cite{MukherjeeNPR06} and \cite{RakhlinMuPo05}).

The harder part is showing that a uniform-RO stable AERM is \emph{necessary}
for learnability. This is done in several steps.

First, we show that consistent AERMs have to be average-RO stable:
\begin{theorem} \label{thm:ucons}
For an AERM, the following are equivalent: \\
\indent {\bf $\bullet$} Universal average-RO stability.\\
\indent {\bf $\bullet$} Universal consistency. \\
\indent {\bf $\bullet$} Universal generalization.
\end{theorem}

The exact conversion rate of Theorem \ref{thm:ucons} is specified in the
corresponding proof (Sub-section \ref{sec:proofs}), and are all polynomial.
In particular, an $\epscon$-universal consistent $\epsapprox$-AERM is
average-RO stable with rate
\begin{equation*}
\epsstable(n) \leq \epsapprox(n) + 3\epscon(n^{1/4}) + \tfrac{4 B}{\sqrt{n}}.
\end{equation*}

Next, we show that if we seek universally consistent and generalizing learning
rules, then we must consider only AERMs:
\begin{theorem}\label{thm:AERM}
If a rule $\A$ is universally consistent with rate $\epscon(n)$ and
generalizing with rate $\epsgen(n)$, then it is universally an AERM with rate
$$\epsapprox(n) \leq \epsgen(n) + 3\epscon(n^{1/4}) + \frac{4 B}{\sqrt{n}}$$
\end{theorem}

Now, recall that learnability is defined as the existence of some universally
consistent learning rule. Such a rule might not be generalizing, stable or even
an AERM (see example \ref{ex:non_AERM_rule} below). However, it turns out that
if a universally consistent learning rule exist, then there is \emph{another}
learning rule for the same problem, which is generalizing
(\lemref{prop}). Thus, by Theorems \ref{thm:ucons}-\ref{thm:AERM},
this rule must also be average-RO stable AERM. In fact, by another
application of \lemref{prop}, such an AERM must also be uniform-RO stable,
leading to \thmref{thm:main}.


\section{Randomization, Convexification, and a Generic Learning Rule}\label{sec:randomization} \index{learning rule!statistical!randomized}

In this section we show that when one considers randomized learning rules, it is possible to get stronger results and we even propose a generic randomized algorithm for statistical learning problem that guarantees non-trivial learning rate whenever the problem is learnable.

\subsection{Stronger Results with Randomized Learning Rules}

The strongest result we were able to obtain for characterizing learnability so
far is \thmref{thm:main}, which stated that a problem is learnable if and only
if there exists a universally uniform-RO stable AERM. In fact, this result was
obtained under the assumption that the learning rule $\A$ is deterministic:
given a fixed sample $S$, $\A$ returns a single specific hypothesis
$\h$. However, we might relax this assumption and also consider
\emph{randomized} learning rules: given any fixed $S$, $\A(S)$ returns a
distribution over the hypothesis class $\cH$.

With this relaxation, we will see that we can obtain a stronger version of
\thmref{thm:main}, and even provide a generic learning algorithm (at least for
computationally unbounded learners) which successfully learns any learnable
problem.

To simplify notation, we will override the notations $\ell(\A(S),\z)$, $L(\A(S))$
and $\L_S(\A(S))$ to mean $\Es{\h\sim\A(S)}{\ell(\h,\z)}$, $\Es{\h\sim\A(S)}{L(\h)}$ and
$\Es{\h\sim\A(S)}{\L_S(\h)}$. In other words, $\A$ returns a
distribution over $\cH$ and
$\ell(\A(S),\z)$ for some fixed $S,\z$ is the expected loss of a random hypothesis
picked according to that distribution, with respect to $\z$. Similarly,
$L(\A(S))$ for some fixed $S$ is the expected generalization error, and
$\L_S(\A(S))$ is the expected empirical risk on the fixed
sample $S$. With this slight abuse of notation, all our previous definitions
hold. For instance, we still define a learning rule $\A$ to be consistent with
rate $\epscon(n)$ if $\Es{S\sim \Dcal^n}{L(\A(S))-L^*}\leq \epscon(n)$, only
now we actually mean
\[
\Es{S\sim\Dcal^n}{\Es{\h \sim\A(S)}{L(\h)-L^*}}\leq \epscon(n).
\]
The definitions for AERM, generalization etc. also hold with this subtle change
in meaning.

An alternative way to view randomization is as a method to \emph{linearize} the
learning problem. In other words, randomization implicitly replaces the
arbitrary hypothesis class $\cH$ by the space of probability distributions over $\cH$,
$$
\mathcal{M} = \left\{\alpha : \cH \to [0,1]  ~~\textrm{s.t.}~\int \alpha[\h] = 1\right\} ~,
$$ and
replaces the arbitrary function $\ell(\h;\z)$ by a \emph{linear} function in its
first argument
$$
\ell(\alpha;\z) = \Es{\h \sim \alpha}{\ell(\h,\z)}= \int \ell(\h;\z) \alpha[\h] ~.
$$
 Linearity of the loss and convexity of $\mathcal{M}$ are the
 key mechanism which allows us to obtain our stronger
 results. Moreover, if the learning problem is already convex (i.e.,
 $f$ is convex and $\cH$ is covex), we can
achieve the same results using a deterministic learning rule, as the
following claim demonstrates:
\begin{claim}
Assume that the hypothesis class $\cH$ is convex subset of a vector space,
such that $\Es{\h \sim \A(S)}{\h}$ is a well-defined element of $\cH$ for any
$S$. Moreover, assume that $\ell(\h;\z)$ is convex in $\h$. Then from any
(possibly randomized) learning rule $\A$, it is possible to construct a
deterministic learning rule $\A'$, such that $\ell(\A'(S),\z)\leq \ell(\A(S),\z)$ for
any $S,\z$. As a result, it also holds that $\L_S(\A'(S))\leq \L_S(\A(S))$ and
$L(\A'(S))\leq L(\A(S))$.
\end{claim}

\begin{proof}
Given a sample $S$, define $\A'(S;\z)$ as the single hypothesis
$\Es{\h \sim \A(S)}{\h}$. The proof of the theorem is immediate by Jensen's inequality:
since $\ell()$ is convex in its first argument,
\[
\ell(\A'(S);\z)=\ell(\Es{\h \sim \A(S)}{\h},\z)\leq \Es{\h \sim \A(S)}{\ell(\h,\z)},
\]
where the r.h.s. is in fact $\ell(\A(S),\z)$ by the abuse of notation we have
defined previously.
\end{proof}

Although linearization is the real mechanism at play here, we find it more
convenient to display our results and proofs in the language of randomized
learning rules.

Allowing randomization allows us to obtain results with respect to the
following very strong notion of stability\footnote{This definition of stability
  is very similar to the so-called ``uniform stability'', discussed in
  \cite{BousquetEl02}, although \cite{BousquetEl02} consider deterministic
  learning rules. }:

\begin{definition}\label{def:stronglystable}\index{stability!(replace one) RO!strongly uniform}
A rule $\A$ is {\bf strongly-uniform-RO stable} with rate $\epsstable(n)$ if
for all samples $S$ of $n$ points, for all $i$, and any $\z',\z'_i\in \Z$,
it holds that
\begin{equation*}
\abs{ \ell(\A(S^{(i)});\z') - \ell(\A(S);\z')} \le
\epsstable(n).
\end{equation*}
\end{definition}

The strengthening of \thmref{thm:main} that we will prove here is the
following:
\begin{theorem}\label{thm:strong_nain}
A learning problem is learnable if and only if there exists a (possibly
randomized) learning rule which is an always AERM and strongly-uniform-RO
stable.
\end{theorem}

Compared to \thmref{thm:main}, we have replaced universal AERM by the stronger
notion of an always AERM, and uniform-RO stability by strongly-uniform-RO
stability. This makes the result strong enough to formulate a generic learning
algorithm, as we will see later on.

The theorem is an immediate consequence of \thmref{thm:main} and the following
lemma:

\begin{lemma}\label{prop2}
For any deterministic learning rule $\A$, there exists a randomized learning
rule $\A'$ such that:
\begin{itemize}
\item For any $\Dcal$, if $\A$ is $\epscon$-consistent under
$\Dcal$ then $\A'$ is $\epscon(\lfloor \sqrt{n}
\rfloor)$ consistent under $\Dcal$.
\item $\A'$ universally generalizes with rate $4B/\sqrt{n}$.
\item If $\A$ is uniform-RO stable with rate $\epsstable(n)$, then $\A'$ is
  strongly-uniform-RO stable with rate $\epsstable(\lfloor \sqrt{n} \rfloor)$.
\item If $\A$ is universally $\epscon$-consistent, then $\A'$ is an always AERM
  with rate $2\epscon(\lfloor \sqrt{n} \rfloor)$.
\end{itemize}
Moreover, $\A'$ is a symmetric learning rule (it does not depend on the order
of elements in the sample on which it is applied).
\end{lemma}

\subsection{A Generic Learning Rule}\label{subsec:rand_alg}

Recall that a symmetric learning rule $\A$ is such that $\A(S)=\A(S')$ whenever
$S,S'$ are identical samples up to permutation. When we deal with randomized
learning rules, we assume that the distribution of $\A(S)$ is identical to the
distribution of $\A(S')$. Also, let $\bar{\cH}$ denote the set of all
distributions on $\cH$. An element $\bar{\h}\in \bar{H}$ will be thought of as
a possible outcome of a randomized learning rule.

Consider the following learning rule: given a sample size $n$, find a
minimizer over all symmetric\footnote{The algorithm would still work, with
  slight modifications, if we minimize over all functions - symmetric or
  not. However, the search space would be larger.} functions $\A:\Z^n
\rightarrow \bar{\cH}$ of
\begin{equation}
\sup_{S\in \Z^{n}} \left(\L_S(\A(S))-\L_S(\hemp_{S})\right) ~+~\sup_{S\in
  \Z^n,\z'}\left|\ell(\A(S);\z')-\ell(\A(S^{(i)};\z')\right|,\label{eq:algorithm}
\end{equation}
with $i$ being an arbitrary fixed element in $\{1,\ldots,n\}$. Once such a
function $\A$ is found, return $\A_{n}(S)$.

\begin{theorem}\label{thm:generic}
If a learning problem is learnable (namely, there exist a universally
consistent learning rule with rate $\epscon(n)$), the learning algorithm
described above is universally consistent with rate
\[
4\epscon(\lfloor \sqrt{n}\rfloor)+\frac{8B}{\sqrt{n}}.
\]
\end{theorem}

The main drawback of the algorithm we described is that it is completely
infeasible: in practice, we cannot hope to efficiently perform minimization of
\eqref{eq:algorithm} over all functions from $\Z^n$ to $\bar{\cH}$.
Nevertheless, we believe it is conceptually important for three reasons:
First, it hints that generic methods to develop learning algorithms might be
possible in the General Learning Setting (similar to the more specific
supervised classification setting); Second, it shows that stability might play
a crucial role in the way such methods will work; And third, that stability
might act in a similar manner to regularization. Indeed, \eqref{eq:algorithm}
can be seen as a ``regularized ERM'' in the space of learning rules (i.e.,
functions from samples to hypotheses): if we take just the first term in
\eqref{eq:algorithm}, $\sup_{S\in
  \Z^n}\left(\L_S(\A(S))-\L_S(\hemp_S)\right)$, then its minimizer is
trivially the ERM learning rule. If we take just the second term in
\eqref{eq:algorithm}, $\sup_{S\in
  \Z^n,\z}\left|\ell(\A(S);\z')-\ell(\A(S^{(i)});\z')\right|$, then its minimizers
are trivial learning rules which return the same hypothesis irrespective of
the training sample. Minimizing a sum of both terms forces us to choose a
learning rule which is an ``almost''-ERM but also stable - a learning rule which
must exist if the problem is learnable at all, as \thmref{thm:strong_nain}
proves.

\section{Detailed Results and Proofs}\label{sec:proofs}

\subsection{Detailed Proof of Main Result (Section \ref{sec:main_results})}
In this subsection we provide proofs for the main results contained in Section \ref{sec:main_results}.  We first establish that for AERMs, average-RO stability and generalization
are equivalent.

\subsubsection{Equivalence of Stability and Generalization}\label{sec:stabgen}

It will be convenient to work with a weaker version of generalization as an
intermediate step: We say a rule $\A$ {\bf on-average generalizes} with rate
$\epspgen(n)$ under distribution $\Dcal$ if for all $n$,
\begin{equation}\label{eq:pgfen}
\abs{\Es{S\sim\Dcal^n}{ L(\A(S)) - \L_{S}(\A(S)) } } \le \epspgen(n) .
\end{equation}
It is straightforward to see that generalization implies on-average
generalization with the same rate.  We show that for AERMs, the converse is
also true, and also that on-average generalization is equivalent to average-RO
stability. This establishes the equivalence between generalization and average-RO stability (for AERMs).

\begin{lemma}[\bf on-average generalization $\Leftrightarrow$ average-RO stability]\label{lem:pseudo}
If $\A$ is on-average generalizing with rate $\epspgen(n)$ then it is
average-RO stable with rate $\epspgen(n)$.  If $\A$ is average-RO stable with
rate $\epsstable(n)$ then it is on-average generalizing with rate
$\epsstable(n)$.
\end{lemma}
\begin{proof}
For any $i$, $\z_i$ and $\z_i'$ are both drawn i.i.d. from $\Dcal$, we have
that
\[
\Es{S \sim \Dcal^n}{\ell(\A(S);\z_i)} = \Es{S \sim \Dcal^n, \z_i' \sim
\Dcal}{\ell(\A(S^{(i)});\z_i')}. \]
Hence,
\begin{align*}
\Es{S \sim \Dcal^n}{\L_S(\A(S))} & = \Es{S \sim \Dcal^n}{\frac{1}{n}\sum_{i=1}^n \ell(\A(S);\z_i)}\\
& = \frac{1}{n}\sum_{i=1}^n \Es{S \sim \Dcal^n}{ \ell(\A(S);\z_i)} \\
& = \frac{1}{n} \sum_{i=1}^n \Es{S \sim \Dcal^n, \z_i' \sim
\Dcal}{\ell(\A(S^{(i)});\z_i')}
\end{align*}
Also note that $L(\A(S)) = \Es{\z_i' \sim \Dcal}{\ell(\A(S);\z_i')}= \frac{1}{n}
\sum_{i=1}^n \Es{\z_i' \sim \Dcal}{\ell(\A(S);\z_i')}$. Hence we can conclude that
$$
\Es{S\sim\Dcal^n}{ \L(\A(S)) - \L_{S}(\A(S)) }  = \frac{1}{n} \sum_{i=1}^n
\Es{S\sim\Dcal^n, (\z_1',...,\z_n') \sim \Dcal^n}{ \ell(\A(S);\z_i') -
\ell(\A(S^{(i)});\z_i') }
$$
Hence we have the required result.
\end{proof}

For the next result, we will need the following two short utility lemmas.

\begin{ulemma}\label{lemma:varbd}
For i.i.d.~$X_i$, $\abs{X_i}\leq B$ and $X = \frac{1}{n}\sum_{i=1}^{n} X_i$ we
have $\E{\abs{X - \E{X}}} \le B/{\sqrt{n}}$.
\end{ulemma}
\begin{proof}
$\E{\abs{X - \E{X}}}\leq \sqrt{\E{\abs{X - \E{X}}^2}}\leq \sqrt{\mathrm{Var}[X]}=\sqrt{\mathrm{Var}[X_i]/n}\leq
B/{\sqrt{n}}$.
\end{proof}

\begin{ulemma}\label{lemma:absXY}
Let $X,Y$ be random variables s.t. $X \underset{a.s.}{\leq} Y$.  Then
$\E{\abs{X}} \leq \abs{\E{X}} + 2 \E{\abs{Y}}$.
\end{ulemma}
\begin{proof}
$~~~~~
\E{\abs{X}}=\E{\abs{(Y-X)-Y}}\leq \E{Y-X}+\E{|Y|}\leq \abs{\E{X}}+2\abs{\E{Y}}.
$
\end{proof}

\begin{lemma}[\bf AERM + on-average generalization $\Rightarrow$ generalization]\label{lem:pgentogen}
If $\A$ is an AERM with rate $\epsapprox(n)$ and on-average generalizes with
rate $\epspgen(n)$ under $\Dcal$, then $\A$ generalizes with rate
$\epspgen(n)+2\epsapprox(n)+\tfrac{2B}{\sqrt{n}}$ under $\Dcal$.
\end{lemma}
\begin{proof}
Recall that $L^*=\inf_{\h\in \cH}L(\h)$. For an arbitrarily small $\nu>0$,
let $\h_{\nu}$ be a fixed hypothesis such that $L(\h_{\nu})\leq F^{*}+\nu$. Using respective optimalities of $\hemp_S$ and $L^*$ we can bound:
\begin{align*}
&\L_S(\A(S)) - L(\A(S))\\
&=\L_S(\A(S)) - \L_S(\hemp_S) + \L_S(\hemp_S) -
\L_S(\h_{\nu}) + \L_S(\h_{\nu}) - L(\h_{\nu}) + L(\h_{\nu})-L(\A(S)) \\
&\leq \L_S(\A(S)) - \L_S(\hemp_S) + \L_S(\h_{\nu}) - L(\h_{\nu})+\nu = Y_{\nu}
\end{align*}
Where the final equality defines a new random variable $Y_{\nu}$.  By Lemma
\ref{lemma:varbd} and the AERM guarantee we have $\E{\abs{Y_{\nu}}} \leq
\epsapprox(n) + B/\sqrt{n}+\nu$.  From \lemref{lemma:absXY} we can conclude
that
\[
\E{\abs{\L_S(\A(S)) - L(\A(S))}}
\leq
\abs{\E{\L_S(\A(S)) - L(\A(S))}} + 2 \E{\abs{Y_{\nu}}}
\leq \epspgen(n)+2\epsapprox(n)+\tfrac{2B}{\sqrt{n}}+\nu.
\]
Notice that the l.h.s. is a fixed quantity which does not depend on
$\nu$. Therefore, we can take $\nu$ in the r.h.s. to zero, and the result
follows.
\end{proof}

Combining \lemref{lem:pseudo} and \lemref{lem:pgentogen}, we have now {\bf
  established the stability$\leftrightarrow$generalization parts of
  \thmref{thm:cons} and \thmref{thm:ucons}} (in fact, even a slightly stronger
converse than in \thmref{thm:ucons}, as it does not require universality).


\subsubsection{A Sufficient Condition for Consistency}

It is fairly straightforward to see that generalization (or even
on-average generalization) of an AERM implies its consistency:

\begin{lemma}[\bf AERM+generalization$\Rightarrow$consistency]\label{lem:mainimp}
If $\A$ is AERM with rate $\epsapprox(n)$ and it on-average generalizes with
rate $\epspgen(n)$ under $\Dcal$ then it is consistent with rate $\epspgen(n)
+\epsapprox(n)$ under $\Dcal$.
\end{lemma}
\begin{proof}
For any $\nu>0$, let $\h_{\nu}$ be a hypothesis such that $L(\h_{\nu})\leq
L^*+\nu$. We have
\begin{align*}
&\E{L(\A(S))  - L^*}   =  \E{L(\A(S))  - \L_S(\h_{\nu})+\nu}  \\
&=  \E{ L(\A(S)) - \L_S(\A(S)) }   +  \E{\L_S(\A(S)) - \L_S(\h_{\nu})}+\nu  \\
&\le  \E{L(\A(S)) - \L_S(\A(S)) }   +  \E{\L_S(\A(S)) - \L_S(\hemp_S)}+\nu \\
&\le \epspgen(n) + \epsapprox(n)+\nu.
\end{align*}
Since this upper bound holds for any $\nu$, we can take $\nu$ to zero, and the
result follows.
\end{proof}

Combined with the results of \lemref{lem:pseudo}, this completes the
\textbf{proof of \thmref{thm:cons}} and the \textbf{ stability $\rightarrow$
consistency and generalization $\rightarrow$ consistency parts of
\thmref{thm:ucons}}.

\subsubsection{Converse Direction}\label{converse_direction}

\lemref{lem:pseudo} already provides a converse result, establishing that
stability is necessary for generalization.  However, as it will turn out, in
order to establish that stability is also necessary for \emph{universal
  consistency}, we must prove that universal consistency of an AERM implies
\emph{universal} generalization. The assumption of \emph{universal}
consistency for the AERM is crucial here: mere consistency of an AERM with
respect to a specific distribution does \emph{not} imply generalization nor
stability with respect to that distribution. The following example briefly
illustrates this point.

\begin{example}\label{ex:non_average_loo_AERM}
There exists a learning problem and a distribution on the instance space, such
that the ERM (or any AERM) is consistent with rate $\epscon(n)=0$, but does not
generalize and is not average-RO stable (namely,
$\epsgen(n),\epsstable(n)=\Omega(1)$).
\end{example}
\begin{proof}
Let the instance space be $[0,1]$, the hypothesis space consist of all finite
subsets of $[0,1]$, and define the objective function as $\ell(\h,z) = \ind{z\notin
\h})$. Consider any continuous distribution on the instance space. Since the
underlying distribution $\Dcal$ is continuous, we have $L(\h)=1$
for any hypothesis $h$. Therefore, any learning rule (including any AERM) will
be consistent with $L(\A(S))=1$. On the other hand, the ERM here always
achieves $\L_S(\hat{\h}_S)=0$, so any AERM cannot generalize, or even
on-average-generalize (by \lemref{lem:pgentogen}), hence cannot be average-RO
stable (by \lemref{lem:pseudo}).
\end{proof}

The main tool we use to prove our desired converse result is the following
lemma. It is here that we crucially use the universal consistency assumption
(i.e., consistency with respect to \emph{any} distribution). Intuitively, it
states that if a problem is learnable at all, then although the ERM rule might
fail, its empirical risk is a consistent estimator of the
minimal achievable risk.

\begin{lemma}[\bf Main Converse Lemma]\label{maincon}
If a problem is learnable, namely there exists a universally consistent rule
$\A$ with rate $\epscon(n)$, then under any distribution,
\begin{equation}\label{eq:main_converse}
\E{\left|\L_S(\hat{\h}_S) - L^*\right|} \le \eps{emp}(n)  \ \ \ \ \ \ \ \
\textrm{where}
\end{equation}
 $$\eps{emp}(n) =   2 \epscon(n') + \tfrac{2 B}{\sqrt{n}} + \tfrac{2 B {n'}^2}{n}$$
for any $n'$ such that $2\leq n'\leq n/2$.
\end{lemma}
\begin{proof}
Let $I = \{I_1,\ldots,I_{n'}\}$ be a random sample of $n'$ indexes in the
range $1..n$ where each $I_i$ is independently uniformly distributed, and $I$
is independent of $S$. Let $S' = \{z_{I_i}\}_{i=1}^{n'}$, i.e.~a sample of
size $n'$ drawn from the uniform distribution over samples in $S$ (with
replacements). We first bound the probability that $I$ has no repeated indexes
(``duplicates''):
\begin{equation}\label{eq:duplicates}
\Pr{I \textrm{ has duplicates}} \leq  \frac{\sum_{i=1}^{n'} (i-1)}{n} \leq
\frac{{n'}^2}{2n}
\end{equation}

Conditioned on not having duplicates in $I$, the sample $S'$ is actually
distributed according to $\Dcal^{n'}$, i.e.~can be viewed as a sample from the
original distribution.  We therefore have by universal consistency:
\begin{equation}\label{1}
\mathbb{E}\left[\abs{L(\A(S')) - L^*} \;\middle|\; \textrm{no
dups}\right] \le \epscon(n')
\end{equation}

But viewed as a sample drawn from the uniform distribution over instances in
$S$, we also have:
\begin{align}\label{2}
\Es{S'}{ \abs{\L_S(\A(S')) - \L_S(\hemp_S)}} \le \epscon(n')
\end{align}
Conditioned on having no duplications in $I$, the set of those
samples in $S$ not chosen by $I$ (i.e.~$S\setminus S'$) is independent of $S'$, and $\abs{S\setminus
S'}= n-n'$, and so by \lemref{lemma:varbd}:
\begin{align}\label{3}
\Es{S}{ \abs{L(\A(S')) - \L_{S\setminus S'}(\A(S'))}} \le \frac{B}{\sqrt{n-n'}}
\end{align}
Finally, if there are no duplicates, then for any hypothesis, and in
particular for $\A(S')$ we have:
\begin{equation}\label{4}
\abs{\L_S(\A(S'))-\L_{S\setminus S'}(\A(S'))} \leq  \frac{2 B n'}{n}
\end{equation}

Combining \eqref{1},\eqref{2},\eqref{3} and \eqref{4}, accounting for a
maximal discrepancy of $B$ when we do have duplicates, and assuming $2\leq
n'\leq n/2$, we get the desired bound.
\end{proof}

Equipped with \lemref{maincon}, we are now ready to show that universal
consistency of an AERM implies universal generalization and that any
universally consistent and generalizing rule must be an AERM.  What we show is
actually a bit stronger: that if a problem is learnable, and so
\lemref{maincon} holds, then for any distribution $\Dcal$ separately,
consistency of an AERM under $\Dcal$ implies generalization under $\Dcal$ and
also any consistent and generalizing rule under $\Dcal$ must be an AERM.

\begin{lemma}[\bf learnable+AERM+consistent$\Rightarrow$generalizing]\label{lem:contogen}
If \eqref{eq:main_converse} in \lemref{maincon} holds with rate $\epsempp(n)$,
and $\A$ is an $\epsapprox$-AERM and $\epscon$-consistent under $\Dcal$, then
it is generalizing under $\Dcal$ with rate
$\epsempp(n)+\epsapprox(n)+\epscon(n)$.
\end{lemma}
\begin{proof}
\begin{align*}
\E{\abs{\L_S(\A(S)) - L(\A(S))}} &\leq \E{\abs{\L_S(\A(S)) - \L_S(\hemp_S)}}
+\E{\abs{L^* - L(\A(S))}} +
\E{\abs{\L_S(\hemp_S) - L^*}} \\ & \leq
\epsapprox(n)+\epscon(n)+\epsempp(n) ~.
\end{align*}
\end{proof}

\begin{lemma}[\bf learnable+consistent+generalizing$\Rightarrow$AERM]\label{lem:contoaerm}
If \eqref{eq:main_converse} in \lemref{maincon} holds with rate $\epsempp(n)$,
and $\A$ is $\epscon$-consistent and $\epsgen$-generalizing under $\Dcal$, then
it is AERM under $\Dcal$ with rate $\epsempp(n)+\epsgen(n)+\epscon(n)$.
\end{lemma}
\begin{proof}
\begin{align*}
\E{\abs{\L_S(\A(S)) - \L_S(\hemp_S)}} & \leq \E{\abs{\L_S(\A(S)) - L(\A(S))}}
+\E{\abs{L(\A(S)) - L^*}} +\E{\abs{L^* - \L_S(\hemp_S)}} \\ & \leq \epsgen(n) + \epscon(n) +
\epsempp(n)~.
\end{align*}
\end{proof}

\lemref{lem:contogen} establishes that universal consistency of an AERM implies
universal generalization, and thus {\bf completes the proof of
  \thmref{thm:ucons}}.  \lemref{lem:contoaerm} {\bf establishes
  \thmref{thm:AERM}}.  To get the rates in \subsecref{subsec:main_results}, we
use $n'=n^{1/4}$ in \lemref{maincon}.

\lemref{lem:mainimp}, \lemref{lem:contogen} and \lemref{lem:contoaerm} together
establish an interesting relationship:
\begin{corollary}\label{cor:any2}
For a (universally) learnable problem, for any distribution $\Dcal$ and
learning rule $\A$, any two of the following imply the third :\\
\indent {\bf $\bullet$} $\A$ is an AERM under $\Dcal$. \\
\indent {\bf $\bullet$} $\A$ is consistent under $\Dcal$. \\
\indent {\bf $\bullet$} $\A$ generalizes under $\Dcal$.
\end{corollary}
\noindent Note, however, that any one property by itself is possible, even
universally:\\
\begin{itemize}
\item In \subsecref{subsec:ugc_fails}, we have discussed an example where the
  ERM learning rule is neither consistent nor generalizing, despite the problem
  being learnable.
\item In the next subsection (Example \ref{ex:non_AERM_rule}) we demonstrate a universally
  consistent learning rule which is neither generalizing nor an AERM.
\item A rule returning a fixed hypothesis always generalizes, but of course
  need not be consistent nor an AERM.
\end{itemize}

In contrast, for learnable supervised classification problems,
it is not possible for a learning rule to be just universally consistent,
without being an AERM and without generalization.  Nor is it possible for a
learning rule to be a universal AERM for a learnable problem, without being
generalizing and consistent.

Corollary \ref{cor:any2} can also provide a \emph{certificate} of
non-learnability.  In other words, for the problem in Example
\ref{ex:non_average_loo_AERM} we show a specific distribution for which there
is a consistent AERM that does not generalize.  We can conclude that there is
{\em no} universally consistent learning rule for the problem, otherwise the
corollary is violated.

\subsubsection{Existence of a Stable Rule}


\thmref{thm:ucons} and \thmref{thm:AERM}, which we just completed proving,
already establish that for AERMs, universal consistency is equivalent to
universal average-RO stability.  Existence of a universally average-RO
stable AERM is thus sufficient for learnability.  In order to prove that it is
also necessary, it is enough to show that existence of a universally consistent
learning rule implies existence of a universally consistent AERM.  This AERM
must then be average-RO stable by \thmref{thm:ucons}.

We actually show how to transform a consistent rule to a consistent and
generalizing rule (\lemref{prop} below).  If this rule is universally
consistent, then by \lemref{lem:contoaerm} we can then conclude it
must be an
AERM, and by \lemref{lem:pseudo} it must be average-RO stable.

\begin{lemma}\label{prop}
For any rule $\A$ there exists a rule $\A'$, such that:\\ \indent $\bullet$
$\A'$ universally generalizes with rate $\frac{3B}{\sqrt{n}}$.\\ \indent
$\bullet$ For any $\Dcal$, if $\A$ is $\epscon$-consistent under $\Dcal$ then
$\A'$ is $\epscon(\lfloor \sqrt{n} \rfloor)$ consistent under
$\Dcal$.\\ \indent $\bullet$ $\A'$ is uniformly-RO-stable with rate $\frac{2
  B}{\sqrt{n}}$.
\end{lemma}
\begin{proof}
For a sample $S$ of size $n$, let $S'$ be a sub-sample consisting of some
$\lfloor \sqrt{n} \rfloor$ observation in $S$.  To simplify the
presentation, assume that $\lfloor \sqrt{n} \rfloor$ is an integer. Define $\A'(S)=\A(S')$. That
is, $\A'$ applies $A$ to only $\sqrt{n}$ of the observation in
$S$.

\paragraph{$\A'$ generalizes:}  We can decompose:
\begin{align*}
\L_{S}(\A(S')) - & L(\A(S')) =  \tfrac{1}{ \sqrt{n} }
(\L_{S'}(\A(S')) - L(\A(S')))  + (1 - \tfrac{1}{ \sqrt{n} })
(\L_{S \setminus S'}(\A(S')) - L(\A(S')))
\end{align*}
The first term can be bounded by $2B/\sqrt{n}$.  As for the
second term, $S\setminus S'$ is statistically independent of $S'$ and so we can
use \lemref{lemma:varbd} to bound its expected magnitude to obtain:
\begin{align}
\E{\abs{\L_{S}(\A(S')) - L(\A(S'))}}  \leq \tfrac{2B}{ \sqrt{n} }
+ (1-\tfrac{1}{ \sqrt{n} })\tfrac{B}{\sqrt{n- \sqrt{n}
}} \leq \tfrac{3B}{\sqrt{n}}
\end{align}

\paragraph{$\A'$ is consistent:} If $\A$ is consistent, then:
\begin{align*}
\E{L(\A'(S)) - \inf_{\h \in \cH} L(\h)} = \E{L(\A(S')) - \inf_{\h \in
\cH} L(\h)} \leq \epscon(\sqrt{n})
\end{align*}

\paragraph{$\A'$ is uniformly-RO-stable:} Since $\A'$ only
uses the first $\sqrt{n}$ samples of $S$, for any $i > \sqrt{n}$ we have
$\A'(S^{(i)}) = \A'(S)$ and so:
\begin{align*}
\frac{1}{n} \sum_{i=1}^n & \abs{ \ell(\A'(S^{(i)});\z') - \ell(\A'(S);\z')}  =
\frac{1}{n} \sum_{i=1}^{\sqrt{n}}\abs{ \ell(\A'(S^{(i)});\z') - \ell(\A'(S);\z')} \le \frac{2
B}{\sqrt{n}}
\end{align*}
\end{proof}

\paragraph{Proof of Converse in \thmref{thm:main}} If there exists
a universally consistent rule with rate $\epscon(n)$, by \lemref{prop}
there exists $\A'$ which is $\epscon(\sqrt{n})$- universally consistent,
$\frac{2 B}{\sqrt{n}}$-generalizing and $\frac{2
B}{\sqrt{n}}$-uniformly-RO-stable.  Further by \lemref{lem:contoaerm} and
 \lemref{maincon} (with $n'=n^{1/4}$), we can conclude that $A'$ is
$\epsapprox$-universally AERM where,
$$
\epsapprox(n) \le 3 \epscon(n^{1/4}) +  \frac{8 B}{\sqrt{n}}.
$$
Hence we get the specified rate for the converse direction. To see that if
there exists a rule that is a universal AERM and stable it is consistent, we
simply use \lemref{lem:mainimp}.

As a final note, the following example shows that while learnability is
equivalent to the existence of stable and consistent AERM's (\thmref{thm:main}
and \thmref{thm:ucons}), there might still exist other learning rules, which
are neither stable, nor generalize, nor AERM's. In this sense, our results
characterize learnability, but do not characterize all learning rules which
``work''.
\begin{example}\label{ex:non_AERM_rule}
There exists a learning problem with a universally consistent learning rule,
which is not average-RO stable, generalizing nor an AERM.
\end{example}
\begin{proof}
Let the  instance space be $[0,1]$. Let the hypothesis space consist of all
finite subsets of $[0,1]$, and the objective function be the indicator
function $\ell(\h,z)=\ind{z\in \h}$. Consider the following learning rule: given a
sample $S\subseteq [0,1]$, the learning rule checks if there are any two
identical instances in the sample. If so, the learning rule returns the empty
set $\emptyset$. Otherwise, it returns the sample.

Consider any continuous distribution on $[0,1]$. In that case, the probability
of having two identical instances is $0$. Therefore, the learning rule always
returns a countable non-empty set $\A(S)$, with $\L_S(\A(S))=1$, while
$\L_S(\emptyset)=0$ (so it is not an AERM) and $L(\A(S))=0$ (so it does not
generalize). Also, $\ell(\A(S),z_i)=1$ while $\ell(\A(S^{(i)}),z_i)=0$ with
probability $1$, so it is not average-RO stable either.

However, the learning rule is universally consistent. If the underlying
distribution is continuous on $[0,1]$, then the returned hypothesis is $S$,
which is countable hence , $L(S)=0=\inf_{\h}L(\h)$. For discrete distributions,
let $M_1$ denote the proportion of instances in the sample which appear
exactly once, and let $M_0$ be the probability mass of instances which did not
appear in the sample. Using \cite[Theorem 3]{McAllesterSc00}, we have that for
any $\delta$, it holds with probability at least $1-\delta$ over a sample of
size $n$ that
\begin{align*}
|M_0-M_1| \leq O\left(\tfrac{\log(n/\delta)}{\sqrt{n}}\right),
\end{align*}
uniformly for any discrete distribution. If this occurs, then either
$M_1<1$, or $M_0\geq 1-O(\log(n/\delta)/\sqrt{n})$. But in the first event, we
get duplicate instances in the sample, so the returned hypothesis is the
optimal $\emptyset$, and in the second case, the returned hypothesis is the
sample, which has a total probability mass of at most
$O(\log(n/\delta)/\sqrt{n})$, and therefore $L(\A(S))\leq
O(\log(n/\delta)/\sqrt{n})$. As a result, regardless of the underlying
distribution, with probability of at least $1-\delta$ over the sample,
\[
L(\A(S))\leq O\left(\tfrac{\log(n/\delta)}{\sqrt{n}}\right).
\]
Since the r.h.s. converges to $0$ with $n$ for any $\delta$, it is easy to see
that the learning rule is universally consistent.
\end{proof}

\subsection{Other Proofs}

\begin{proof}[Proof of \thmref{thm:ERMHazan}]
To prove the theorem, we use a stability argument.
Denote
$$ \empF^{(i)}(\h) = \frac{1}{n}\left(\ell(\h,\z'_i)+\sum_{j \ne i}\ell(\h,\z_j)
\right).
$$ the empirical average with $\z_i$ replaced by an independently and
identically drawn $\z'_i$, and consider its minimizer:
\begin{equation*}
\hemp_{S}^{(i)} = \arg\min_{\h \in \cH} \empF^{(i)}(\h).
\end{equation*}
We first use strong convexity and Lipschitz-continuity to establish
that empirical minimization is stable in the following sense:
\begin{equation}\label{eq:finalcalim01}
\forall \z \in \Z,~~\abs{\ell(\hemp_{S},\z)-\ell(\hemp_{S}^{(i)},\z)} \leq
\tfrac{4L^2}{\lambda n} ~.
\end{equation}
We have that
\begin{align}
\lefteqn{\empL(\hemp_{S}^{(i)}) - \empL(\hemp_{S})} \notag \\ &=
\frac{\ell(\hemp_{S}^{(i)},\z_i) - \ell(\hemp_{S},\z_i)}{n} + \frac{\sum_{j \ne i}
  \left(\ell(\hemp_{S}^{(i)},\z_j) - \ell(\hemp_{S},\z_j)\right)}{n} \notag \\ & =
\frac{\ell(\hemp_{S}^{(i)},\z_i) - \ell(\hemp_{S},\z_i)}{n} + \frac{
  \ell(\hemp_{S},\z_i') - \ell(\hemp_{S}^{(i)},\z_i')}{n} \notag \\ & \quad +
\left(\empF^{(i)}(\hemp_{S}^{(i)}) -
\empF^{(i)}(\hemp_{S})\right) \notag \\ & \le
\frac{|\ell(\hemp_{S}^{(i)},\z_i) - \ell(\hemp_{S},\z_i)|}{n} + \frac{
  |\ell(\hemp_{S},\z_i') - \ell(\hemp_{S}^{(i)},\z_i')|}{n} \notag \\ & \le \frac{2
  \Lipf}{n} \norm{\hemp_{S}^{(i)} - \hemp_{S}} \label{eqn:proof3}
\end{align}
where the first inequality follows from the fact that $\hemp_{S}^{(i)}$ is the
minimizer of $\empF^{(i)}(\h)$ and for the second inequality we use Lipschitz
continuity.  But from strong convexity of $\empL(\h)$ and the fact that
$\hemp_{S}$ minimizes $\empL(\h)$ we also have that
\begin{equation}\label{eq:proof4}
\empL(\hemp_{S}^{(i)}) \geq \empL(\hemp_{S}) + \tfrac{\lambda}{2}
\norm{\hemp_{S}^{(i)} - \hemp_{S}}^2.
\end{equation}
Combining \eqref{eq:proof4} with \eqref{eqn:proof3} we get
$\norm{\hemp_{S}^{(i)} - \hemp_{S}} \leq 4L/(\lambda n)$ and combining
this with Lipschitz continuity of $f$ we obtain that \eqref{eq:finalcalim01} holds.
Later in this chapter we show that a stable
ERM is sufficient for learnability. More formally,
\eqref{eq:finalcalim01} implies that the ERM is  uniform-RO stability
(\defref{def:alwaysstable}) with rate $\epsstable(n) = 4L^2/(\lambda
n)$ and therefore \thmref{thm:cons} implies that the ERM is consistent with rate $\leq \epsstable(n)$, namely
\[
\Es{S\sim\Dcal^n}{L(\hemp_{S})-L^*}\leq \tfrac{4L^2}{\lambda n} ~.
\]
This concludes the proof.
\end{proof}

We now turn to the proof of \thmref{thm:RERM}.

\begin{proof}[Proof of \thmref{thm:RERM}]
Let $r(\h;\z) = \tfrac{\lambda}{2}\|\h\|^2 + \ell(\h;\z)$ and let $R(\h) =
\Es{\z}{r(\h,\z)}$.  Note that $\hemp_{\lambda}$ is the empirical minimizer for
the stochastic optimization problem defined by $r(\h;\z)$.

We apply Theorem \ref{thm:ERMHazan} to $r(\h;\z)$, to this end note that since
$f$ is $L$-Lipschitz and $\forall \h \in \cH,\ \|\h\| \le B$ we see that $r$
is in fact $L + \lambda B$-Lipschitz. Applying Theorem \ref{thm:ERMHazan},
we see that
\[
\tfrac{\lambda}{2}\norm{\hemp_{\lambda}}^2 + \Es{S}{ L(\hemp_\lambda) } =
\Es{S}{R(\hemp_{\lambda})} \leq \inf_{\h} R(\h) + \frac{4(\Lipf + \lambda
B)^2}{\lambda n}
\]
Now note that $\inf_{\h}R(\h)\leq \inf_{\h \in \H}L(\h)+\frac{\lambda}{2}B^2=L^*+\frac{\lambda}{2}B^2$, and so we get that
\begin{align*}
\Es{S}{L(\hemp_\lambda)} & \le  \inf_{\h \in \H}L(\h) + \frac{\lambda}{2} B^2 + \frac{4(\Lipf +
\lambda B)^2}{\lambda n} \\ & \le  \inf_{\h \in \H}L(\h) + \frac{\lambda}{2} B^2 + \frac{8 \Lipf^2 }{\lambda n} + \frac{8 \lambda B^2}{n}
\end{align*}
Plugging in the value of $\lambda$ given in the theorem statement we see that
\begin{align*}
L(\hemp_\lambda) & \le \inf_{\h \in \H} \L(\h) + 4 \sqrt{\frac{\Lipf^2 B^2}{n}} +
\frac{32}{ n}\sqrt{\frac{\Lipf^2 B^2}{n}}
\end{align*}
This gives us the required bound.
\end{proof}

\begin{proof}[Proof of Theorem \ref{thm:erm}]
\lemref{lem:mainimp} and \lemref{lem:contogen} from subsection
\ref{converse_direction} already tell us that for ERM's, universal consistency
is equivalent to universal generalization. Moreover, \lemref{lem:pgentogen}
implies that for ERM's, generalization is equivalent to on-average
generalization (see \eqref{eq:pgfen} for the exact definition). Thus, is left
to prove that for ERM's, generalization implies LOO stability, and LOO stability implies on-average generalization.
stability.

First, suppose the ERM learning rule is generalizing with rate
$\epsgen(n)$. Note that $\ell(\hemp_{S^{\setminus i}};z_i)-\ell(\hemp_{S};z_i)$ is
always nonnegative. Therefore the LOO stability of the ERM can be upper bounded
as follows:
\begin{align*}
&\frac{1}{n}\sum_{i=1}^{n}\E{|\ell(\hemp_{S^{\setminus i}};z_i)-\ell(\hemp_{S};z_i)|}\\
&=\frac{1}{n}\sum_{i=1}^{n}\E{\ell(\hemp_{S^{\setminus i}};z_i)-\ell(\hemp_{S};z_i)}\\
&=\frac{1}{n}\sum_{i=1}^{n}\E{L(\hemp_{S^{\setminus i}})}-\E{\frac{1}{n}\sum_{i=1}^{n}\ell(\hemp_{S};z_i)}\\
&\leq \frac{1}{n}\sum_{i=1}^{n}\E{L_{S^{\setminus i}}(\hemp_{S^{\setminus i}})+\epsgen(n-1)}-\E{L_{S}(\hemp_{S})}\\
&= \epsgen(n-1)+\E{\frac{1}{n}\sum_{i=1}^{n}L_{S^{\setminus i}}(\hemp_{S^{\setminus i}})-L_{S}(\hemp_{S})}\\
&\leq \epsgen(n-1).
\end{align*}

For the opposite direction, suppose the ERM learning rule is LOO stable with
rate $\epsstable(n)$. Notice that we can get any sample of size $n-1$ by
picking a sample $S$ of size $n$ and discarding any instance $i$.
Therefore, the on-average generalization rate of the ERM for samples
of size $n-1$ is equal to the following:
\begin{align*}
&\left|\E{L(\hemp_{S^{\setminus i}})-L_{S^{\setminus i}}(\hemp_{S^{\setminus i}})}\right|\\
&=\left|\frac{1}{n}\sum_{i=1}^{n}\E{L(\hemp_{S^{\setminus i}})-L_{S^{\setminus i}}(\hemp_{S^{\setminus i}})}\right|\\
&=\left|\frac{1}{n}\sum_{i=1}^{n}\E{\ell(\hemp_{S^{\setminus i}};z_i)}-\frac{1}{n}\sum_{i=1}^{n}\E{L_{S^{\setminus i}}(\hemp_{S^{\setminus i}})}\right|\\
\end{align*}
Now, note that for the ERM's of $S$ and $S^{\setminus i}$ we have $\abs{L_{S^{\setminus
i}}(\hemp_{S^{\setminus i}})  - L_{S}(\hemp_{S})} \le \tfrac{2
B}{n}$. Therefore, we can upper bound the above by
\begin{align*}
&\left|\frac{1}{n}\sum_{i=1}^{n}\E{\ell(\hemp_{S^{\setminus i}};z_i)}-\E{L_{S}(\hemp_{S})}\right|+\frac{2B}{n}\\
&=\left|\frac{1}{n}\sum_{i=1}^{n}\E{\ell(\hemp_{S^{\setminus i}};z_i)-\ell(\hemp_{S},;z_i)}\right|\\
&\leq \epsstable(n)
\end{align*}
using the assumption that the learning rule is $\epsstable(n)$-stable.
\end{proof}

\begin{proof}[Proof of Lemma \ref{prop2}]
Consider the learning rule $A'$ which given a sample $S$, returns a uniform distribution over $A(S')$, where $S'$ ranges over all subsets of $S$ of size $\lfloor \sqrt{n} \rfloor$.

The fact that $\A'$ is symmetric is trivial. We now prove the other assertions
in the lemma.
\paragraph{$\A'$ is consistent:} First note that $L(\A'(S)) =
\Es{S'}{L(\A(S'))}$, and so:
\[
\Es{S}{\abs{L(\A'(S))-L^*}} ~\leq~
\Es{S,S'}{\abs{L(\A(S'))-L^*}}
~=~ \Es{[S']}{\Es{S|[S']}{\abs{L(\A(S'))-L^*}}}
\]
where $[S']$ designates a choice of indices for $S'$. This decomposition of the
random choice of $S'$ (e.g., first deciding on the indices and only then
sampling $S$) allows us think of $[S']$ and $S$ as statistically
independent. Given a fixed choice of indices $[S']$, $S'$ is simply an
i.i.d. sample of size $\lfloor \sqrt{n} \rfloor$. Therefore, if $\A$ is
consistent, $|L(\A(S'))-L^*|\leq \epscon(\lfloor \sqrt{n} \rfloor)$, this
holds for any possible fixed $[S']$, and therefore
\[
\Es{[S']}{\Es{S|[S']}{\abs{L(\A(S'))-L^*}}} ~=~ \Es{[S']}{\epscon(\lfloor
  \sqrt{n} \rfloor)} ~\leq~ \epscon(\lfloor \sqrt{n} \rfloor).
\]

\paragraph{$\A'$ generalizes:} For convenience, let
$b(S,S')=\abs{\L_S(\A(S'))-L(\A(S'))}$. Using similar arguments and notation as
above:
\begin{align*}
&\Es{S}{\abs{\L_S(\A'(S))-L(\A'(S))}} \\
&\leq \Es{[S']}{\Es{S|[S']}{b(S,S')}}\\
&\leq
\Es{[S']}{\Es{S|[S']}{\frac{\lfloor \sqrt{n} \rfloor}{n}b(S',S')}
~+~\Es{S|[S']}{\left(1-\frac{\lfloor \sqrt{n} \rfloor}{n}\right)b(S\setminus
S',S')}}\\
&\leq \Es{[S']}{\frac{\lfloor \sqrt{n} \rfloor}{n}2B+\left(1-\frac{\lfloor \sqrt{n} \rfloor}{n}\right)\frac{B}{\sqrt{n-\lfloor \sqrt{n} \rfloor}+1}},
\end{align*}
where the last line follows from \lemref{lemma:varbd} and the fact that
$b(S,S')\leq 2B$ for any $S,S'$. It is not hard to show that the expression
above is at most $4B/\sqrt{n}$, assuming $n\geq 1$.

\paragraph{$\A'$ is strongly-uniform-RO stable:} For any sample $S$, any $i$ and replacement instance $\z_i$, and any instance $\z'$, we have that
\[
\left|\ell(\A'(S^{(i)});\z')-\ell(\A'(S);\z')\right|~\leq~
\Es{S'}{\left|\ell(\A(S'^{(i)});\z')-\ell(\A(S');\z')\right|},
\]
where we take $S'^{(i)}$ in the expectation to mean $S'$ if $i\notin
[S']$. Notice that if $i\notin [S']$, then
$\ell(\A(S'^{(i)});\z_i)-\ell(\A(S');\z_i)$ is trivially $0$. Thus, we can upper
bound the expression above by
\[
\Es{S'}{\left|\ell(\A(S'^{i});\z')-\ell(\A(S');\z')\right|~~\Big|~~i\in [S']}.
\]
Since $S'$ is chosen uniformly over all $\lfloor \sqrt{n} \rfloor$-subsets of
$S$, all permutations of $[S']$ are equally happen to occur, and therefore the
above is equal to
\[
\Es{S'}{\frac{1}{\lfloor \sqrt{n} \rfloor}\sum_{j\in
    S'}\left|\ell(\A(S'^{(j)});\z')-\ell(\A(S');\z')\right|} ~\leq~
\Es{S'}{\epsstable(\lfloor \sqrt{n} \rfloor)} ~=~ \epsstable(\lfloor \sqrt{n}
\rfloor).
\]

\paragraph{$\A'$ is an always AERM:} For any fixed sample $S$, we note that \begin{align*}
&|\L_S(\A'(S))-\L_S(\hemp_{S})| ~=~
  \Es{S'}{\L_S(\A(S'))-\L_S(\hemp_{S})}\\ &=\Es{S' \sim \mathcal{U}(S)^{\lfloor
      \sqrt{n}\rfloor}}{\L_S(\A(S'))-\L_S(\hemp_{S})\;|\;\textrm{no dups}},
\end{align*}
where $\mathcal{U}(S)^{\lfloor \sqrt{n} \rfloor}$ signifies the distribution of
i.i.d. samples of size $\lfloor \sqrt{n} \rfloor$, picked uniformly at random
(with replacement) from $\lfloor \sqrt{n} \rfloor$, and 'no dups' signifies the
event that no element in $S$ was picked twice. By the law of total expectation,
this is at most
\[
\frac{\Es{S' \sim \mathcal{U}(S)^{\lfloor \sqrt{n}
      \rfloor}}{\L_S(\A(S'))-\L_S(\hemp_{S})}}{\Pr{\textrm{no dups}}}.
\]
Since the learning rule $\A$ is universally consistent, it is in particular
consistent with respect to the distribution $\mathcal{U}(S)$, and therefore the
expectation in the expression above is at most $\epscon(\lfloor \sqrt{n}
\rfloor)$. As to $\Pr{\textrm{no dups}}$, an analysis identical to the one
performed in the proof of \lemref{maincon} (see \eqref{eq:duplicates})
implies that it is at least $1-(\lfloor \sqrt{n} \rfloor)^2/n\geq
1/2$. Overall, we get that $\L_S(\A'(S))-\L_S(\hemp_S)\leq 2\epscon(\lfloor
\sqrt{n} \rfloor)$, so in particular
\[
\frac{\Es{S' \sim \mathcal{U}(S)^{\lfloor \sqrt{n} \rfloor}}{\L_S(\A(S'))-\L_S(\hemp_{S})}}{\Pr{\textrm{no dups}}} \leq 2\epscon(\lfloor \sqrt{n} \rfloor),
\]
from which the claim follows.

\end{proof}

\begin{proof}[Proof of Theorem \ref{thm:generic}]
By \lemref{prop2}, if a learning problem is learnable, there exists a (possibly
randomized) symmetric learning rule $\A'$, which is an always AERM and
strongly-uniform-RO stable. More specifically, we have that
\[
\sup_{S\in \Z^{n}} \left(\L_S(\A'(S))-\L_S(\hemp_{S})\right) \leq
2\epscon(\lfloor \sqrt{n} \rfloor),
\]
as well as
\[
\sup_{S\in \Z^n,\z'}\left|\ell(\A'(S);\z')-\ell(\A'(S^{(i)});\z')\right| \leq
\frac{4B}{\sqrt{n}}.
\]
In particular, there exists some symmetric $\A:Z^n \rightarrow \bar{\cH}$,
for which the expression in \eqref{eq:algorithm} is at most
\[
2\epscon(\lfloor \sqrt{n} \rfloor)+\frac{4B}{\sqrt{n}}.
\]
Therefore, by definition, the $\A$ found satisfies
\begin{equation}
\sup_{S\in \Z^{n}} \left(\L_S(\A_n(S))-\L_S(\hemp_{S})\right) ~\leq ~
2\epscon(\lfloor \sqrt{n}\rfloor)+\frac{4B}{\sqrt{n}},\label{eq:A_AERM}
\end{equation}
as well as
\begin{equation}
\sup_{S\in \Z^n}\left|\ell(\A_n(S);\z')-\ell(\A_n(S^{(i)});\z')\right| ~\leq~
2\epscon(\lfloor \sqrt{n} \rfloor)+\frac{4B}{\sqrt{n}}.\label{eq:A_stable}
\end{equation}
In \thmref{thm:ucons}, we have seen that a universally average-RO stable AERM
learning rule has to be universally consistent. The inequalities above
essentially say that $\A$ is in fact both strongly-uniform-RO stable (and in
particular, universally average-RO stable) and an AERM, and thus is a
universally consistent learning rule. Formally speaking, this is not entirely
accurate, because $\A$ is defined only with respect to samples of size $n$,
and hence is not formally a learning rule which can be applied to samples of
any size. However, the analysis we have done earlier in fact carries through
also for learning rules $\A$ which are defined just on a specific sample size
$n$. In particular, the analysis of \lemref{lem:pseudo} and
\lemref{lem:mainimp} hold verbatim for $\A$ (with trivial modifications due to
the fact that $\A$ is randomized), and together imply that since
\eqref{eq:A_AERM} and \eqref{eq:A_stable} hold, then
\[
\E{L(\A(S))-L^*}\leq 4\epscon(\lfloor \sqrt{n}\rfloor)+\frac{8B}{\sqrt{n}}.
\]
Therefore, our learning algorithm is consistent with rate $4\epscon(\lfloor
\sqrt{n}\rfloor)+\frac{8B}{\sqrt{n}}$.
\end{proof}



\section{Discussion}
In this chapter we begun exploring the issue of statistical learnability in the General Setting, and uncovered important relationships between learnability and stability.  However problems are left open and avenues left to explore, some of which are listed below.

First, a natural question that might arise is whether it is possible to come up with well-known machine learning applications, where learnability is achievable despite uniform convergence failing to hold. Subsequently in a very recent work \cite{DanSabBenSha11} it was shows that for multi-class learning problems with large number of classes, problems could still be learnable while uniform convergence fails and ERM approach may not be successful (at least not all ERM's are good).

In \secref{subsec:rand_alg}, we have managed to obtain a completely generic learning algorithm: an algorithm which in principle allows us to learn any learnable problem. However, the algorithm suffers from the severe drawback that in general, it requires unbounded computational power and is not in any succinct form. 
Can we derive an algorithm in a simple form, or characterize classes of learning problems where our algorithm, or some other generic learning algorithm utilizing the notion of stability, can be written for instance, as 
a regularized ERM learning rule?

On a related vein, it would be interesting to develop learning algorithms
(perhaps for specific settings rather than generic learning problems) which
directly use stability in order to learn. Convex regularization is one such
mechanism, as discussed earlier to induce stability. Are there other mechanisms, which use the notion of stability in a different way?

Another issue is that even the existence of uniform-RO stable AERM (or
strongly-uniform-RO stable, always-AERM allowing for convexity/randomization)
is not as elegant and simple as having finite VC dimension or fat-shattering
dimension. It would be very interesting to derive equivalent but more
``combinatorial'' conditions for learnability.

%% file: onln.tex
\chapter{Online Learning/Optimization} \label{chp:online}

In the online learning framework, the learner is faced with a sequence of data appearing at discrete time intervals. In contrast to statistical' learning scenario where the learner is being evaluated after the sequence is completely revealed, in the online framework the learner is evaluated at every round. Furthermore, in the statistical learning scenario the data source is typically assumed to be drawn $i.i.d.$ while in the online framework we relax or eliminate any stochastic assumptions on the data source. As such, the online learning problem can be phrased as a repeated two-player game between the learner (player) and the adversary (Nature).

Unlike the statistical learning framework, there has been surprisingly little work on characterizing learnability and developing generic tools to obtain rates for the online learning framework. Littlestone \cite{Lit88} has shown that, in the setting of prediction of binary outcomes, a certain combinatorial property of the binary-valued function class characterizes learnability in the realizable case (that is, when the outcomes presented by the adversary are given according to some function in the class $\F$). The result has been extended to the non-realizable case by Shai Ben-David, D\'avid P\'al and Shai Shalev-Shwartz \cite{BenPalSha09} who named this combinatorial quantity the {\em Littlestone's dimension}. Coincident with \cite{BenPalSha09}, \emph{minimax} analysis of online convex optimization yielded new insights into the value of the game, its minimax dual representation, as well as algorithm-independent upper and lower bounds \cite{AbeAgaBarRak09, SriTew10}. In this chapter we will build tools analogous to the ones we have to analyze statistical learning problems like Rademacher complexity, covering numbers etc that work for online learning framework.

Section \ref{sec:problem} introduces the problem at hand formally and provides various key definitions like that of online learning algorithm. Section \ref{sec:value} formally defines value of the online learning game and uses it to define learnability of an online learning problem and provides the main minimax theorem which is key in getting many results. In Section \ref{sec:radon} we formally define the sequential Rademacher complexity and shows that it can be used to bound the value and hence get bounds on optimal rates for online learning problems. Structural properties of this complexity measure is also provided. This is perhaps the most important tool we introduce in this chapter. Section \ref{sec:coveron} introduces sequential covering numbers and sequential fat-shattering dimension and shows relation between them and how they can be used to bound the sequential Rademacher complexity. Section \ref{sec:ucvgon} shows how these complexity tools provide a martingale uniform convergence story for online learning. Section \ref{sec:supon} shows how these complexity measures can be used to characterize online learnability of supervised learning problems and goes on to provide a generic algorithm for online supervised learning problems. Section \ref{sec:exampleon} provides various examples illustrating how the results can be used to obtain bounds for various online learning problems.

\section{The Online Learning Problem}\label{sec:problem}
The online learning problem is a continual learning process that proceeds in rounds where in each round adversary picks an instance, learner in turn picks a hypothesis. At the end of the round, the learner pays loss for picking the particular hypothesis against the instance chosen by the adversary for that round. Specifically the online learning protocol can be written as :

\begin{center}\index{learning protocol!online learning}
\fbox{
\begin{minipage}[t]{0.52\textwidth} 
{\bf Online Learning Protocol : } \vspace{0.05in}

{\bf for } $t=1$ { \bf to } n \vspace{-0.12in}
\begin{itemize}
\item[] Learner picks hypothesis $\h_t \in \bcH$  \vspace{-0.12in}
\item[] Adversary simultaneously picks instance $\z_t \in \Z$ \vspace{-0.12in}
\item[] Learner pays loss $\ell(\h_t,\z_t)$ \vspace{-0.12in}
\end{itemize}
{\bf end for}   
\end{minipage}
}
\end{center}
Notice that unlike the statistical learning framework, the instances need not be selected statistically according to some fixed distribution. The adversary at round $t$ can select the instance $\z_t$ in an adversarial worst case fashion based on previous instances $\z_1,\ldots,\z_{t-1}$ and based on previous hypotheses $\h_1,\ldots,\h_{t-1}$ selected by the learner. The automatic question that will rise in the reader's mind would be what is the goal of the learner in this online learning framework. The goal we consider for this online learning framework at the end of $n$ rounds is to have low ``regret" w.r.t. the best single hypothesis from target class $\cH$ that the learner could have picked at hind-sight after knowing $\z_1,\ldots,\z_n$. That is, the regret after $n$ rounds is defined as :
\begin{align}\label{eq:regret}\index{regret}
\Reg_n\left((\h_1,\z_1),\ldots,(\h_n,\z_n)\right) = \frac{1}{n} \sum_{t=1}^n \ell(\h_t,\z_t) - \inf_{\h \in \cH} \frac{1}{n} \sum_{t=1}^n \ell(\h,\z_t)
\end{align}

\paragraph{On the Importance of Randomization :}
Unlike the statistical learning framework, due to the adversarial nature of the online learning framework, randomization of learning rules  become necessary even for very simple (non-convex) problems. To illustrate this let us consider the simple problem of binary classification w.r.t. target hypothesis class $\cH$ consisting of exactly two functions, the constant mapping to $1$ and the constant mapping to $-1$. Now if we consider only deterministic rule, the adversary (knowing the learning rule of the learner), at each round can pick label $y_t$ to be opposite of what the learner predicts. At the end of $n$ rounds, the average loss suffered by the learner is $\frac{1}{n}\sum_{t=1}^n \ell(\h_t,(x_t,y_t)) = 1$. However $\inf_{\h \in \cH} \frac{1}{n}\sum_{t=1}^n \ell(\h,(x_t,y_t)) \le \frac{1}{2}$ Hence regret is lower bounded by $1/2$. This example can easily be extended to other non-convex (in $\cH$) supervised learning problems. Thus we see that in the online learning framework considering randomized learning rules is imperative to get useful results. However whenever $\bcH$ is a vector space and loss $\ell$ is convex in its first argument, simple application of Jensen's inequality shows that it is enough to consider only deterministic learning rules. 

Just like we introduced the notion of "Statistical Learning Rule" to refer to the way the learner picks hypothesis for statistical learning problems, we now formally define online learning rules to refer to learner's strategy for an online learning problem. However as discussed in the previous paragraph, we shall right from the start define learning rule to be a randomized one. 

\begin{definition}\index{learning rule}\index{learning rule!online!randomized}
A ``Randomized Online Learning Rule" $\Algo : \bigcup_{n \in \mathbb{N} \cup \{\}} \left(\bcH^{n} \times \Z^n\right) \mapsto \Delta(\bcH)$ is a mapping from sequences of hypothesis, instance pair in $\bcH \times \Z$ to the set of all Borel distributions over hypothesis set $\bcH$.
\end{definition}

Given a ``Randomized Online Learning Rule" $\Algo$, the way the learner uses this for picking the hypothesis for each round is as follows. On any round $1$, learner simply sampling $\h_1 \sim \Algo(\{\})$ which is some fixed distribution over the set $\bcH$. Further recursively at each round $t$, the learner uses the learning rule on instances $z_1,\ldots,z_{t-1}$ seen so far and hypotheses $\h_1,\ldots,\h_{t-1}$ sampled so far to pick the hypothesis for current round as $\h_t \sim \Algo(\h_{1:t-1},\z_{1:t-1})$. For a randomized learning algorithm our aim will be to ensure that under expectation over randomization, the regret of the learner is small.

We shall often refer to the the ``Randomized Online Learning Rule" as player/learner's strategy as well. Further, whenever the output of the learning rule is deterministic, ie. whenever for each input sequence, the learning rule picks a particular hypothesis with probability one, we will refer to the rule as a ``Deterministic Online Learning Rule". \index{learning rule!online} Notice that if the learning rule is deterministic then at any round $t$, giving hypotheses $\h_1,\ldots,\h_{t-1}$ as argument to the learning is redundant as they can be calculated using just the instances. Hence we see that when we talk of ``Deterministic Learning Rule", it is no different in form from ``Statistical Learning Rules" which are mapping from sequence of instances to hypothesis set $\bcH$.

\section{Online Learnability and the Value of the Game} \label{sec:value}\index{value of the online learning game|(}
How can one define learnability in the online learning framework? Of course, we will refer to a problem as online learnable if there exists a randomized online learning algorithm which can guarantee diminishing expected regret against any strategy of the adversary. We can use the concept of value of a game to succinctly write down that it means for a problem to be online learnable. We will assume that $\Delta(\bcH)$, the set of all Borel probability measures on $\bcH$ is weakly compact. Note that if $\bcH$ is a compact set or is the unit ball of a reflexive Banach space then one can guarantee that $\Delta(\bcH)$ is automatically weakly compact. Hence this restriction is automatically true in most practical cases. We consider randomized learners who predicts a distribution $q_t\in\Delta(\bcH)$ on every round. We can define the value of the game as
\begin{align}  
	\label{eq:def_val_game}
	\Val_n(\cH, \Z)  = \inf_{q_1\in \Delta(\bcH)}\sup_{\z_1\in\Z} \underset{\h_1 \sim q_1}{\En} \cdots  \inf_{q_n\in  \Delta(\bcH)}\sup_{\z_n\in\Z} \underset{\h_n \sim q_n}{\En}\left[ \frac{1}{n} \sum_{t=1}^n \ell(\h_t,\z_t) - \inf_{\h \in \cH}\frac{1}{n}\sum_{t=1}^n \ell(\h,\z_t)\right]
\end{align}
where $\h_t$ has distribution $q_t$. We consider here the {\em adaptive} adversary who gets to choose instance $\z_t$ at round $t$ based on the history of moves $\h_{1:t-1}$ and $\z_{1:t-1}$. 

The above definition is stated in the extensive form, but can be equivalently written in a strategic form. Just like we used $\Algo$ to describe the learner's strategy, analogously we can define the adversarial strategy as a sequence of $\tau : \bigcup_{n \in \mathbb{N}} \left( \H^{n} \times \Z^{n} \right) \mapsto \Delta(\Z)$ where $\Delta(\Z)$ refers to the set of all Borel probability distribution on $\Z$. The value can then be written as
\begin{align}  
	\Val_n(\cH, \Z)  = \inf_{\Algo}\sup_{\tau} \En \left\{ \frac{1}{n}\sum_{t=1}^n \ell(\h_t,\z_t) - \inf_{\h \in \cH} \frac{1}{n} \sum_{t=1}^n \ell(\h,\z_t) \right\} \ .
\end{align}
where it is understood that each $\h_t$ and $\z_t$ are successively drawn according to law $\Algo(\h_{1:t-1},\z_{1:t-1})$ and $\tau(\h_{1:t-1},\z_{1:t-1})$ respectively. While the strategic notation is more succinct, it hides the important sequential structure of the problem. This is the reason why we opt for the more explicit, yet more cumbersome, extensive form. We are now ready to formally define online learnability of a problem. 

\begin{definition}\index{learnability!online}
	We say that a problem is {\em online learnable} with respect to the given instance space $\Z$ against target hypothesis set $\H$ if 
$$ 
\limsup_{n\to \infty} \Val_n(\H, \Z) = 0 ~.
$$
\end{definition}

The first key step is to appeal to the minimax theorem and exchange the pairs of infima and suprema in \eqref{eq:def_val_game}. This dual formulation is easier to analyze because the choice of the player comes \emph{after} the choice of the mixed strategy of the adversary.  We remark that the minimax theorem holds under a very general assumption of weak compactness of $\Delta(\bcH)$. The assumptions on $\bcH$ that translate into weak compactness of $\Delta(\bcH)$ are discussed in Appendix~\ref{app:minimax}. Compactness under weak topology allows us to appeal to Theorem~\ref{thm:minimax} stated below. 

\begin{theorem}\label{thm:minimax}
	Let $\bcH$ and $\Z$ be the sets satisfying the necessary conditions for the minimax theorem to hold, then
\begin{align}
	\Val_n(\cH, \Z)& = \inf_{q_1\in \Delta(\bcH)}\sup_{\z_1\in\Z} \underset{\h_1 \sim q_1}{\En} \cdots  \inf_{q_n\in  \Delta(\bcH)}\sup_{\z_n\in\Z} \underset{\h_n \sim q_n}{\En}\left[ \frac{1}{n} \sum_{t=1}^n \ell(\h_t,\z_t) - \inf_{\h \in \cH}\frac{1}{n}\sum_{t=1}^n \ell(\h,\z_t)\right] \notag \\
	&=\sup_{p_1 \in \Delta(\Z)} \underset{\z_1\sim p_1}{\En} \ldots \sup_{p_n \in \Delta(\Z)} \underset{\z_n\sim p_n}{\En} \left[
	  \frac{1}{n}\sum_{t=1}^n \inf_{\h_t \in \bcH}
	  	\underset{\z_t \sim p_t}{\En}[\ell(\h_t,\z_t)] - \inf_{\h\in\cH} \frac{1}{n} \sum_{t=1}^n \ell(\h,\z_t)
	\right] \label{eq:value_equality}
\end{align}	
\end{theorem}
The question of learnability in the online learning model is now reduced to the study of $\Val_n(\H,\Z)$, taking Eq.~\eqref{eq:value_equality} as the starting point. In particular, under our definition, showing that the value grows sublinearly with $n$ is equivalent to showing learnability.

One of the key notions introduced in this chapter is the complexity which we term {\em Sequential Rademacher complexity}. A natural generalization of Rademacher complexity \cite{koltchinskii2000rademacher, BarMed03, Men03fewnslt}, the sequential analogue possesses many of the nice properties of its classical cousin. The properties are proved in Section~\ref{sec:structural} and then used to show learnability for many of the examples in Section~\ref{sec:examples}. The first step, however, is to show that Sequential Rademacher complexity upper bounds the value of the game. This is the subject of the next section.

\index{value of the online learning game|)}

\section{Sequential Rademacher Complexity} \label{sec:radon} \index{Rademacher complexity!sequential}
We propose the following definition of sequential Rademacher complexity of any function class $\F \subset \reals^{\Z}$. The key difference from the classical notion is the dependence of the sequence of data on the sequence of signs (Rademacher random variables). As shown in the sequel, this dependence captures the sequential nature of the problem.
\begin{definition} 
	The \emph{Sequential Rademacher Complexity} of a function class $\F \subseteq \reals^\Z$ is defined as
$$
\Radon_n(\mathcal{F}) = \sup_{\tz} \Es{\epsilon}{ \sup_{f \in \F} \frac{1}{n} \sum_{t=1}^{n} \epsilon_t f(\tz_t(\epsilon))}
$$
where the outer supremum is taken over all $\Z$-valued trees ($\tz$) of depth $n$ and $\epsilon=(\epsilon_1,\ldots, \epsilon_n)$ is a sequence of i.i.d. Rademacher random variables.
\end{definition}

In statistical learning, Rademacher complexity is shown to control uniform deviations of means and expectations, and this control is key for learnability in the ``batch'' setting. We now show that Sequential Rademacher complexity upper-bounds the value of the game, suggesting its importance for online learning (see Section~\ref{sec:supervised} for a lower bound). 

\begin{theorem}\label{thm:valrad}
The minimax value of a randomized game is bounded as
$$
\Val_n(\H,\Z) \le 2 \Radon_n(\F)
$$
where the function class $\F$ is given by, $\F = \{ \z \mapsto \ell(\h,\z) :  \h \in \cH \}~.$
\end{theorem}

\begin{proof}
From Eq.~\eqref{eq:value_equality},
\begin{align}
\Val_n(\H,\Z) & = \sup_{p_1} \En_{\z_1 \sim p_1} \ldots \sup_{p_n} \En_{\z_n \sim p_n}\left[ \frac{1}{n} \sum_{t=1}^n  \inf_{\h_t \in \bcH} \Es{\z_t \sim p_t}{\ell(\h_t,\z_t)} - \inf_{\h \in \cH} \frac{1}{n} \sum_{t=1}^n \ell(\h,\z_t) \right] \notag \\
& = \sup_{p_1} \En_{\z_1 \sim p_1} \ldots \sup_{p_n} \En_{\z_n \sim p_n}\left[ \sup_{\h \in \cH} \left\{ \frac{1}{n} \sum_{t=1}^n  \inf_{\h_t \in \bcH} \Es{\z_t \sim p_t}{\ell(\h_t,\z_t)} -  \frac{1}{n} \sum_{t=1}^n \ell(\h,\z_t) \right\} \right] \notag \\
& \le \sup_{p_1} \En_{\z_1 \sim p_1} \ldots \sup_{p_n} \En_{\z_n \sim p_n}\left[ \sup_{\h \in \cH} \left\{ \frac{1}{n} \sum_{t=1}^n \Es{\z_t \sim p_t}{\ell(\h,\z_t)} -  \frac{1}{n} \sum_{t=1}^n \ell(\h,\z_t) \right\} \right]  \label{eq:beforeexpequal}
\end{align}
The last step, in fact, is the first time we deviated from keeping equalities. The upper bound is obtained by replacing each infimum by a particular choice $\h$. Now renaming variables we have,
\begin{align*}
\Val_n(\H,\Z) &=\sup_{p_1} \En_{\z_1 \sim p_1} \ldots \sup_{p_n} \En_{\z_n \sim p_n}\left[ \sup_{\h \in \cH} \left\{ \frac{1}{n} \sum_{t=1}^n \Es{\z'_t \sim p_t}{\ell(\h,\z'_t)} -  \frac{1}{n} \sum_{t=1}^n \ell(\h,\z_t) \right\} \right] \\
& \le \sup_{p_1} \En_{\z_1 \sim p_1} \ldots \sup_{p_n} \En_{\z_n \sim p_n}\left[ \En_{\z'_1 \sim p_1} \ldots \En_{\z'_n \sim p_n} \sup_{\h \in \cH} \left\{ \frac{1}{n} \sum_{t=1}^n \ell(\h,\z'_t) -  \frac{1}{n}\sum_{t=1}^n \ell(h,\z_t) \right\} \right] \\
& \leq \sup_{p_1} \En_{\z_1, \z'_1 \sim p_1} \ldots \sup_{p_n} \En_{\z_n, \z'_n \sim p_n}\left[  \sup_{\h \in \cH} \left\{ \frac{1}{n}\sum_{t=1}^n \ell(\h,\z'_t) - \frac{1}{n} \sum_{t=1}^n \ell(\h,z_t) \right\} \right] ~.
\end{align*}
where the last two steps are using Jensen inequality for the supremum. 

By the Key Technical Lemma (see Lemma~\ref{lem:technical_symmetrization} below) with $\phi(u) = u$,
\begin{align*}
	\sup_{p_1} \En_{\z_1, \z'_1 \sim p_1} \ldots \sup_{p_n} \En_{\z_n, \z'_n \sim p_n}& \left[ \sup_{\h \in \cH}\left\{ \frac{1}{n} \sum_{t=1}^n  \ell(\h,\z_t') - \ell(\h,\z_t) \right\} \right]\\
& \leq \sup_{\z_{1},\z_{1}'} \En_{\epsilon_{1}} \ldots \sup_{\z_n, \z'_n} \Es{\epsilon_n}{\sup_{h \in\cH} \frac{1}{n} \sum_{t=1}^{n} \epsilon_t \left(\ell(\h,\z'_t) - \ell(\h,\z_t) \right)}
\end{align*}

Thus,
\begin{align*}
\Val_n(\F)& \leq \sup_{\z_{1},\z_{1}'} \En_{\epsilon_{1}} \ldots \sup_{\z_n, \z'_n} \Es{\epsilon_n}{\sup_{\h \in\cH} \frac{1}{n} \sum_{t=1}^{n} \epsilon_t \left(\ell(\h,\z'_t) - \ell(\h\z_t) \right)}  \\
& \le \sup_{\z_{1},\z_{1}'} \En_{\epsilon_{1}} \ldots \sup_{\z_n, \z'_n} \Es{\epsilon_n}{\sup_{\h \in\cH}\left\{  \frac{1}{n} \sum_{t=1}^{n} \epsilon_t \ell(\h,\z'_t) \right\} + \sup_{\h \in \cH}\left\{\frac{1}{n} \sum_{t=1}^n  -\epsilon_t  \ell(\h,\z_t) \right\}}  \\
& = 2 \sup_{\z_{1}} \En_{\epsilon_{1}} \ldots \sup_{\z_n} \En_{\epsilon_n} \left[ \sup_{\h \in\cH}  \frac{1}{n} \sum_{t=1}^{n} \epsilon_t \ell(\h,\z_t)  \right]
\end{align*}
Now, we need to move the suprema over $\z_t$'s outside. This is achieved via an idea similar to skolemization in logic. We basically exploit the identity
\begin{equation*}
\Es{\epsilon_{1:t-1}}{ \sup_{\z_t} G(\epsilon_{1:t-1},\z_t) } = \sup_{\tz_t} \Es{\epsilon_{1:t-1}}{G(\epsilon_{1:t-1},\tz_t(\epsilon_{1:t-1})) }
\end{equation*}
that holds for any $G:\{\pm1\}^{t-1}\times \Z \mapsto \reals$. On the right the supremum is over functions $\tz_t: \{\pm1\}^{t-1} \to \Z$.
Using this identity once, we get,
\begin{align*}
\Val_n(\F) & \le \frac{2}{n} \sup_{\z_1,\z_2} \left\{
	\Es{\epsilon_1,\epsilon_2}{
		\sup_{\z_3} \ldots \sup_{\z_n} \left\{
			\Es{\epsilon_n}{
				\sup_{\h\in\cH}\left\{
					\epsilon_1\ell(\h,\z_1) + \epsilon_2 \ell(\h,\tz_2(\epsilon_1)) + \sum_{t=3}^{n} \epsilon_t \ell(\h,\z_t)
				\right\}
			}
		\right\} \ldots
	}
\right\}
\end{align*}
Now, use the identity $n-2$ more times to successively move the supremums over $\z_3,\ldots,\z_n$ outside, to get
\begin{align*}
\Val_n(\F) & \le \frac{2}{n} \sup_{\z_1,\tz_2,\ldots,\tz_n} \Es{\epsilon_{1},\ldots, \epsilon_n}{\sup_{\h \in\cH}\left\{  \epsilon_1\ell(\h,\z_1) + \sum_{t=2}^{n} \epsilon_t \ell(\h,\tz_t(\epsilon_{1:t-1})) \right\} } \\
& = 2 \sup_{\tz}\Es{\epsilon_{1},\ldots, \epsilon_n}{\sup_{\h \in\cH}\left\{  \frac{1}{n} \sum_{t=1}^{n} \epsilon_t \ell(\h,\tz_t(\epsilon)) \right\} }
\end{align*}
where the last supremum is over $\Z$-valued trees of depth $n$. Thus we have proved the required statement.
\end{proof}

Theorem~\ref{thm:valrad} relies on the following technical lemma, which will be used again in Section~\ref{sec:universal}. Its proof requires considerably more work than the classical symmetrization proof \cite{Dudley99, Men03fewnslt} due to the non-i.i.d. nature of the sequences.

\begin{lemma}[Key Technical Lemma]
	\label{lem:technical_symmetrization}
	Let $(\z_1,\ldots,\z_n)\in\Z^n$ be a sequence distributed according to $\D$ and let $(\z'_1, \ldots, \z'_n)\in\Z^n$ be a tangent sequence\index{tangent sequence}. Let $\phi:\reals\mapsto\reals$ be a measurable function. Then 
	\begin{align*}
		 & \sup_{p_1} \underset{\z_1, \z'_1 \sim p_1}{\En} \ldots \sup_{p_n} \underset{\z_n, \z'_n \sim p_n}{\En}\left[ \phifunc{\sup_{\h\in\cH} \frac{1}{n} \sum_{t=1}^n  \Delta_\h (\z_t, \z'_t) } \right] \\
		 & \hspace{2in} \leq \sup_{\z_{1},\z_{1}'} \En_{\epsilon_{1}} \ldots \sup_{\z_n, \z'_n} \En_{\epsilon_n} \left[ \phifunc {\sup_{\h\in\cH} \frac{1}{n} \sum_{t=1}^{n} \epsilon_t  \Delta_\h (\z_t, \z'_t)}  \right]
	\end{align*}	
	where $\epsilon_1,\ldots,\epsilon_n$ are independent (of each other and everything else) Rademacher random variables and $\Delta_\h(x_t,x'_t) = \frac{1}{n}\left(\ell(\h,\z'_t)-\ell(\h,\z_t)\right)$. The inequality also holds when an absolute value of the sum is introduced on both sides.
\end{lemma}

Before proceeding, let us give some intuition behind the attained bounds. Theorem~\ref{thm:minimax} establishes an upper bound on the value of the game in terms of a stochastic process on $\Z$. In general, it is difficult to get a handle on the behavior of this process. The key idea is to relate this process to a symmetrized version, and then pass to a new process obtained by fixing a binary tree $\tz$ and then following a path in $\tz$ using i.i.d. coin flips. In some sense, we are replacing the $\sigma$-algebra generated by the random process of Theorem~\ref{thm:minimax} by a simpler process generated by Rademacher random variables and a tree $\tz$. It can be shown that the original process and the simpler process are in fact close in a certain sense, yet the process generated by the Rademacher random variables is much easier to work with. It is precisely due to symmetrization that the trees we consider in this work are binary trees and not full game trees. Passing to binary trees allows us to define covering numbers, combinatorial parameters, and  other analogues of the classical notions from statistical learning theory.

\subsection{Structural Results}
\label{sec:structural}

Being able to bound complexity of a function class by a complexity of a simpler class is of great utility for proving bounds. In statistical learning theory, such structural results are obtained through properties of Rademacher averages \cite{Men03fewnslt, BarMed03}. In particular, the contraction inequality due to Ledoux and Talagrand \cite[Corollary 3.17]{LedouxTalagrand91}, allows one to pass from a composition of a Lipschitz function with a class to the function class itself. This wonderful property permits easy convergence proofs for a vast array of problems.

We show that the notion of Sequential Rademacher complexity also enjoys many of the same properties. In Section~\ref{sec:examples}, the effectiveness of the results is illustrated on a number of examples. 

The next lemma bounds the Sequential Rademacher complexity for the product of function classes. 
\begin{lemma}\label{lem:inflip}
Let $\F = \F_1 \times \ldots \times \F_k$ where each $\F_j \subset [-1,1]^{\Z}$. Also let $\phi : \reals^k \times \Z \mapsto \reals$ be such that $\phi(\cdot,z)$ is $L$-Lipschitz w.r.t. $\|\cdot\|_\infty$ norm for any $z\in\Z$. Then we have that 
$$
\Radon_n(\phi \circ \F) \le 8\,L\,\left(1+ 4\sqrt{2}\log^{3/2}(en^2)\right) \sum_{j=1}^k \Radon_n(\F_j) 
$$
as long as $\Radon_n(\F_j) \ge 1$ for each $j$.
\end{lemma}

As a special case of the result, we get a sequential counterpart of the Ledoux-Talagrand \cite{LedouxTalagrand91} contraction inequality. 

\begin{lemma}
	\label{lem:contraction}
	Fix a class $\F\subseteq [-1,1]^\Z$ with $\Radon_n(\F)\ge 1$ and a function $\phi:\reals\times \Z\mapsto\reals$.
	Assume, for all $z \in \Z$,
	$\phi(\cdot,z)$ is a Lipschitz function with a constant $L$.
	$$ \Radon (\phi(\F)) \leq 8\,L\,\left(1+4\sqrt{2}\log^{3/2}(en^2)\right) \cdot\Radon_n(\F)$$
	where $\phi(\F) = \{z \mapsto \phi(f(z),z): f\in \F\}$.
\end{lemma}
We remark that the lemma above encompasses the case of a Lipschitz $\phi:\reals\mapsto\reals$, as stated in \cite{LedouxTalagrand91, BarMed03}. However, here we get an extra logarithmic factor in $n$ which is absent in the classical case. Whether
the same can be proved in the sequential case remains an open question.

We state another useful corollary of Lemma~\ref{lem:inflip}.

\begin{corollary}
	\label{cor:radem_binary}
For a fixed binary function $b : \{\pm1\}^k \mapsto \{\pm1\}$ and classes $\F_1, \ldots, \F_k$ of $\{\pm 1\}$-valued functions,
$$
\Radon_n(g(\F_1,\ldots,\F_k)) \le \mc{O}\left(\log^{3/2}(n)\right) \sum_{j=1}^k \Radon_n(\F_j) 
$$
\end{corollary}

In the next proposition, we summarize some useful properties of Sequential Rademacher complexity (see \cite{Men03fewnslt, BarMed03} for the results in the i.i.d. setting)

\begin{proposition}
	\label{prop:rademacher_properties}
	Sequential Rademacher complexity satisfies the following properties.
	\begin{enumerate}
		\item If $\F\subset \G$, then $\Radon_n(\F) \leq \Radon_n(\G)$.
		\item $\Radon_n(\F) = \Radon (\conv(\F))$.
		\item $\Radon_n(c\F) = |c|\Radon_n(\F)$ for all $c\in\reals$.
		\item For any $h$, $\Radon_n(\F+h) =  \Radon_n(\F)$ where $\F+h = \{f+h: f\in\F\}$.
	\end{enumerate}
\end{proposition}

\section{Sequential Covering Number and Combinatorial Parameters} \label{sec:coveron}
In statistical learning theory, learnability for binary classes of functions is characterized by the Vapnik-Chervonenkis combinatorial dimension \cite{VapChe71}. For real-valued function classes, the corresponding notions are the scale-sensitive dimensions, such as $P_\gamma$ \cite{ABCH97, BarLonWil96}. For online learning, the notion characterizing learnability for binary prediction in the realizable case has been introduced by Littlestone \cite{Lit88} and extended to the non-realizable case of binary prediction by Shai Ben-David, D\'avid P\'al and Shai Shalev-Shwartz \cite{BenPalSha09}. Next, we define the Littlestone's dimension  \cite{Lit88,BenPalSha09} and propose its scale-sensitive versions for real-valued function classes. In the sequel, these combinatorial parameters are shown to control the growth of covering numbers on trees. In the setting of prediction, the combinatorial parameters are shown to exactly characterize learnability (see Section~\ref{sec:supervised}).

\begin{definition}\index{shattering!sequential}
A $\Z$-valued tree $\tz$ of depth $d$ is \emph{shattered} by a function class $\F \subseteq \{\pm 1\}^{\Z}$ if for all $\epsilon \in \{\pm1\}^{d}$, there exists $f \in \F$ such that $f(\tz_t(\epsilon)) = \epsilon_t$ for all $t \in [d]$. 
	The \emph{Littlestone dimension} $\ldim(\F, \Z)$ is the largest $d$ such that $\F$ shatters an $\Z$-valued tree of depth $d$. \index{Littlestone dimension}
\end{definition}

\begin{definition}\index{fat-shattering!sequential}
An $\Z$-valued tree $\tz$ of depth $d$ is \emph{$\alpha$-shattered} by a function class $\F \subseteq \reals^\Z$ , if there exists an $\reals$-valued tree $\ts$ of depth $d$ such that 
$$
\forall \epsilon \in \{\pm1\}^d , \ \exists f \in \F \ \ \ \trm{s.t. } \forall t \in [d], \  \epsilon_t (f(\tz_t(\epsilon)) - \ts_t(\epsilon)) \ge \alpha/2
$$
The tree $\ts$ is called the \emph{witness to shattering}. The \emph{fat-shattering dimension} $\faton_\alpha(\F, \Z)$ at scale $\alpha$ is the largest $d$ such that $\F$ $\alpha$-shatters an $\Z$-valued tree of depth $d$. 
\end{definition}

With these definitions it is easy to see that $\faton_\alpha(\F, \Z) = \ldim(\F, \Z)$ for a binary-valued function class $\F \subseteq \{\pm1\}^\Z$ for any $0 <\alpha \leq 2$.

When $\Z$ and/or $\F$ is understood from the context, we will simply write $\faton_\alpha$ or $\faton_\alpha(\F)$ instead of $\faton_\alpha(\F, \Z)$. Furthermore, we will write $\faton_\alpha(\F, \tz)$ for $\faton_\alpha(\F, \Img(\tz))$. In other words, $\faton_\alpha(\F,\tz)$ is the largest $d$ such that $\F$ $\alpha$-shatters a tree $\ty$ of depth $d$ with $\Img(\ty)\subseteq \Img(\tz)$.

Let us mention that if trees $\tz$ are defined by constant mappings $\tz_t(\epsilon) = \z_t$, the combinatorial parameters coincide with the Vapnik-Chervonenkis dimension and with the scale-sensitive dimension $P_\gamma$. Therefore, the notions we are studying are strict ``temporal'' generalizations of the VC theory.

As in statistical learning theory, the combinatorial parameters are only useful if they can be shown to capture that aspect of $\F$ which is important for learnability. In particular, a ``size'' of a function class is known to be related to complexity of learning from i.i.d. data., and the classical way to measure ``size'' is through a cover or of a packing set. We propose the following definitions for online learning.

\begin{definition}\index{covering number!sequential}
	\label{def:cover}
A set $V$ of $\reals$-valued trees of depth $n$ is \emph{an $\alpha$-cover} (with respect to $\ell_p$-norm) of $\F \subseteq \reals^\Z$ on a tree $\tz$ of depth $n$ if
$$
\forall f \in \F,\ \forall \epsilon \in \{\pm1\}^n \ \exists \tv \in V \  \mrm{s.t.}  ~~~~ \left( \frac{1}{n} \sum_{t=1}^n |\tv_t(\epsilon) - f(\tz_t(\epsilon))|^p \right)^{1/p} \le \alpha
$$
The \emph{covering number} of a function class $\F$ on a given tree $\tz$ is defined as 
$$
\Nonl_p(\alpha, \F, \tz) = \min\{|V| :  V \ \trm{is an }\alpha-\text{cover w.r.t. }\ell_p\trm{-norm of }\F \trm{ on } \tz\}.
$$
Further define $\Nonl_p(\alpha, \F , n) = \sup_\tz \Nonl_p(\alpha, \F, \tz) $, the maximal $\ell_p$ covering number of $\F$ over depth $n$ trees. 
\end{definition}

In particular, a set $V$ of $\reals$-valued trees of depth $n$ is a \emph{$0$-cover} of $\F \subseteq \reals^\Z$ on a tree $\tz$ of depth $n$ if
$$
\forall f \in \F,\ \forall \epsilon \in \{\pm1\}^n \ \exists \v \in V \  \mrm{s.t.}  ~~~~ \v_t(\epsilon) =  f(\tz_t(\epsilon))
$$
We denote by $\Nonl(0,\F,\tz)$ the size of a smallest $0$-cover on $\tz$ and $\Nonl(0,\F,n) = \sup_\tz \Nonl(0,\F,\tz)$.

Let us discuss a subtle point. The $0$-cover should not be mistaken for the size $|\F(\tz)|$ of the projection of $\F$ onto the tree $\tz$, and the same care should be taken when dealing with $\alpha$-covers. Let us illustrate this with an example. Consider a tree $\tz$ of depth $n$ and suppose for simplicity that $|\Img(\tz)| = 2^n-1$, i.e. the values of $\tz$ are all distinct. Suppose $\F$ consists of $2^{n-1}$ binary-valued functions defined as zero on all of $\Img(\tz)$ except for a single value of $\Img(\tz_n)$. In plain words, each function is zero everywhere on the tree except for a single leaf. While the projection $\F(\tz)$ has $2^{n-1}$ distinct trees, the size of a $0$-cover is only $2$. It is enough to take an all-zero function $g_0$ along with a function $g_1$ which is zero on all of $\Img(\tz)$ except $\Img(\tz_n)$ (i.e. on the leaves). It is easy to verify that $g_0(\tz)$ and $g_1(\tz)$ provide a $0$-cover for $\F$ on $\tz$, and therefore, unlike $|\F(\tz)|$, the size of the cover does not grow with $n$. The example is encouraging: our definition of a cover captures the fact that the function class is ``simple'' for any given path.

Next, we naturally propose a definition of a packing.

\begin{definition} \index{packing number!sequential!weak}
A set $V$ of $\reals$-valued trees of depth $n$ is said to be $\alpha$-separated if
	$$
	\forall \v\in V,\ \exists \epsilon\in\{\pm 1\}^n  \  \mrm{s.t.} \; \; \; \forall \w\in V\setminus\{\v\}  ~~~~  \left( \frac{1}{n}\sum_{t=1}^n |\v_t(\epsilon) - \w_t(\epsilon)|^p\right)^{1/p} > \alpha 
	$$
	The \emph{weak packing number} $\cD_p(\alpha, \F, \tz)$ of a function class $\F$ on a given tree $\tz$ is the size of the largest $\alpha$-separated subset of $\{f(\tz) : f\in\F\}$.
	
\end{definition}

\begin{definition} \index{packing number!sequential!strong}
A set $V$ of $\reals$-valued trees of depth $n$ is said to be \emph{strongly} $\alpha$-separated if
	$$
	\exists \epsilon\in\{\pm 1\}^n  \  \mrm{s.t.} \; \; \; \forall \v, \w \in V, \v\neq\w  ~~~~  \left( \frac{1}{n}\sum_{t=1}^n |\v_t(\epsilon) - \w_t(\epsilon)|^p\right)^{1/p} > \alpha 
	$$
	The \emph{strong packing number} $\M_p(\alpha, \F, \tz)$ of a function class $\F$ on a given tree $\tz$ is the size of the largest strongly $\alpha$-separated subset of $\{f(\tz) : f\in\F\}$.
\end{definition}

Note the distinction between the packing number and the strong packing number. For the former, it must be that every member of the packing is $\alpha$-separated from every other member on \emph{some} path.
For the latter, there must be a path on which every member of the packing is $\alpha$-separated from every other member. This distinction does not arise in the classical scenario of ``batch'' learning.
We observe that if a tree $\tz$ is defined by constant mappings $\tz_t = x_t$, the two notions of packing and strong packing coincide, i.e. $\cD_p(\alpha, \F, \tz)=\M_p(\alpha, \F, \tz)$. The following lemma gives a relationship between covering numbers and the two notions of packing numbers. The form of this should be familiar, except for the distinction between the two types of packing numbers.

\begin{lemma}
	\label{lem:packing_covering_ineq}
	For any $\F\subseteq \reals^\Z$, any $\Z$-valued tree $\tz$ of depth $n$, and any $\alpha>0$
	$$\M_p(2\alpha, \F, \tz) \leq \Nonl_p (\alpha, \F, \tz) \leq \cD_p(\alpha, \F, \tz).$$
\end{lemma}

It is important to note that the gap between the two types of packing can be as much as $2^n$.

\subsection{A Combinatorial Upper Bound}

We now relate the combinatorial parameters introduced in the previous section to the size of a cover. In the binary case ($k=1$ below), a reader might notice a similarity of Theorems~\ref{thm:sauer_multiclass}~and~\ref{thm:sauer_multiclass_0_cover} to the classical results due to Sauer \cite{Sauer72}, Shelah \cite{Shelah1972} (also, Perles and Shelah), and Vapnik and Chervonenkis \cite{VapChe71}. There are several approaches to proving what is often called the Sauer-Shelah lemma. We opt for the inductive-style proof (e.g. Alon and Spencer \cite{AloSpe00}). Dealing with trees, however, requires more work than in the VC case \index{VC dimension}.

\begin{theorem}
	\label{thm:sauer_multiclass}
	Let $\F \subseteq {\{0,\ldots, k\}}^\Z$ be a class of functions with $\faton_2(\F) = d$. Then 
	$$ \Nonl_\infty(1/2, \F , n) \leq \sum_{i=0}^d {n\choose i} k^i \leq \left(ekn \right)^d.$$ 
	Furthermore, for $n\geq d$ 
	$$\sum_{i=0}^d {n\choose i} k^i \leq \left(\frac{ekn}{d}\right)^d.$$
\end{theorem}

Armed with Theorem~\ref{thm:sauer_multiclass}, we can approach the problem of bounding the size of a cover at an $\alpha$ scale by a discretization trick. For the classical case of a cover based on a set points, the discretization idea appears in \cite{ABCH97, MenVer03}. When passing from the combinatorial result to the cover at scale $\alpha$ in Corollary~\ref{cor:l2_norm_bound}, it is crucial that Theorem~\ref{thm:sauer_multiclass} is in terms of $\faton_2(\F)$ and not $\faton_1(\F)$. This point can be seen in the proof of Corollary~\ref{cor:l2_norm_bound} (also see \cite{MenVer03}): the discretization process can assign almost identical function values to discrete values which differ by $1$. This explains why the combinatorial result of Theorem~\ref{thm:sauer_multiclass} is proved for the $2$-shattering dimension.

We now show that the covering numbers are bounded in terms of the fat-shattering dimension.
\begin{corollary}
	\label{cor:l2_norm_bound}
	Suppose $\F$ is a class of $[-1,1]$-valued functions on $\Z$. Then for any $\alpha >0$, any $n>0$, and any $\Z$-valued tree $\tz$ of depth $n$,
	$$ \Nonl_1(\alpha, \F, \tz) \leq \Nonl_2(\alpha, \F, \tz) \leq \Nonl_\infty(\alpha, \F, \tz) \leq \left(\frac{2e n}{\alpha}\right)^{\faton_{\alpha} (\F) }$$
\end{corollary}

With a proof similar to Theorem~\ref{thm:sauer_multiclass}, a bound on the $0$-cover can be proved in terms of the $\faton_1(\F)$ combinatorial parameter. Of particular interest is the case $k=1$, when $\faton_1 (\F)= \ldim(\F)$.
\begin{theorem}
	\label{thm:sauer_multiclass_0_cover}
	Let $\F \subseteq {\{0,\ldots, k\}}^\Z$ be a class of functions with $\faton_1(\F) = d$. Then 
	$$ \Nonl (0, \F , n) \leq \sum_{i=0}^d {n\choose i} k^i \leq \left(ekn \right)^d.$$ 
	Furthermore, for $n\geq d$ 
	$$\sum_{i=0}^d {n\choose i} k^i \leq \left(\frac{ekn}{d}\right)^d.$$
	In particular, the result holds for binary-valued function classes ($k=1$), in which case $\faton_1(\F)=\ldim(\F)$.
\end{theorem}

When bounding deviations of means from expectations uniformly over the function class, the usual approach proceeds by a symmetrization argument \cite{GinZin84} followed by passing to a cover of the function class and a union bound (e.g. \cite{Men03fewnslt}). Alternatively, a more refined {\em chaining} analysis integrates over covering at different scales (e.g. \cite{Sara00}). By following the same path, we are able to prove a number of similar results for our setting. In the next section we present a bound similar to Massart's finite class lemma \cite[Lemma 5.2]{Mas00}, and in the following section this result will be used when integrating over different scales for the cover.

\subsection{Finite Class Lemma and the Chaining Method}

\begin{lemma}\label{lem:fin}
For any finite set $V$ of $\reals$-valued trees of depth $n$ we have that
$$
\Es{\epsilon}{\max_{\v \in V} \sum_{t=1}^n \epsilon_t \v_t(\epsilon)} \le \sqrt{2 \log(|V|) \max_{\v \in V} \max_{\epsilon \in \{\pm1\}^n} \sum_{t=1}^n \v_t(\epsilon)^2}
$$
\end{lemma}

A simple consequence of the above lemma is that if $\F \subseteq [0,1]^\Z$ is a finite class, then for any given tree $\tz$ we have that
$$
\Es{\epsilon}{\max_{f \in \F} \frac{1}{n}\sum_{t=1}^n \epsilon_t f(\tz_t(\epsilon))} \le \frac{1}{n}\Es{\epsilon}{\max_{\v \in \F(\tz) } \sum_{t=1}^n \epsilon_t \v_t(\epsilon)} \le \sqrt{\frac{2  \log |\F|}{n} } \ .
$$
Note that if $f\in \F$ is associated with an ``expert'', this result combined with Theorem~\ref{thm:valrad} yields a bound given by the exponential weighted average forecaster algorithm (see \cite{PLG}). In Section~\ref{sec:examples} we discuss this case in more detail. However, as we show next, Lemma~\ref{lem:fin} goes well beyond just finite classes and can be used to get an analog of Dudley entropy bound \cite{Dudley67} for the online setting through a chaining argument.

\begin{definition}\index{Dudley integrated complexity!sequential}
	The \emph{Integrated complexity} of a function class $\F \subseteq [-1,1]^\Z$ is defined as
$$
\Dudleyon_n (\F) = \inf_{\alpha}\left\{4 \alpha + 12\int_{\alpha}^{1} \sqrt{ \frac{ \log \ \mathcal{N}_2(\delta, \F,n )}{n} } d \delta \right\} .
$$
\end{definition}

To prove the next theorem, we consider covers of the class $\F$ at different scales that form a geometric progression. We zoom into a given function $f \in \F$ using covering elements at successive scales. This zooming in procedure is visualized as forming a chain that consists of links connecting elements of covers at successive scales. The Rademacher complexity of $\F$ can then be bounded by controlling the Rademacher complexity of the link classes, i.e. the class consisting of differences of functions from covers at neighbouring scales. This last part of the argument is the place where our proof becomes a bit more involved than the classical case.

\begin{theorem}\label{thm:dudley}
For any function class $\F\subseteq [-1,1]^\Z$,
\begin{align*}
\Radon_n(\F) \le \Dudleyon_n(\F) \ .
\end{align*}
\end{theorem}

If a fat-shattering dimension of the class can be controlled, Corollary~\ref{cor:l2_norm_bound} together with Theorem~\ref{thm:dudley} yield an upper bound on the value. 

We can now show that, in fact, the two complexity measures $\Radon_n(\F)$ and $\Dudleyon_n(\F)$ are equivalent, up to a logarithmic factor. Before stating this result formally, we prove the following lemma which asserts that the fat-shattering dimensions at ``large enough'' scales cannot be too large. 
\begin{lemma}
	\label{lem:simprel}
	For any $\beta > \frac{2}{n} \Radon_n(\F)$, we have that $\faton_\beta(\F) < n$. 
\end{lemma}

The following lemma complements Theorem~\ref{thm:dudley}.

\begin{lemma}
	\label{lem:dudley_lower_by_rad}
	For any function class $\F\subseteq [-1,1]^\Z$, we have that
	$$
	\Dudleyon_n(\F) \leq  8\ \Radon_n(\F)\left( 1 + 4 \sqrt{2} \ \log^{3/2}\left(e n^2\right)  \right)
	$$
	as long as $\Radon_n(\F) \ge \frac{1}{n}$.
\end{lemma}

\section{Martingale Uniform Convergence}\label{sec:ucvgon}
As we discussed in the previous chapter, in the statistical learning setting  learnability of supervised learning problem is equivalent to the so called uniform Glivenko-Cantelli property (uniform convergence) of the class. The property refers to the empirical averages converging to expected value of the function for any fixed distribution (samples drawn i.i.d.) and uniformly over the function class almost surely. Tools like classical Rademacher complexity, statistical covering numbers and statistical fat-shattering dimension can be used in analyzing rate of uniform convergence of empirical average to expected value of the function for samples drawn i.i.d from any fixed distribution. Analogously sequential counterparts of these complexity measures can be used to bound rate of uniform convergence of average value of the function to average conditional expectation of the function values for arbitrary distributions over sequence of random variables. In fact in the proof of Theorem \ref{thm:valrad} we already encountered this general martingale uniform convergence in expectation. Specifically Equation \ref{eq:beforeexpequal} which we showed is bounded by the sequential Rademacher complexity. We now formally define Universal Uniform Convergence which is analogous to the usual definition of uniform Glivenko-Cantelli property for general dependent processes and in this section we will show how the sequential complexity measures provide tools for bounding this martingale version of uniform convergence.

\begin{definition}\index{uniform convergence!universal}
A function class $\F$ satisfies a \emph{Universal Uniform Convergence} if for all $\alpha > 0$,
\begin{align*}
\lim_{N \rightarrow \infty} \sup_{\D} \mathbb{P}_\D\left( \sup_{n \ge N} \sup_{f \in \mathcal{F}} \frac{1}{n} \left|\sum_{t=1}^n \left(f(\z_t) - \mathbb{E}_{t-1}[f(\z_t)] \right) \right| > \alpha\right) = 0
\end{align*}
where the supremum is over distributions $\D$ over infinite sequences $(x_1,\ldots,x_n,\ldots )$
\end{definition}

We remark that the notion of uniform Glivenko-Cantelli classes is recovered if the supremum is taken over i.i.d. distributions. The theorem below shows that finite fat shattering dimension at all scales is a sufficient condition for \emph{Universal Uniform Convergence}.

\begin{theorem}
	\label{thm:universal}
	Let $\F$ be a class of $[-1,1]$-valued functions.
	If $\faton_\alpha(\F)$ is finite for all $\alpha > 0$, then $\F$ satisfies Universal Uniform Convergence.
\end{theorem}

The proof follows from the Lemma~\ref{lem:symmetrized_probability} and Lemma~\ref{lem:pollard_probability} below, while Lemma~\ref{lem:dudley_probability} is an even stronger version of Lemma~\ref{lem:pollard_probability}. We remark that Lemma~\ref{lem:symmetrized_probability} is the ``in-probability'' version of sequential symmetrization technique of Theorem~\ref{thm:valrad} and Lemma~\ref{lem:dudley_probability} is the ``in-probability'' version of Theorem~\ref{thm:dudley}. 

\begin{lemma}
	\label{lem:symmetrized_probability}
	Let $\F$ be a class of $[-1,1]$-valued functions. Then for any $\alpha>0$
	$$
	\mathbb{P}_{\D}\left( \frac{1}{n} \sup_{f \in \mathcal{F}} \left|\sum_{t=1}^n \left(f(\z_t) - \mathbb{E}_{t-1}[f(\z_t)] \right) \right| > \alpha \right)  \le 4 \sup_{\tz}\ \mathbb{P}_{\epsilon}\left( \frac{1}{n} \sup_{f \in \mathcal{F}} \left| \sum_{t=1}^{n} \epsilon_t f(\tz_t(\epsilon)) \right| > \alpha/4\right)
	$$
\end{lemma}

\begin{lemma}\label{lem:pollard_probability}
Let $\F$ be a class of $[-1,1]$-valued functions. For any $\Z$-valued tree $\tz$ of depth $n$ and $\alpha>0$
\begin{align*}
\mbb{P}_{\epsilon}\left( \frac{1}{n} \sup_{f \in \mathcal{F}} \left| \sum_{t=1}^{n} \epsilon_t f(\tz_t(\epsilon)) \right| > \alpha/4 \right) \leq 2 \mathcal{N}_1(\alpha/8, \F, \tz) e^{- n \alpha^2 /128} \leq 2\left( \frac{16e n}{\alpha}\right)^{\faton_{\alpha/8}} e^{- n\alpha^2 /128}
\end{align*}
\end{lemma}

Next, we show that the sequential Rademacher complexity is, in some sense, the ``right'' complexity measure even when one considers high probability statements.

\begin{lemma}
	\label{lem:dudley_probability}
	Let $\F$ be a class of $[-1,1]$-valued functions and suppose $\faton_\alpha(\F)$ is finite for all $\alpha > 0$. Then for any $\theta>\sqrt{8/n}$, for any $\Z$-valued tree $\tz$ of depth $n$,
	\begin{align*}
	&\mbb{P}_\epsilon\left( \sup_{f\in\F} \left|\frac{1}{n}\sum_{t=1}^n \epsilon_t f(\tz_t(\epsilon))\right| > 128\left(1+\theta\sqrt{n} \log^{3/2}(2n) \right)\cdot \Radon_n(\F) \right) \\ 
	&\leq \mbb{P}_\epsilon\left( \sup_{f\in\F} \left|\frac{1}{n}\sum_{t=1}^n \epsilon_t f(\tz_t(\epsilon))\right| > \inf_{\alpha > 0}\left\{ 4 \alpha + 12 \theta \int_{\alpha}^{1} \sqrt{\frac{\log \Nonl_\infty(\delta,\F,n)}{n}} d \delta \right\} \right) \\
	&\leq L e^{- \frac{n \theta^2  }{4}}
	\end{align*}
 	where $L$ is a constant such $L > \sum_{j=1}^\infty \Nonl_\infty(2^{-j},\F,n)^{-1}$ \ .
\end{lemma}

While throughout this chapter we are mostly concerned with expected versions of minimax regret and the corresponding complexities, the above lemmas can be employed to give an analogous in-probability treatment. To obtain such in-probability statements, the value is defined as the minimax probability of regret exceeding a threshold.

\section{Charecterizing Learnability of Supervised Learning Problem}\label{sec:supon}
In this section we study the specific case of online supervised learning problem. In this setting, the instance space $\Z$ is of form $\Z = \X \times \Y$ where $\X$ is some arbitrary input space and $\Y \subset \reals$. We shall assume that $\Y \subseteq [-1,1]$ (of course the bound of $1$ can be changed to an arbitrary value). The target hypothesis set $\cH$ for the supervised learning problem corresponds to a set of functions that map input instances from $\X$ to some predicted label in $[-1,1]$, that is $\cH \subset [-1,1]^{\X}$. $\bcH \subset [-1,1]^\X$ can be an arbitrary superset of $\H$ and the results in fact hold even if $\bcH = \H$ (ie. proper learning setting). The loss function we consider is for the form 
$$
\ell(\h,(x,y)) = |\h(x) - y|
$$
In the online supervised learning problem at each round $t$, the player picks hypothesis $\h_t \in \bcH$ and the adversary provides input target pair $(x_t,y_t)$ and the player suffers loss $|\h_t(x_t) - y_t|$. Note that if $\bcH \subseteq \{\pm 1\}^\X$ and each $y_t \in \{\pm1\}$ then the problem boils down to binary classification problem. 

Though we use the absolute loss in this section, it is easy to see that all the results hold (with modified rates) for any loss $\ell(\h(x),y)$ which is such that for all $\h$, $x$ and $y$,  
$$
\phi(\ell(\hat{y},y)) \le |\hat{y} - y| \le \Phi(\ell(\hat{y},y))
$$
where $\Phi$ and $\phi$ are monotonically increasing functions. For instance the squared loss is a classic example. 

To formally define the value of the online supervised learning game, fix a set of labels $\Y \subseteq [-1,1]$. For the sake of brevity, we shall use the notation
$\Val^{\trm{S}}_n(\H) = \Val_n(\H,\X \times \Y)$.
Binary classification is, of course, a special case when $\Y = \{\pm1\}$ and $\bcH \subseteq \{\pm1\}^\X$. In that case, we simply use $\Val^{\trm{Binary}}_n(\H)$ for $\Val^{\trm{S}}_n(\H)$.

\begin{proposition}\label{prop:uplow}
For the supervised learning game played with a target hypothesis class $\H \subseteq [-1,1]^\X$, for any $n\geq 2$
\begin{align}\label{eq:uplow}
\frac{1}{4\sqrt{2}} \sup_{\alpha}\left\{\alpha \sqrt{n \min\left\{\faton_{\alpha}(\H), n\right\}} \right\} \le \frac{1}{2}\Val^S_n(\H) \leq \Dudleyon_n(\H)  & \le \inf_{\alpha}\left\{4 \alpha + \frac{12}{\sqrt{n}} \int_{\alpha}^{1} \sqrt{ \faton_\beta(\H) \log\left(\frac{2 e n}{\beta}\right)}\ d \beta \right\}  \nonumber\\
&  \le 58 \log^{\frac{3}{2}}n\ \Radon_n(\H)\ .
\end{align}
Moreover, the lower bound $\Radon_n(\H) \leq \Val^{\trm{S}}_n(\H)$ on the value of the supervised game also holds.
\end{proposition}

The proposition above implies that finiteness of the fat-shattering dimension is necessary and sufficient for learnability of a supervised game. Further, all the complexity notions introduced so far are within a logarithmic factor from each other whenever the problem is learnable. These results are summarized in the next theorem.

\begin{theorem}\label{thm:tight}
	For any target hypothesis class $\H \subseteq [-1,1]^\X$, the following statements are equivalent
	\begin{enumerate}
	\item Target hypothesis class $\H$ is online learnable in the supervised setting.
	\item For any $\alpha > 0$, $\faton_\alpha(\H)$ is finite.
	\end{enumerate}
	Moreover, if the function class is online learnable, then the value of the supervised game $\Val^{\trm{S}}_n(\H)$, the Sequential Rademacher complexity $\Radon_n(\H)$, and the Integrated complexity $\Dudleyon_n(\H)$ are within a multiplicative factor of $\mc{O}(\log^{3/2} n)$ of each other.
\end{theorem}
\begin{proof}
	The equivalence of \emph{1} and \emph{2} follows directly from Proposition~\ref{prop:uplow}. As for relating the various complexity measures, note that again by Proposition~\ref{prop:uplow}, $\Radon_n(\H)  \le \Val^{\trm{S}}_n(\H)$, $\Val^{\trm{S}}_n(\H) \le 2 \Dudleyon_n(\H)$ and $\Dudleyon_n(\H) \le 58 \log^{\frac{3}{2}}n\ \Radon_n(\H)$. Hence $\Val^{\trm{S}}_n(\H)$ and $\Dudleyon_n(\H)$ are sandwiched between $\Radon_n(\H)$ and $O(\log^{3/2} n) \Radon_n(\H)$ which concludes the proof.
\end{proof}

\begin{corollary}
For the binary classification game played with function class $\F$ we have that
$$
K_1  \sqrt{n \min\left\{\ldim(\H), n\right\}}  \le \Val^{\trm{Binary}}_n(\H)  \le K_2  \sqrt{n\ \ldim(\H) \log n}
$$
for some universal constants $K_1,K_2$.
\end{corollary}

We wish to point out that the lower bound of Proposition~\ref{prop:uplow}  holds for arbitrary class $\bcH$ that is a superset of the target hypothesis class $\H$. Since a proper learning rule can always be seen also as an improper learning rule, we trivially have that if class is properly online learnable in the supervised setting then it is improperly online learnable too. However by the above mentioned fact that the lower bound of Proposition~\ref{prop:uplow} holds for arbitrary class $\bcH$, we also have the non-trivial reverse implication that :
if a class is improperly online learnable in the supervised setting, it is online learnable.

It is natural to ask whether being able to learn in the online model is different from learning in a batch model (in the supervised setting). The standard example (e.g. \cite{Lit88,BenPalSha09}) is the class of step functions on a bounded interval, which has a VC dimension $1$, but is not learnable in the online setting. Indeed, it is possible to verify that the Littlestone's dimension is not bounded. Interestingly, the closely-related class of ``ramp'' functions (modified step functions with a Lipschitz transition between $0$'s and $1$'s) \emph{is} learnable in the online setting (and in the batch case). We extend this example as follows. By taking a convex hull of step-up and step-down functions on a unit interval, we arrive at a class of functions of bounded variation, which is learnable in the batch model, but not in the online learning model. However, the class of \emph{Lipschitz} functions of bounded variation is learnable in both models. Online learnability of the latter class is shown with techniques analogous to Section~\ref{sec:isotron}.

\subsection{Generic Algorithm for Supervised Learning Problem}

We shall now present a generic improper learning algorithm for the online supervised setting that achieves a low regret bound whenever the function class is online learnable. For any $\alpha > 0$ define an $\alpha$-discretization of the $[-1,1]$ interval as  $B_\alpha = \{-1+\alpha/2 , -1 + 3 \alpha/2 , \ldots, -1+(2k+1)\alpha/2, \ldots \}$ for $0\leq k$ and $(2k+1)\alpha \leq 4$. Also for any $a \in [-1,1]$ define $\lfloor a \rfloor_\alpha = \argmin{r \in B_\alpha} |r - a|$. For a set of functions $V \subseteq\H$, any $r \in B_\alpha$ and $x \in \X$ define 
$$V (r, x) = \left\{\h \in V \ \middle| \ \h(x) \in (r-\alpha/2, r+\alpha/2]\right\}$$

\begin{algorithm}[H]
\caption{Fat-SOA Algorithm ($\mathcal{F},\alpha$)}
\label{alg:fatSOA}
\begin{algorithmic}
\STATE $V_1 \gets \mathcal{F}$
\FOR{$t=1$ to $n$}
\STATE $R_t(x) = \{r \in B_\alpha : \faton_\alpha(V_t (r, x)) = \max_{r' \in B_\alpha} \faton_\alpha(V_t (r', x))\}$
\STATE For each $x\in\X$, let $\h_t(x) = \frac{1}{|R_t(x)|} \sum_{r \in R_t(x)} r $
\STATE Play $\h_t$ and receive $(x_t,y_t)$
\IF{$|\h_t(x_t) - y_t| \le \alpha$}
\STATE $V_{t+1} = V_{t}$
\ELSE
\STATE $V_{t+1} = V_{t} (\lfloor y_t \rfloor_\alpha, x_t)$
\ENDIF
\ENDFOR 
\end{algorithmic}
\end{algorithm}

\begin{lemma}\label{lem:soa}
Let $\H\subseteq [-1,1]^\X$ be a function class with finite $\faton_\alpha(\H)$. Suppose the learner is presented with a sequence $(x_1, y_1), \ldots, (x_n, y_n)$, where $y_t = \h(x_t)$ for some fixed $\h \in \H$ unknown to the player. Then for $\h_t$'s computed by the Algorithm~\ref{alg:fatSOA} it must hold that
$$
\sum_{t=1}^T \ind{|\h_t(x_t) - y_t| > \alpha} \le \faton_\alpha(\H).
$$
\end{lemma}

Lemma~\ref{lem:soa} proves a bound on the performance of Algorithm~\ref{alg:fatSOA} in the realizable setting. We now provide an algorithm for the agnostic setting. We achieve this by generating ``experts'' in a way similar to \cite{BenPalSha09}. Using these experts along with the exponentially weighted average (EWA) algorithm we shall provide the generic algorithm for online supervised learning. The EWA (Algorithm~\ref{alg:experts}) and its regret bound are provided in the appendix for completeness (p.~\pageref{alg:experts}).
\begin{algorithm}[H]
\caption{Expert ($\mathcal{F},\alpha, 1 \le i_1 < \ldots < i_L \le n, Y_{1},\ldots,Y_{L}$)}
\label{alg:fatSOA_experts}
\begin{algorithmic}
\STATE $V_1 \gets \mathcal{F}$
\FOR{$t=1$ to $n$}
\STATE $R_t(x) = \{r \in B_\alpha : \faton_\alpha(V_t (r, x)) = \max_{r' \in B_\alpha} \faton_\alpha(V_t (r', x))\}$
\STATE For each $x\in\X$, let $f'_t(x) = \frac{1}{|R_t(x)|} \sum_{r \in R_t(x)} r $
\IF{$t \in \{i_1,\ldots, i_L\}$}
\STATE $\forall x \in \X, \h_t(x) = Y_{j}$ where $j$ is s.t. $t = i_j$
\STATE Play $\h_t$ and receive $x_t$
\STATE $V_{t+1} = V_{t}(\h_t(x_t), x_t)$
\ELSE
\STATE Play $\h_t = \h'_t$ and receive $x_t$
\STATE $V_{t+1} = V_t$
\ENDIF
\ENDFOR 
\end{algorithmic}
\end{algorithm}

For each $L \le \faton_\alpha(\H)$ and every possible choice of $1 \le i_1 < \ldots < i_L \le n$ and $Y_1,\ldots,Y_L \in B_\alpha$ we generate an expert. Denote this set of experts as $E_n$. Each expert outputs a function $\h_t \in \H$ at every round $n$. Hence each expert $e \in E_n$ can be seen as a sequence  $(e_1,\ldots,e_n)$ of mappings $e_t : \X^{t-1} \mapsto \H$. The total number of unique experts is clearly
$$
|E_n| = \sum_{L=0}^{\faton_\alpha} {n \choose L} \left(\left|B_\alpha\right| - 1\right)^L \le \left(\frac{2 n}{\alpha} \right)^{\faton_\alpha}
$$ 

\begin{lemma}\label{lem:approx}
For any $\h \in \H$ there exists an expert $e \in E_n$ such that for any $t \in [n]$,
$$
|\h(x_t) - e(x_{1:t-1})(x_t)| \le \alpha
$$
\end{lemma}
\begin{proof}
By Lemma \ref{lem:soa}, for any function $\h \in \H$, the number of rounds on which $|\h_t(x_t) - \h(x_t)| > \alpha$ for the output of the fat-SOA algorithm $\h_t$ is bounded by $\faton_\alpha(\H)$. Further on each such round there are $|B_\alpha| - 1$ other possibilities. For any possible such sequence of ``mistakes", there is an expert that predicts the right label on those time steps and on the remaining time agrees with the fat-SOA algorithm for that target function. Hence we see that there is always an expert $e \in E_n$ such that 
$$
|\h(x_t) - e(x_{1:t-1})(x_t)| \le \alpha
$$
\end{proof}

\begin{theorem}
For any $\alpha > 0$ if we run the exponentially weighted average (EWA) algorithm with the set $E_n$ of experts then the expected regret of the algorithm is bounded as
$$
\E{\Reg_n} \le \alpha + \sqrt{ \frac{\faton_\alpha \log\left(\frac{2 n}{\alpha} \right)}{n}}
$$
\end{theorem}
\begin{proof}
For any $\alpha \ge 0$  if we run EWA with corresponding set of experts $E_n$ then we can guarantee that regret w.r.t. best expert in the set $E_n$ is bounded by $\sqrt{n \faton_\alpha \log\left(\frac{2 n}{\alpha} \right)}$. However by Lemma \ref{lem:approx} we have that the regret of the best expert in $E_n$ w.r.t. best function in function class $\H$ is at most $\alpha n$. Combining we get the required result.
\end{proof}

The above theorem holds for a fixed $\alpha$. To provide a regret statement that optimizes over $\alpha$ we consider $\alpha_i$'s of form $2^{-i}$ and assign weights $p_i = \frac{6}{\pi^2} i^{-2}$ to experts generated in above theorem for each $\alpha_i$ and run EWA on the entire set of experts with these initial weights. Hence we get the following corollary.

\begin{corollary}
	\label{cor:generic_bound_all_alpha}
Let $\H\subseteq [-1,1]^\X$. The expected regret of the algorithm described above is bounded as
$$
\E{\Reg_n} \le \inf_{\alpha}\left\{\alpha  + \sqrt{\frac{ \faton_\alpha \log\left(\frac{2 n}{\alpha} \right)}{n}} + \frac{\left(3 + 2 \log \log\left(\frac{1}{\alpha}\right)\right)}{\sqrt{n}}\right\}
$$
\end{corollary}

\section{Examples}\label{sec:exampleon}

\subsection{Example: Margin Based Regret}\index{margin}
In the classical statistical setting, margin bounds provide guarantees on expected zero-one loss of a classifier based on the empirical margin zero-one error. These results form the basis of the theory of large margin classifiers (see \cite{SchapireFrBaLe97,KoltchinskiiPa02}). Recently, in the online setting, margin bounds have been shown through the concept of margin via the Littlestone dimension \index{Littlestone dimension} \cite{BenPalSha09}. We show that our machinery can easily lead to margin bounds for the binary classification games for general function classes $\F$ based on their sequential Rademacher Complexity. We use ideas from \cite{KoltchinskiiPa02} to do this.

\begin{proposition}
	\label{prop:margin}
For any target hypothesis class $\H \subset [-1,1]^{\Z}$, there exists a randomized online algorithm $\Algo$ such that for any sequence $z_1,\ldots,z_n$ where each $z_t = (x_t,y_t) \in \Z \times \{\pm 1\}$, played by the adversary,
$$
\E{\frac{1}{n}\sum_{t=1}^n \Es{\h_t \sim \Algo(z_{1:t-1})}{\ind{\h_t(x_t) y_t < 0 }}}  \le \inf_{\gamma > 0}\left\{\inf_{\h \in \H} \frac{1}{n}\sum_{t=1}^n \ind{\h(x_t) y_t < \gamma} +  \frac{4}{\gamma} \Radon_n(\H) + \frac{\left(3 + \log \log\left(\frac{1}{\gamma}\right) \right)}{\sqrt{n}}\right\}
$$
\end{proposition}

\subsection{Example : Neural Networks} \index{neural networks}
We provide below a bound on sequential Rademacher complexity for  classic multi-layer neural networks thus showing they are learnable in the online setting. The model of neural network we consider below and the bounds we provide are analogous to the ones considered in the batch setting in \cite{BarMed03}. We now consider a $k$-layer $1$-norm neural network. To this end let function class $\F_1$ be given by 
$$
\F_1 = \left\{x \mapsto \sum_{j} w^{1}_j x_j  ~~\Big|~~ \|w\|_1 \le B_1\right\}
$$
and further for each $2 \le i \le k$ define
$$
\F_i = \left\{x \mapsto \sum_{j} w^{i}_j \sigma\left( f_j(x)\right) ~~\Big|~~ \forall j\   f_j \in \F_{i-1} , \|w^{i}\|_1 \le B_i\right\}
$$

\begin{proposition}
	\label{prop:NN}
Say $\sigma : \reals \mapsto [-1,1]$ is $L$-Lipschitz, then
$$
\Radon_n(\F_k) \le \left(\prod_{i=1}^k 2 B_i \right) L^{k-1} X_\infty  \sqrt{\frac{2 \log d}{n}}
$$
where $X_\infty$ is such that $\forall x \in \Z$, $\|x\|_\infty \le X_\infty$ and $\Z \subset \reals^d$
\end{proposition}

\subsection{Example: Decision Trees}\index{decision trees}
We consider here the supervised learning game where adversary provides instances from input space $\X$ and binary labels $\pm 1$ corresponding to the input and the player plays decision trees of depth no more than $d$ with decision functions from set $\mc{C} \subset \{\pm 1\}^\X$ of binary valued functions. The following proposition shows that there exists a randomized learning algorithm which under certain circumstances could have low regret for the supervised learning (binary) game played with class of decision trees of depth at most $d$ with decision functions from $\mc{C}$. The proposition is analogous to the one in \cite{BarMed03} considered in the batch (classical) setting.

\begin{proposition}
	\label{prop:DT}
Denote by $\mc{T}$ the class of decision trees of depth at most 
$d$ with decision functions in $\mc{C}$. There exists a randomized online learning algorithm $\Algo$ such that for any sequence of instances $z_1 = (x_1,y_1), \ldots, z_n = (x_n,y_n) \in (\X \times \{\pm 1\})^n$ played by the adversary,
\begin{align*}
\E{\frac{1}{n}\sum_{t=1}^n \Es{\tau_t \sim \Algo(z_{1:t-1})}{\ind{\tau_t(x_t) \ne y_t }}}& \le \inf_{\tau \in \mc{T}} \frac{1}{n} \sum_{t=1}^n  \ind{\tau(x_t) \ne y_t} \\
& ~~~~~ + \mc{O}\left(\sum_{l} \min \left(\tilde{C}_n(l) , d \log^{3/2}(n)\ \Radon_n(\mc{H})\right) + \frac{\left(3 + 2   \log(N_{\textrm{leaf}})\right) }{\sqrt{n}}\right)
\end{align*}
where $\tilde{C}_n(l)$ denotes the number of instances which reach the leaf $l$ and are correctly classified in the decision tree $t$ that minimizes  $\sum_{t=1}^n  \ind{\tau(x_t) \ne y_t}$ and let $N_{\textrm{leaf}}$ be the number of leaves in this tree. 
\end{proposition}

\subsection{Example: Online Transductive Learning}\index{online transductive learning}
\label{sec:transductive}
Let $\F$ be a class of functions from $\Z$ to $\reals$. 
Let 
\begin{align}
	\label{eq:def_metric_entropy}
\Nhat_\infty (\alpha, \F) = \min\left\{|G| : G\subseteq \reals^\Z \trm{ s.t. }  \forall f\in\F \ \ \exists g\in G ~\trm{ satisfying }~ \|f-g\|_\infty \leq \alpha \right\}.
\end{align}
be the $\ell_\infty$ covering number at scale $\alpha$, where the cover is pointwise on all of $\Z$. It is easy to verify that  
\begin{align}
	\label{eq:bound_by_metric_entropy}
	\forall n, \ \ \ N_\infty(\alpha, \F, n) \leq \Nhat_\infty(\alpha, \F)
\end{align}
Indeed, let $G$ be a minimal cover of $\F$ at scale $\alpha$. We claim that the set $V=\{\v^g = g(\tz): g\in G\}$ of $\reals$-valued trees is an $\ell_\infty$ cover of $\F$ on $\tz$. Fix any $\epsilon\in\{\pm1\}^n$ and $f\in\F$, and let $g\in G$ be such that $\|f-g\|_\infty \leq \alpha$. Then clearly $|\v^g_t(\epsilon)-f(\tz_t(\epsilon))|$ for any $1\leq t\leq n$, which concludes the proof.

This simple observation can be applied in several situations. First, consider the problem of {\em transductive learning}, where the set $\Z=\{z_1,\ldots,z_m\}$ is a finite set. To ensure online learnability, it is sufficient to consider an assumption on the dependence of $\Nhat_\infty (\alpha, \F)$ on $\alpha$. An obvious example of such a class is a VC-type class with $\Nhat_\infty(\alpha, \F)  \leq (c/\alpha)^d$ for some $c$ which can depend on $m$. Assume that $\F\subset [-1,1]^\Z$. Substituting this bound on the covering number into
$$
\Dudleyon_n (\F) = \inf_{\alpha}\left\{4  \alpha + 12\int_{\alpha}^{1} \sqrt{ \frac{\log \ \mathcal{N}_2(\delta, \F,n ) }{n} } d \delta \right\}
$$
and choosing $\alpha =0$, we observe that the value of the supervised game is upper bounded by $2\Dudleyon_n(\F) \leq 4\sqrt{\frac{d \log c}{n}}$ by Proposition~\ref{prop:uplow}. It is easy to see that if $m$ is fixed and the problem is learnable in the batch (e.g. PAC) setting, then the problem is learnable in the online transductive model.

In the transductive setting considered by Kakade and Kalai \cite{KakadeKalai05}, it is assumed that $m\leq n$ and $\F$ are binary-valued. If $\F$ is a class with VC dimension $d$, the Sauer-Shelah lemma ensures that the $\ell_\infty$ cover is smaller than $(em/d)^d \leq (en/d)^d$. Using the previous argument with $c=en$, we obtain a bound of $4\sqrt{dn\log (en)}$ for the value of the game, matching \cite{KakadeKalai05} up to a constant $2$.



We also consider the problem of prediction of individual sequences, which has been studied both in information theory and in learning theory. In particular, in the case of binary prediction, Cesa-Bianchi and Lugosi \cite{CesaBianLugo99} proved upper bounds on the value of the game in terms of the (classical) Rademacher complexity and the (classical) Dudley integral. The particular assumption made in \cite{CesaBianLugo99} is that experts are \emph{static}. That is, their prediction only depends on the current round, not on the past information. Formally, we define static experts as mappings $f:\{1,\ldots,n\}\mapsto \Y=[-1,1]$, and let $\F$ denote a class of such experts. Defining $\Z = \{1,\ldots, n\}$ puts us in the setting considered earlier with $m=n$. We immediately obtain $4\sqrt{dn\log (en)}$, matching the results on \cite[p. 1873]{CesaBianLugo99}. We mention that the upper bound in Theorem 4 in \cite{CesaBianLugo99} is tighter by a $\log n$ factor if a sharper bound on the $\ell_2$ cover is considered. Finally, for the case of a finite number of experts, clearly $\Nhat_\infty \leq N$ which gives the classical $O(\sqrt{n\log N})$ bound on the value of the game \cite{PLG}.

\subsection{Example: Isotron}\index{isotron}
\label{sec:isotron}

Recently, Kalai and Sastry \cite{KalSas09} introduced a method called \emph{Isotron} for learning Single Index Models (SIM). These models generalize linear and logistic regression, generalized linear models, and classification by linear threshold functions. For brevity, we only describe the Idealized SIM problem from \cite{KalSas09}. In its ``batch'' version, we assume that the data is revealed at once as a set  $\{(\x_i, y_i)\}_{t=1}^n \in \reals^d \times \reals$ where $y_t= u(\inner{\w,\x_t})$ for some unknown $\w\in \reals^d$ of bounded norm and an unknown non-decreasing $u:\reals\mapsto\reals$ with a bounded Lipschitz constant. Given this data, the goal is to iteratively find the function $u$ and the direction $\w$, making as few mistakes as possible. The error is measured as $\frac{1}{n}\sum_{t=1}^n (\h_i(\x_t) - y_t)^2$, where $\h_i(\x) = u_i(\inner{\w_i,\x})$ is the iterative approximation found by the algorithm on the $i$th round. The elegant computationally efficient method presented in \cite{KalSas09} is motivated by Perceptron, and a natural open question posed by the authors is whether there is an online variant of Isotron. Before even attempting a quest for such an algorithm, we can ask a more basic question: is the (Idealized) SIM problem even learnable in the online framework? After all, most online methods deal with convex functions, but $u$ is only assumed to be Lipschitz and non-decreasing. We answer the question easily with the tools we have developed.

We are interested in online learnability of the above described problem. Specifically, the setting of the problem is a supervised learning one where instance set is $\Z = \X \times \Y$ where $\X = B_2$ (the unit Euclidean ball in $\reals^d$) and $\Y=[-1,1]$. The target hypothesis set is given by $\H = \mc{U} \times B_2$ where $\mc{U} = \{ u : [-1,1] \mapsto [-1,1] : u \textrm{ is }1\textrm{-Lipschitz}\}$. Finally the loss function is the squared loss, that is given hypothesis $\h = (u,\x)$ and instance $\z = (\x,y)$, the loss function if given by $\ell(\h, \z) = \left(u(\ip{\x}{\w}) - y\right)^2$. It is evident that the loss class $\F$ associated with the problem is a composition with three levels: the squared loss, the Lipschitz non-decreasing function, and the linear function. That is 
\begin{align} 
	\label{eq:def_isotron_class}
	\F = \{ (\x,y) \mapsto (y-u(\inner{\w, \x}))^2 \ | \ u\in \mc{U}, \ \|\w\|_2\leq 1 \}
\end{align}
The proof of the following Proposition boils down to showing that the covering number of the class does not increase much under these compositions.

\begin{proposition}\label{prop:isotron}
The target class $\cH = \mc{U} \times B_2$ is online learnable in the supervised setting. Moreover, $$\Val_n(\cH,\X\times\Y) = O\left(\sqrt{\frac{\log^{3} n}{n}}\right)~~.$$ 
\end{proposition}

\section{Detailed Proofs and More Results}

\subsection{Proofs}
\begin{proof}[\textbf{Proof of Theorem~\ref{thm:minimax}}]
	For brevity, denote $\psi(\z_{1:n}) = \inf_{\h\in\H} \frac{1}{n} \sum_{t=1}^n \ell(\h,\z_t)$.	The first step in the proof is to appeal to the minimax theorem for every couple of $\inf$ and $\sup$:
\begin{align*}
& \Val_n(\H,\Z)  = \inf_{q_1\in \Delta(\bcH)}\sup_{p_1 \in \Delta(\Z)} \underset{\underset{\z_1 \sim p_1}{\h_1\sim q_1}}{\En}\ldots\inf_{q_n\in \Delta(\bcH)}\sup_{p_n \in \Delta(\Z)} \underset{\underset{\z_n \sim p_n}{\h_n\sim q_n}}{\En} \left\{ \frac{1}{n}\sum_{t=1}^n \ell(\h_t,\z_t) - \psi(\z_{1:n}) \right\} \\
	&~~~ = \sup_{p_1 \in \Delta(\Z)} \inf_{q_1\in \Delta(\bcH)} \underset{\underset{\z_1 \sim p_1}{\h_1\sim q_1}}{\En}\ldots\sup_{p_n \in \Delta(\Z)} \inf_{q_n\in \Delta(\bcH)} \underset{\underset{z_n \sim p_n}{\h_n\sim q_n}}{\En} \left\{ \frac{1}{n}\sum_{t=1}^n \ell(\h_t,\z_t) - \psi(\z_{1:n}) \right\} \\
	&~~~=\sup_{p_1\in \Delta(\Z)}\inf_{\h_1 \in \bcH}\En_{x_1\sim p_1} \ldots \sup_{p_n \in \Delta(\Z)}\inf_{\h_n \in \bcH} \En_{\z_n\sim p_n} \left\{ \frac{1}{n}\sum_{t=1}^n \ell(\h_t,\z_t) - \psi(\z_{1:n})\right\} 
\intertext{From now on, it will be understood that $\z_t$ has distribution $p_t$. Moving the expectation with respect to $\z_n$ and then the infimum with respect to $\h_n$ inside the expression, we arrive at,}
&~~~ =	\sup_{p_1\in \Delta(\Z)}\inf_{\h_1 \in \bcH}\underset{\z_1}{\En} \ldots \sup_{p_{n-1}}\inf_{\h_{n-1} \in \bcH}\underset{\z_{n-1}}{\En}\sup_{p_n \in \Delta(\Z)}\left\{ \frac{1}{n}\sum_{t=1}^{n-1} \ell(\h_t,\z_t) + \frac{1}{n}\inf_{\h_n \in \bcH}\En_{\z_n} \ell(\h_n,\z_n) - \En_{\z_n}\psi(\z_{1:n})\right\} \\
	&~~~=\sup_{p_1 \in \Delta(\Z)}\inf_{\h_1 \in \bcH}\En_{\z_1} \ldots \sup_{p_{n-1}}\inf_{\h_{n-1}}\En_{\z_{n-1}}\sup_{p_n \in \Delta(\Z)} \En_{\z_n}\left[ \frac{1}{n}\sum_{t=1}^{n-1} \ell(\h_t,\z_t) + \frac{1}{n} \inf_{\h_n \in \bcH}\En_{\z_n} \ell(\h_n,\z_n) - \psi(\z_{1:n})\right]
\intertext{Repeating the procedure for step $n-1$,}
	&~~~= \sup_{p_1 \in \Delta(\Z)}\inf_{\h_1 \in \bcH}\underset{\z_1}{\En} \ldots \sup_{p_{n-1}}\inf_{\h_{n-1}}\underset{\z_{n-1}}{\En} \left[ \frac{1}{n} \sum_{t=1}^{n-1} \ell(\h_t,\z_t) + \sup_{p_n \in \Delta(\Z)}\En_{\z_n} \left[ \frac{1}{n} \inf_{\h_n \in \bcH}\En_{\z_n} \ell(\h_n,\z_n)- \psi(\z_{1:n})\right]\right] \\
	&~~~=\sup_{p_1 \in \Delta(\Z)}\inf_{\h_1 \in \bcH}\underset{\z_1}{\En} \ldots \sup_{p_{n-1} \in \Delta(\Z)} \left\{ \frac{1}{n}\sum_{t=1}^{n-2} \ell(\h_t,\z_t) + \frac{1}{n}\left[\inf_{\h_{n-1} \in \bcH} \En_{\z_{n-1}} \ell(\h_{n-1},\z_{n-1}) \right] \right.\\
	&\left.\hspace{1.9in} + \En_{\z_{n-1}} \sup_{p_n} \underset{z_n}{\En} \left[ \frac{1}{n} \inf_{\h_n \in \bcH}\En_{\z_n} \ell(\h_n,\z_n)- \psi(\z_{1:n})\right]\right\} \\
	&~~~=\sup_{p_1\in \Delta(\Z)}\inf_{\h_1 \in \bcH}\En_{\z_1} \ldots \sup_{p_n \in \Delta(\Z)} \En_{\z_n} \left\{ \frac{1}{n}\sum_{t=1}^{n-2} \ell(\h_t,\z_t) + \frac{1}{n} \sum_{t=n-1}^n \inf_{\h_{t} \in \bcH} \En_{\z_{t}} \ell(\h_{t},\z_{t})- \psi(\z_{1:n})\right\}
\end{align*}
Continuing in this fashion for $n-2$ and all the way down to $t=1$ proves the theorem.
\end{proof}

\begin{proof}[\textbf{Proof of the Key Technical Lemma (Lemma~\ref{lem:technical_symmetrization})}]
We start by noting that since $ \z_n , \z'_n$ are both drawn from $p_n$,
	\begin{align*}
	\Es{ \z_n,\z'_n \sim p_n}{\Phifunc{ \sum_{t=1}^{n} \Delta_\h( \z_t,\z'_t) }} & =   \Es{ \z_n,\z'_n \sim p_n}{\Phifunc{ \sum_{t=1}^{n-1} \Delta_\h( \z_t,\z'_t) + \Delta_\h( \z_n,\z'_n)}} \\
	& = \Es{\z'_n, \z_n \sim p_n}{\Phifunc{ \sum_{t=1}^{n-1} \Delta_\h( \z_t,\z'_t) + \Delta_\h( \z_n,\z'_n)}} \\
	& = \Es{ \z_n,\z'_n \sim p_n}{\Phifunc{ \sum_{t=1}^{n-1} \Delta_\h( \z_t,\z'_t) - \Delta_\h( \z_n,\z'_n)}} \ ,
	\end{align*}
	where the last line is by antisymmetry of $\Delta_\h$. Since the first and last lines are equal, they are both equal to their average and hence
	\begin{align*}
	\Es{ \z_n, \z'_n \sim p_n}{ \Phifunc{ \sum_{t=1}^{n-1} \Delta_\h( \z_t,\z'_t) }}
	&= \Es{ \z_n, \z'_n \sim p_n}{\Es{\epsilon_n}{ \Phifunc{ \sum_{t=1}^{n-1} \Delta_\h( \z_t,\z'_t) + \epsilon_n \Delta_\h( \z_n,\z'_n) } }}\ .
	\end{align*}
Hence we conclude that 
	\begin{align*}
	&\sup_{p_n} \Es{ \z_n, \z'_n \sim p_n}{ \Phifunc{ \sum_{t=1}^{n} \Delta_\h( \z_t,\z'_t) }}\\
	&= \sup_{p_n} \Es{ \z_n, \z'_n \sim p_n}{\Es{\epsilon_n}{ \Phifunc{ \sum_{t=1}^{n-1} \Delta_\h( \z_t,\z'_t) + \epsilon_n \Delta_\h( \z_n,\z'_n) } }} \\
	& \le  \sup_{ \z_n, \z'_n}\Es{\epsilon_n}{ \Phifunc{ \sum_{t=1}^{n-1} \Delta_\h( \z_t,\z'_t) + \epsilon_n \Delta_\h( \z_n,\z'_n) } } \ .
	\end{align*}
	Using the above and noting that $ \z_{n-1}, \z'_{n-1}$ are both drawn from $p_{n-1}$ and hence similar to previous step introducing Rademacher variable $\epsilon_{n-1}$ we get that
	\begin{align*}
	& \sup_{p_{n-1}} \En_{\z_{n-1}, \z'_{n-1} \sim p_{n-1}} \sup_{p_n} \Es{\z_n, \z'_n \sim p_n}{ \Phifunc{ \sum_{t=1}^{n} \Delta_\h(\z_t,\z'_t) }}  \\
	&\le 
	\sup_{p_{n-1}} \Es{\z_{n-1}, \z'_{n-1} \sim p_{n-1}}{ \sup_{\z_n, \z'_n}\Es{\epsilon_n}{ \Phifunc{ \sum_{t=1}^{n-1} \Delta_\h(\z_t,\z'_t) + \epsilon_n \Delta_\h(\z_n,\z'_n)}}} \\
	&= \sup_{p_{n-1}} \En_{\z_{n-1}, \z'_{n-1} \sim p_n}\Es{\epsilon_{n-1}}{\sup_{\z_n, \z'_n} \Es{\epsilon_n}{ \Phifunc{ \sum_{t=1}^{n-2} \Delta_\h(\z_t,\z'_t) + \epsilon_{n-1} \Delta_\h(\z_{n-1},\z'_{n-1}) + \epsilon_n \Delta_\h(\z_n,\z'_n) } }} \\
	&\le \sup_{\z_{n-1}, \z'_{n-1}}\Es{\epsilon_{n-1}}{\sup_{\z_n, \z'_n} \Es{\epsilon_n}{ \Phifunc{ \sum_{t=1}^{n-2} \Delta_\h(\z_t,\z'_t) + \epsilon_{n-1} \Delta_\h(\z_{n-1},\z'_{n-1}) + \epsilon_n \Delta_\h(\z_n,\z'_n) } }} \ .
	\end{align*}
Proceeding in similar fashion introducing Rademacher variables all the way upto $\epsilon_1$ we finally get the required statement that
\begin{align*}
&\sup_{p_1} \En_{\z_1,\z'_1 \sim p_1} \ldots \sup_{p_n} \Es{\z_n, \z'_n \sim p_n}{\Phifunc{ \sum_{t=1}^{n} \Delta_\h(\z_t,\z'_t) }} \\
&\le \sup_{\z_1,\z'_1}\left\{ \Es{\epsilon_1}{ \ldots \sup_{\z_n,\z'_n}\left\{ \Es{\epsilon_n}{\Phifunc{ \sum_{t=1}^{n} \epsilon_t \Delta_\h(\z_t,\z'_t) }}\right\} \ldots }\right\}
\end{align*}
\end{proof}

\begin{proof}[\textbf{Proof of Lemma~\ref{lem:inflip}}]
Without loss of generality assume that the Lipschitz constant $L=1$ because the general case follows by scaling $\phi$. Now note that by Theorem \ref{thm:dudley} we have that 
\begin{align}\label{eq:usedud}
\Radon_n(\phi \circ \F)  \le  \inf_{\alpha}\left\{4  \alpha + 12\int_{\alpha}^{1} \sqrt{\frac{ \log \ \mathcal{N}_2(\delta, \phi \circ \F,n ) }{n} } d \delta \right\}
\end{align}
Now we claim that we can bound 
$$\log \ \mathcal{N}_2(\delta, \phi \circ \F,n )  \le \sum_{j=1}^k \log \ \mathcal{N}_\infty (\delta, \F_j,n ) $$ 
To see this we first start by noting that
\begin{align*}
&\quad \sqrt{\frac{1}{n} \sum_{t=1}^n \left( \phi(f(\z_t(\epsilon)),\z_t(\epsilon)) - \phi(\v_t(\epsilon),\z_t(\epsilon)) \right)^2}\\ 
& \le \sqrt{\frac{1}{n} \sum_{t=1}^n \max_j \left( f_j(\z_t(\epsilon))) - \v^j_t(\epsilon) \right)^2} \\
& \le \sqrt{\max_{t \in [n]} \max_j \left( f_j(\z_t(\epsilon))) - \v^j_t(\epsilon) \right)^2}\\
& \le \max_{j \in [k] , t \in[n]}|f_j(\z_t(\epsilon))) - \v^j_t(\epsilon) |
\end{align*}
Now suppose we have $V_1, \ldots, V_k$ that are minimal $\ell_\infty$-covers for $\F_1, \ldots, \F_k$ on the tree $\z$ at level $\delta$. Consider the set:
\[
	V_\phi = \{ \v_\phi \::\: \v \in V_1 \times \ldots \times V_k \}
\]
where $\v_\phi$ is the tree such that $(\v_\phi)_t(\epsilon) = \phi(\v_t(\epsilon),\z_t(\epsilon))$.
Then, for any $f = (f_1, \ldots,f_k) \in \F$ (with representatives $(\v^1,\ldots,\v^k) \in V_1 \times \ldots \times V_k$)
and any $\epsilon \in \{\pm1\}^n$, we have,
\begin{align*}
&\quad \sqrt{\frac{1}{n} \sum_{t=1}^n \left( \phi(f(\z_t(\epsilon)),\z_t(\epsilon)) - (\v_\phi)_t(\epsilon) \right)^2} \\
& =\sqrt{\frac{1}{n} \sum_{t=1}^n \left( \phi(f(\z_t(\epsilon)),\z_t(\epsilon)) - \phi(\v_t(\epsilon),\z_t(\epsilon)) \right)^2} \\
& \le  \max_{j \in [k]} \max_{t \in[n]}|f_j(\z_t(\epsilon))) - \v^j_t(\epsilon) | \le \delta
\end{align*}
Thus we see that $V_\phi$ is an $\ell_\infty$-cover at scale $\delta$ for $\phi \circ \F$ on $\z$. Hence
\begin{align*}
\log \ \mathcal{N}_2(\delta, \phi \circ \F,n )  & \le \log \ \mathcal{N}_\infty(\delta, \phi \circ \F,n ) \\
&\le \log(|V|) = \sum_{j=1}^k \log(|V_j|) \\
& = \sum_{j=1}^k \log\ \mathcal{N}_\infty(\delta, \F_j,n )
\end{align*}
as claimed. Now using this in Equation \ref{eq:usedud} we have that
\begin{align*}
\Radon_n(\phi \circ \F)  & \le  \inf_{\alpha}\left\{4  \alpha + 12\int_{\alpha}^{1} \sqrt{\frac{\sum_{j=1}^k \log \ \mathcal{N}_\infty(\delta,\F_j,n ) }{n}} d \delta \right\}\\
& \le  \inf_{\alpha}\left\{4 \alpha + 12 \sum_{j=1}^k \int_{\alpha}^{1} \sqrt{\frac{ \log \ \mathcal{N}_\infty(\delta,\F_j,n ) }{n}} d \delta \right\}\\
& \le \sum_{j=1}^k \inf_{\alpha}\left\{4 \alpha + 12\int_{\alpha}^{1} \sqrt{\frac{ \log \ \mathcal{N}_\infty(\delta,\F_j,n ) }{n} } d \delta \right\}
\end{align*}
Now applying Lemma\ref{lem:dudley_lower_by_rad} we conclude, as required, that
$$
\Radon_n(\phi \circ \F)  \le 8\left(1 + 4\sqrt{2} \log^{3/2}(en^2) \right) \sum_{j=1}^k \Radon_n(\F_j)
$$
as long as $\Radon_n(\F_j) \ge 1$ for each $j$.
\end{proof}

\begin{proof}[\textbf{Proof of Corollary~\ref{cor:radem_binary}}]
We first extend the binary function $b$ to a function $\bar{b}$ to any $x \in \reals^k$ as follows :
$$
\bar{b}(x) = \left\{\begin{array}{cl}
(1 - \|x - a\|_\infty)b(a) & \textrm{if }\|x - a\|_\infty < 1 \textrm{ for some }a \in \{\pm 1\}^k \\
0 & \textrm{otherwise} 
\end{array} \right.
$$ 
First note that $\bar{b}$ is well-defined since all points in the $k$-cube are separated by $L_\infty$ distance $2$. Further note that $\bar{b}$ is $1$-Lipschitz w.r.t. the $L_\infty$ norm and so applying Lemma \ref{lem:inflip} we conclude the statement of the corollary.
\end{proof}

\begin{proof}[\textbf{Proof of Proposition~\ref{prop:rademacher_properties}}]
	The most difficult of these is Part 4, which follows immediately by Lemma~\ref{lem:contraction} by taking $\phi(\cdot,z)$ there to be simply $\phi(\cdot)$. The other items follow similarly to Theorem~15 in \cite{Men03fewnslt} and we provide the proofs for completeness. Note that, unlike Rademacher complexity defined in \cite{Men03fewnslt}, Sequential Rademacher complexity does not have the absolute value around the sum.
	
	Part 1 is immediate because for any fixed tree $\tz$ and fixed realization of $\{\epsilon_i\}$,
	\[  
		\sup_{f\in\F} \sum_{t=1}^n \epsilon_t f(\tz_t(\epsilon)) \leq \sup_{f\in\G} \sum_{t=1}^n \epsilon_t f(\tz_t(\epsilon)) \ ,
	\]
 	Now taking expectation over $\epsilon$ and supremum over $\tz$ completes the argument.

	To show Part 2, first observe that, according to Part 1, 
	\[
		\Radon_n(\F)\leq \Radon_n(\conv(\F))\ .
	\]
	Now, any $h\in \conv(\F)$ can be written as $h=\sum_{j=1}^m \alpha_j f_j$ with $\sum_{j=1}^m \alpha_j = 1$,
	$\alpha_j\geq 0$. Then, for fixed tree $\tz$ and sequence $\epsilon$,
	\[
		\sum_{t=1}^n \epsilon_t h(\tz_t(\epsilon) = \sum_{j=1}^m \alpha_j \sum_{t=1}^n \epsilon_t f_j(\tz_t(\epsilon)
		\leq \sup_{f\in\F} \sum_{t=1}^n \epsilon_t f(\tz_t(\epsilon))
	\]
	and thus
	\[
		\sup_{h\in\conv(\F)} \sum_{t=1}^n \epsilon_t h(\tz_t(\epsilon) \leq
		\sup_{f\in\F} \sum_{t=1}^n \epsilon_t f(\tz_t(\epsilon) \ .
	\]
	Taking expectation over $\epsilon$ and supremum over $\tz$ completes the proof.

	To prove Part 3, first observe that the statement is easily seem to hold for $c\geq 0$. That is, $\Radon_n(c\F) = c\Radon_n(\F)$ follows directly from
	the definition. Hence, it remains to convince ourselves of the statement for $c=-1$. That is, $\Radon_n(-\F) = \Radon_n(\F)$. To prove this, consider a tree $\tz^R$ that is a reflection of $\tz$. That is, $\tz^R_t(\epsilon) = \tz_t(-\epsilon)$ for all $t\in[n]$. It is then enough to observe that
\begin{align*}
	\Es{\epsilon}{ \sup_{f\in -\F} \sum_{t=1}^n \epsilon_t f(\tz_t(\epsilon)) } 
	= \Es{\epsilon}{ \sup_{f\in \F} \sum_{t=1}^n -\epsilon_t f(\tz_t(\epsilon)) } \\
	= \Es{\epsilon}{ \sup_{f\in \F} \sum_{t=1}^n \epsilon_t f(\tz_t(-\epsilon)) }
	= \Es{\epsilon}{ \sup_{f\in \F} \sum_{t=1}^n \epsilon_t f(\tz^R_t(\epsilon)) }
\end{align*}
where we used the fact that $\epsilon$ and $-\epsilon$ have the same distribution. As $\tz$ varies over all trees, $\tz^R$ also varies over all trees.
Hence taking the supremum over $\tz$ above finishes the argument.

	Finally, for Part 5, 
	\[
		\sup_{f\in\F} \left\{ \sum_{t=1}^n \epsilon_t \left(f+h\right)(\tz_t(\epsilon)) \right\}= 
	\left\{
		\sup_{f\in \F } \sum_{t=1}^n \epsilon_t f(\tz_t(\epsilon))
	\right\} 
	+ 
	\left\{
		\sum_{t=1}^n \epsilon_t h(\tz_t(\epsilon))
	\right\}
	\]
	Note that, since $h(\tz_t(\epsilon))$ only depends on $\epsilon_{1:t-1}$, we have
	\begin{align*}
		\Es{\epsilon}{ \epsilon_t h(\tz_t(\epsilon)) } 
		= \Es{\epsilon_{1:t-1}}{ \E{ \epsilon_t | \epsilon_{1:t-1} } h(\tz_t(\epsilon) } = 0\ . 
	\end{align*}
	Thus, 
	\[
		\Radon_n(\F+h) = \Radon_n(\F) \ .
	\]
\end{proof}

\begin{proof}[\textbf{Proof of Lemma~\ref{lem:fin}}]
For any $\lambda > 0$, we invoke Jensen's inequality to get
\begin{align*}
M(\lambda) &:= \exp\left\{\lambda  \Es{\epsilon}{ \max_{\v \in V} \sum_{t=1}^{n} \epsilon_t \v_{t}(\epsilon)  } \right\} 
\le \Es{\epsilon}{ \exp\left\{\lambda \max_{\v \in V} \sum_{t=1}^{n}  \epsilon_t \v_{t}(\epsilon) \right\}} \\
&\le \Es{\epsilon}{\max_{\v \in V}  \exp\left\{\lambda \sum_{t=1}^{n} \epsilon_t \v_{t}(\epsilon) \right\}} 
\le \Es{\epsilon}{\sum_{\v \in V}  \exp\left\{\lambda \sum_{t=1}^{n} \epsilon_t \v_{t}(\epsilon)  \right\}} 
\end{align*}
With the usual technique of peeling from the end,
\begin{align*}
M(\lambda)&\le \sum_{\v \in V} \Es{\epsilon_1, \ldots, \epsilon_n}{ \prod_{t=1}^{n}  \exp\left\{\lambda  \epsilon_t \v_{t}(\epsilon_{1:t-1})  \right\}} \\
&= \sum_{\v \in V} \Es{\epsilon_1, \ldots, \epsilon_{n-1}}{ \prod_{t=1}^{n-1} \exp\left\{ \lambda \epsilon_t  \v_{t}(\epsilon_{1:t-1}) \right\} \times  \left(\frac{ \exp\left\{ \lambda \v_{n}(\epsilon_{1:n-1}) \right\} + \exp\left\{ - \lambda \v_{n}(\epsilon_{1:n-1})\right\}}{2}\right)} \\
&\le \sum_{\v \in V} \Es{\epsilon_1, \ldots, \epsilon_{n-1}}{ \prod_{t=1}^{n-1} \exp\left\{ \lambda \epsilon_t \v_{t}(\epsilon_{1:t-1}) \right\} \times  \exp\left\{ \frac{\lambda^2 \v_{n}(\epsilon_{1:n-1})^2 }{2}\right\} }
\end{align*}
where we used the inequality $\frac{1}{2}\left\{\exp(a)+\exp(-a)\right\} \leq \exp(a^2/2)$, valid for all $a\in \reals$. Peeling off the second term is a bit more involved:
\begin{align*}
&M(\lambda)\leq \sum_{\v \in V} \Es{\epsilon_1, \ldots, \epsilon_{n-2}}{ \prod_{t=1}^{n-2} \exp\left\{ \lambda \epsilon_t  \v_{t}(\epsilon_{1:t-1}) \right\} \times  \right. \\
&\hspace{4cm}\left. \frac{1}{2}\left(\exp\left\{\lambda \v_{n-1}(\epsilon_{1:n-2})\right\} \exp\left\{ \frac{\lambda^2 \v_{n}((\epsilon_{1:n-2},1))^2 }{2}\right\} \right.\right.\\
&\hspace{4.5cm}\left.\left. + \exp\left\{- \lambda \v_{n-1}(\epsilon_{1:n-2})\right\} \exp\left\{ \frac{\lambda^2 \v_{n}((\epsilon_{1:n-2},-1))^2 }{2}\right\} \right) }\\
\end{align*}
Consider the term inside:
\begin{align*}
&\frac{1}{2}\left(\exp\left\{\lambda \v_{n-1}(\epsilon_{1:n-2})\right\} \exp\left\{ \frac{\lambda^2 \v_{n}((\epsilon_{1:n-2},1))^2 }{2}\right\} + \exp\left\{- \lambda \v_{n-1}(\epsilon_{1:n-2})\right\} \exp\left\{ \frac{\lambda^2 \v_{n}((\epsilon_{1:n-2},-1))^2 }{2}\right\} \right)\\
&\leq \max_{\epsilon_{n-1}} \left( \exp\left\{ \frac{\lambda^2 \v_{n}((\epsilon_{1:n-2},\epsilon_{n-1}))^2 }{2} \right\} \right) \frac{ \exp\left\{\lambda \v_{n-1}(\epsilon_{1:n-2})\right\} + \exp\left\{- \lambda \v_{n-1}(\epsilon_{1:n-2})\right\} }{2} \\
&\leq \max_{\epsilon_{n-1}} \left( \exp\left\{ \frac{\lambda^2 \v_{n}((\epsilon_{1:n-2},\epsilon_{n-1}))^2 }{2} \right\} \right) \exp\left\{\frac{\lambda^2 \v_{n-1}(\epsilon_{1:n-2})^2}{2}\right\} \\
& = \exp\left\{ \frac{\lambda^2 \max_{\epsilon_{n-1} \in \{\pm 1\}} \left(\v_{n-1}(\epsilon_{1:n-2})^2 + \v_{n}(\epsilon_{1:n-1})^2\right) }{2} \right\}  
\end{align*}

Repeating the last steps, we show that for any $i$,
\begin{align*}
M(\lambda) &\le \sum_{\v \in V} \Es{\epsilon_1, \ldots, \epsilon_{i-1}}{ \prod_{t=1}^{i-1} \exp\left\{ \lambda\epsilon_t  \v_{t}(\epsilon_{1:t-1}) \right\} \times \exp\left\{ \frac{\lambda^2 \max_{\epsilon_{i} \ldots \epsilon_{n-1} \in \{\pm 1\}} \sum_{t=i}^n \v_{t}(\epsilon_{1:t-1})^2  }{2} \right\}  }
\end{align*}
We arrive at
\begin{align*}
 M(\lambda) & \le \sum_{\v \in V}  \exp\left\{ \frac{\lambda^2 \max_{\epsilon_{1} \ldots \epsilon_{n-1} \in \{\pm 1\}} \sum_{t=1}^n \v_{t}(\epsilon_{1:t-1})^2  }{2} \right\}  \\
 & \le  |V|   \exp\left\{ \frac{\lambda^2 \max_{\v \in V} \max_{\epsilon \in \{\pm1\}^n} \sum_{t=1}^n \v_{t}(\epsilon)^2}{2} \right\} 
\end{align*}
Taking logarithms on both sides, dividing by $\lambda$ and setting $\lambda = \sqrt{\frac{2 \log(|V|)}{\max_{\v \in V} \max_{\epsilon \in \{\pm1\}^n } \sum_{t=1}^n \v_{t}(\epsilon)^2}}$ we conclude that
$$
\Es{\epsilon_1, \ldots, \epsilon_n}{ \max_{\v \in V} \sum_{t=1}^{n} \epsilon_t\v_t(\epsilon) } \le \sqrt{2 \log(|V|)  \max_{\v \in V} \max_{\epsilon \in \{\pm1\}^n} \sum_{t=1}^n \v_{t}(\epsilon)^2 }
$$ 
\end{proof}

\begin{proof}[\textbf{Proof of Lemma~\ref{lem:packing_covering_ineq}}]
	We prove the first inequality. Let $\{\w^1,\ldots, \w^M\}$ be a largest strongly $2\alpha$-separated set of $\F(\tz)$ with $M=\M_p(2\alpha, \F, \tz)$. Let $\{\v^1,\ldots, \v^N\}$ be a smallest $\alpha$-cover of $\F$ on $\tz$ with $N=\Nonl_p(\alpha, \F, \tz)$. For the sake of contradiction, assume $M > N$. Consider a path $\epsilon\in\{\pm1\}^n$ on which all the trees $\{\w^1,\ldots, \w^M\}$ are $(2\alpha)$-separated. By the definition of a cover, for any $\w^i$ there exists a tree $\v^j$ such that 
$$\left( \frac{1}{n} \sum_{t=1}^n |\v^j_t(\epsilon) - \w^i_t(\epsilon)|^p \right)^{1/p} \le \alpha.$$
Since $M>N$, there must exist distinct $\w^i$ and $\w^k$, for which the covering tree $\v^j$ is the same for the given path $\epsilon$. By triangle inequality, 
$$\left( \frac{1}{n} \sum_{t=1}^n |\w^i_t(\epsilon) - \w^k_t(\epsilon)|^p \right)^{1/p} \le 2\alpha,$$
which is a contradiction. We conclude that $M\leq N$.
	
	Now, we prove the second inequality. Consider a maximal $\alpha$-packing 
	$V \subseteq \F(\tz)$ of size $\cD_p(\alpha, \F, \tz)$.
	Since this is a \emph{maximal} $\alpha$-packing, for any $f\in\F$, there is no path on which $f(\tz)$ is $\alpha$-separated from every member of the packing. In other words, for every path $\epsilon\in\{\pm 1\}^{n}$, there is a member of the packing $\v\in V$ such that 
	$$ \left( \frac{1}{n}\sum_{t=1}^n |\v_t(\epsilon) - f(\tz_t(\epsilon))|^p\right)^{1/p} \leq \alpha$$
	which means that the packing $V$ is a cover.
\end{proof}

\begin{proof}[\textbf{Proof of Theorem~\ref{thm:sauer_multiclass}}]
For any $d\geq 0$ and $n\geq 0$, define the function
$$ g_k(d, n) = \sum_{i=0}^d {n\choose i} k^i.$$
It is not difficult to verify that this function satisfies the recurrence
$$ g_k(d, n) = g_k(d, n-1)+ k g_k(d-1, n-1)$$
for all $d, n\geq 1$. To visualize this recursion, consider a $k\times n$ matrix and ask for ways to choose at most $d$ columns followed by a choice among the $k$ rows for each chosen column. The task can be decomposed into (a) making the $d$ column choices out of the first $n-1$ columns, followed by picking rows (there are $g_k(d, n-1)$ ways to do it) or (b) choosing $d-1$ columns (followed by row choices) out of the first $n-1$ columns and choosing a row for the $n$th column (there are $k g_k(d-1, n-1)$ ways to do it). This gives the recursive formula.

In what follows, we shall refer to an $L_\infty$ cover at scale $1/2$ simply as a $1/2$-cover. The theorem claims that the size of a minimal $1/2$-cover is at most $g_k(d, n)$. The proof proceeds by induction on $n+d$.

{\bf Base:} For $d=1$ and $n=1$, there is only one node in the tree, i.e. the tree is defined by the constant $\tz_1\in \Z$. Functions in $\F$ can take up to $k+1$ values on $\tz_1$, i.e. $\Nonl(0, \F , 1) \leq k+1$ (and, thus, also for the $1/2$-cover). Using the convention ${n\choose 0} = 1$, we indeed verify that $g_k(1, 1) = 1+ k = k+1$. The same calculation gives the base case for $n=1$ and any $d\in{\mathbb N}$. Furthermore, for any $n\in{\mathbb N}$ if $d=0$, then there is no point which is $2$-shattered by $\F$. This means that functions in $\F$ differ by at most $1$ on any point of $\Z$. Thus, there is a $1/2$ cover of size $1 = g_k(0, n)$, verifying this base case. 

{\bf Induction step:}
Suppose by the way of induction that the statement holds for $(d, n-1)$ and $(d-1, n-1)$. 
Consider any tree $\tz$ of depth $n$ with $\faton_2(\F, \tz)=d$. Define the partition $\F=\F_0\cup\ldots\cup \F_k$ with $\F_i = \{f\in\F: f(\tz_1) = i\}$ for $i\in \{0,\ldots,k\}$, where $\tz_1$ is the root of $\tz$. Let $n = \left|\{i: \faton_2(\F_i, \tz) = d \} \right|$.

Suppose first, for the sake of contradiction, that $\faton_2(\F_i,\tz)=\faton_2(\F_j, \tz)=d$ for $|i-j| \geq 2$. Then there exist two trees $\z$ and $\v$ of depth $d$ which are $2$-shattered by $\F_i$ and $\F_j$, respectively, and with $\Img(\z),\Img(\v)\subseteq\Img(\tz)$. Since functions within each subset $\F_i$ take on the same values on $\tz_1$, we conclude that $\tz_1 \notin \Img(\z), \tz_1\notin\Img(\v)$. This follows immediately from the definition of shattering. We now \emph{join} the two shattered $\z$ and $\v$ trees with $\tz_1$ at the root and observe that $\F_i\cup\F_j$ $2$-shatters this resulting tree of depth $d+1$, which is a contradiction. Indeed, the witness  $\reals$-valued tree $\mbf{s}$ is constructed by joining the two witnesses for the $2$-shattered trees $\z$ and $\v$ and by defining the root as $\mbf{s}_1 = (i+j)/2$. It is easy to see that $\mbf{s}$ is a witness to the shattering. Given any $\epsilon\in\{\pm 1\}^{d+1}$, there is a function $f^i\in\F_i$ which realizes the desired separation under the signs $(\epsilon_2,\ldots, \epsilon_{d+1})$ for the tree $\z$ and there is a function $f^j\in\F_j$ which does the same for $\v$. Depending on $\epsilon_1=+1$ or $\epsilon_1=-1$, either $f^i$ or $f^j$ realize the separation over $\epsilon$.

We conclude that the number of subsets of $\F$ with fat-shattering dimension equal to $d$ cannot be more than two (for otherwise at least two indices will be separated by $2$ or more). We have three cases: $n=0$, $n=1$, or $n=2$, and in the last case it must be that the indices of the two subsets differ by $1$. 

First, consider any $\F_i$ with  $\faton_2(\F_i,\tz)\leq d-1$, $i\in\{0,\ldots,k\}$. By induction, there are $1/2$-covers $V^{\ell}$ and $V^{r}$ of $\F_i$ on the subtrees $\tz^{\ell}$ and $\tz^{r}$, respectively, both of size at most $g_k(d-1, n-1)$. Informally, out of these $1/2$-covers we can create a $1/2$-cover $V$ for $\F_i$ on $\tz$ by pairing the $1/2$-covers in $V^{\ell}$ and $V^{r}$. The resulting cover of $\F_i$ will be of size $g_k(d-1,n-1)$. Formally, consider a set of pairs $(\v^{\ell}, \v^r)$ of trees, with $\v^{\ell}\in V^{\ell}$, $\v^r\in V^r$ and such that each tree in $V^{\ell}$ and $V^{r}$ appears in at least one of the pairs. Clearly, this can be done using at most $g_k(d-1, n-1)$ pairs, and such a pairing is not unique. We join the subtrees in every pair $(\v^{\ell}, \v^r)$ with a constant $i$ as the root, thus creating a set $V$ of trees, $|V|\leq g_k(d-1, n-1)$. We claim that $V$ is a $1/2$-cover for $\F_i$ on $\tz$.   Note that all the functions in $\F_i$ take on the same value $i$ on $\tz_1$ and by construction $\v_1=i$ for any $\v\in V$. Now, consider any $f\in\F_i$ and $\epsilon\in\{\pm1\}^n$. Without loss of generality, assume $\epsilon_1=-1$. By assumption, there is a $\v^{\ell}\in V^{\ell}$ such that $|\v^{\ell}_t(\epsilon_{2:n}) - f(\tz_{t+1}(\epsilon_{1:n}))| \leq 1/2$ for any $t\in[n-1]$. By construction $\v^{\ell}$ appears as a left subtree of at least one tree in $V$, which, therefore, matches the values of $f$ for $\epsilon_{1:n}$. The same argument holds for $\epsilon_1=+1$ by finding an appropriate subtree in $V^r$. We conclude that $V$ is a $1/2$-cover of $\F_i$ on $\tz$, and this holds for any $i\in\{0,\ldots,k\}$ with $\faton_2(\F_i,\tz)\leq d-1$. Therefore, the total size of a $1/2$-cover for the union $\cup_{i:\faton_2(\F_i, \tz) \leq d-1} \F_i$ is at most $(k+1-n) g_k(d-1, n-1)$. If $n=0$, the induction step is proven because $g_k(d-1, n-1) \leq g_k(d, n-1)$ and so the total size of the constructed cover is at most
$$(k+1)g_k(d-1, n-1) \leq g_k(d, n-1)+ kg_k(d-1, n-1) = g_k(d, n).$$ 

Now, consider the case $n=1$ and let $\faton_2(\F_i, \tz) = d$. An argument exactly as above yields a $1/2$-cover for $\F_i$, and this cover is of size at most $g_k(d, n-1)$ by induction. The total $1/2$-cover is therefore of size at most
$$g_k(d, n-1)+ kg_k(d-1, n-1) = g_k(d, n).$$ 

Lastly, for $n=2$, suppose $\faton_2(\F_i,\tz)=\faton_2(\F_j, \tz)=d$ for $|i-j| =1$. Let $\F' = \F_i \cup \F_j$. Note that $\faton_2(\F', \tz) = d$. Just as before, the $1/2$-covering for $\tz$ can be constructed by considering the $1/2$-covers for the two subtrees. However, when joining any $(\v^{\ell}, \v^r)$, we take $(i+j)/2$ as the root. It is straightforward to check that the resulting cover is indeed an $1/2$-cover, thanks to the relation $|i-j|=1$. The size of the constructed cover is, by induction, $g_k(d, n-1)$, and the induction step follows. This concludes the induction proof, yielding the main statement of the theorem.

Finally, the upper bound on $g_k(d, n)$ is
$$ \sum_{i=1}^d {n\choose i}k^i \leq \left(\frac{kn}{d}\right)^d \sum_{i=1}^d {n\choose i} \left(\frac{d}{n}\right)^i \leq \left(\frac{kn}{d}\right)^d \left(1+\frac{d}{n}\right)^n \leq \left(\frac{ekn}{d}\right)^d $$
whenever $n\geq d$.

\end{proof}

\begin{proof}[\textbf{Proof of Theorem~\ref{thm:sauer_multiclass_0_cover}}]
The proof is very close to the proof of Theorem~\ref{thm:sauer_multiclass}, with a few key differences. As before, for any $d\geq 0$ and $n\geq 0$, define the function $g_k(d, n) = \sum_{i=0}^d {n\choose i} k^i.$

The theorem claims that the size of a minimal $0$-cover is at most $g_k(d, n)$. The proof proceeds by induction on $n+d$.

{\bf Base:} For $d=1$ and $n=1$, there is only one node in the tree, i.e. the tree is defined by the constant $\tz_1\in \Z$. Functions in $\F$ can take up to $k+1$ values on $\tz_1$, i.e. $\Nonl(0, \F , 1) \leq k+1$. Using the convention ${n\choose 0} = 1$, we indeed verify that $g_k(1, 1) = 1+ k = k+1$. The same calculation gives the base case for $n=1$ and any $d\in{\mathbb N}$. Furthermore, for any $n\in{\mathbb N}$ if $d=0$, then there is no point which is $1$-shattered by $\F$. This means that all functions in $\F$ are identical, proving that there is a $0$-cover of size $1 = g_k(0, n)$.

{\bf Induction step:}
Suppose by the way of induction that the statement holds for $(d, n-1)$ and $(d-1, n-1)$. 
Consider any tree $\tz$ of depth $n$ with $\faton_1(\F, \tz)=d$. Define the partition $\F=\F_0\cup\ldots\cup \F_k$ with $\F_i = \{f\in\F: f(\tz_1) = i\}$ for $i\in \{0,\ldots,k\}$, where $\tz_1$ is the root of $\tz$. 

We first argue that $\faton_1(\F_i, \tz)=d$ for at most one value $i\in\{0,\ldots,k\}$. By the way of contradiction, suppose we do have $\faton_1(\F_i,\tz)=\faton_1(\F_j, \tz)=d$ for $i\neq j$. Then there exist two trees $\z$ and $\v$ of depth $d$ $1$-shattered by $\F_i$ and $\F_j$, respectively, and with $\Img(\z),\Img(\v)\subseteq\Img(\tz)$. Since functions within each subset $\F_i$ take on the same values on $\tz_1$, we conclude that $\tz_1 \notin \Img(\z), \tz_1\notin\Img(\v)$. This follows immediately from the definition of shattering. We now \emph{join} the two shattered $\z$ and $\v$ trees with $\tz_1$ at the root and observe that $\F_i\cup\F_j$ $1$-shatters this resulting tree of depth $d+1$, which is a contradiction. Indeed, the witness  $\reals$-valued tree $\mbf{s}$ is constructed by joining the two witnesses for the $1$-shattered trees $\z$ and $\v$ and by defining the root as $\mbf{s}_1 = (i+j)/2$. It is easy to see that $\mbf{s}$ is a witness to the shattering. Given any $\epsilon\in\{\pm 1\}^{d+1}$, there is a function $f^i\in\F_i$ which realizes the desired separation under the signs $(\epsilon_2,\ldots, \epsilon_{d+1})$ for the tree $\z$ and there is a function $f^j\in\F_j$ which does the same for $\v$. Depending on $\epsilon_1=+1$ or $\epsilon_1=-1$, either $f^i$ or $f^j$ realize the separation over $\epsilon$.

We conclude that $\faton_1(\F_i,\tz)=d$ for at most one $i\in\{0,\ldots,k\}$. Without loss of generality, assume $\faton_1(\F_0, \tz) \leq d$ and $\faton_1(\F_i, \tz) \leq d-1$ for $i\in\{1,\ldots,k\}$.  By induction, for any $\F_i$, $i\in\{1,\ldots,k\}$, there are $0$-covers $V^{\ell}$ and $V^{r}$ of $\F_i$ on the subtrees $\tz^{\ell}$ and $\tz^{r}$, respectively, both of size at most $g_k(d-1, n-1)$. Out of these $0$-covers we can create a $0$-cover $V$ for $\F_i$ on $\tz$ by pairing the $0$-covers in $V^{\ell}$ and $V^{r}$. Formally, consider a set of pairs $(\v^{\ell}, \v^r)$ of trees, with $\v^{\ell}\in V^{\ell}$, $\v^r\in V^r$ and such that each tree in $V^{\ell}$ and $V^{r}$ appears in at least one of the pairs. Clearly, this can be done using at most $g_k(d-1, n-1)$ pairs, and such a pairing is not unique. We join the subtrees in every pair $(\v^{\ell}, \v^r)$ with a constant $i$ as the root, thus creating a set $V$ of trees, $|V|\leq g_k(d-1, n-1)$. We claim that $V$ is a $0$-cover for $\F_i$ on $\tz$.   Note that all the functions in $\F_i$ take on the same value $i$ on $\tz_1$ and by construction $\v_1=i$ for any $\v\in V$. Now, consider any $f\in\F_i$ and $\epsilon\in\{\pm1\}^n$. Without loss of generality, assume $\epsilon_1=-1$. By assumption, there is a $\v^{\ell}\in V^{\ell}$ such that $\v^{\ell}_t(\epsilon_{2:n}) = f(\tz_{t+1}(\epsilon_{1:n}))$ for any $t\in[n-1]$. By construction $\v^{\ell}$ appears as a left subtree of at least one tree in $V$, which, therefore, matches the values of $f$ for $\epsilon_{1:n}$. The same argument holds for $\epsilon_1=+1$ by finding an appropriate subtree in $V^r$. We conclude that $V$ is a $0$-cover of $\F_i$ on $\tz$, and this holds for any $i\in\{1,\ldots,k\}$.

Therefore, the total size of a $0$-cover for $\F_1\cup\ldots\cup F_k$ is at most $k g_k(d-1, n-1)$. A similar argument yields a $0$-cover for $\F_0$ on $\tz$ of size at most $g_k(d, n-1)$ by induction. Thus, the size of the resulting $0$-cover of $\F$ on  $\tz$ is at most 
$$g_k(d, n-1)+ kg_k(d-1, n-1) = g_k(d, n),$$ 
completing the induction step and yielding the main statement of the theorem.

The upper bound on $g_k(d, n)$ appears in the proof of Theorem~\ref{thm:sauer_multiclass}.

\end{proof}

\begin{proof}[\textbf{Proof of Corollary~\ref{cor:l2_norm_bound}}]
	The first two inequalities follow by simple comparison of norms. It remains to prove the bound for the $\ell_\infty$ covering. For any $\alpha > 0$ define an $\alpha$-discretization of the $[-1,1]$ interval as  $B_\alpha = \{-1+\alpha/2 , -1 + 3 \alpha/2 , \ldots, -1+(2k+1)\alpha/2, \ldots \}$ for $0\leq k$ and $(2k+1)\alpha \leq 4$. Also for any $a \in [-1,1]$ define $\lfloor a \rfloor_\alpha = \argmin{r \in B_\alpha} |r - a|$ with ties being broken by choosing the smaller discretization point. For a  function $f:\Z\mapsto [-1,1]$ let the function $\lfloor f \rfloor_\alpha$ be defined pointwise as $\lfloor f(x) \rfloor_\alpha$, and let $\lfloor \F \rfloor_\alpha = \{\lfloor f \rfloor_\alpha : f\in\F\}$. First, we prove that $\Nonl_\infty(\alpha, \F, \tz) \leq \Nonl_\infty(\alpha/2, \lfloor \F \rfloor_\alpha, \tz)$. Indeed, suppose the set of trees $V$ is a minimal $\alpha/2$-cover of $\lfloor \F \rfloor_\alpha$ on $\tz$. That is,
	$$
	\forall f_{\alpha} \in \lfloor \F \rfloor_\alpha,\ \forall \epsilon \in \{\pm1\}^n \ \exists \v \in V \  \mrm{s.t.}  ~~~~ |\v_t(\epsilon) -  f_{\alpha}(\tz_t(\epsilon))| \leq \alpha/2
	$$
	Pick any $f\in\F$ and let $f_\alpha = \lfloor f \rfloor_\alpha$. Then $\|f-f_\alpha \|_\infty \leq \alpha/2$. Then for all $\epsilon \in \{\pm1\}^n $ and any $t \in [n]$
	$$\left|f(\tz_t(\epsilon))- \v_t(\epsilon)\right| \leq \left|f(\tz_t(\epsilon))- f_{\alpha}(\tz_t(\epsilon))\right| + \left|f_\alpha(\tz_t(\epsilon))- \v_t(\epsilon)\right| \leq \alpha,$$
	and so $V$ also provides an $L_\infty$ cover at scale $\alpha$.
	
	We conclude that $\Nonl_\infty(\alpha, \F, \tz) \leq \Nonl_\infty(\alpha/2, \lfloor \F \rfloor_\alpha, \tz) =  \Nonl_\infty( 1/2, {\mathcal G}, \tz) $ where $G = \frac{1}{\alpha} \lfloor \F \rfloor_\alpha$. The functions of ${\mathcal G}$ take on a discrete set of at most $\lfloor 2/\alpha \rfloor + 1$ values. Obviously, by adding a constant to all the functions in ${\mathcal G}$, we can make the set of values to be $\{0, \ldots, \lfloor 2/\alpha \rfloor \}$. We now apply  Theorem~\ref{thm:sauer_multiclass} with an upper bound $\sum_{i=0}^d {n\choose i} k^i \leq \left(ekn\right)^d$ which holds for any $n>0$. This yields  $\Nonl_\infty(1/2, {\mathcal G}, \tz) \leq \left(2e n/\alpha \right)^{\faton_{2}({\mathcal G})}$. 
	
	It remains to prove $\faton_2({\mathcal G}) \leq \faton_\alpha (\F)$, or, equivalently (by scaling) $\faton_{2\alpha} (\lfloor \F \rfloor_\alpha)  \leq \faton_\alpha (\F)$. 
	To this end, suppose there exists an $\reals$-valued tree $\tz$ of depth $d=\faton_{2\alpha}(\lfloor \F \rfloor_\alpha)$ such that there is an witness tree $\s$ with
	$$
	\forall \epsilon \in \{\pm1\}^d , \ \exists f_{\alpha} \in \lfloor \F \rfloor_\alpha \ \ \ \trm{s.t. } \forall t \in [d], \  \epsilon_t (f_{\alpha}(\tz_t(\epsilon)) - \s_t(\epsilon)) \ge \alpha
	$$
	Using the fact that for any $f\in\F$ and $f_\alpha = \lfloor f \rfloor_\alpha$ we have $\|f-f_\alpha \|_\infty \leq \alpha/2$, it follows that
	$$
	\forall \epsilon \in \{\pm1\}^d , \ \exists f \in \F \ \ \ \trm{s.t. } \forall t \in [d], \  \epsilon_t (f(\tz_t(\epsilon)) - \s_t(\epsilon)) \ge \alpha/2
	$$
	That is, $\s$ is a witness to $\alpha$-shattering by $\F$. Thus for any $\tz$, 
	$$\Nonl_\infty(\alpha, \F, \tz) \leq \Nonl_\infty( \alpha/2, \lfloor \F \rfloor_\alpha, \tz) \leq \left(\frac{2e n}{\alpha} \right)^{\faton_{2\alpha} (\lfloor \F \rfloor_\alpha) } \leq \left(\frac{2e n}{\alpha}\right)^{\faton_{\alpha} (\F) }$$
\end{proof}

\begin{proof}[\textbf{Proof of Theorem~\ref{thm:dudley}}]
Define $\beta_0 = 1$ and $\beta_j = 2^{-j}$. For a fixed tree $\tz$ of depth $n$, let $V_j$ be an  $\ell_2$-cover at scale $\beta_j$. For any path $\epsilon \in \{\pm1\}^n$ and any $f \in \F$ , let $\v[f,\epsilon]^j \in V_j$ the element of the cover such that
$$
\sqrt{\frac{1}{n} \sum_{t=1}^{n} |\v[f,\epsilon]^{j}_{t}(\epsilon) - f(\tz_t(\epsilon))|^2 }  \le \beta_j
$$
By the definition such a $\v[f,\epsilon]^j \in V_j$ exists, and we assume for simplicity this element is unique (ties can be broken in an arbitrary manner). Thus,  $f\mapsto \v[f,\epsilon]^j$ is a well-defined mapping for any fixed $\epsilon$ and $j$. As before, $\v[f,\epsilon]^j_t$ denotes the $t$-th mapping of  $\v[f,\epsilon]^j$. For any  $t  \in [n]$, we have
 $$
 f(\tz_t(\epsilon)) =  f(\tz_t(\epsilon)) - \v[f,\epsilon]^{N}_{t}(\epsilon) + \sum_{j=1}^{N} ( \v[f,\epsilon]^{j}_{t}(\epsilon) - \v[f,\epsilon]^{j-1}_{t}(\epsilon))
 $$
 where $\v[f,\epsilon]^{0}_{t}(\epsilon) = 0$.  Hence, 
{\small
\begin{align}
	\label{eq:dudley_step_decomposition}
\Es{\epsilon}{\sup_{f \in \F}  \sum_{t=1}^n \epsilon_t  f(\tz_t(\epsilon))}  
&= \Es{\epsilon}{ \sup_{f \in \F}  \sum_{t=1}^{n} \epsilon_t  \left( f(\tz_{t}(\epsilon)) - \v[f,\epsilon]^N_{t}(\epsilon) + \sum_{j=1}^N (\v[f,\epsilon]^j_{t}(\epsilon) - \v[f,\epsilon]^{j-1}_{t}(\epsilon)) \right)  } \nonumber\\
&= \Es{\epsilon}{ \sup_{f \in \F}  \sum_{t=1}^n \epsilon_t  \left( f(\tz_t(\epsilon)) - \v[f,\epsilon]^N_{t}(\epsilon)\right) + \sum_{t=1}^n \epsilon_t  \left( \sum_{j=1}^N (\v[f,\epsilon]^j_{t}(\epsilon) - \v[f,\epsilon]^{j-1}_{t}(\epsilon)) \right) } \nonumber\\
& \le \Es{\epsilon}{ \sup_{f \in \F}  \sum_{t=1}^n \epsilon_t  \left( f(\tz_t(\epsilon)) - \v[f,\epsilon]^N_{t}(\epsilon)\right) } + \Es{\epsilon}{ \sup_{f \in \F}\sum_{t=1}^n \epsilon_t  \left( \sum_{j=1}^N (\v[f,\epsilon]^j_{t}(\epsilon) - \v[f,\epsilon]^{j-1}_{t}(\epsilon)) \right) } 
\end{align}
}
The first term above can be bounded via the Cauchy-Schwarz inequality as
$$ 
\Es{\epsilon}{ \sup_{f \in \F}  \sum_{t=1}^n \epsilon_t  \left( f(\tz_t(\epsilon)) - \v[f,\epsilon]^N_{t}(\epsilon)\right) } 
\leq n\ \Es{\epsilon}{\sup_{f \in \F}  \sum_{t=1}^n \frac{\epsilon_t}{\sqrt{n}} \frac{\left( f(\tz_t(\epsilon)) - \v[f,\epsilon]^N_{t}(\epsilon)\right)}{\sqrt{n}} }
\leq n\ \beta_N .
$$

The second term in \eqref{eq:dudley_step_decomposition} is bounded by considering  successive refinements of the cover. The argument, however, is more delicate than in the classical case, as the trees $\v[f,\epsilon]^j$, $\v[f,\epsilon]^{j-1}$ depend on the particular path. Consider all possible pairs of $\v^s\in V_j$ and $\v^r\in V_{j-1}$, for $1\leq s\leq |V_j|$, $1\leq r \leq |V_{j-1}|$, where we assumed an arbitrary enumeration of elements. For each pair $(\v^s,\v^r)$, define a real-valued tree $\w^{(s,r)}$ by
\begin{align*}
\w^{(s,r)}_t(\epsilon) = 
	\begin{cases} 
	\v^s_t(\epsilon)-\v^r_t(\epsilon) & \text{if there exists } f\in\F \mbox{ s.t. } \v^s = \v[f,\epsilon]^{j}, \v^r = \v[f, \epsilon]^{j-1} \\
	0 &\text{otherwise.}
	\end{cases}
\end{align*}
for all $t\in [n]$ and $\epsilon\in\{\pm1\}^n$. It is crucial that $\w^{(s,r)}$ can be non-zero only on those paths $\epsilon$ for which $\v^s$ and $\v^r$ are indeed the members of the covers (at successive resolutions) close to $f(\tz(\epsilon))$ (in the $\ell_2$ sense) {\em for some} $f\in\F$. It is easy to see that $\w^{(s,r)}$ is well-defined. Let the set of trees $W_j$ be defined as
\begin{align*}
	W_j = \left\{ \w^{(s,r)}: 1\leq s\leq |V_j|, 1\leq r \leq |V_{j-1}| \right\}
\end{align*}

Now, the second term in \eqref{eq:dudley_step_decomposition} can be written as
\begin{align*}
	\Es{\epsilon}{\sup_{f \in \F}\sum_{t=1}^n \epsilon_t  \sum_{j=1}^N (\v[f,\epsilon]^j_{t}(\epsilon) - \v[f,\epsilon]^{j-1}_{t}(\epsilon))   } 
	&\leq \sum_{j=1}^N \Es{\epsilon}{\sup_{f \in \F}\sum_{t=1}^n \epsilon_t   (\v[f,\epsilon]^j_{t}(\epsilon) - \v[f,\epsilon]^{j-1}_{t}(\epsilon)) } \\
	&\leq \sum_{j=1}^N \Es{\epsilon}{\max_{\w \in W_j}\sum_{t=1}^n \epsilon_t \w_t(\epsilon) } 
\end{align*}
The last inequality holds because for any $j\in [N]$, $\epsilon\in\{\pm1\}^n$ and $f\in\F$ there is some $\w^{(s,r)} \in W_j$ with $\v[f,\epsilon]^j=\v^s$, $\v[f,\epsilon]^{j-1} = \v^r$ and
$$  \v^s_t(\epsilon)-\v^r_t(\epsilon) = \w^{(s,r)}_t(\epsilon) \ \ \forall t\leq n.$$

Clearly, $|W_j| \leq |V_j|\cdot |V_{j-1}|$. To invoke Lemma \ref{lem:fin}, it remains to bound the magnitude of all $\w^{(s,r)}\in W_j$ along all paths. For this purpose, fix $\w^{(s,r)}$ and a path $\epsilon$. If there exists $f\in\F$ for which $\v^s=\v[f,\epsilon]^j$ and $\v^r=\v[f,\epsilon]^{j-1}$, then $\w^{(s,r)}_t(\epsilon) = \v[f,\epsilon]^{j}_t - \v[f,\epsilon]^{j-1}_t$ for any $t\in[n]$. By triangle inequality
$$\sqrt{\sum_{t=1}^n \w^{(s,r)}_t(\epsilon)^2}\leq \sqrt{\sum_{t=1}^n (\v[f,\epsilon]_t^j(\epsilon)-f(\tz_t(\epsilon)))^2} + \sqrt{\sum_{t=1}^n (\v[f,\epsilon]_t^{j-1}(\epsilon)-f(\tz_t(\epsilon)))^2} \leq \sqrt{n}(\beta_{j}+\beta_{j-1})=3\sqrt{n}\beta_j.$$
If there exists no such $f\in\F$ for the given $\epsilon$ and $(s,r)$, then $\w^{(s,r)}_t(\epsilon)$ is zero for all $t\geq t_o$, for some $1\leq t_o < n$, and thus 
$$\sqrt{\sum_{t=1}^n \w^{(s,r)}_t(\epsilon)^2} \leq \sqrt{\sum_{t=1}^n \w^{(s,r)}_t(\epsilon')^2}$$
for any other path $\epsilon'$ which agrees with $\epsilon$ up to $t_o$. Hence, the bound 
$$\sqrt{\sum_{t=1}^n \w^{(s,r)}_t(\epsilon)^2}\leq 3\sqrt{n}\beta_j$$
holds for all $\epsilon\in\{\pm1\}^n$ and all $\w^{(s,r)}\in W_j$. 

Now, back to \eqref{eq:dudley_step_decomposition}, we put everything together and apply Lemma \ref{lem:fin}:
\begin{align*}
	\Es{\epsilon}{\sup_{f \in \F}  \sum_{t=1}^n \epsilon_t  f(\tz_t(\epsilon))} 
	& \le n \ \beta_N + \sqrt{n}\ \sum_{j=1}^N  3 \beta_j \sqrt{2  \log(|V_j|\ |V_{j-1}|)} \\
& \le n\ \beta_N + \sqrt{n}\ \sum_{j=1}^N  6 \beta_j \sqrt{ \log(|V_j|)}\\
& \le n\ \beta_N + 12\ \sqrt{n}\ \sum_{j=1}^N ( \beta_j - \beta_{j+1})  \sqrt{\log \mathcal{N}_2(\beta_j, \mathcal{F}, \tz ) \ }\\
& \le n\ \beta_N + 12 \ \int_{\beta_{N+1}}^{\beta_{0}}   \sqrt{n\ \log\ \mathcal{N}_2(\delta, \mathcal{F},\tz ) \ } d \delta
\end{align*}  

where the last but one step is because $2 (\beta_j - \beta_{j+1}) = \beta_j$. Now for any $\alpha > 0$, pick $N = \sup\{j : \beta_j > 2 \alpha\}$. In this case we see that by our choice of $N$, $\beta_{N+1} \le 2 \alpha$ and so $\beta_N = 2 \beta_{N+1} \le 4 \alpha$. Also note that since $\beta_{N} > 2 \alpha$, $\beta_{N+1} = \frac{\beta_N}{2}> \alpha$. Hence dividing throughout by $n$ we conclude that
\begin{align*}
\Radon_n(\F) \le \inf_{\alpha}\left\{4 \alpha + 12 \int_{\alpha}^{1} \sqrt{ \frac{\log \ \mathcal{N}_2(\delta, \mathcal{F},n ) }{n} } d \delta \right\}
\end{align*}
\end{proof}

\begin{proof}[\textbf{Proof of Theorem~\ref{thm:universal}}]

Let $(x'_1,\ldots,x'_n)$ be a sequence tangent to $(x_1,\ldots,x_n)$. Recall the notation $\Es{t-1}{f(x'_t)} = \Ebr{ f(x'_t) | x_1,\ldots,x_{t-1}}$. By Chebychev's inequality, for any $f \in \mathcal{F}$,
\begin{align*}
\Ps{\D}{  \frac{1}{n}\left|\sum_{t=1}^n \left(f(x'_t) - \Es{t-1}{f(x'_t)} \right) \right| > \alpha/2 \Big| \frac{ }{ } x_1,\ldots, x_n } 
& \le \frac{\E{ \left(\sum_{t=1}^n \left(f(x'_t) - \Es{t-1}{f(x'_t)} \right) \right)^2 \Big| x_1,\ldots, x_n} }{n^2 \alpha^2/4}\\
& = \frac{\sum_{t=1}^n  \E{ \left( f(x'_t) - \Es{t-1}{f(x'_t)} \right)^2 \big| x_1,\ldots, x_n} }{n^2 \alpha^2/4}\\
& \le \frac{4n}{n^2 \alpha^2/4}  = \frac{16}{n \alpha^2}.
\end{align*}
The second step is due to the fact that the cross terms are zero:
 $$ \Ebr{\left(f(x'_t) - \Es{t-1}{f(x'_t)} \right) \left(f(x'_s) - \Es{s-1}{f(x'_s)} \right) \big| x_1,\ldots, x_n } = 0 \ .$$
Hence
$$
\inf_{f \in \mathcal{F}} \mathbb{P}_\D\left[  \frac{1}{n}\left|\sum_{t=1}^n \left(f(x'_t) - \mathbb{E}_{t-1}[f(x'_t)] \right) \right| \le \alpha/2 \left| \frac{ }{ } x_1,\ldots, x_n\right.\right]  \ge 1 - \frac{16}{n \alpha^2}
$$
Whenever $\alpha^2 \ge \frac{32}{n}$ we can conclude that 
\begin{align*}
\inf_{f \in \mathcal{F}} \mathbb{P}_{\D}\left[  \frac{1}{n}\left|\sum_{t=1}^n \left(f(x'_t) - \mathbb{E}_{t-1}[f(x'_t)] \right) \right| \le \alpha/2 \left| \frac{ }{ } x_1,\ldots, x_n\right.\right]  \ge \frac{1}{2}
\end{align*}
Now given a fixed $x_1,...,x_n$ let $f^*$ be the function that maximizes $\frac{1}{n}\left|\sum_{t=1}^n \left(f(x_t) - \mathbb{E}_{t-1}[f(x'_t)] \right) \right|$. Note that $f^*$ is a deterministic choice given $x_1,...,x_n$. Hence
\begin{align*}
\frac{1}{2} & \le \inf_{f \in \mathcal{F}} \mathbb{P}_{\D}\left[  \frac{1}{n}\left|\sum_{t=1}^n \left(f(x'_t) - \mathbb{E}_{t-1}[f(x'_t)] \right) \right| \le \alpha/2 \left| \frac{ }{ } x_1,\ldots, x_n\right.\right]  \\
& \le \mathbb{P}_{\D}\left[  \frac{1}{n}\left|\sum_{t=1}^n \left(f^*(x'_t) - \mathbb{E}_{t-1}[f^*(x'_t)] \right) \right| \le \alpha/2 \Big| \frac{ }{ } x_1,\ldots, x_n \right] 
\end{align*}
Let $A = \left\{(x_1,\ldots,x_n) \middle| \frac{1}{n}\sup_{f \in \F} |\sum_{t=1}^n f(x_t) - \Es{t-1}{f(x'_t)}| > \alpha\right\}$. Since the above inequality holds for any $x_1,\ldots, x_n$ we can assert that 
\begin{align*}
\frac{1}{2} & \le  \mathbb{P}_{\D}\left[  \frac{1}{n}\left|\sum_{t=1}^n \left(f^*(x'_t) - \mathbb{E}_{t-1}[f^*(x'_t)] \right) \right| \le \alpha/2 \Big|  (x_1,\ldots,x_n) \in A \right] 
\end{align*}

Hence we conclude that
\begin{align*}
\frac{1}{2}& \mathbb{P}_{\D}\left[ \sup_{f \in \mathcal{F}} \frac{1}{n}\left|\sum_{t=1}^n \left(f(x_t) - \mathbb{E}_{t-1}[f(x'_t)] \right) \right| > \alpha \right]  \\
& \le \mathbb{P}_{\D}\left[  \frac{1}{n}\left|\sum_{t=1}^n \left(f^*(x'_t) - \mathbb{E}_{t-1}[f^*(x'_t)] \right) \right| \le \alpha/2 \left| \frac{ }{ } (x_1,\ldots, x_n) \in A\right.\right]  \\
&\hspace{2cm} \times \mathbb{P}_{\D}\left[ \frac{1}{n} \sup_{f \in \F} \left|\sum_{t=1}^n \left(f(x_t) - \mathbb{E}_{t-1}[f(x'_t)] \right) \right| > \alpha \right]  \\
& \le \mathbb{P}_{\D}\left[ \frac{1}{n}\left|\sum_{t=1}^n \left(f^*(x_t) - f^*(x'_t) \right) \right| > \alpha/2 \right] \\
& \le \mathbb{P}_{\D}\left[ \frac{1}{n}\sup_{f \in \mathcal{F}}\left|\sum_{t=1}^n \left(f(x_t) - f(x'_t) \right) \right| > \alpha/2 \right] 
\end{align*}

Now we apply Lemma~\ref{lem:technical_symmetrization} with $\phi(u) = \ind{u  > \alpha/2}$ and $\Delta_f(x_t, x_t') = f(x_t)- f(x'_t)$,
\begin{align}
	\label{eq:sym_with_indicators}
	&\E{\ind{\sup_{f \in \F} \left|\sum_{t=1}^n  f(x_t) - f(x'_t) \right| \geq \alpha/2}} \nonumber\\
	&\hspace{1in}\leq 
	\sup_{x_{1},x_{1}'}\left\{\Es{\epsilon_{1}}{\ldots \sup_{x_n, x'_n}\left\{\Es{\epsilon_n}{ 
		\ind{\sup_{f \in\F} \left| \sum_{t=1}^{n} \epsilon_t \left(f(x_t) - f(x'_t) \right) \right| \geq \alpha/2}
		}\right\}\ldots }\right\}
\end{align}

The next few steps are similar to the proof of Theorem~\ref{thm:valrad}. Since 
$$\sup_{f \in\F} \left| \sum_{t=1}^{n} \epsilon_t \left(f(x_t) - f(x'_t) \right) \right| \leq \sup_{f \in\F}\left| \sum_{t=1}^{n} \epsilon_t f(x_t) \right|+\sup_{f \in\F}\left| \sum_{t=1}^{n} \epsilon_t f(x'_t) \right|$$
it is true that
$$\ind{\sup_{f \in\F} \left| \sum_{t=1}^{n} \epsilon_t \left(f(x_t) - f(x'_t) \right) \right| \geq \alpha/2}
\leq
\ind{\sup_{f \in\F} \left| \sum_{t=1}^{n} \epsilon_t f(x_t) \right| \geq \alpha/4}
+
\ind{\sup_{f \in\F} \left| \sum_{t=1}^{n} \epsilon_t f(x'_t) \right| \geq \alpha/4}
$$

The right-hand side of Eq.~\eqref{eq:sym_with_indicators} then splits into two equal parts:
\begin{align*}
	&\sup_{x_{1}}\left\{\Es{\epsilon_{1}}{\ldots \sup_{x_n}\left\{\Es{\epsilon_n}{ 
	\ind{\sup_{f \in\F} \left| \sum_{t=1}^{n} \epsilon_t f(x_t) \right| \geq \alpha/4}
	}\right\}\ldots }\right\} \\
	&+
	\sup_{x'_{1}}\left\{\Es{\epsilon_{1}}{\ldots \sup_{x'_n}\left\{\Es{\epsilon_n}{ 
	\ind{\sup_{f \in\F} \left| \sum_{t=1}^{n} \epsilon_t f(x'_t) \right| \geq \alpha/4}
	}\right\}\ldots }\right\} \\
	&=2\sup_{x_{1}}\left\{\Es{\epsilon_{1}}{\ldots \sup_{x_n}\left\{\Es{\epsilon_n}{ 
	\ind{\sup_{f \in\F} \left| \sum_{t=1}^{n} \epsilon_t f(x_t) \right| \geq \alpha/4}
	}\right\}\ldots }\right\}
\end{align*}

Moving to the tree representation, 
\begin{align*}
\mathbb{P}_{\D}\left[ \frac{1}{n}\sup_{f \in \mathcal{F}}\left|\sum_{t=1}^n \left(f(x_t) - f(x'_t) \right) \right| > \alpha/2 \right] 
&\le 2\sup_{\tz} \mathbb{E}_{\epsilon}\left[\ind{\frac{1}{n} \sup_{f \in \mathcal{F}} \left| \sum_{t=1}^{n} \epsilon_t f(\tz_t(\epsilon)) \right| > \alpha/4}\right]\\
&= 2\sup_{\tz} \mathbb{P}_{\epsilon}\left[\frac{1}{n} \sup_{f \in \mathcal{F}} \left| \sum_{t=1}^{n} \epsilon_t f(\tz_t(\epsilon)) \right| > \alpha/4\right] 
\end{align*}
We can now conclude that
$$
\mathbb{P}_{\D}\left[ \frac{1}{n} \sup_{f \in \mathcal{F}} \left|\sum_{t=1}^n \left(f(x_t) - \mathbb{E}_{t-1}[f(x_t)] \right) \right| > \alpha \right]  \le 4 \sup_{\tz}\ \mathbb{P}_{\epsilon}\left[ \frac{1}{n} \sup_{f \in \mathcal{F}} \left| \sum_{t=1}^{n} \epsilon_t f(\tz_t(\epsilon)) \right| > \alpha/4\right] 
$$

Fix an $\Z$-valued tree $\tz$ of depth $n$. By assumption $\faton_\alpha(\F) < \infty$ for any $\alpha > 0$. Let $V$ be a minimum $\ell_1$-cover of $\F$ over $\tz$ at scale $\alpha/8$. Corollary~\ref{cor:l2_norm_bound} ensures that
$$
|V| = \mathcal{N}_1(\alpha/8, \F, \tz) \le \left( \frac{16en}{\alpha}\right)^{\faton_{\frac{\alpha}{8}} }
$$
and for any $f \in \mathcal{F}$ and $\epsilon\in\{\pm1\}^n$, there exists $\v[f,\epsilon] \in V$ such that 
$$
\frac{1}{n} \sum_{t=1}^{n} |f(\tz_t(\epsilon)) - \v[f,\epsilon]_t(\epsilon)| \le \alpha/8
$$
on the given path $\epsilon$.
Hence
\begin{align*}
& \Ps{\epsilon}{\frac{1}{n} \sup_{f \in \mathcal{F}} \left| \sum_{t=1}^{n} \epsilon_t f(\tz_t(\epsilon)) \right| > \alpha/4 } \\
& ~~~~~~~~~~~ = \Ps{\epsilon}{\frac{1}{n}\sup_{f \in \mathcal{F}} \left| \sum_{t=1}^{n} \epsilon_t \left( f(\tz_t(\epsilon)) - \v[f,\epsilon]_t(\epsilon) +  \v[f,\epsilon]_t(\epsilon) \right)\right| > \alpha/4} \\
& ~~~~~~~~~~~ \le \Ps{\epsilon}{\frac{1}{n}\sup_{f \in \mathcal{F}} \left| \sum_{t=1}^{n} \epsilon_t \left( f(\tz_t(\epsilon)) - \v[f,\epsilon]_t(\epsilon) \right) \right| + \frac{1}{n}\sup_{f \in \mathcal{F}} \left| \sum_{t=1}^{n} \epsilon_t \v[f,\epsilon]_t(\epsilon) \right| > \alpha/4}\\
& ~~~~~~~~~~~ \le \Ps{\epsilon}{\frac{1}{n}\sup_{f \in \mathcal{F}} \left| \sum_{t=1}^{n} \epsilon_t \v[f,\epsilon]_t(\epsilon) \right| > \alpha/8}
\end{align*}
For fixed $\epsilon=(\epsilon_1, \ldots,\epsilon_n)$, 
$$ \frac{1}{n}\sup_{f \in \mathcal{F}} \left| \sum_{t=1}^{n} \epsilon_t \v[f,\epsilon]_t(\epsilon) \right| > \alpha/8 ~~~~~\Longrightarrow~~~~~ \frac{1}{n}\max_{\v \in V} \left| \sum_{t=1}^{n} \epsilon_t \v_t(\epsilon) \right| > \alpha/8$$
and, therefore, for any $\tz$,
\begin{align*}
	&\Ps{\epsilon}{\frac{1}{n} \sup_{f \in \mathcal{F}} \left| \sum_{t=1}^{n} \epsilon_t f(\tz_t(\epsilon)) \right| > \alpha/4 } 
~~~\le~~~ \Ps{\epsilon}{\frac{1}{n}\max_{\v \in V} \left| \sum_{t=1}^{n} \epsilon_t \v_t(\epsilon) \right| > \alpha/8 }\\
& \le \sum_{\v \in V} \Ps{\epsilon}{\frac{1}{n} \left| \sum_{t=1}^{n} \epsilon_t \v_t(\epsilon) \right| > \alpha/8 } 
~~~\le~~~ 2|V| e^{- n \alpha^2 /128} 
~~~\le~~~ 2\left( \frac{16e n}{\alpha}\right)^{\faton_{\alpha/8}} e^{- n \alpha^2 /128}
\end{align*}
Hence we conclude that for any $\D$
$$
\mathbb{P}_{\D}\left[ \frac{1}{n} \sup_{f \in \mathcal{F}}\left|\sum_{t=1}^n \left(f(x_t) - \mathbb{E}_{t-1}[f(x_t)] \right) \right| > \alpha \right]  \le 8 \left( \frac{16e n}{\alpha}\right)^{\faton_{\alpha/8}} e^{- n \alpha^2 /128}
$$

Now applying Borel-Cantelli lemma proves the required result as 
$$\sum_{n=1}^{\infty} 8 \left( \frac{16e n}{\alpha}\right)^{\faton_{\alpha/8}} e^{- n \alpha^2 /128} < \infty \ .$$
\end{proof}

\begin{proof}[\textbf{Proof of Proposition~\ref{prop:uplow}}]

Using Theorem~\ref{thm:dudley} we get the bound $\Radon_n(\F_{\trm{S}}) \le \Dudleyon_n(\F_{\trm{S}})$. Further the fact that absolute loss is $1$-Lipschitz implies that an $\epsilon$ cover of class $\H$ is also an $\epsilon$ cover for loss class $\F_{\trm{S}}$ and so $\Dudleyon_n(\F_{\trm{S}}) \le \Dudleyon_n(\H)$. This gives the first upper bound of value in terms of $\Dudleyon_n(\H)$. The second inequality in the upper bound is a direct consequence of using Corollary~\ref{cor:l2_norm_bound} in $\Dudleyon_n(\H)$. Now before we prove the final inequality in the upper bound we first prove the lower bound because we shall use ideas from the lower bound to get the final inequality in the upper bound.

For the lower bound, we use a construction similar to \cite{BenPalSha09}. We construct a particular distribution which induces a lower bound on regret for any algorithm. For any $\alpha \ge 0$ by definition of fat-shattering dimension, there exists a tree $\tz$ of depth $d=\faton_\alpha(\H)$ that can be $\alpha$-shattered by $\H$. For simplicity, we assume $n = k d$ where $k$ is some non-negative integer, and the case $n\leq d$ is discussed at the end of the proof. Now, define the $j$th block of time $T_j = \{(j-1)k+1, \ldots, jk\}$. 

Now the strategy of Nature (Adversary) is to first pick $\tilde{\epsilon} \in \{\pm 1\}^{n}$ independently and uniformly at random. Further let $\epsilon\in \{\pm1\}^d$ be defined as $\epsilon_j = \text{sign}\left(\sum_{t\in T_j} \tilde{\epsilon}_t\right)$ for $1\leq j\leq d$, the block-wise modal sign of $\tilde{\epsilon}$. Now note that by definition of $\alpha$-shattering, there exists a witness tree $\s$ such that for any $\epsilon\in \{\pm1\}^d$ there exists $\h_{\epsilon}\in\H$ with $\epsilon_j(\h_{\epsilon}(\tz_j(\epsilon))-\s_j(\epsilon)) \geq \alpha/2$ for all $1\leq j\leq d$.  Now let the random sequence $(x_1,y_1),\ldots,(x_n,y_n)$ be defined by $x_t = \tz_j(\epsilon)$ for all $t\in T_j$ and $j\in \{1,\ldots,d\}$ and $y_t = \tilde{\epsilon}_t$. In the remainder of the proof we show that any algorithm suffers large expected regret.

Now consider any player strategy (possibly randomized) making prediction $\hat{y}_t\in[-1,1]$ at round $t$. Note that if we consider block $j$, $y_t=\tilde{\epsilon}_t$ is $\pm 1$ uniformly at random. This means that irrespective of what $\hat{y}_t$ the player plays, the expectation over $\tilde{\epsilon}_t$ of the loss the player suffers at round $t$ is
$$
\En_{\tilde{\epsilon}_t}{|\hat{y}_t - y_t|} = 1
$$
Hence on block $j$, the expected loss accumulated by any player is $k$ and so for any player strategy (possibly randomized), 
\begin{align}\label{eq:player}
\E{\sum_{t=1}^n |\hat{y}_t - y_t|} & = \sum_{j=1}^d k = d k = n
\end{align}

On the other hand since $x_t = \tz_j(\epsilon)$, we know that there always exists a function for any $\epsilon \in \{\pm1\}^d$, say $\h_\epsilon$ such that $\epsilon_j(\h_{\epsilon}(\tz_j(\epsilon))-\s_j(\epsilon)) \geq \alpha/2$. Hence
\begin{align*}
\E{\inf_{\h \in \H}  \sum_{t=1}^n |\h(x_t) - y_t|} & \le  \sum_{j = 1}^d \E{\sum_{t \in T_j} |\h_{\epsilon}(x_t) - y_t|} \\
& = \sum_{j = 1}^d \E{\sum_{t \in T_j}|\h_{\epsilon}(\tz_j(\epsilon)) - y_t|}\\
& \le \sum_{j = 1}^d \E{\max_{c_j \in [\s_j(\epsilon) + \epsilon_j \frac{\alpha}{2}, \epsilon_j]}\sum_{t \in T_j}|c_j - y_t|}
\end{align*}
where the last step is because for all of block $j$, $\h_{\epsilon}(\tz_j(\epsilon))$ does not depend on $t$ and lies in the interval\footnote{We use the convention that $[a,b]$ stands for $[b,a]$ whenever $a> b$. } $[\s_j(\epsilon) + \epsilon_j \frac{\alpha}{2}, \epsilon_j]$ (i.e. the majority side) and so by replacing it by the maximal $c_j$ in the same interval for that block we only make the quantity bigger. Now for a block $j$, define the number of labels that match the sign of $\epsilon_j$ (the majority) as $M_j = \sum_{t \in T_j} \ind{y_t = \epsilon_j}$. 
Since $y_t = \tilde{\epsilon}_t \in\{\pm 1\}$, observe that the function $g(c_j)=\sum_{t \in T_j}|c_j - y_t|$ is linear on the interval $[-1,1]$ with its minimum at the majority sign $\epsilon_j$. Hence, the maximum over $[\s_j(\epsilon) + \epsilon_j \frac{\alpha}{2}, \epsilon_j]$ must occur at $c_j=\s_j(\epsilon) + \epsilon_j \frac{\alpha}{2}$. Substituting,

\begin{align*}
	\max_{c_j \in [\s_j(\epsilon) + \epsilon_j \frac{\alpha}{2}, \epsilon_j]}\sum_{t \in T_j}|c_j - y_t| 
	&= M_j \left|\s_j(\epsilon) + \epsilon_j \frac{\alpha}{2} - \epsilon_j \right| + (k-M_j) \left|\s_j(\epsilon) + \epsilon_j \frac{\alpha}{2} + \epsilon_j \right| \\
	&=M_j \left|\epsilon_j \s_j(\epsilon) + \frac{\alpha}{2} - 1 \right| + (k-M_j)\left|\epsilon_j \s_j(\epsilon) + \frac{\alpha}{2} + 1 \right| \\
	&=M_j\left(1-\epsilon_j \s_j(\epsilon) -\frac{\alpha}{2}\right) + (k-M_j)\left(1+\epsilon_j \s_j(\epsilon) + \frac{\alpha}{2} \right) \\
	&=k + (k-2M_j)\left( \epsilon_j \s_j(\epsilon) + \frac{\alpha}{2}\right)
\end{align*}

Hence,
\begin{align*}
\E{\inf_{\h \in \H}  \sum_{t=1}^n |\h(x_t) - y_t|} & \le  d k + \sum_{j = 1}^d \E{ \epsilon_j \s_j(\epsilon) (k - 2 M_j) + \frac{\alpha}{2} (k - 2 M_j) }\\
 & =  d k + \sum_{j = 1}^d \E{ \epsilon_j \s_j(\epsilon) (k - 2 M_j)} + \frac{\alpha}{2} \sum_{j = 1}^d \E{ k - 2 M_j }
\end{align*}
Further note that $k - 2M_j = -|\sum_{t \in T_j} \tilde{\epsilon}_t|$ and so $\epsilon_j (k - 2M_j) = -\sum_{t \in T_j} \tilde{\epsilon}_t$ and so the expectation 
$$
\E{\epsilon_j \s_j(\epsilon) (k - 2 M_j)} = \E{\Es{\tilde{\epsilon}_{k(j-1)+ 1 : j k}}{\epsilon_j \s_j(\epsilon) (k - 2 M_j)}} = 0
$$ 
because $\s_j(\epsilon)$ is independent of $\tilde{\epsilon}_t$ for $t\in T_j$. Hence we see that
\begin{align}\label{eq:compare}
\E{\inf_{\h \in \H}  \sum_{t=1}^n |\h(x_t) - y_t|} & \le  d k +  \frac{\alpha}{2} \sum_{j = 1}^d \E{ k - 2 M_j }
\end{align}

Combining Equations \eqref{eq:player} and \eqref{eq:compare} we can conclude that for any player strategy,
\begin{align}
&\E{\sum_{t=1}^n |\hat{y}_t - y_t|}- \E{\inf_{\h \in \H}  \sum_{t=1}^n |\h(x_t) - y_t|} ~~\ge~~ \frac{\alpha}{2} \sum_{j = 1}^d \E{ 2 M_j - k } \notag \\
&=~~ \frac{\alpha}{2} \E{\sum_{j=1}^d \left|\sum_{t\in T_j} \tilde{\epsilon}_t \right|}
~~=~~ \frac{\alpha}{2} \sum_{j=1}^d \En \left|\sum_{t\in T_j} \tilde{\epsilon}_t \right|
~~\ge~~ \frac{\alpha d}{2} \sqrt{\frac{k}{2}}  
~~=~~ \alpha  \sqrt{\frac{n d}{8}} ~~=~~  \alpha  \sqrt{\frac{n \ \faton_\alpha}{8}} \label{eq:radlowbnd}
\end{align}
by Khinchine's inequality (e.g. \cite[Lemma A.9]{PLG}), yielding the theorem statement for $n\geq \faton_\alpha$. For the case of $n < \faton_\alpha$, the proof is the same with $k=1$ and the depth of the shattered tree being $n$, yielding a lower bound of $\alpha n/\sqrt{8}$. Dividing throughout by $n$ completes the lower bound. 

Now we move to the final inequality in the upper bound, but before we proceed notice that since $y_t$ are Rademacher random variables. Hence from Equation \ref{eq:radlowbnd} we see that
\begin{align*}
\alpha  \sqrt{\frac{\faton_\alpha}{8\ n}}  & \le \E{\frac{1}{n} \sum_{t=1}^n |\hat{y}_t - y_t|}- \E{\inf_{\h \in \H} \frac{1}{n} \sum_{t=1}^n |\h(x_t) - y_t|} \\
& =  \E{\sup_{\h \in \H} \frac{1}{n} \sum_{t=1}^n (1 - |\h(x_t) - y_t|)} \\
& =  \E{\sup_{\h \in \H} \frac{1}{n} \sum_{t=1}^n \tilde{\epsilon}_t \h(x_t)} \\
& =  \E{\sup_{\h \in \H} \frac{1}{n} \sum_{t=1}^n \tilde{\epsilon}_t \h(\tz'_t(\tilde{\epsilon}))} 
\end{align*}
Where $\tz'_t(\tilde{\epsilon}) = \tz_{\lceil \frac{t}{k} \rceil}(\epsilon)$ where each $\epsilon_j = \text{sign}\left(\sum_{t\in T_j} \tilde{\epsilon}_t\right)$ (ie. $x_t$'s can be seen as nodes of a tree formed by taking the $\tz$ tree which is of depth $d$ and making it into a depth $d$ tree by expanding each node of the tree into a subtree of depth $k$).  Hence we conclude that :
\begin{align*}
\alpha  \sqrt{\frac{\faton_\alpha}{8\ n}}  & \le \E{\frac{1}{n} \sum_{t=1}^n |\hat{y}_t - y_t|}- \E{\inf_{\h \in \H} \frac{1}{n} \sum_{t=1}^n |\h(x_t) - y_t|} \\
& =  \E{\sup_{\h \in \H} \frac{1}{n} \sum_{t=1}^n \tilde{\epsilon}_t \h(\tz'_t(\tilde{\epsilon}))} \\
& \le \sup_{\tz}  \Es{\epsilon}{\sup_{\h \in \H} \frac{1}{n} \sum_{t=1}^n \epsilon_t \h(\tz_t(\epsilon))} =  \Radon_n(\H)
\end{align*}
Thus effectively we have shown that $\sup_{\alpha} \alpha  \sqrt{\frac{\min\{\faton_\alpha,n\}}{8\ n}} \le \Radon_n(\H)$. Using this we see that if $\hat{\alpha}$ is the solution to $\faton_{\hat{\alpha}} = n$ then we have that $\hat{\alpha} \le \Radon_n(\H)$ and for any $\beta > \hat{\alpha}$, $\sqrt{\frac{\faton_\beta}{n}} \le \frac{2 \sqrt{2}\ \Radon_n(\H)}{\beta}$.
Hence using this we conclude that,
\begin{align*}
 \inf_{\alpha}\left\{4 \alpha + \frac{12}{\sqrt{n}} \int_{\alpha}^{1} \sqrt{ \faton_\beta(\H) \log\left(\frac{2 e n}{\beta}\right)}\ d \beta \right\} & \le  4 \hat{\alpha} + 12 \int_{\hat{\alpha}}^{1} \frac{2 \sqrt{2}\ \Radon_n(\H)\ \sqrt{  \log\left(\frac{2 e n}{\beta}\right)}}{\beta} \ d \beta \\
 & \le 4 \Radon_n(\H) + 36 \sqrt{2}\ \Radon_n(\H)\ \sqrt{\log(n)} \int_{\hat{\alpha}}^{1} \frac{1}{\beta} \ d \beta \\ 
  & \le 4 \Radon_n(\H) + 36 \sqrt{2}\ \Radon_n(\H)\ \log^{\frac{3}{2}}(n)  \\ 
    & \le 58\ \Radon_n(\H)\ \log^{\frac{3}{2}}(n) 
\end{align*}
This concludes the upper bound. To show that $\Radon_n(\H) \le \Val^{S}_n(\H)$, consider the adversary strategy where adversary picks three $\tx$ of depth $n$. Now at round $1$ the adversary presents as instance $x_1 = \tx_1$ and $y_1 = \epsilon_1$ where $\epsilon_1$ is a Rademacher random variable. Now at round $2$ the adversary presents instance $x_2 = \tx_2(\epsilon_1)$ and $y_2 = \epsilon_2$ where again $\epsilon_2$ is a Rademacher random variable. In a similar fashion at round $t$ adversary presents instance $x_t = \tx_t(\epsilon_{1:t-1})$ and $y_t = \epsilon_t$. Therefore we see that
\begin{align*}
\Val_n^{\trm{S}}(\H) & \ge \Es{\epsilon}{\frac{1}{n} \sum_{t=1}^n |\hat{y}_t - \epsilon_t|}- \Es{\epsilon}{\inf_{\h \in \H} \frac{1}{n} \sum_{t=1}^n |\h(\tx_t(\epsilon)) -\epsilon_t|}\\
& = 1 - \Es{\epsilon}{\inf_{\h \in \H} \frac{1}{n} \sum_{t=1}^n |\h(\tx_t(\epsilon)) -\epsilon_t|}\\
& =  \Es{\epsilon}{\inf_{\h \in \H} \frac{1}{n} \sum_{t=1}^n \left(1 - |\h(\tx_t(\epsilon)) -\epsilon_t|\right)}\\
& = \Es{\epsilon}{\inf_{\h \in \H} \frac{1}{n} \sum_{t=1}^n \epsilon_t \h(\tx_t(\epsilon)) }
\end{align*}
Since choice of tree $\tx$ is arbitrary we conclude that $\Val_n^{\trm{S}}(\H) \ge \Radon_n(\H)$

\end{proof}

\begin{proof}[\textbf{Proof of Lemma~\ref{lem:soa}}]
First, we claim that for any $x\in \X$, $\faton_\alpha(V_t(r, x)) = \faton_\alpha(V_t)$ for at most two $r , r' \in B_\alpha$.\footnote{The argument should be compared to the combinatorial argument in Theorem~\ref{thm:sauer_multiclass}.}  Further if there are two such $r , r' \in B_\alpha$ then $r , r'$ are consecutive elements of $B_\alpha$ (i.e. $|r - r'| = \alpha$). Suppose, for the sake of contradiction, that $\faton_\alpha(V_t(r, x)) = \faton_\alpha(V_t(r', x))= \faton_\alpha(V_t)$ for distinct $r,r' \in B_\alpha$ that are not consecutive (i.e. $|r - r'| \ge 2 \alpha$). Then let $s = (r+r')/2$ and without loss of generality suppose $r>r'$. By definition for any $\h\in V_t(r, x)$, 
$$
\h(x) \geq r-\alpha/2 = (r' + r)/2 + (r - r')/2 - \alpha/2 \ge s + \alpha/2
$$ 
Also for any  $\h'\in V_t(r', x)$ we also have,
$$
\h'(x) \leq r' + \alpha/2  = (r' + r)/2 + (r' - r)/2 + \alpha/2 \le s -  \alpha/2
$$ 
Let $\tv$ and $\tv'$ be trees of depth $\faton_\alpha(V_t)$ $\alpha$-shattered by $V_t(r, x)$ and $V_t(r', x)$, respectively. To get a contradiction, form a new tree $\tv''$ of depth $\faton_\alpha(V_t)+1$ by joining trees $\tv$ and $\tv'$ with the constant function $\tv''_1 = x$ as the root and $\tv$ and $\tv'$ as the lest and right subtrees respectively. It is straightforward that this tree is shattered by $V_t(r, x)\cup V_t(r', x)$, a contradiction.

Notice that the times $t \in [n]$ for which $|\h_t(x_t) - y_t| > \alpha$ are exactly those times when we update current set $V_{t+1}$. We shall show that whenever an update is made, $\faton_\alpha(V_{t+1}) < \faton_\alpha(V_{t})$ and hence claim that the total number of times $|\h_t(x_t) - y_t| > \alpha$ is bounded by $\faton_\alpha(\F)$.

At any round we have three possibilities. First is when $\faton_\alpha(V_t(r, x_t)) < \faton_\alpha(V_t)$  for all $r \in B_\alpha$. In this case, clearly, an update results in $\faton_\alpha(V_{t+1}) = \faton_\alpha(V_t (\lfloor y_t \rfloor_\alpha, x_t))  < \faton_\alpha(V_{t})$.

The second case is when $\faton_\alpha(V_t(r, x_t)) = \faton_\alpha(V_t)$ for exactly one $r \in B_\alpha$. In this case the algorithm chooses $\h_t(x_t) = r$.  If the update is made, $|\h_t(x_t)-y_t|>\alpha$ and thus $\lfloor y_t \rfloor_\alpha \ne \h_t(x_t)$. We can conclude that 
$$\faton_\alpha(V_{t+1}) = \faton_\alpha(V_t (\lfloor y_t \rfloor_\alpha, x_t)) < \faton_\alpha(V_t(\h_t(x_t), x_t)) = \faton_\alpha(V_t)
$$

The final case is when $\faton_\alpha(V_t(r, x_t)) = \faton_\alpha(V_t(r', x_t))= \faton_\alpha(V_t)$ and $|r-r'|=\alpha$. In this case, the algorithm chooses $\h_t(x_t) = \frac{r + r'}{2}$. Whenever $y_t$ falls in either of these two consecutive intervals given by $r$ or $r'$, we have $|\h_t(x_t) - y_t|\leq \alpha$, and hence no update is made. Thus, if an update is made,  $\lfloor y_t \rfloor_\alpha \ne r$ and $\lfloor y_t \rfloor_\alpha \ne r'$. However, for any element or $B_\alpha$ other than $r,r'$, the fat shattering dimension is less than that of $V_t$. That is 
$$\faton_\alpha(V_{t+1}) = \faton_\alpha(V_t (\lfloor y_t \rfloor_\alpha, x_t)) < \faton_\alpha(V_t(r, x_t)) = \faton_\alpha(V_t(r', x_t)) = \faton_\alpha(V_t).$$
We conclude that whenever we update, $\faton_\alpha(V_{t+1}) < \faton_\alpha(V_t)$, and so we can conclude that algorithm's prediction is more than $\alpha$ away from $y_t$ on at most $\faton_\alpha(\F)$ number of rounds.
\end{proof}

\begin{proof}[\textbf{Proof of Corollary~\ref{cor:generic_bound_all_alpha}}]
For the choice of weights $p_i = \frac{6}{\pi^2} i^{-2}$ we see from Proposition \ref{prop:experts} that for any $i$,
$$
\E{\Reg_n} \le \alpha_i  + \sqrt{\frac{\faton_{\alpha_i} \log\left(\frac{2 n}{\alpha_i} \right)}{n}} + \frac{1}{\sqrt{n}}\left(3 + 2 \log(i)\right)
$$
Now for any $\alpha > 0$ let $i_\alpha$ be such that $\alpha \le 2^{-i_\alpha}$ and for any $i < i_\alpha$, $\alpha > 2^{-i_\alpha}$. Using the above bound on expected regret we have that
$$
\E{\Reg_n} \le \alpha_{i_\alpha}  + \sqrt{\frac{ \faton_{\alpha_{i_\alpha}} \log\left(\frac{2 n}{\alpha_{i_\alpha}} \right)}{n}} + \frac{1}{\sqrt{n}}\left(3 + 2 \log(i_\alpha)\right)
$$
However for our choice of $i_\alpha$ we see that $i_\alpha \le \log(1/\alpha)$ and further $\alpha_{i_\alpha} \le \alpha$. Hence we conclude that
$$
\E{\Reg_n} \le \alpha  + \sqrt{\frac{ \faton_{\alpha} \log\left(\frac{2 n}{\alpha} \right)}{n}} + \frac{1}{\sqrt{n}}\left(3 + 2 \log \log\left(\frac{1}{\alpha}\right)\right)
$$
Since choice of $\alpha$ was arbitrary we take infimum and get the result.
\end{proof}

\begin{proof}[\textbf{Proof of Proposition~\ref{prop:margin}}]
Fix a $\gamma > 0$ and use loss 
$$
\ell(\hat{y},y) = \left\{\begin{array}{ll}
1 &  \hat{y} y \le 0\\
1-\hat{y}y/\gamma & 0 < \hat{y} y < \gamma \\
0 & \hat{y} y \ge \gamma
\end{array}\right.
$$
First note that since the loss is $1/\gamma$-Lipschitz, we can use Theorem~\ref{thm:valrad} and the Rademacher contraction Lemma~\ref{lem:contraction} to show that for each $\gamma > 0$ there exists a randomized strategy $\Algo^\gamma$ such that
$$
\E{\sum_{t=1}^n  \Es{\h_t \sim \Algo^\gamma_t(z_{1:t-1})}{\ell(\h_t(x_t), y_t)}}  \le \inf_{\h \in \H} \sum_{t=1}^n \ell(\h(x_t), y_t) +  \frac{2}{\gamma} \Radon_n(\H) 
$$
Now note that the loss is lower bounded by the Zero-one loss $\ind{\hat{y} y <0}$ and is upper bounded by the margin Zero-one loss $\ind{\hat{y} y <\gamma}$. Hence we see that for this strategy,
\begin{align} \label{eq:gamreg}
\E{\sum_{t=1}^n  \Es{\h_t \sim \Algo^\gamma_t(z_{1:t-1})}{\ind{\h_t(x_t) y_t < 0}}}  \le \inf_{\h \in \H} \sum_{t=1}^n \ind{\h(x_t) y_t < \gamma} +  \frac{2}{\gamma} \Radon_n(\H) 
\end{align}
Hence for each fixed $\gamma$ for randomized strategy given by $\Algo^\gamma$ we have the above bound. Now we discretize over $\gamma$'s as $\gamma_i = 1 /2^{i}$ and using the output of the randomized strategies $\Algo^{\gamma_1}, \Algo^{\gamma_2}, \ldots$ that attain the regret bounds given in \eqref{eq:gamreg} as experts and running experts algorithm given in Algorithm~\ref{alg:experts} with initial weight for expert $i$ as $p_i = \frac{6}{\pi^2 i^2}$ then using Proposition~\ref{prop:experts} we get that
for this randomized strategy $\Algo$, such that for any $i$
$$
\E{\sum_{t=1}^n  \Es{\h_t \sim \Algo(z_{1:t-1})}{\ind{\h_t(x_t) y_t < 0}}}  \le \inf_{\h \in \H} \sum_{t=1}^n \ind{\h(x_t) y_t < \gamma_i} +  \frac{2}{\gamma_i} \Radon_n(\H)  +\sqrt{n}\left(1 + 2 \log\left(\frac{i \pi}{\sqrt{6}}\right)\right)
$$
Now for any $\gamma > 0$ let $i_\gamma$ be such that $\gamma \le 2^{-i_\gamma}$ and for any $i < i_\gamma$, $\gamma > 2^{-i_\gamma}$. Then using the above bound we see that
$$
\E{\sum_{t=1}^n  \Es{\h_t \sim \Algo(z_{1:t-1})}{\ind{\h_t(x_t) y_t < 0}}}  \le \inf_{\h \in \H} \sum_{t=1}^n \ind{\h(x_t) y_t < 2 \gamma} +  \frac{2}{\gamma} \Radon_n(\H)  + \sqrt{n}\left(1 + 2 \log\left(\frac{i \pi}{\sqrt{6}}\right)\right)
$$
However note that $i_\gamma \le \log(1/\gamma)$ and so we can conclude that
$$
\E{\sum_{t=1}^n  \Es{\h_t \sim \Algo(z_{1:t-1})}{\ind{\h_t(x_t) y_t < 0}}}  \le \inf_{\h \in \H} \sum_{t=1}^n \ind{\h(x_t) y_t < 2 \gamma} +  \frac{2}{\gamma} \Radon_n(\H)  + \sqrt{n}\left(1 + 2 \log\left(\frac{ \pi \log(1/\gamma)}{\sqrt{6}}\right)\right)
$$
Dividing throughout by $n$ concludes the proof.
\end{proof}

\begin{proof}[\textbf{Proof of Proposition~\ref{prop:NN}}]
We shall prove that for any $i \in [k]$,
$$
\Radon_n(\F_i) \le 2 L B_i \Radon_n(\F_{i-1})
$$
To see this note that
\begin{align}\label{eq:rec}
\Radon_n(\F_i) &= \frac{1}{n}\sup_{\tz} \Es{\epsilon}{\sup_{\underset{ \forall j  f_j \in \F_{i-1}}{w^{i} : \|w^i\|_1 \le B_i}}\sum_{t=1}^n \epsilon_t \left(\sum_{j} w^i_j \sigma\left(f_j(\tz_t(\epsilon))\right)\right) }\notag \\
& \le \frac{1}{n} \sup_{\tz} \Es{\epsilon}{\sup_{\underset{ \forall j  f_j \in \F_{i-1}}{w^{i} : \|w^i\|_1 \le B_i}} \|w^i\|_1 \max_j\left|\sum_{t=1}^n \epsilon_t \sigma\left(f_j(\tz_t(\epsilon))\right)\right|}&(\trm{H\"older's inequality})\notag \\
& \le \frac{1}{n} \sup_{\tz} \Es{\epsilon}{ B_i \sup_{ f \in \F_{i-1}}\left|\sum_{t=1}^n \epsilon_t \sigma\left(f(\tz_t(\epsilon))\right)\right|}\notag \\
& = \frac{1}{n} \sup_{\tz} \Es{\epsilon}{ B_i \sup_{ f \in \F_{i-1}}\max\left\{\sum_{t=1}^n \epsilon_t \sigma\left(f(\tz_t(\epsilon))\right), -\sum_{t=1}^n \epsilon_t \sigma\left(f(\tz_t(\epsilon))\right)\right\}}\notag \\
& = \frac{1}{n} \sup_{\tz} \Es{\epsilon}{ B_i \max\left\{\sup_{ f \in \F_{i-1}}\sum_{t=1}^n \epsilon_t \sigma\left(f(\tz_t(\epsilon))\right), \sup_{ f \in \F_{i-1}} \sum_{t=1}^n - \epsilon_t \sigma\left(f(\tz_t(\epsilon))\right)\right\}}\notag \\
& \le \frac{1}{n} \sup_{\tz} \Es{\epsilon}{ B_i \sup_{ f \in \F_{i-1}} \sum_{t=1}^n \epsilon_t \sigma\left(f(\tz_t(\epsilon))\right)} + \sup_{\tz} \Es{\epsilon}{ B_i \sup_{ f \in \F_{i-1}} \sum_{t=1}^n - \epsilon_t \sigma\left(f(\tz_t(\epsilon))\right)} & (\sigma(0) =0 \trm{ and }0 \in \F_i)  \notag \\
& = \frac{2 B_i}{n}  \sup_{\tz} \Es{\epsilon}{  \sup_{ f \in \F_{i-1}}\sum_{t=1}^n \epsilon_t \sigma\left(f(\tz_t(\epsilon))\right)}&(\trm{Proposition }\ref{prop:rademacher_properties})\notag \\
& \le \frac{2 B_i L}{n}  \sup_{\tz} \Es{\epsilon}{  \sup_{ f \in \F_{i-1}}\sum_{t=1}^n \epsilon_t f(\tz_t(\epsilon))} &(\trm{Lemma }\ref{lem:contraction})\notag \\
& = 2 B_i L \Radon_n(\F_{i-1})
\end{align}
To finish the proof we note that 
\begin{align*}
\Radon_n(\F_1) &= \sup_{\tz} \Es{\epsilon}{  \sup_{ w \in \reals^d : \|w\|_1 \le B_1}\frac{1}{n}\sum_{t=1}^n \epsilon_t w^\top \tz_t(\epsilon)}\\
& \le \sup_{\tz} \Es{\epsilon}{  \sup_{ w \in \reals^d : \|w\|_1 \le B_1}\|w\|_1 \left\|\frac{1}{n} \sum_{t=1}^n \epsilon_t \tz_t(\epsilon)\right\|_\infty} \\
& \le B_1 \sup_{\tz} \Es{\epsilon}{ \max_{i\in [d]}\left\{\frac{1}{n} \sum_{t=1}^n \epsilon_t \tz_t(\epsilon)[i] \right\}}
\end{align*}
Note that the instances $x \in \Z$ are vectors in $\reals^d$ and so for a given instance tree $\tz$, for any $i \in [d]$, $\tz[i]$ given by only taking the $i^{th}$ co-ordinate is a valid real valued tree. Hence using Lemma \ref{lem:fin} we conclude that
\begin{align*}
\Radon_n(\F_1) & \le B_1 \sup_{\tz} \Es{\epsilon}{ \max_{i\in [d]}\left\{\frac{1}{n} \sum_{t=1}^n \epsilon_t \tz_t(\epsilon)[i]\right\}}\\
& \le B_1 \sqrt{\frac{2 X_\infty^2 \log d}{n}}
\end{align*}
Using the above and Equation \ref{eq:rec} we conclude the proof.
\end{proof}

\begin{proof}[\textbf{Proof of Proposition~\ref{prop:DT}}]
For a tree of depth $d$, the indicator function of a leaf is a conjunction of no more than $d$ decision functions. More specifically, if the decision tree consists of decision nodes chosen from a class $\mc{C}$ of binary-valued functions, the indicator function of leaf $l$ (which takes value $1$ at a point $x$ if $x$ reaches $l$, and $0$ otherwise) is a conjunction of $d_l$ functions from $\mc{C}$, where $d_l$ is the depth of leaf $l$. We can represent the function computed by the tree as the sign of 
$$
g(x) = \sum_{l} w_l \sigma_l \bigwedge_{i=1}^{d_l} c_{l,i}(x)
$$
where the sum is over all leaves $l$, $w_l > 0$,  $\sum_l w_l = 1$, $\sigma_l \in \{\pm 1\}$ is the label of leaf $l$, $c_{l,i} \in \mc{C}$, and the conjunction is understood to map to $\{0, 1\}$. 
Now note that if we fix some $L > 0$ then we see that the loss 
$$
\phi_L(\alpha) = \left\{\begin{array}{cl}
1 & \textrm{if }\alpha \le 0\\
1 - L\alpha & \textrm{if }0 < \alpha \le 1/L\\
0 & \textrm{otherwise}
\end{array} \right.
$$ 
is $L$-Lipschitz and so by Theorem \ref{thm:valrad} and Lemma \ref{lem:contraction} we have that for every $L > 0$, there exists a randomized strategy $\Algo^L$ for the player, such that for any sequence $z_1 = (x_1,y_1), \ldots, z_n = (x_n,y_n)$,
$$
\E{\sum_{t=1}^n \Es{\tau_t \sim \Algo^L(z_{1:t-1})}{\phi_L(y_t \tau_t(x_t)}} \le \inf_{\tau \in \mc{T}} \sum_{t=1}^n \phi_L(y_t \tau(x_t)) + L \Radon_n(\mc{T})
$$
Now note that $\phi_L$ upper bounds the step function and so
$$
\E{\sum_{t=1}^n \Es{\tau_t \sim \Algo^L(z_{1:t-1})}{\ind{\tau_t(x_t) \ne y_t}}} \le \inf_{\tau \in \mc{T}} \sum_{t=1}^n \phi_L(y_t \tau(x_t)) + L \Radon_n(\mc{T})
$$
Now say $\tau^* \in \mc{T}$ is the minimizer of $\sum_{t=1}^n \ind{\tau(x_t) \ne y_t}$ then note that 
\begin{align*}
\sum_{t=1}^n \phi_L(y_t \tau^*(x_t))  &= \sum_{t=1}^n \ind{\tau(x_t) \ne y_t} + \sum_l \tilde{C}_n(l) \phi_L(w_l)\\
& \le \sum_{t=1}^n \ind{\tau^*(x_t) \ne y_t} + \sum_l \tilde{C}_n(l) \max(0, 1 - Lw_l)\\
& \le \sum_{t=1}^n \ind{\tau^*(x_t) \ne y_t} + \sum_l  \max\left(0, (1 - Lw_l)\tilde{C}_n(l) \right)\\
& = \inf_{\tau \in \mc{T}} \sum_{t=1}^n \ind{\tau(x_t) \ne y_t} + \sum_l  \max\left(0, (1 - Lw_l)\tilde{C}_n(l) \right)
\end{align*}
Hence we see that
$$
\E{\sum_{t=1}^n \Es{\tau_t \sim \Algo^L(z_{1:t-1})}{\ind{\tau_t(x_t) \ne y_t}}} \le \inf_{\tau \in \mc{T}} \sum_{t=1}^n \ind{\tau(x_t) \ne y_t} + \sum_l  \max\left(0, (1 - Lw_l)\tilde{C}_n(l) \right)
$$
Now if we discretize over $L$ as $L_i = i$ for all $i \in \mbb{N}$ and run experts algorithm \ref{alg:experts} with output of randomized strategies, $\Algo^{L_1}, \Algo_{L_2}, \ldots$ as our experts and weight of expert $i$ with $p_i = \frac{6}{\pi^2}i^{-2}$  so that $\sum_i p_i = 1$ then we get that for this randomized strategy $\Algo$, we have from Proposition \ref{prop:experts} that for all $L \in \mbb{N}$,
\begin{align*}
&\E{\sum_{t=1}^n \Es{\tau_t \sim \Algo(z_{1:t-1})}{\ind{\tau_t(x_t) \ne y_t}}}\\ 
&\le \inf_{\tau \in \mc{T}} \sum_{t=1}^n \ind{\tau(x_t) \ne y_t} + \sum_l  \max\left(0, (1 - Lw_l)\tilde{C}_n(l) \right) + L \Radon_n(\mc{T}) + \sqrt{n} + 2 \sqrt{n} \log(L \pi / \sqrt{6})
\end{align*}
Now we pick $L = |\{l : \tilde{C}_n(l) > 2 \Radon_n(\mc{T})\}| =: N_{\textrm{leaf}}$ and also pick $w_l = 0$ if $\tilde{C}_n(l) \le 2 \Radon_n(\mc{T})$ and $w_l = 1/L$ otherwise. Hence we see that
\begin{align*}
\E{\sum_{t=1}^n \Es{\tau_t \sim \Algo(z_{1:t-1})}{\ind{\tau_t(x_t) \ne y_t}}} & \le \inf_{\tau \in \mc{T}} \sum_{t=1}^n \ind{\tau(x_t) \ne y_t} + \sum_l  \tilde{C}_n(l) \ind{\tilde{C}_n(l) \le 2 \Radon_n(\mc{T})} \\
& \hspace{5pt} +  2 \Radon_n(\mc{T}) \sum_l \ind{\tilde{C}_n(l) > 2 \Radon_n(\mc{T})} + \sqrt{n}+ 2 \sqrt{n}  \log(N_{\textrm{leaf}} \pi / \sqrt{6})\\
&\hspace{-1.5cm}= \inf_{\tau \in \mc{T}} \sum_{t=1}^n \ind{\tau(x_t) \ne y_t} + \sum_l \min(\tilde{C}_n(l), 2 \Radon_n(\mc{T})) + \sqrt{n}\left(1 + 2   \log(N_{\textrm{leaf}} \pi / \sqrt{6})\right)
\end{align*}
Now finally we can apply Corollary \ref{cor:radem_binary} to bound $\Radon_n(\mc{T}) \le d \mc{O}(\log^{3/2}n)\  \Radon_n(\mc{H})$ and thus conclude the proof by plugging this into the above.
\end{proof}

\subsection{Exponentially Weighted Average (EWA) Algorithm on Countable Experts}
We consider here a version of the exponentially weighted experts algorithm for countable (possibly infinite) number of experts and provide a bound on the expected regret of the randomized algorithm. The proof of the result closely follows the finite case (e.g. \cite[Theorem 2.2]{PLG}).

Say we are provided with countable experts $E_1, E_2 , \ldots$ where each expert can herself be thought of as a randomized/deterministic player strategy which, given history, produces an element of $\F$ at round $t$. Here we also assume that $\F \subset [0,1]^\Z$ contains only non-negative functions (corresponds to loss class). Denote by $f^i_t$ the function output by expert $i$ at round $t$ given the history. The EWA algorithm we consider needs access to the countable set of experts and also needs an initial weighting on each expert $p_1,p_2,\ldots$ such that $\sum_i p_i =1$. 

\begin{algorithm}
\caption{EWA ($E_1,E_2,\ldots$, $p_1,p_2,\ldots$)}
\label{alg:experts}
\begin{algorithmic}
\STATE Initialize each $w^1_i \gets p_i$
\FOR{$t=1$ to $n$}
\STATE Pick randomly an expert $i$ with probability $w^t_i$
\STATE Play $f_t = f^t_i$
\STATE Receive $\z_t$
\STATE Update for each $i$, $w^{t+1}_i = \frac{w^t_i e^{- \eta f^t_i(\z_t)}}{\sum_{i} w^t_i e^{- \eta f^t_i(\z_t)}}$
\ENDFOR 
\end{algorithmic}
\end{algorithm}

\begin{proposition}
	\label{prop:experts}
For the exponentially weighted average forecaster (Algorithm~\ref{alg:experts}) with $\eta = n^{-1/2}$ yields
$$
\E{\sum_{t=1}^n f_t(\z_t)} \le \sum_{t=1}^n f^t_i(\z_t) + \frac{\sqrt{n}}{8} + \sqrt{n} \log\left( 1/p_i\right)
$$
for any $i \in \mathbb{N}$.
\end{proposition}
\begin{proof}
Define $W_t = \sum_{i} p_i e^{- \eta \sum_{j=1}^t f^j_i(\z_t)}$. Then note that
$$
\log\left(\frac{W_t}{W_{t-1}}\right) = \log\left(\frac{\sum_{i} p_i e^{- \eta \sum_{j=1}^t f^j_i(\z_t)}}{W_{t-1}}\right) = \log\left(\sum_i w^{t-1}_i e^{- \eta f_i^t (\z_t)}\right)
$$
Now using Hoeffding's inequality (see \cite[Lemma 2.2]{PLG}) we have that 
$$
\log\left(\frac{W_t}{W_{t-1}}\right) \le - \eta \sum_i  w^{t-1}_i f^t_i(\z_t) + \frac{\eta^2}{8} = -\eta \E{f_t(\z_t)} + \frac{\eta^2}{8}
$$
Summing over $t$ we get
\begin{align}\label{eq:sumW}
\log(W_n) - \log(W_0) = \sum_{t=1}^n \log\left(\frac{W_t}{W_{t-1}}\right) \le -\eta \E{\sum_{t=1}^n f_t(\z_t)} + \frac{n \eta^2}{8}
\end{align}
Note that $W_0 = \sum_{i} p_i = 1$ and so $\log(W_0) = 0$.
Also note that for any $i \in \mbb{N}$, 
$$
\log(W_n) = \log\left(\sum_i p_i e^{-\eta \sum_{t=1}^n f^t_i(\z_t)} \right) \ge \log\left(p_i ^{-\eta \sum_{t=1}^n f^t_i(\z_t)} \right) = \log(p_i)  - \eta \sum_{t=1}^n f^t_i(\z_t)
$$
Hence using this with Equation \ref{eq:sumW} we see that
$$
\log(p_i)  - \eta \sum_{t=1}^n f^t_i(\z_t) \le -\eta \E{\sum_{t=1}^n f_t(\z_t)} + \frac{n \eta^2}{8}
$$
Rearranging we get
\begin{align*}
\E{\sum_{t=1}^n f_t(\z_t)} \le \sum_{t=1}^n f^t_i(\z_t) + \frac{\eta n}{8}+ \frac{1}{\eta} \log\left(\frac{1}{p_i}\right)
\end{align*}
Using $\eta = \frac{1}{\sqrt{n}}$ we get the desired bound.
\end{proof}

\begin{proof}[Proof of Proposition \ref{prop:isotron}]
First, by the classical result of Kolmogorov and Tihomirov \cite{kolmogorov1959}, the class $\G$ of all bounded Lipschitz functions has small metric entropy: $\log\Nhat_\infty(\alpha, \G) = \Theta(1/\alpha)$. For the particular class of non-decreasing $1$-Lipschitz functions, it is trivial to verify that the entropy is in fact bounded by $2/\alpha$. 

Next, consider the class $\F = \{ \inner{w, x} \ | \ \|w\|_2\leq 1 \}$ over the Euclidean ball. By Proposition~\ref{prop:rad_linear_functions}, $\Radon_n(\F) \leq \sqrt{\frac{2}{n}}$. Using the lower bound of Proposition~\ref{prop:uplow}, $\faton_\alpha \leq 64/\alpha^2$ whenever $\alpha > 8/\sqrt{n}$. This implies that $\Nonl_\infty(\alpha,\F,n)\leq (2en/\alpha)^{64/\alpha^2}$ whenever $\alpha > 8/\sqrt{n}$. Note that this bound does not depend on the ambient dimension of $\Z$.

Next, we show that a composition of $\G$ with any small class $\F\subset [-1,1]^\Z$ also has a small cover. To this end, suppose $\Nonl_\infty (\alpha, \F, n)$ is the covering number for $\F$. Fix a particular tree $\tz$ and let $V=\{\v_1,\ldots, \v_N\}$ be an $\ell_\infty$ cover of $\F$ on $\tz$ at scale $\alpha$. Analogously, let $W=\{g_1,\ldots,g_M\}$ be an $\ell_\infty$ cover of $\G$ with $M = \Nhat_\infty(\alpha, \G)$. Consider the class $\G\circ \F = \{g\circ f: g\in \G, f\in \F\}$. The claim is that $\{g(\v): \v\in V, g\in W\}$ provides an $\ell_\infty$ cover for $\G\circ\F$ on $\tz$. Fix any $f\in\F, g\in \G$ and $\epsilon\in\{\pm1\}^n$. Let $\v\in V$ be such that $\max_{t\in[n]} |f(\tz_t(\epsilon))-\v_t(\epsilon)| \leq \alpha$, and let $g'\in W$ be such that $\|g-g'\|_\infty\leq\alpha$. Then, using the fact that functions in $\G$ are $1$-Lipschitz, for any $t\in [n]$, 
$$|g(f(\tz_t(\epsilon))) - g'(\v_t(\epsilon))| \leq |g(f(\tz_t(\epsilon))) - g'(f(\tz_t(\epsilon))| + |g'(f(\tz_t(\epsilon)) - g'(\v_t(\epsilon))| \leq 2\alpha \ .$$
Hence, $\Nonl_\infty(2\alpha, \G\circ\F, n) \leq   \Nhat_\infty (\alpha, \G) \times \Nonl_\infty (\alpha, \F, n)$. 

Finally, we put all the pieces together. By Lemma~\ref{lem:contraction}, the Sequential Rademacher complexity of $\cH$ is bounded by 4 times the Sequential Rademacher complexity of the class
$$\G\circ\F = \{ u(\inner{w, x}) \ | \ u:[-1,1]\mapsto [-1,1] \mbox{ is $1$-Lipschitz }, \ \|w\|_2\leq 1 \} $$
since the squared loss is $4$-Lipschitz on the space of possible values. The latter complexity is then bounded by
$$\Dudleyon_n(\G\circ\F) \leq \frac{32}{\sqrt{n}} + 12\int_{8/\sqrt{n}}^{1} \sqrt{\frac{ \log \ \Nonl(\delta, \G\circ\F, n) }{n}} d \delta \leq \frac{32}{\sqrt{n}}+ 12 \int_{8/\sqrt{n}}^1 \sqrt{\frac{2}{n \delta} + \frac{64}{n \delta^2}\log(2en)} d\delta \ .$$
We conclude that the value of the game $\Val_n(\cH,\Z\times\Y) = O(\sqrt{\frac{\log^{3} n}{n}})$.
\end{proof}


\section{Discussion}
While in this chapter we introduced tools for analyzing rates for online learning problems analogous to the various complexity measures for statistical learning framework, as we saw in the previous chapter, for statistical learning framework these tools fail to characterize learnability in general. Similar situation is true for the online setting too. While these tools characterize learnability of online supervised learning problems and can also be used to obtain rates for online convex learning problems indirectly, in general they could fail to characterize learnability general of online learning problems. 

In the previous chapter we then turned to the notion of online stability to characterize learnability and even provided a generic learning algorithm. Is there some notion of stability that can be used to characterize learnability  in the online learning framework? Can we provide a generic algorithm for general learning problems in the online framework?

Another interesting avenue to explore is the question of fast rates for online learning problems. In the statistical learning framework the notion of Localized Rademacher complexity introduced in \cite{BartlettBoMe05} can often be used to obtain fast rates. Just like we provided an analog to Rademacher complexity for online learning, can we provide an analog of localized complexity measures, specifically a local sequential Rademacher complexity that can then be used to obtain fast rates for online learnign problems?


%% file: convex.tex

\chapter{Convex Learning and Optimization Problem Setup}\label{chp:cvx}
In the first part of this dissertation we mainly focused on the question of learnability (and learning rates) in both statistical and online settings, that is whether the problems were at all learnable using some algorithm. We did not take into consideration tractability of the learning rules we considered and whether the problem is learnable using some efficient learning algorithm. The generic learning rules/algorithms presented in the first part are not at all tractable. In this part of the dissertation, we try to address the issue of tractability of learning algorithms in both statistical and online learning settings. To do this, we restrict ourselves to so called convex learning or optimization problems. In this chapter we introduce the convex learning problems we will encounter in the second part of this dissertation and associated notations.

\section{Convex Problems}\label{sec:cvxprob}
Let us now give the basic setup for the convex learning and optimization problems we consider in the second part of this dissertation. Of course when we say convex problem, we mean that the set of target hypothesis $\H$ is a convex set and for each given instance $\z \in \Z$, the loss function $\ell(\h,\z)$ is convex in $\h$. To describe more formal, we consider an arbitrary real vector space $\B$ and denote its dual by $\Bd$. Now the target hypothesis class $\cH \subset \B$ we shall consider throughout will be a {\bf convex and centrally symmetric} subset of $\B$. Also consider the set $\X \subset \Bd$ to be a {\bf bounded, convex and centrally symmetric} subset of the dual $\Bd$. The role of set $\X$ will become clear in the following paragraph. It will be convenient for us, to relate the notion of a convex centrally symmetric sets to their corresponding (semi)norms. To do this, recall the definition of \index{Minkowski functional}  the Minkowski functional of a set $\mc{K}$ of a real vector space $\B$. It is defined as 
$$\norm{\v}_{\mc{K}} := \inf\left\{ \alpha > 0 : \v \in \alpha \mc{K} \right\}$$  
Now it can be seen that if $\mc{K}$ is convex and centrally symmetric (i.e. $\mc{K} = -\mc{K}$), then $\norm{\cdot}_{\mc{K}}$ is a semi-norm. Further, for instance in $\reals^d$, if the set $\mc{K}$ is bounded then $\norm{\cdot}_{\mc{K}}$ is in fact a norm. Our assumption on the sets $\H$ and $\X$ ensure that $\norm{\cdot}_{\H}$ and $\norm{\cdot}_{\X}$ (the Minkowski functionals of the sets $\H$ and $\X$) are semi-norms. For simplicity we shall further assume that $\norm{\cdot}_\H$ and $\norm{\cdot}_\X$ are in fact norms. Even though we do
this for simplicity, we remark that all our results go through for
semi-norms too.  We use $\Xd$ and $\Hd$ to represent the duals of
$\X$ and $\H$ respectively, i.e.~the unit balls of the dual norms
$\norm{\cdot}^*_{\X}$ and $\norm{\cdot}^*_{\H}$. 

As mentioned we consider convex learning and optimization problem where target set $\H$ is the unit ball of norm $\norm{\cdot}_{\H}$ and for each instance $\z \in \Z$ the loss $\ell(\cdot,\z)$ is convex. Now the convex problems we consider are of three flavors. The first case we consider is the one where, for each instance $\z \in \Z$, the sub-gradients of the convex function $\ell(\cdot,\z)$ belong to the set $\X$. Notice that when $\X = \Hd$ this case exactly corresponds to convex $1$-lipschitz problems.  Most prior work on online learning, statistical learning and convex optimization problems considers this case when $\H$ is the unit ball of some Banach space, and $\X$ is the unit ball of the dual space.  However, we analyze the general problem where $\X \in \Bd$ is not necessarily the dual ball of $\H$ . The second flavor of problems we consider are the ones where for each $\z \in \Z$, the function $\ell(\cdot,\z)$ is in fact uniformly convex w.r.t. norm $\norm{\cdot}_{\Xd}$. The final flavor of problems we consider are problems where for each $\z \in \Z$, the function $\ell(\cdot,\z)$ are non-negative, convex and smooth w.r.t. norm $\norm{\cdot}_{\X}$. The next section formally defines the key set of convex function classes we consider for the convex learning and optimization problems throughout the second part.

\section{Various Convex Learning/Optimization Problems}

Below we provide examples of various convex learning/optimization problems we consider in this work.

\begin{example}[Lipschitz Convex functions]
The set $\Z = \Z_{\mrm{Lip}}(\X)$ correspond to the set of all convex loss functions such that for any $\z \in \Z$ and $\h \in \cH$, $\nabla_{\h} \ell(\h,\z) \in \X$ where $\X$ is some set in the dual vector space. In short
$$
\Z_{\mrm{Lip}}(\X) = \{\z :  \ell(\cdot,\z) \textrm{ is convex and }\forall \h \in \cH, \nabla \ell(\h,\z) \in \X\}
$$
\end{example}


\begin{example}[Linear functions]
The set $\Z = \Z_{\mrm{lin}}$ consists of linear functions on $\bcH$ from the set $\X$.
$$
\Z_{\mrm{lin}}(\X) = \{\z : \ell(\cdot,\z) = \ip{\x}{\cdot} \textrm{ where } \x \in \X\}
$$
\end{example}

\begin{example}[Supervised Learning with Linear predictors]
The set $\Z = \Z_{\mrm{sup}}(\X)$ consists of functions of form
$$
\Z_{\mrm{sup}}(\X) = \{\z : \ell(\h,\z) = |\inner{\h,\x} - y| \textrm{ where } \x \in \X, y \in [-b,b]\}
$$
\end{example}

In the first part we introduced generic supervised learning problem with absolute loss and arbitrary function class for prediction that mapped input $x$ to reals. At first glance the above class might seem specific given that the predictor is always linear. However at second glance, if we consider supervised learning with absolute loss and we require loss function for every instance to be convex in $\h$, then necessarily the predictor has to be linear. To see this note that we need $|h(\x) - y|$ to be convex in $\h$, when both $y = 1$ and $y = -1$ which basically means predictor has to be both convex and concave in $\h$ and so is linear. The same argument can be extended to other common margin losses like squared loss, logistic loss etc. Given a convex loss function $\phi$ one can also more generally define a class $\Z_{\phi} =
\left\{\z : \ell(\h,\z)  = \phi(\inner{\x,\h},y) : \x \in \X, y \in [-b,b],  \right\}$ for any 1-Lipschitz loss function $\phi : \reals \times \reals \mapsto \reals$, and this class would also be a subset of $\Z_{\mrm{Lip}}$.  In fact, this setting includes supervised learning fairly generally, including problems such as multitask learning and matrix completion, where in all cases $\X$
specifies the data domain\footnote{Note that any convex supervised
 learning problem can necessarily be viewed as linear classification
 with some convex constraint $\H$ on the predictors.}.

\begin{example}[Non-Negative Smooth Convex Loss]
The set $\Z$ corresponds  to non-negative convex functions that are smooth w.r.t. to the norm $\norm{\cdot}_{\Xd}$, that is 
$$
\Z_{\mrm{smt}(H)}(\X) = \left\{\z :  \ell(\cdot,\z) \ge 0 \textrm{ is convex and }\forall\, \h, \h' \in \bcH, \norm{\nabla \ell(\h,\z) - \nabla \ell(\h',\z)}_{\X} \le H \norm{\h - \h'}_{\Xd}\right\}
$$
\end{example}

A subset of the above instance class are cases of supervised learning problems $\Z_{\mrm{\phi}}$ where $\phi$ is non-negative and smooth function on the reals like logistic loss, smoothed hinge loss and squared loss.  In the chapters to come we show how for non-negative smooth convex instances one can get faster learning rates in both online and statistical convex learning frameworks when optimal loss itself is small.

\begin{example}[Regularized Convex Loss]
The set $\Z$ consists of functions of form
$$
\Z_{\mrm{reg}}(\X) = \{\z : \ell(\cdot,\z) = \phi(\cdot,\z) + R(\h) \textrm{ where }  \phi(\cdot,\z) \textrm{ is convex, }\forall \h \in \cH, \nabla \phi(\h,\z) \in \X \textrm{ and }R\textrm{ is convex}\}
$$
\end{example}

The class $\Z_{\mrm{reg}}(\X)$ captures regularized convex objective classes where $R : \bcH \mapsto \reals$ is a convex regularizer used to enforce structure or prior into the learning problem. Commonly regularizer chosen are strongly convex or more generally uniformly convex. The following instance class captures more specifically these classes.

\begin{example}[Uniformly Convex Loss]
The set $\Z$ corresponds  to uniformly convex functions,
$$
\Z_{\mrm{ucvx}(\sigma,q)}(\X) = \left\{\z :  \ell(\cdot,\z) = \phi(\cdot,\z) + R(\cdot) \textrm{ is }(\sigma,q)\textrm{-uniformly convex w.r.t. }\norm{\cdot}_{\Xd}, \forall\z \in\Z\,  \nabla \phi(\cdot,\z)  \in \X \right\}
$$
\end{example}

In the above example, we assume that $\ell$ is uniformly convex but however we assumes that $\ell$ can be decomposed as $\ell(\cdot,\z) = \phi(\cdot,\z) + R(\cdot) $ and only assumes that $\nabla \phi(\cdot,\z) \in \X$.  The reason we did not directly assume in the above that $\nabla \ell$ is itself in $\X$ is so that we can capture many regularized learning problems where the regularizer $R$ is not Lipschitz and so $\ell$ is not Lipschitz but however the loss function of interest $\phi$ in these examples are Lipschitz.

Another class of important loss functions are non-negative smooth convex loss functions. Linear predictors with logistic loss ,squared loss are common examples of such problems. The following example captures such classes.

\begin{example}[Bounded Convex functions]
The set $\Z = \Z_{\mrm{bnd}}$ corresponds ot the set of all convex loss functions that are that are bounded by $b$ on $\cH$, that is
$$
\Z_{\mrm{bnd}} = \{\z :  \ell(\cdot,\z) \textrm{ is convex and bounded by }b \textrm{ on }\cH\}
$$
\end{example}

\begin{remark}\label{rem:detcnvx}
It must be noted that for convex optimization problem, both in the online and statistical setting, we can use Jensen's inequality to show that for every randomized algorithm there exists a deterministic algorithm (that play the expected action of the randomized algorithm) that achieves learning rates that are at most as bad as that of the randomized algorithm. Hence for the convex optimization problem it suffices to only consider deterministic learning algorithms.
\end{remark}

Owing to the above remark, we see that since we only need to consider deterministic learning algorithms, both in statistical and online convex learning settings, any learning algorithm $\Algo$ is specified by a mapping $\Algo : \bigcup_{n \in \mathbb{N}} \Z^{n-1} \mapsto \bcH$. Also note that while we assumed that the set $\H$ was convex and centrally symmetric and for ease even assumed that $\H$ is the unit ball of norm $\norm{\cdot}_{\H}$, no such assumptions are made on hypothesis set $\bcH$ other than that it contains $\cH$ and so it could even be all of $\B$.

 \section{Discussion}
Notice that the sets $\H$ and $\X$ that we consider are arbitrary convex centrally symmetric sets and need not be related to each other a priori. While the special case of when $\H = \Xd$ is what is usually encountered in majority of the theoretical analysis existing literature, in many applications $\H$ and $\X$ are not dual to each other. Here we provide a generic theoretical analysis of the non-dual case. Note that when $\H = \Xd$, $\Z_{\mrm{Lip}}(\X)$ corresponds to usual convex Lipschitz problem, $ \Z_{\mrm{smt}}(\X)$ corresponds to usual smooth convex class and finally  $\Z_{\mrm{ucvx}(\sigma,2)}(\X)$ corresponds to $\sigma$-strongly convex functions. Overall, the convex learning and optimization problems that we introduced in this chapter and will study in the remaining chapters cover majority of the problems considered in previous works. Perhaps the one case not covered in this thesis is the case of exp-concave loss functions for which in finite (low) dimensional cases, one can get faster learning rates in both online and statistical learning cases.

%
%
%

%% file: md.tex
\chapter{Mirror Descent Methods}\label{chp:md}
Perhaps one of the most popular convex optimization algorithm most readers would be familiar with is the gradient descent algorithm. The gradient descent algorithm proceeds by starting with an initial point and iteratively updating it by taking steps in the direction of the negative gradient of the function to optimize at the current point. The gradient descent method is a natural algorithm for problems in Euclidean space. The mirror descent algorithm \cite{NemirovskiYu78} is a natural generalization of gradient descent method for general convex learning problems. Section \ref{sec:md} describes the basic update step of the mirror descent algorithm. Section \ref{sec:mod} provides bounds on regret of mirror descent method for generic online convex learning problems, online smooth convex learning problems and online uniformly convex learning problems. The section \ref{sec:smd} which follows shows how mirror descent algorithm can be used for statistical convex learning problem and associated learning guarantee. Section \ref{sec:offmd} shows how mirror descent can also be used for offline convex optimization problems. Following that Section \ref{sec:mdproof} provides proofs of all the results of this chapter and finally we conclude this chapter with some discussion in Section \ref{sec:mddiscuss}.

\section{The Mirror Descent Update}\label{sec:md}

Given a strictly convex function $\Psi : \B \mapsto \reals$, the Mirror Descent algorithm, $\Algo_{\mathrm{MD}}$ is given by the update
\begin{align} \label{eq:update1}
& \h_{t+1} = \argmin{\h \in \bcH} \breg{\Psi}{\h}{\h_{t}} + \eta \inner{\nabla \ell(\h_t,\z_t), \h - \h_t}\\
\textrm{or equivalently~~~~~~} & \h'_{t+1} = \nabla \Psi^*\left( \nabla \Psi(\h_t) - \eta \nabla \ell(\h_t,\z_t) \right),~~ \h_{t+1} = \argmin{\h \in \bcH} \breg{\Psi}{\h}{\h'_{t+1}}
\end{align}
where $\breg{\Psi}{\h}{\h'} := \Psi(\h) - \Psi(\h') -
\inner{\nabla \Psi(\h'), \h - \h'}$ is the Bregman divergence and $\Psi^*$ is
the convex conjugate of $\Psi$. As an example notice that when $\Psi(\h) = \frac{1}{2}\norm{\h}_2^2$ then we get back the online gradient descent algorithm. It is worth noting that the perceptron algorithm can be viewed as a conservative variant of online gradient descent with hinge loss function. Also when $\H$ is the $d$ dimensional simplex and $\Psi(\h) = \sum_{i=1}^d \h_i \log(1/\h_i)$, then we get the multiplicative weights update algorithm. In general the function $\Psi$ used in mirror descent is often referred to as the proxy-function. \index{mirror descent!proxy function} 

The mirror descent algorithm is an $O(1)$, memory single pass, first order method that only needs some sub-gradient for each update. Often times in practice, each mirror descent update step has time complexity same as that of calculating a single gradient and hence overall runtime is linear in number of rounds (for online learning case) or number of samples (statistical learning case).This makes the mirror descent algorithm attractive from a computational viewpoint.

Before we proceed we would like to point out that in the setting we consider, the hypothesis set from which learner is allowed to pick, $\bcH \subset \B$, need not be the same as the target hypothesis set $\cH$ and only needs to be a superset. Of course when $\bcH = \cH$ then the update above corresponds to the usual mirror descent update. However when $\bcH$ is all of $\B$ then notice that the projection step is mute and essentially the update becomes,
$$
\h_{t+1} = \nabla \Psi^*\left( \nabla \Psi(\h_t) - \eta \nabla \ell(\h_t,\z_t) \right)~.
$$
This is especially attractive because in many machine learning applications while we would like to do as well as the best hypothesis from some target class, we don't really care if the hypothesis learner picks itself is selected from this target hypothesis set as long as it gives good results. Hence if we set up the problem such that $\bcH = \B$ then we avoid extra computational time on projection step which could at times be expensive.

%

\section{Online Mirror Descent}\index{mirror descent} \label{sec:omd}

In this section we describe the online mirror descent algorithm $\Algomd : \bigcup_{n \in \mathbb{N}} \Z^{n-1} \mapsto \bcH$ which simply uses the mirror descent update given in the previous section and returns the $h_t$'s in each round. That is, $\Algomd(\{\}) = \h_1$
and further for any $t \in \mathbb{N}$ and any $\z_1,\ldots,\z_{t-1} \in \Z$,
$$
\Algomd(\z_1,\ldots,\z_{t-1}) = h_t~.
$$

A key tool in the analysis mirror descent is the notion of strong
convexity or more generally uniform convexity of the function $\Psi$. Recall the definition of uniform convexity :

 \begin{definition}\index{uniform convexity}
  A function $\Psi:\B \rightarrow \reals$ is said to be $q$-uniformly convex w.r.t. $\|\cdot\|$ if for any $\h,\h' \in \B$:  
$$
 \forall_{\alpha\in[0,1]} \;\; \Psi\left(\alpha \h+ (1 - \alpha)\h'\right) \le \alpha \Psi(\h) + (1 - \alpha) \Psi(\h')  - \tfrac{\alpha (1 - \alpha)}{q} \norm{\h - \h'}^q
$$
\end{definition}

%
%

We are interested in bounding the regret of the of the mirror descent algorithm given by,
\begin{align*}
\Reg_n(\Algomd,\z_1,\ldots,\z_n) &:= \frac{1}{n} \sum_{t=1}^n \ell(\Algomd(\z_{1:{t-1}}),\z_t)  - \inf_{\h \in \cH} \frac{1}{n} \sum_{t=1}^n \ell(\h,\z_t) \\
& = \frac{1}{n} \sum_{t=1}^n \ell(\h_t,\z_t)  - \inf_{\h \in \cH} \frac{1}{n} \sum_{t=1}^n \ell(\h,\z_t) \ .
\end{align*}
Now assuming we can find an appropriate $q$-uniformly convex function on $\B$, below we provide bounds on regret of mirror descent method for generic online convex learning problems, non-negative smooth convex learning problems and uniformly convex learning problems. \\

{\bf Convex Losses with Sub-gradients in $\X$ :}
We first start with the case when the convex costs at each round are such that their sub-gradients lie in the set $\X$. Note that instance sets $\Z_{\mrm{Lip}}(\X)$, $\Z_{\mrm{supp}}(\X)$ and $\Z_{\mrm{lin}}(\X)$ are examples of instance classes that fall in this set.
\begin{lemma}\label{lem:md}
  Let $\Psi : \B \mapsto \reals$ be non-negative and $q$-uniformly
  convex w.r.t. norm $\norm{\cdot}_{\X^*}$. For the Mirror Descent
  algorithm with this $\Psi$, using $\h_1 = \argmin{\h \in \H}
  \Psi(\h)$ and $\eta = \left(\tfrac{\sup_{\h \in \H}\Psi(\h)}{ n
      B}\right)^{1/p} $ we can guarantee that for any $\z_{1}, \ldots,
  \z_n$ s.t. $\frac{1}{n} \sum_{t=1}^n \norm{\nabla \ell(\cdot,\z_t)}_{\X}^p \le
  1~~~$ (where $p = \tfrac{q}{q-1}$),
$$
\Reg(\Algomd,\z_1,\ldots,\z_n) \le 2 \left( \frac{\sup_{\h \in \H} \Psi(\h)}{n} \right)^{\frac{1}{q} } \ .
$$
\end{lemma}
Note that in our case we have that for each $z \in \z$, $\nabla \ell(\cdot,\z) \in\X$, i.e.~$\norm{\nabla \ell(\cdot,\z)}_{\X} \leq 1$, and so certainly $\frac{1}{n} \sum_{t=1}^n
\norm{\nabla \ell(\cdot , \z_t}_{\X}^p \le 1$. \\

{\bf Non-Negative Smooth Convex Losses :} 
Next we deal with the case when the convex costs in each round need not be such that their sub gradients are from set $\X$ but rather are such that the costs are non-negative and $H$-smooth w.r.t. to the norm $\norm{\cdot}_{\X}$. That is :
$$
\forall \z \in \Z , \forall \h,\h' \in \bcH,  \norm{\nabla \ell(\h,\z) - \nabla \ell(\h',\z)}_{\X} \le H \norm{\h - \h'}_{\Xd}~.
$$
For non-negative smooth losses, one has a property called self-bounding property that for any $\h$ and any $\z$,
$$
\norm{\nabla \ell(\h,\z)}_{\X} \le \sqrt{4 H \ell(\h,\z)}
$$
In \cite{SreSriTew10a} we had made the observation that any non-negative smooth convex loss satisfies this above self bounding property. Shalev-Shwartz \cite{Shalev07}  showed that the self bounding property can be used this to provide optimistic rates on regret of mirror gradient descent for the dual case and using strongly convex function. By optimistic rates we refer to rates that improve when average loss of best hypothesis is small. The following lemma provides regret bounds for mirror descent algorithm with optimistic rates using similar lines of reasoning as in\cite{SreSriTew10a,Shalev07}.

\begin{lemma}\label{lem:mdsmooth}
  Let $\Psi : \B \mapsto \reals$ be non-negative and $q$-uniformly
  convex w.r.t. norm $\norm{\cdot}_{\X^*}$. For any $\overline{L^*} \ge 0$, using the Mirror Descent
  algorithm with function $\Psi$ for making updates, using $\h_1 = \argmin{\h \in \H}  \Psi(\h)$ and 
 \begin{align*}
\eta = \left\{ \begin{array}{ll}
\left(\frac{p \sup_{\h \in \H} \Psi(\h)}{n}\right)^{1/p} \frac{1}{\sqrt{4 H \overline{L^*}}} & \textrm{if } \overline{L^*} \ge  \frac{16 H}{p^{2/p}} \left(\frac{\sup_{\h \in \H} \Psi(\h)}{n} \right)^{2/q} \vspace{0.04in}\\
\left(\frac{p}{2} \right)^{\frac{p}{2}} \frac{1}{4 H} \left(\frac{\sup_{\h \in \H} \Psi(\h)}{n} \right)^{\frac{2 - p}{p}} & \textrm{otherwise}
\end{array}\right.
\end{align*}
we can guarantee that for any sequence $\z_1,\ldots,\z_n \in \Z$ and any $\hopt \in \H$ such that $\frac{1}{n} \sum_{t=1}^n \ell(\hopt,\z_t) \le \overline{L^*}$, 
$$
\Reg(\Algomd,\z_1,\ldots,\z_n) \le \sqrt{64 H \overline{L^*}} \ \left(\frac{\sup_{\h \in \H} \Psi(\h)}{n}\right)^{1/q}  + 40 H \left(\frac{\sup_{\h \in \H} \Psi(\h)}{n}\right)^{2/q}  \ .
$$
\end{lemma}

{\bf Uniformly Convex Losses : }
We now consider the case when loss functions are uniformly convex, specifically $(\sigma,q')$-uniformly convex. Up to now in this chapter we considered $\bcH$ to be any superset of $\cH$ including $\bcH = \cH$. For the case of uniformly convex loss instance class $\Z_{\mrm{ucvx}(\sigma,q)}(\X)$ alone we will assume that $\bcH = \B$ so that we don't need a projection step and $\h'_{t} = \h_t$ in the mirror descent update. Under this setting we have the following upper bound for regret of mirror descent algorithm.

\begin{lemma}\label{lem:mducvx}
 Let $\Psi : \B \mapsto \reals$ be non-negative and $q$-uniformly
  convex w.r.t. norm $\norm{\cdot}_{\X^*}$. Let $\psi : \B \mapsto \reals$  be non-negative and $q'$ uniformly convex w.r.t. norm $\norm{\cdot}_{\X^*}$. For each $t \in [n]$, define 
$$
\tilde{\Psi}_t(\cdot) = \frac{1}{\eta} \Psi(\cdot) + R(\cdot)  + \sigma\, t\ \psi(\cdot)
$$
Then for the choice of
$$
\eta = \left\{ \begin{array}{ll}
 \left(\frac{\sup_{\h\in \H}\Psi(\h)}{n}\right)^{1/p} & \textrm{if } n \ge \left((2 - p') \sigma^{p'-1} \sup_{\h \in \cH}\Psi^{1/q}(\h)\right)^{\frac{1}{2 - p' - 1/p}}  \\
\infty & \textrm{otherwise}
\end{array}\right.
$$
we can guarantee that for any sequence $\z_1,\ldots,\z_n \in \Z_{\mrm{ucvx}(\sigma,q')}$, if $q' > 2$ : 
$$
\Reg(\Algomd,\z_1,\ldots,\z_n) \le  \min\left\{\frac{2 \left(\sup_{\h \in \cH} \Psi(\hopt)\right)^{1/q}}{n^{1/q}} , \frac{2}{(2 - p') \sigma^{p'-1} n^{p' - 1}} \right\}  + \frac{\sup_{\h \in \H} R(\h)}{n} ~.
$$
where $p' = \frac{q'}{q' - 1}$ and for $q' = 2$ :
$~~~
\Reg(\Algomd,\z_1,\ldots,\z_n) \le  \frac{2 \log n}{\sigma n}   + \frac{\sup_{\h \in \H} R(\h)}{n} ~.
$
\end{lemma}

{\bf Choice of $\Psi$ : Construction I }
The Mirror Descent bound suggests that as long as we can find an appropriate function $\Psi$ that is  uniformly convex w.r.t. $\norm{\cdot}_\X^*$ we can get a diminishing regret guarantee using Mirror Descent. This 
suggests constructing the following function:
\begin{align}\label{eq:construction1}
\tilde{\Psi}_q := \argmin{\substack{\psi : \psi \textrm{ is }
  q\textrm{-uniformly convex}\\ \textrm{w.r.t. } \norm{\cdot}_{\X^*}\textrm{ on }\H \textrm{ and }\psi \ge 0 }} \sup_{\h \in \H} \Psi(\h) \ .
\end{align}

If no $q$-uniformly convex function exists then $\tilde{\Psi}_q =
\infty$ is assumed by default. The above function is in a sense the
best choice for the Mirror Descent bound in \eqref{lem:md}.  The
question then is: when can we find such appropriate functions and what
is the best rate we can guarantee using Mirror Descent?

\section{Stochastic Mirror Descent }\index{stochastic mirror descent} \index{online to batch} \label{sec:smd}

In the previous section we saw that mirror descent can be used successfully for online convex learning. In general, especially for convex problems, any online method can be converted into an algorithm for statistical learning with same learning rate guarantee and this is often referred to as online to batch conversion. Refer to \cite{CesaBianchiCoGe04} for more details about online to batch conversion. Hence one can also use the mirror descent method for statistical convex learning problems. This algorithm is often referred to as stochastic mirror descent algorithm. The stochastic mirror descent algorithm is given as follows : 
$$
\overline{\Algomd}(\{\}) = \h_1,~~~~~~~~~~~~~~~~~~ \forall t \in [n],~~~ \overline{\Algomd}(\z_1,\ldots,\z_t) = \frac{1}{t} \sum_{i=1}^t \h_i~.
$$

\begin{proposition}\label{prop:o2b}
For any hypothesis set $\bcH \subset \B$,  target hypothesis class $\H \subset \B$ and convex learning problem specified by instance space $\Z$ and any fixed distribution $\D$ on instance space $\Z$ :  
$$
 \Es{S \sim \D^n}{\L_\D\left(\overline{\Algomd}(\z_{1:n})\right)  - \inf_{\h \in \H} \L_\D(\h)} \le  \Es{S \sim \D^n}{ \Reg_n(\Algomd,\z_1,\ldots,\z_n)}
$$
\end{proposition}
\begin{proof}
Note that 
\begin{align*}
\Es{S \sim \D^n}{ \Reg_n(\Algomd,\z_1,\ldots,\z_n)} & = \Es{S \sim \D^n}{\frac{1}{n} \sum_{t=1}^n \ell(\Algomd(\z_{1:{t-1}}),\z_t)  - \inf_{\h \in \cH} \frac{1}{n} \sum_{t=1}^n \ell(\h,\z_t)}\\
& =  \frac{1}{n} \sum_{t=1}^n \Es{S \sim \D^n}{\ell(\Algomd(\z_{1:{t-1}}),\z_t)}  - \inf_{\h \in \cH} \L_\D(\h)\\
& =  \frac{1}{n} \sum_{t=1}^n \Es{S \sim \D^n}{ \Es{z_t \sim \D}{\ell(\Algomd(\z_{1:{t-1}}),\z_t)}}  - \inf_{\h \in \cH} \L_\D(\h)\\
& =  \frac{1}{n} \sum_{t=1}^n \Es{S \sim \D^n}{\L_\D(\Algomd(\z_{1:{t-1}}))}  - \inf_{\h \in \cH} \L_\D(\h)\\
& =  \Es{S \sim \D^n}{ \frac{1}{n} \sum_{t=1}^n \L_\D(\Algomd(\z_{1:{t-1}}))  - \inf_{\h \in \cH} \L_\D(\h)}\\
\intertext{}
& \ge \Es{S \sim \D^n}{  \L_\D\left(\frac{1}{n} \sum_{t=1}^n \Algomd(\z_{1:{t-1}})\right)  - \inf_{\h \in \cH} \L_\D(\h)}\\
& = \Es{S \sim \D^n}{  \L_\D\left(\overline{\Algomd}(\z_{1:n})\right)  - \inf_{\h \in \cH} \L_\D(\h)}
\end{align*}
where the inequality step above is due to Jensen's inequality. Thus we get the statement of the proposition.
\end{proof}

Owing to the above proposition we see that one can get the same learning guarantees for statistical convex learning problems as for the online counterpart. Specifically we get the following lemmas for statistical convex learning problems.

\begin{lemma}\label{lem:smd}
  Let $\Psi : \B \mapsto \reals$ be non-negative and $q$-uniformly
  convex w.r.t. norm $\norm{\cdot}_{\X^*}$. For the Mirror Descent
  algorithm with this $\Psi$, using $\h_1 = \argmin{\h \in \H}
  \Psi(\h)$ and $\eta = \left(\tfrac{\sup_{\h \in \H}\Psi(\h)}{ n
      B}\right)^{1/p} $ we can guarantee that for any  distribution $\D$ over $\Z$ s.t. $\Es{\z \sim \D}{\norm{\nabla \ell(\cdot,\z)}_{\X}^p} \le  1~~~$  we have that :
$$
\Reg(\Algomd,\z_1,\ldots,\z_n) \le 2 \left( \frac{\sup_{\h \in \H} \Psi(\h)}{n} \right)^{\frac{1}{q} } \ . ~~~~~ (\textrm{where} p = \tfrac{q}{q-1})
$$
\end{lemma}
\begin{proof}
The statement follows by using Proposition \ref{prop:o2b} along with line of proof in Lemma \ref{lem:md}.
\end{proof}

Similar bound can be given for stochastic convex learning of non-negative smooth convex losses as follows :

\begin{lemma}\label{lem:smdsmooth}
  Let $\Psi : \B \mapsto \reals$ be non-negative and $q$-uniformly
  convex w.r.t. norm $\norm{\cdot}_{\X^*}$. For any $L^* \ge 0$, using the stochastic Mirror Descent algorithm with function $\Psi$ for making updates, using $\h_1 = \argmin{\h \in \H}  \Psi(\h)$ and 
 \begin{align*}
\eta = \left\{ \begin{array}{ll}
\left(\frac{p \sup_{\h \in \H} \Psi(\h)}{n}\right)^{1/p} \frac{1}{\sqrt{4 H L^*}} & \textrm{if } L^* \ge  \frac{16 H}{p^{2/p}} \left(\frac{\sup_{\h \in \H} \Psi(\h)}{n} \right)^{2/q} \vspace{0.04in}\\
\left(\frac{p}{2} \right)^{\frac{p}{2}} \frac{1}{4 H} \left(\frac{\sup_{\h \in \H} \Psi(\h)}{n} \right)^{\frac{2 - p}{p}} & \textrm{otherwise}
\end{array}\right.
\end{align*}
we can guarantee that for any distribution $\D$ over $\Z$ such that $\inf_{\h \in \H} \L(\h) \le L^*$, 
$$
L(\overline{\Algomd}) - \inf_{\h \in \H} L(\h) \le \sqrt{64 H L^*} \ \left(\frac{\sup_{\h \in \H} \Psi(\h)}{n}\right)^{1/q}  + 40 H \left(\frac{\sup_{\h \in \H} \Psi(\h)}{n}\right)^{2/q}  \ .
$$
\end{lemma}
\begin{proof}
The statement follows by using Proposition \ref{prop:o2b} along with line of proof in Lemma \ref{lem:mdsmooth}.
\end{proof}

Finally bounds for stochastic learning of uniformly convex losses can also be given. 

\begin{lemma}\label{lem:smducvx}
 Let $\Psi : \B \mapsto \reals$ be non-negative and $q$-uniformly
  convex w.r.t. norm $\norm{\cdot}_{\X^*}$. Let $\psi : \B \mapsto \reals$  be non-negative and $q'$ uniformly convex w.r.t. norm $\norm{\cdot}_{\X^*}$. For each $t \in [n]$, define 
$$
\tilde{\Psi}_t(\cdot) = \frac{1}{\eta} \Psi(\cdot) + R(\cdot)  + \sigma\, t\ \psi(\cdot)
$$
Then for the choice of
$$
\eta = \left\{ \begin{array}{ll}
 \left(\frac{\sup_{\h\in \H}\Psi(\h)}{n}\right)^{1/p} & \textrm{if } n \ge \left((2 - p') \sigma^{p'-1} \sup_{\h \in \cH}\Psi^{1/q}(\h)\right)^{\frac{1}{2 - p' - 1/p}}  \\
\infty & \textrm{otherwise}
\end{array}\right.
$$
we can guarantee that for any distribution $\D$ over instance space $\Z_{\mrm{ucvx}(\sigma,q')}(\X)$, if $q' > 2$ : 
$$
L(\overline{\Algomd}) - \inf_{\h \in \H} L(\h) \le  \min\left\{\frac{2 \left(\sup_{\h \in \cH} \Psi(\hopt)\right)^{1/q}}{n^{1/q}} , \frac{2}{(2 - p') \sigma^{p'-1} n^{p' - 1}} \right\}  + \frac{\sup_{\h \in \H} R(\h)}{n} ~.
$$
where $p' = \frac{q'}{q' - 1}$ and for $q' = 2$ :
$~~~
L(\overline{\Algomd}) - \inf_{\h \in \H} L(\h) \le  \frac{2 \log n}{\sigma n}   + \frac{\sup_{\h \in \H} R(\h)}{n} ~.
$
\end{lemma}
\begin{proof}
The statement follows by using Proposition \ref{prop:o2b} along with line of proof in Lemma \ref{lem:mducvx}.
\end{proof}

\section{Mirror Descent for Offline Optimization}\index{mirror descent!offline optimization} \label{sec:offmd}
Notice that one can think of offline convex optimization as a special case of stochastic convex optimization where the set of distributions we consider are point distributions on a single instance of the instance space $\Z$. Using this observation we see that $\overline{\Algomd}$ can directly be used for offline convex optimization problems with same guarantee on sub-optimality. In this thesis, for the problem of offline convex optimization we only consider the instance class $\Z_{\mrm{Lip}}(\X)$. The following corollary is a direct consequence of Lemma \ref{lem:smd}.

\begin{corollary}\label{cor:offmd}
  Let $\Psi : \B \mapsto \reals$ be non-negative and $q$-uniformly
  convex w.r.t. norm $\norm{\cdot}_{\X^*}$. For the Stochastic Mirror Descent  algorithm with this $\Psi$, using $\h_1 = \argmin{\h \in \H}
  \Psi(\h)$ and $\eta = \left(\tfrac{\sup_{\h \in \H}\Psi(\h)}{ n
      B}\right)^{1/p} $ we can guarantee that for any  instance $\z \in \Z_{\mrm{Lip}}$,  we have that :
$$
\ell(\overline{\Algomd}(\nabla \ell(h_1,\z), \ldots, \nabla \ell(h_n,\z)),\z) - \inf_{\h \in \H} \ell(\h,\z) \le 2 \left( \frac{\sup_{\h \in \H} \Psi(\h)}{n} \right)^{\frac{1}{q} } \ . ~~~~~ (\textrm{where} p = \tfrac{q}{q-1})
$$

\end{corollary}

\section{Detailed Proofs}\label{sec:mdproof}

\begin{proof}[Proof of Lemma \ref{lem:md} (generalized MD guarantee)]
Note that for any $\hopt \in \H$,
\begin{align*}
& \eta \left(\sum_{t=1}^n \ell(\h_t,\z_t) -  \sum_{t=1}^n \ell(\hopt,\z_t) \right)  \le \sum_{t=1}^n \ip{\eta \nabla \ell(\h_t,\z_t)}{\h_t - \hopt}\\
& ~~~~~~~~~~~~~~~ = \sum_{t=1}^n\left( \ip{\eta \nabla \ell(\h_t,\z_t)}{\h_t - \h'_{t+1}} + \ip{\nabla \ell(\h_t,\z_t)}{\h'_{t+1} - \hopt} \right)\\
& ~~~~~~~~~~~~~~~ = \sum_{t=1}^n\left( \ip{\eta \nabla \ell(\h_t,\z_t)}{\h_t - \h'_{t+1}} + \ip{\nabla \Psi(\h_{t}) - \nabla \Psi(\h'_{t+1})}{\h'_{t+1} - \hopt} \right)\\
& ~~~~~~~~~~~~~~~ \le \sum_{t=1}^n\left( \norm{\eta \nabla \ell(\h_t,\z_t)}_{\X} \norm{\h_t - \h'_{t+1}}_{\Xd} + \ip{\nabla \Psi(\h_{t}) - \nabla \Psi(\h'_{t+1})}{\h'_{t+1} - \hopt} \right)\\
& ~~~~~~~~~~~~~~~ \le \sum_{t=1}^n\left( \frac{\eta^p}{p}\norm{\nabla \ell(\h_t,\z_t)}_{\X}^p + \frac{1}{q} \norm{\h_t - \h'_{t+1}}_{\Xd}^q +  \ip{\nabla \Psi(\h_{t}) - \nabla \Psi(\h'_{t+1})}{\h'_{t+1} - \hopt} \right)
\end{align*}
Using simple manipulation we can show that
\begin{align*}
\ip{\nabla \Psi(\h_{t}) - \nabla \Psi(\h_{t+1})}{\h_{t+1} - \hopt} = \breg{\Psi}{\hopt}{\h_t} - \breg{\Psi}{\hopt}{\h_{t+1}} - \breg{\Psi}{\h_{t+1}}{\h_t}
\end{align*}
where given any $\h,\h' \in \B$, 
$$\breg{\Psi}{\h}{\h'} := \Psi(\h) - \Psi(\h') - \ip{\nabla \Psi(\h')}{\h - \h'}$$
is the Bregman divergence between $\h$ and $\h'$ w.r.t. function $\Psi$.
Hence,
{
\begin{align*}
& \eta \left( \sum_{t=1}^n \ell(\h_t,\z_t) -  \sum_{t=1}^n \ell(\hopt,\z_t) \right)  \\
&  ~~~~~~~~~~~~~~~ \le \sum_{t=1}^n\left( \frac{\eta^p}{p}\norm{\nabla \ell(\h_t,\z_t)}_{\X}^p + \frac{1}{q} \norm{\h_t - \h'_{t+1}}_{\Xd}^q + \ip{\nabla \Psi(\h_{t}) - \nabla \Psi(\h'_{t+1})}{\h'_{t+1} - \hopt} \right)\\
& ~~~~~~~~~~~~~~~ = \sum_{t=1}^n\left( \frac{\eta^p}{p}\norm{\nabla \ell(\h_t,\z_t)}_{\X}^p + \frac{1}{q} \norm{\h_t - \h'_{t+1}}_{\Xd}^q + \breg{\Psi}{\hopt}{\h_t} - \breg{\Psi}{\hopt}{\h'_{t+1}} - \breg{\Psi}{\h'_{t+1}}{\h_t}  \right) \\
& ~~~~~~~~~~~~~~~ \le \sum_{t=1}^n\left( \frac{\eta^p}{p}\norm{\nabla \ell(\h_t,\z_t)}_{\X}^p + \frac{1}{q} \norm{\h_t - \h'_{t+1}}_{\Xd}^q + \breg{\Psi}{\hopt}{\h_t} - \breg{\Psi}{\hopt}{\h_{t+1}} - \breg{\Psi}{\h'_{t+1}}{\h_t}  \right) \\
& ~~~~~~~~~~~~~~~ = \sum_{t=1}^n\left( \frac{\eta^p}{p}\norm{\nabla \ell(\h_t,\z_t)}_{\X}^p + \frac{1}{q} \norm{\h_t - \h'_{t+1}}_{\Xd}^q   - \breg{\Psi}{\h'_{t+1}}{\h_t}  \right) + \breg{\Psi}{\hopt}{\h_1} - \breg{\Psi}{\hopt}{\h_{n+1}}\\
& ~~~~~~~~~~~~~~~ \le \sum_{t=1}^n\left( \frac{\eta^p}{p}\norm{\nabla \ell(\h_t,\z_t)}_{\X}^p + \frac{1}{q} \norm{\h_t - \h'_{t+1}}_{\Xd}^q   - \breg{\Psi}{\h'_{t+1}}{\h_t}  \right) + \Psi(\hopt) 
\end{align*}
}
Now since $\Psi$ is $q$-uniformly convex w.r.t. $\norm{\cdot}_{\Xd}$,
for any $\h , \h' \in \Bd$, $\breg{\Psi}{\h'}{\h} \ge \frac{1}{q} \norm{\h - \h'}^q_{\Xd}$. Hence we conclude that
\begin{align*}
\sum_{t=1}^n \ell(\h_t,\z_t) -  \sum_{t=1}^n \ell(\hopt,\z_t) & \le\frac{\eta^{p-1}}{p}  \sum_{t=1}^n \norm{\nabla \ell(\h_t,\z_t)}_{\X}^p  + \frac{\Psi(\hopt)}{\eta} \\
& \le\frac{\eta^{p-1} B n}{p}  + \frac{\sup_{\h \in \H} \Psi(\h)}{\eta} \\
& \le\frac{\eta^{p-1} B n}{p}  + \frac{\sup_{\h \in \H} \Psi(\h)}{\eta} 
\end{align*}
Plugging in the value of $\eta = \left(\frac{\sup_{\h \in \H} \Psi(\h)}{n B}\right)^{1/p}$ we get : 
\begin{align*}
\sum_{t=1}^n \ell(\h_t,\z_t) -  \sum_{t=1}^n \ell(\hopt,\z_t) & \le 2 \left( \sup_{\h \in \H} \Psi(\h)\right)^{1/q} (B n)^{1/p} 
\end{align*}
dividing throughout by $n$ conclude the proof.
\end{proof}

\begin{proof}[Proof of Lemma \ref{lem:mdsmooth}]
Note that for any $\hopt \in \H$,
\begin{align*}
& \eta \left(\sum_{t=1}^n \ell(\h_t,\z_t) -  \sum_{t=1}^n \ell(\hopt,\z_t) \right)  \le \sum_{t=1}^n \ip{\eta \nabla \ell(\h_t,\z_t)}{\h_t - \hopt}\\
& ~~~~~~~~~~~~~~~ = \sum_{t=1}^n\left( \ip{\eta \nabla \ell(\h_t,\z_t)}{\h_t - \h'_{t+1}} + \ip{\eta \nabla \ell(\h_t,\z_t)}{\h'_{t+1} - \hopt} \right)\\
& ~~~~~~~~~~~~~~~ = \sum_{t=1}^n\left( \ip{\eta \nabla \ell(\h_t,\z_t)}{\h_t - \h'_{t+1}} + \ip{\nabla \Psi(\h_{t}) - \nabla \Psi(\h'_{t+1})}{\h'_{t+1} - \hopt} \right)\\
& ~~~~~~~~~~~~~~~ \le \sum_{t=1}^n\left( \norm{\eta \nabla \ell(\h_t,\z_t)}_{\X} \norm{\h_t - \h'_{t+1}}_{\Xd} + \ip{\nabla \Psi(\h_{t}) - \nabla \Psi(\h'_{t+1})}{\h'_{t+1} - \hopt} \right)\\
& ~~~~~~~~~~~~~~~ \le \sum_{t=1}^n\left( \frac{\eta^p}{p}\norm{\nabla \ell(\h_t,\z_t)}_{\X}^p + \frac{1}{q} \norm{\h_t - \h'_{t+1}}_{\Xd}^q +  \ip{\nabla \Psi(\h_{t}) - \nabla \Psi(\h'_{t+1})}{\h'_{t+1} - \hopt} \right)
\end{align*}
Using simple manipulation we can show that
\begin{align*}
\ip{\nabla \Psi(\h_{t}) - \nabla \Psi(\h_{t+1})}{\h_{t+1} - \hopt} = \breg{\Psi}{\hopt}{\h_t} - \breg{\Psi}{\hopt}{\h_{t+1}} - \breg{\Psi}{\h_{t+1}}{\h_t}
\end{align*}
where given any $\h,\h' \in \B$, 
$$\breg{\Psi}{\h}{\h'} := \Psi(\h) - \Psi(\h') - \ip{\nabla \Psi(\h')}{\h - \h'}$$
is the Bregman divergence between $\h$ and $\h'$ w.r.t. function $\Psi$.
Hence,
{
\begin{align*}
& \eta \left( \sum_{t=1}^n \ell(\h_t,\z_t) -  \sum_{t=1}^n \ell(\hopt,\z_t) \right)  \\
&  ~~~~~~~~~~~~~~~ \le \sum_{t=1}^n\left( \frac{\eta^p}{p}\norm{\nabla \ell(\h_t,\z_t)}_{\X}^p + \frac{1}{q} \norm{\h_t - \h'_{t+1}}_{\Xd}^q + \ip{\nabla \Psi(\h_{t}) - \nabla \Psi(\h'_{t+1})}{\h'_{t+1} - \hopt} \right)\\
& ~~~~~~~~~~~~~~~ = \sum_{t=1}^n\left( \frac{\eta^p}{p}\norm{\nabla \ell(\h_t,\z_t)}_{\X}^p + \frac{1}{q} \norm{\h_t - \h'_{t+1}}_{\Xd}^q + \breg{\Psi}{\hopt}{\h_t} - \breg{\Psi}{\hopt}{\h'_{t+1}} - \breg{\Psi}{\h'_{t+1}}{\h_t}  \right) \\
& ~~~~~~~~~~~~~~~ \le \sum_{t=1}^n\left( \frac{\eta^p}{p}\norm{\nabla \ell(\h_t,\z_t)}_{\X}^p + \frac{1}{q} \norm{\h_t - \h'_{t+1}}_{\Xd}^q + \breg{\Psi}{\hopt}{\h_t} - \breg{\Psi}{\hopt}{\h_{t+1}} - \breg{\Psi}{\h'_{t+1}}{\h_t}  \right) \\
& ~~~~~~~~~~~~~~~ = \sum_{t=1}^n\left( \frac{\eta^p}{p}\norm{\nabla \ell(\h_t,\z_t)}_{\X}^p + \frac{1}{q} \norm{\h_t - \h'_{t+1}}_{\Xd}^q   - \breg{\Psi}{\h'_{t+1}}{\h_t}  \right) + \breg{\Psi}{\hopt}{\h_1} - \breg{\Psi}{\hopt}{\h_{n+1}}\\
& ~~~~~~~~~~~~~~~ \le \sum_{t=1}^n\left( \frac{\eta^p}{p}\norm{\nabla \ell(\h_t,\z_t)}_{\X}^p + \frac{1}{q} \norm{\h_t - \h'_{t+1}}_{\Xd}^q   - \breg{\Psi}{\h'_{t+1}}{\h_t}  \right) + \Psi(\hopt) 
\end{align*}
}
Now since $\Psi$ is $q$-uniformly convex w.r.t. $\norm{\cdot}_{\Xd}$,
for any $\h , \h' \in \Bd$, $\breg{\Psi}{\h'}{\h} \ge \frac{1}{q} \norm{\h - \h'}^q_{\Xd}$. Hence we conclude that
\begin{align*}
\sum_{t=1}^n \ell(\h_t,\z_t) -  \sum_{t=1}^n \ell(\hopt,\z_t) & \le\frac{\eta^{p-1}}{p}  \sum_{t=1}^n \norm{\nabla \ell(\h_t,\z_t)}_{\X}^p  + \frac{\Psi(\hopt)}{\eta} 
\end{align*}
Now by smoothness of objective, using the  Lemma \ref{tem:selfbnd} we have that each $\norm{\nabla \ell(\h_t,\z_t)}_{\X} \le \sqrt{4 H \ell(\h_t,\z_t)}$. Using this in the above we get that,
\begin{align*}
\frac{1}{n} \sum_{t=1}^n \ell(\h_t,\z_t) -  \frac{1}{n} \sum_{t=1}^n \ell(\hopt,\z_t) & \le\frac{\eta^{p-1}}{p} \frac{1}{n} \sum_{t=1}^n \left(4 H \ell(\h_t,\z_t) \right)^{\frac{p}{2}}  + \frac{1}{n} \frac{\Psi(\hopt)}{\eta} \\
& \le \frac{(4 H)^{\frac{p}{2}}  \eta^{p-1} }{p} \left( \frac{1}{n} \sum_{t=1}^n\ell(\h_t,\z_t) \right)^{\frac{p}{2}}  + \frac{1}{n} \frac{\Psi(\hopt)}{\eta}  \\
& \le \frac{(4 H)^{\frac{p}{2}}  \eta^{p-1} }{p} \left( \frac{1}{n} \sum_{t=1}^n \ell(\h_t,\z_t) - \frac{1}{n} \sum_{t=1}^n \ell(\hopt,\z_t) \right)^{\frac{p}{2}} + \frac{(4 H)^{\frac{p}{2}}  \eta^{p-1} }{p} \left(\frac{1}{n} \sum_{t=1}^n \ell(\hopt,\z_t) \right)^{\frac{p}{2}}  + \frac{1}{n} \frac{\Psi(\hopt)}{\eta}  
\end{align*}
We now use the fact that if for any $x, B \ge 0$ if $x \le B x^\alpha + A$ for some $\alpha \in (\frac{1}{2},1]$ then as long as $B < A^{1 - \alpha}$,  $x \le \frac{A}{ 1 - B A^{\alpha -1}}$. Using this with $x = \frac{1}{n} \sum_{t=1}^n \ell(\h_t,\z_t) -  \frac{1}{n} \sum_{t=1}^n \ell(\hopt,\z_t)$, $B =\frac{(4 H)^{\frac{p}{2}}  \eta^{p-1} }{p}$ and $A = \frac{1}{n} \frac{\Psi(\hopt)}{\eta}  + \frac{(4 H)^{\frac{p}{2}}  \eta^{p-1} }{p} \left(\frac{1}{n} \sum_{t=1}^n \ell(\hopt,\z_t)\right)^{p/2}$ we conclude that for any $\eta$ such that,
\begin{align}\label{eq:etacond}
\frac{(4 H)^{\frac{p}{2}}  \eta^{p-1} }{p}  \le \frac{1}{2} \left( \frac{1}{n} \frac{\Psi(\hopt)}{\eta}  + \frac{(4 H)^{\frac{p}{2}}  \eta^{p-1} }{p} \left(\frac{1}{n} \sum_{t=1}^n \ell(\hopt,\z_t)\right)^{p/2}\right)^{1 - \frac{p}{2}}~,
\end{align}
we will have that :   
\begin{align}\label{eq:bound}
\frac{1}{n} \sum_{t=1}^n \ell(\h_t,\z_t) - \frac{1}{n} \sum_{t=1}^n \ell(\hopt,\z_t) & \le \frac{\frac{1}{n} \frac{\Psi(\hopt)}{\eta}  + \frac{(4 H)^{\frac{p}{2}}  \eta^{p-1} }{p} \left(\frac{1}{n} \sum_{t=1}^n \ell(\hopt,\z_t)\right)^{p/2}}{1 - \frac{(4 H)^{\frac{p}{2}}  \eta^{p-1} }{p} \left(\frac{1}{n} \frac{\Psi(\hopt)}{\eta}  + \frac{(4 H)^{\frac{p}{2}}  \eta^{p-1} }{p} \left(\frac{1}{n} \sum_{t=1}^n \ell(\hopt,\z_t)\right)^{p/2}\right)^{\frac{p}{2} - 1}}\\
& \le 2\left(\frac{1}{n} \frac{\Psi(\hopt)}{\eta}  + \frac{(4 H)^{\frac{p}{2}}  \eta^{p-1} }{p} \left(\frac{1}{n} \sum_{t=1}^n \ell(\hopt,\z_t)\right)^{p/2}\right)
\end{align}
To this end we now choose the step size as
\begin{align*}
\eta = \left\{ \begin{array}{cl}
\left(\frac{p \Psi(\hopt)}{n}\right)^{1/p} \frac{1}{\sqrt{\frac{4 H}{n} \sum_{t=1}^n \ell(\hopt,\z_t)}} & \textrm{if } \frac{1}{n} \sum_{t=1}^n \ell(\hopt,\z_t) \ge  \frac{16 H}{p^{2/p}} \left(\frac{\Psi(\hopt)}{n} \right)^{2/q} \vspace{0.04in}\\
\left(\frac{p}{2} \right)^{\frac{p}{2}} \frac{1}{4 H} \left(\frac{\Psi(\hopt)}{n} \right)^{\frac{2 - p}{p}} & \textrm{otherwise}
\end{array}\right.
\end{align*}
It is easy to verify that the choice of $\eta$ above satisfies condition in Equation \ref{eq:etacond}. To see this note that if we plug in $\eta = \left(\frac{p \Psi(\hopt)}{n}\right)^{1/p} \frac{1}{\sqrt{\frac{4 H}{n} \sum_{t=1}^n \ell(\hopt,\z_t)}}$ into condition in Equation \ref{eq:etacond} and rearrange we get that $\frac{1}{n} \sum_{t=1}^n \ell(\hopt,\z_t) \ge  \frac{16 H}{p^{2/p}} \left(\frac{\Psi(\hopt)}{n} \right)^{2/q}$. Thus whenever $\frac{1}{n} \sum_{t=1}^n \ell(\hopt,\z_t) \ge  \frac{16 H}{p^{2/p}} \left(\frac{\Psi(\hopt)}{n} \right)^{2/q}$, the choice of $\eta$ satisfies the condition. On the other hand, the choice $\eta = \left(\frac{p}{2} \right)^{\frac{p}{2}} \frac{1}{4 H} \left(\frac{\Psi(\hopt)}{n} \right)^{\frac{2 - p}{p}}$ was in the first place derived by making 
$$
\frac{(4 H)^{\frac{p}{2}}  \eta^{p-1} }{p}  \le \frac{1}{2} \left( \frac{1}{n} \frac{\Psi(\hopt)}{\eta}\right)
$$
which is got by zeroing out the second term in the condition. Hence we conclude that this choice of $\eta$ always satisfies the condition in Equation \ref{eq:etacond}. Now we plug in this choice of $\eta$ into the bound in Equation \ref{eq:bound}. Note that whenever $\eta = \left(\frac{p\ \Psi(\hopt)}{n}\right)^{1/p} \frac{1}{\sqrt{\frac{4 H }{n} \sum_{t=1}^n \ell(\hopt,\z_t)}}$, then simply plugging in this choice of $\eta$ into the bound in Equation \ref{eq:bound}, we get
\begin{align}\label{eq:bnd1}
\frac{1}{n} \sum_{t=1}^n \ell(\h_t,\z_t) - \frac{1}{n} \sum_{t=1}^n \ell(\hopt,\z_t) & \le  \frac{\sqrt{64\ H}\ \Psi^{1/q}(\hopt) \sqrt{\frac{1}{n} \sum_{t=1}^n \ell(\hopt,\z_t)} }{p^{1/p} n^{1/q}}  \notag\\
& \le \sqrt{\frac{64 H}{n} \sum_{t=1}^n \ell(\hopt,\z_t)} \ \left(\frac{\Psi(\hopt)}{n}\right)^{1/q} 
\end{align}
On the other hand when
$$
\eta = \left(\frac{p}{2} \right)^{\frac{p}{2}} \frac{1}{4 H} \left(\frac{\Psi(\hopt)}{n} \right)^{\frac{2 - p}{p}}
$$
plugging into bound in Equation \ref{eq:bound} we get,
\begin{align*}
\frac{1}{n} \sum_{t=1}^n \ell(\h_t,\z_t) - \frac{1}{n} \sum_{t=1}^n \ell(\hopt,\z_t) & \le 2\left(\frac{4 H}{\left(\frac{p}{2}\right)^{\frac{p}{2}}} \left(\frac{\Psi(\hopt)}{n}\right)^{2/q}  + \frac{(4 H)^{\frac{2 - p}{2}}  \left(\frac{p}{2} \right)^{\frac{p(p-1)}{2}}  \left(\frac{\Psi(\hopt)}{n} \right)^{\frac{(2 - p)(p-1)}{p}} }{p} \left(\frac{1}{n} \sum_{t=1}^n \ell(\hopt,\z_t)\right)^{p/2}\right)
\end{align*}
However note that we pick $
\eta = \left(\frac{p}{2} \right)^{\frac{p}{2}} \frac{1}{4 H} \left(\frac{\Psi(\hopt)}{n} \right)^{\frac{2 - p}{p}}
$ only when $\frac{1}{n} \sum_{t=1}^n \ell(\hopt,\z_t) < \frac{16 H}{p^{2/p}} \left(\frac{\Psi(\hopt)}{n} \right)^{2/q}$ and so plugging this inequality we get,
\begin{align}\label{eq:bnd2}
\frac{1}{n} \sum_{t=1}^n \ell(\h_t,\z_t) - \frac{1}{n} \sum_{t=1}^n \ell(\hopt,\z_t) & \le 2\left(\frac{4 H}{\left(\frac{p}{2}\right)^{\frac{p}{2}}} \left(\frac{\Psi(\hopt)}{n}\right)^{2/q}  + \frac{2^{p- 2 } 4 H    }{\left(\frac{p}{2}\right)^{\frac{4 - p(p-1)}{2}}}    \left(\frac{\Psi(\hopt)}{n} \right)^{2/q} \right) \notag\\
& \le 40 H \left(\frac{\Psi(\hopt)}{n}\right)^{2/q}   
\end{align}
Combining the above and bound in Equation \ref{eq:bnd1} we conclude that for this choice of $\eta$,
\begin{align*}
\frac{1}{n} \sum_{t=1}^n \ell(\h_t,\z_t) - \frac{1}{n} \sum_{t=1}^n \ell(\hopt,\z_t)  & \le \max\left\{ \sqrt{\frac{64 H}{n} \sum_{t=1}^n \ell(\hopt,\z_t)} \ \left(\frac{\Psi(\hopt)}{n}\right)^{1/q}  , 40 H \left(\frac{\Psi(\hopt)}{n}\right)^{2/q} \right\}\\
& \le \sqrt{\frac{64 H}{n} \sum_{t=1}^n \ell(\hopt,\z_t)} \ \left(\frac{\Psi(\hopt)}{n}\right)^{1/q}  + 40 H \left(\frac{\Psi(\hopt)}{n}\right)^{2/q} 
\end{align*}

\end{proof}

\begin{proof}[Proof of Lemma \ref{lem:mducvx}]
Note that for any $\hopt \in \H$,
\begin{align*}
& \sum_{t=1}^n \ell(\h_t,\z_t) -  \sum_{t=1}^n \ell(\hopt,\z_t)  \le \sum_{t=1}^n \ip{\nabla \ell(\h_t,\z_t)}{\h_t - \hopt} - \frac{\sigma}{q'}\norm{\h_{t} - \hopt}_{\Xd}^{q'}\\
& ~~~~~~~~~~~~~~~ = \sum_{t=1}^n \ip{ \nabla \ell(\h_t,\z_t)}{\h_t - \h'_{t+1}} + \ip{\nabla \ell(\h_t,\z_t)}{\h'_{t+1} - \hopt}   - \frac{\sigma}{q'}\norm{\h_{t} - \hopt}_{\Xd}^{q'}\\
& ~~~~~~~~~~~~~~~ = \sum_{t=1}^n \ip{ \nabla \ell(\h_t,\z_t)}{\h_t - \h'_{t+1}} + \ip{\nabla \tilde{\Psi}_t(\h_{t}) - \nabla \tilde{\Psi}_t(\h'_{t+1})}{\h'_{t+1} - \hopt}  - \frac{\sigma}{q'}\norm{\h_{t} - \hopt}_{\Xd}^{q'}
\end{align*}
Using simple manipulation we can show that
\begin{align*}
\ip{\nabla \tilde{\Psi}_t(\h_{t}) - \nabla \tilde{\Psi}_t(\h_{t+1})}{\h_{t+1} - \hopt} = \breg{\tilde{\Psi}_t}{\hopt}{\h_t} - \breg{\tilde{\Psi}_t}{\hopt}{\h_{t+1}} - \breg{\tilde{\Psi}_t}{\h_{t+1}}{\h_t}
\end{align*}
Hence,
\begin{align*}
& \sum_{t=1}^n \ell(\h_t,\z_t) -  \sum_{t=1}^n \ell(\hopt,\z_t)  \le \sum_{t=1}^n \ip{ \nabla \ell(\h_t,\z_t)}{\h_t - \h'_{t+1}} + \ip{\nabla \tilde{\Psi}_t(\h_{t}) - \nabla \tilde{\Psi}_t(\h'_{t+1})}{\h'_{t+1} - \hopt}  - \frac{\sigma}{q'}\norm{\h_{t} - \hopt}_{\Xd}^{q'}\\
& ~~~~~~~~~~~~~~~ \le \sum_{t=1}^n \ip{ \nabla \ell(\h_t,\z_t)}{\h_t - \h'_{t+1}} + \breg{\tilde{\Psi}_t}{\hopt}{\h_t} - \breg{\tilde{\Psi}_t}{\hopt}{\h'_{t+1}} - \breg{\tilde{\Psi}_t}{\h'_{t+1}}{\h_t} - \frac{\sigma}{q'}\norm{\h_{t} - \hopt}_{\Xd}^{q'} \\
& ~~~~~~~~~~~~~~~ \le \sum_{t=1}^n \ip{ \nabla \ell(\h_t,\z_t)}{\h_t - \h'_{t+1}} + \breg{\tilde{\Psi}_t}{\hopt}{\h_t} - \breg{\tilde{\Psi}_t}{\hopt}{\h_{t+1}} - \breg{\tilde{\Psi}_t}{\h'_{t+1}}{\h_t} - \frac{\sigma}{q'}\norm{\h_{t} - \hopt}_{\Xd}^{q'} \\
& ~~~~~~~~~~~~~~~ \le \sum_{t=1}^n \left(\ip{ \nabla \ell(\h_t,\z_t)}{\h_t - \h'_{t+1}} - \breg{\tilde{\Psi}_t}{\h'_{t+1}}{\h_t} \right)  + \breg{\tilde{\Psi}_1}{\hopt}{\h_1} - 
\frac{\sigma}{q'}\norm{\h_{1} - \hopt}_{\Xd}^{q'}\\
& ~~~~~~~~~~~~~~~~~~~~~~~~  + \sum_{t=2}^n \left(\breg{\tilde{\Psi}_t}{\hopt}{\h_t} - \breg{\tilde{\Psi}_{t-1}}{\hopt}{\h_{t}}  - \frac{\sigma}{q'}\norm{\h_{t} - \hopt}_{\Xd}^{q'} \right)\\
& ~~~~~~~~~~~~~~~ \le \sum_{t=1}^n \left(\ip{ \nabla \ell(\h_t,\z_t)}{\h_t - \h'_{t+1}} - \breg{\tilde{\Psi}_t}{\h'_{t+1}}{\h_t} \right)  + \breg{\tilde{\Psi}_1}{\hopt}{\h_1} - 
\sigma \breg{\psi}{\hopt}{\h_{1}} \\
& ~~~~~~~~~~~~~~~~~~~~~~~~ + \sum_{t=2}^n \left(\breg{\tilde{\Psi}_t}{\hopt}{\h_t} - \breg{\tilde{\Psi}_{t-1}}{\hopt}{\h_{t}}  - \sigma \breg{\psi}{\hopt}{\h_{t}} \right) \\
& ~~~~~~~~~~~~~~~ = \sum_{t=1}^n \left(\ip{ \nabla \ell(\h_t,\z_t)}{\h_t - \h'_{t+1}} - \breg{\tilde{\Psi}_t}{\h'_{t+1}}{\h_t} \right)  + \frac{1}{\eta} \breg{\Psi}{\hopt}{\h_1} +  \breg{R}{\hopt}{\h_1}\\
& ~~~~~~~~~~~~~~~ = \sum_{t=1}^n \left(\ip{ \nabla \ell(\h_t,\z_t)}{\h_t - \h_{t+1}} - \breg{\tilde{\Psi}_t}{\h_{t+1}}{\h_t} \right)  + \frac{1}{\eta} \breg{\Psi}{\hopt}{\h_1} +  \breg{R}{\hopt}{\h_1}\\
& ~~~~~~~~~~~~~~~ = \sum_{t=1}^n \left(\ip{ \nabla \phi(\h_t,\z_t)}{\h_t - \h_{t+1}} - \breg{\frac{\Psi}{\eta} + \sigma t \psi}{\h_{t+1}}{\h_t} + R(\h_t) - R(\h_{t+1}) \right)  + \frac{1}{\eta} \breg{\Psi}{\hopt}{\h_1} \\
& ~~~~~~~~~~~~~~~~~~~~~~~~ +  \breg{R}{\hopt}{\h_1}\\
& ~~~~~~~~~~~~~~~ \le \sum_{t=1}^n \left(\ip{ \nabla \phi(\h_t,\z_t)}{\h_t - \h'_{t+1}} - \frac{1}{\eta\, q} \norm{\hopt - \h'_{t+1}}_{\Xd}^q - \frac{\sigma \, t}{q'} \norm{\hopt - \h'_{t+1}}_{\Xd}^{q'} \right)  + \frac{\Psi(\hopt)}{\eta} \\
& ~~~~~~~~~~~~~~~~~~~~~~~~ + R(\hopt) - R(\h_{n+1})\\
& ~~~~~~~~~~~~~~~ \le \sum_{t=1}^n \inf_{\bf{u}_t + \v_t = \nabla \phi(\h_t,\z_t)} \left\{ \frac{\eta^{p-1}}{p} \norm{\bf{u}_t}_{\X}^{p}  + \frac{1}{p' \, \sigma^{p' - 1}\, t^{p'-1}} \norm{\v_t}_{\X}^{p'} \right\}  + \frac{\Psi(\hopt)}{\eta} + R(\hopt)
\end{align*}
Where in the steps above we used the fact that for any functions $F$ and $G$, $\breg{G +F}{\cdot}{\cdot} = \breg{G}{\cdot}{\cdot} + \breg{F}{\cdot}{\cdot}$ and that fact that for any function $F$ that is $q$-uniformly convex w.r.t. $\norm{\cdot}_{\Xd}$, for any $\h , \h' \in \Bd$, $\breg{F}{\h'}{\h} \ge \frac{1}{q} \norm{\h - \h'}^q_{\Xd}$. The final step is due to Fenchel Young inequality. Now we upper bound the summation term by replacing each infimum over decompositions of $\nabla \ell(\h_t,\z_t)$ into any arbitrary vectors $\bf{u}_t$ and $\v_t$ to vectors of specific form, $\bf{u}_t = (1 - \alpha) \nabla \phi(\h_t,\z_t)$ and $\v_t = \alpha \nabla \phi(\h_t,\z_t)$ for some $\alpha \in [0,1]$. Hence we get for any $\alpha \in [0,1]$,
\begin{align*}
\sum_{t=1}^n \ell(\h_t,\z_t) -  \sum_{t=1}^n \ell(\hopt,\z_t)  & \le  \sum_{t=1}^n  \left( \frac{\eta^{p-1} (1 - \alpha)^p }{p} \norm{\nabla \phi(\h_t,\z_t)}_{\X}^{p}  + \frac{\alpha^{p'}}{p' \, \sigma^{p' - 1}\, t^{p'-1}} \norm{\nabla \phi(\h_t,\z_t)}_{\X}^{p'} \right)  \\
& ~~~~~~~~~~+ \frac{\Psi(\hopt)}{\eta} + R(\hopt)\\
&  \le   \frac{\eta^{p-1} (1 - \alpha)^p n}{p}   + \frac{\alpha^{p'} }{p' \, \sigma^{p' - 1}} \sum_{t=1}^n \frac{1}{t^{p'-1}}   + \frac{\Psi(\hopt)}{\eta} + R(\hopt)\\
&  \le   \frac{\eta^{p-1} (1 - \alpha)^p n}{p}   + \frac{\alpha^{p'} }{p' \, \sigma^{p' - 1}} \frac{n^{2-p'}}{2-p'}   + \frac{\Psi(\hopt)}{\eta} + R(\hopt)\\
&  \le   \eta^{p-1} (1 - \alpha)^p n    + \frac{\alpha^{p'} }{\sigma^{p' - 1}} \frac{n^{2-p'}}{2-p'}   + \frac{\Psi(\hopt)}{\eta} + R(\hopt)\\
&  \le   \eta^{p-1} (1 - \alpha)^p n    + \frac{\alpha^{p'} }{\sigma^{p' - 1}} \frac{n^{2-p'}}{2-p'}   + \frac{\sup_{\h \in \cH} \Psi(\h)}{\eta} + \sup_{\h \in \cH} R(\h)
\end{align*}
Using $\alpha = 1$ whenever $n \ge \left((2 - p') \sigma^{p'-1} \Psi^{1/q}(\hopt)\right)^{\frac{1}{2 - p' - 1/p}} $ and $\alpha = 0$ otherwise and picking 
$$
\eta = \left\{ \begin{array}{ll}
 \left(\frac{\sup_{\h\in \H}\Psi(\h)}{n}\right)^{1/p} & \textrm{if } n \ge \left((2 - p') \sigma^{p'-1} \sup_{\h \in \cH}\Psi^{1/q}(\h)\right)^{\frac{1}{2 - p' - 1/p}}  \\
\infty & \textrm{otherwise}
\end{array}\right.
$$
we get that 
\begin{align*}
\sum_{t=1}^n \ell(\h_t,\z_t) -  \sum_{t=1}^n \ell(\hopt,\z_t)  \le \min\left\{2 \sup_{\h \in \H}  \Psi^{1/q}(\h)\, n^{1/p} + \sup_{\h \in \H} R(\h), \frac{2}{(2 - p') \sigma^{p'-1}} n^{2 - p'} + \sup_{\h \in \H}  R(\h)\right\}
\end{align*}
Dividing throughout by $n$ concludes the proof.
\end{proof}


\section{Discussion}\label{sec:mddiscuss}
The mirror descent algorithm with uniformly convex $\Psi$ functions were introduced by Nemirovski and Yudin in \cite{NemirovskiYu78} for offline convex optimization. Specific upper bounds for offline convex optimization of $\Z_{\mrm{Lip}}$ dual case when $\H$ is the unit $\ell_p$ ball and $\X$ is the dual of $\H$ are provided in  \cite{NemirovskiYu78}. For online convex optimization problem, online gradient descent (Euclidean case) was proposed by Zinkevich in\cite{Zinkevich03}. Faster rates when the losses are strongly convex in the Euclidean case for online gradient descent was proposed in \cite{HazanKaKaAg06}. Mirror descent for general strongly convex objectives with $\log n/ n$ rates was proposed and analyzed in \cite{ShalevSi07_tech}. While in all the above the set $\X$ in the corresponding problems are same as dual of set $\H$ and $\H = \bcH$,
in this chapter we consider the generic case and provide bounds for the non-dual case for online and statistical convex learning and for offline convex optimization. One fact to pay attention to is that the upper bounds are provided assuming one can find appropriate function $\Psi$ that is $q$-uniformly convex w.r.t. norm $\norm{\cdot}_{\Xd}$. Ofcourse the immediate question that arises is ``When can one find such functions $\Psi$ and are the bounds got using such $\Psi$ optimal?". The next three chapters deals with this question for online and statistical convex learning problems and for offline convex optimization problems.

%% file: opton.tex
\chapter{Optimality of Mirror Descent for Online Convex Learning Problem}\label{chp:opton}

In this chapter we will show that the mirror descent method is universal and near optimal for online convex learning problems. Very roughly, the main result we show in this chapter can be stated as : 

{\bf For any online convex learning problem, if some online learning algorithm can guarantee a regret bound of $\mrm{Rate}_n$, then mirror descent algorithm can guarantee regret bounded as $\tilde{O}(\mrm{Rate}_n)$.}
Of course in the remainder of the chapter we will exactly quality this result and show optimality of mirror descent for online learning when losses are s.t. gradients are in set $\X$, for non-negative smooth convex losses and for uniformly convex losses.

Before we proceed we start by noticing that owing to Remark \ref{rem:detcnvx}, it suffices to only consider Deterministic learning algorithms. As a result, the online convex learning problem can be viewed as a multi-round game where on round $t$, the learner first picks a vector $\h_t \in\bcH$. Next, the adversary picks instance $\z_t \in \Z$ where $\Z$ is a class of instances specifying some set of convex functions. At the end of the round, the learner pays instantaneous cost
$\ell(\h_t,\z_t)$. Recall that a deterministic online learning algorithm $\Algo$ for the problem is specified by the mapping
$\Algo : \bigcup_{n \in \mathbb{N}} \Z^{n-1} \mapsto \bcH$. We shall represent the regret of an algorithm $\Algo$ for a given sequence of instances $\z_1,\ldots,\z_n$ by the shorthand : 
$$
\Reg_n(\Algo,\z_1,\ldots,\z_n) := \frac{1}{n} \sum_{t=1}^n \ell(\Algo(\z_{1:{t-1}}),\z_t)  - \inf_{\h \in \cH} \frac{1}{n} \sum_{t=1}^n \ell(\h,\z_t) \ .
$$
The goal of the learner as before, is to minimize the regret at the end of  $n$ rounds.

\index{value of the online learning game}
In Chapter \ref{chp:online} since we considered randomized online learning algorithms we had to be careful in defining the value of the game in Equation \ref{eq:def_val_game}. Since for convex learning problems it suffices to only consider deterministic learning algorithms, it is easier to write down the value of the game for these problems. The value of the online convex learning problem can we written as the best possible guarantee on regret against any sequence of instances that any algorithm can enjoy.  Formally the value can be written as :
\begin{align}\label{eq:value}
\Val_n(\H,\Z) = \inf_{\Algo} \sup_{\z_{1:n}  \in \Z} \Reg_n(\Algo,\z_{1},\ldots,\z_n)
\end{align}

In the Section \ref{sec:valclass} we will see that for online convex learning problems introduced in previous chapter, a problem is online learnable if and only if it is learnable using a gradient-based online learning algorithm. We will also see that value of the linear game acts as key in characterizing optimal learning rates of various other convex learning problems and hence will focus on that. In Chapter \ref{chp:md} we describe the online mirror descent algorithm and provide guarantees for various problems. However these guarantees relied on our ability to be able to pick appropriate uniformly convex function to use with the mirror descent algorithm. In section \ref{sec:martype} we show how the concept of martingale type (a generalization of it as per our need) captures closely the value of linear game and hence can be used to closely characterize rates for the various convex learning problems. In Section \ref{sec:ucvxmtyp} we extend Pisier's result \cite{Pisier75} to show that martingale type of the problem can be used to ensure existence of an appropriate uniformly convex function. Subsequently in Section \ref{sec:mdopt} we put it all together and establish that for the convex problems we consider, there exists appropriate uniformly convex funciton so that mirror descent with this function and right step size is always near optimal (upto to $\log$ factors). Thus we establish universality and nearo optimality of mirror descent. This is also shown for non-negative smooth convex losses and certain uniformly convex losses. In Section \ref{sec:eg} we provide several examples of commonly encountered convex learning problems and establish rates for these problems using mirror descent. Section \ref{sec:onproof} provides detailed proofs for results in this chapter and finally we conclude with some discussion in Section \ref{sec:ondis}.

\section{Value of the Linear Game}\label{sec:valclass}
The value of online learning learning problem plays an important role in characterizing optimal rates of various online convex learning problems. In this section we show how the value of the linear online learning game is related to value of other online convex learning games and also introduce some necessary definitions to build towards showing the main result of this chapter. The following lemma shows how one can use online algorithms for linear problems for other online convex learning problems (problems where sub-gradients are in $\X$).

\begin{lemma}\label{lem:1storacle}
Let $\Algo$ be any online learning algorithm for linear learning problems specified by instance set $\Z_{\mrm{lin}}$. Using this, for any convex learning problem specified by instance set $\Z$ such that for any $\h \in \cH$, $\nabla_\h \ell(\h,\z) \in \X$,  one can construct a new gradient-based learning algorithm $\Algo^{\Oracle^{\mrm{1st}}}$ such that for any $\z_1,\ldots,\z_n \in \Z$,  
$$
\Reg_n(\Algo^{\Oracle^{\mrm{1st}}},\z_{1},\ldots,\z_n) \le \sup_{\z^*_1,\ldots,\z^*_n \in \Z_{\mrm{lin}}} \Reg_n(\Algo,\z^*_{1},\ldots,\z^*_n)
$$
\end{lemma}

A direct consequence of the above lemma is the following corollary that shows that value of the linear game upper bounds value of other online convex learning problems and is in fact equal to the value of the supervised learning game and online convex Lipschitz learning game.

\begin{corollary}\label{cor:val}
For any convex learning problem specified by instance set $\Z$ which is s.t. $\forall \h \in \cH ,\z \in \Z : \nabla_\h \ell(\h,\z) \in \X$, we have that,
$$
\Val_n(\cH,\Z) \le \Val_n(\cH,\Z_{\mathrm{lin}}(\X))
$$
Furthermore
$
\quad \Val_n(\cH,\Z_{\mathrm{Lip}}(\X)) = \Val_n(\cH,\Z_{\mathrm{sup}}(\X)) = \Val_n(\cH,\Z_{\mathrm{lin}}(\X))
\quad $ 
\end{corollary}

The equality in the above corollary can also be extended to most other commonly occurring convex loss function classes like say the hinge loss class  and logistic learning loss class with some extra constant factors. As we see from the above result, the value of the online learning problem for linear instance class $\Z_{\mrm{lin}}(\X)$ is critical in upper and lower bounds on rates of various other convex learning problems. In fact as we will sees later the value of the linear game also plays an important role in characterizing rates of smooth and uniformly convex online learning problems. Owing to this, for any $p \in [1,2]$ we define constant : 
\begin{align}\label{eq:vp}
V_p := \inf\left\{V\ \middle|\ \forall n \in \mathbb{N}, \Val_n(\H,\Z_{\mrm{lin}}(\X)) \le V n^{-\left(1 - \frac{1}{p}\right)}\right\}
\end{align}
Notice that $V_p$ characterizes optimal rate for online linear learning problems (and hence supervised and Lipschitz classes too) up to polynomial.

The main aim of this chapter is to show near optimality of mirror descent algorithm. To this end, similar to $V_p$ for each $p \in [1,2]$ we can define:
\begin{align}\label{eq:mdp}
\MD_p := \inf\left\{D : \exists \Psi, \eta \textrm{ s.t. } \forall n \in \mathbb{N}, \sup_{\z_{1:n}  \in \Z} \Reg_n(\Algomd,\z_{1:n}) \le D  n^{-(1 - \frac{1}{p} )}  \right\}
\end{align}
where the Mirror Descent algorithm in the above definition is run with the
corresponding $\Psi$ and $\eta$. The constant $\MD_p$ is a characterization of the best guarantee the Mirror Descent algorithm can provide by choosing the best $\Psi$ and $\eta$.

A simple consequence of the definitions of $V_p$ and $\MD_p$ is the following proposition.

\begin{proposition}
For any $p \in [1,2]$:
$$
V_p \le \MD_p
$$
\end{proposition}

\section{Value and Martingale Type}\label{sec:martype}
In \cite{SriTew10}, it was shown that the concept of the {\em Martingale type} (also sometimes
called the {\em Haar type}) of
a Banach space and optimal rates for online convex optimization
problem, where $\X$ and $\H$ are duals of each other, are closely
related. In this section we extend the classic notion of Martingale
type of a Banach space (see for instance \cite{Pisier75}) to one that
accounts for the pair $(\Hd,\X)$. Before we proceed with the
definitions we would like to introduce a few necessary notations.
First, throughout we shall use $\epsilon \in \{\pm1\}^\mathbb{N}$ to
represent infinite sequence of signs drawn uniformly at random (i.e.
each $\epsilon_i$ has equal probability of being $+1$ or $-1$). Also
throughout $(\tx_n)_{n \in \mathbb{N}}$ represents an $\Bd$ valued tree of infinite depth, that is a sequence of mappings where each $\tx_n : \{\pm 1\}^{n-1} \mapsto \Bd$. We are now ready to give the extended definition of
Martingale type (or M-type) of a pair $(\Hd,\X)$.

\begin{definition}\label{def:mtype} \index{martingale type}
A pair $(\Hd,\X)$ of subsets of a vector space $\Bd$ is said to be of M-type $p$ if there exists a constant $C \ge 1$ such that for all $\Bd$-valued tree $\tx$ of infinite depth and any $\x_0 \in \Bd$ : 
\begin{align} \label{eq:mtype}
\sup_{n} \E{\norm{\x_0 + \sum_{i=1}^n \epsilon_i \tx_i(\epsilon)}_{\Hd}^p} \le C^p \left(\norm{\x_0}_{\X}^p + \sum_{n \ge 1} \E{\norm{\tx_n(\epsilon)}_{\X}^p} \right)
\end{align}
\end{definition}
The concept is called Martingale type because $(\epsilon_n \tx_n(\epsilon))_{n \in \mathbb{N}}$ is a martingale difference sequence and it can be shown that rate of convergence of martingales in Banach spaces is governed by the rate of convergence of martingales of the form $Z_n = \x_0 + \sum_{i=1}^n \epsilon_i \tx_i(\epsilon)$ (which are incidentally called Walsh-Paley martingales). \index{Walsh-Paley martingales} We point the reader to \cite{Pisier75,Pisier11} for more details.
Further, for any $p \in [1,2]$ we also define,
{\small
\begin{align*}
C_p := \inf\left\{C\ \middle|\ \forall \x_0 \in \Bd, \forall \tx,\ \sup_{n }{\small \E{\norm{\x_0 + \sum_{i=1}^n \epsilon_i \tx_i(\epsilon)}_{\Hd}^p} \le C^p \left(\norm{\x_0}_{\X}^p + \sum_{n \ge 1} \mathbb{E}\norm{\tx_n(\epsilon)}_{\X}^p \right)} \right\}
\end{align*}}
$C_p$ is useful in determining if the pair $(\Hd,\X)$ has Martingale type $p$.

 By the results in Chapter \ref{chp:} (using it specifically for linear class) we have the following theorem:

\begin{theorem}\label{thm:cite} 
For any $\H \in \B$ and any $\X \in \Bd$ and any $n \ge 1$,
\begin{align*}
\sup_{\tx} \E{\norm{\frac{1}{n} \sum_{i=1}^n \epsilon_i \tx_i(\epsilon)}_{\Hd}} \le  \Val_n(\H,\X) \le 2 \sup_{\tx} \E{\norm{ \frac{1}{n} \sum_{i=1}^n \epsilon_i \tx_i(\epsilon)}_{\Hd}}
\end{align*}
where the supremum above is over $\Bd$-valued tree $\tx$ of infinite depth.
\end{theorem}

Our main interest here will is in establishing that low regret implies Martingale type.  To do so, we start with the above theorem to relate value of the online convex optimization game to rate of convergence of martingales in the Banach space. We then extend the result of Pisier in \cite{Pisier75} to the ``non-matching'' setting combining it with the above theorem to finally get :

\begin{lemma}\label{lem:mn}
If for some $r \in (1,2]$ there exists a constant $D > 0$ such that for any $n$,
\begin{align*}
\Val_n(\H,\X) \le D n^{-(1 - \frac{1}{r})}
\end{align*}
then for all $s < r$, we can conclude that any $\x_0 \in \Bd$ and any $\Bd$-valued tree $\tx$ of infinite depth  will satisfy :
\begin{align*}
\sup_n \E{\norm{\x_0 + \sum_{i=1}^n \epsilon_i \tx_i(\epsilon)}^s_{\Hd}} \le \left( \frac{1104\ D}{(r - s)^2}\right)^s \left(\norm{\x_0}_{\X}^s + \sum_{i \ge 1} \E{\norm{\tx_i(\epsilon)}_{\X}^s} \right)
\end{align*} 
That is, the pair $(\H,\X)$ is of martingale type $s$.
\end{lemma}

The following corollary is an easy consequence of the above lemma.
\begin{corollary}
For any $p \in [1,2]$ and any $p' <p$ : 
$
C_{p'} \le  \frac{1104\ \V_p}{(p - p')^2}
$
\end{corollary}

%

\section{Martingale Type and Uniform Convexity}\label{sec:ucvxmtyp}
The classical notion of Martingale type plays a central role in the study of
geometry of Banach spaces. In \cite{Pisier75}, it was shown that a
Banach space has Martingale type $p$ (the classical notion) if and only
if uniformly convex functions with certain properties exist on that
space (w.r.t. the norm of that Banach space). In this section, we
extend this result and show how the Martingale type of a pair $(\Hd,\X)$ are related to existence of certain uniformly
convex functions. Specifically, the following theorem shows that the
notion of Martingale type of pair $(\Hd,\X)$ is equivalent to
the existence of a non-negative function that is uniformly convex w.r.t.
the norm $\norm{\cdot}_{\Xd}$.

\begin{lemma}\label{lem:construct22}
If, for some $p \in (1,2]$, there exists a constant $C > 0$, such that 
for all $\Bd$-valued tree $\tx$ of infinite depth and any $\x_0 \in \Bd$:
\begin{align*}
\sup_{n} \E{\norm{\x_0 + \sum_{i=1}^n \epsilon_i \tx_i(\epsilon)}_{\Hd}^p} \le C^p \left(\norm{\x_0}_{\X}^p + \sum_{n \ge 1} \E{\norm{\tx_n(\epsilon)}_{\X}^p} \right)
\end{align*}
(i.e. $(\Hd,\X)$ has Martingale type $p$), then there exists a convex function $\Psi : \B \mapsto \reals^+$ with $\Psi(0) = 0$, that is $q$-uniformly convex w.r.t. norm $\norm{\cdot}_{\Xd}$ s.t. $\forall \h \in \B$, $ \frac{1}{q} \norm{\h}_{\Xd}^q \le \Psi(\h) \le \frac{C^q}{q} \norm{\h}_{\H}^q $.
\end{lemma}

Define,
\begin{align*}
D_p := \inf\left\{ \left(\sup_{\h \in \H}
    \Psi(\h)\right)^{\frac{p-1}{p}} ~\middle|~ \Psi:\H \mapsto
  \reals^+ \textrm{ is }\tfrac{p}{p-1}\textrm{-uniformly convex
    w.r.t. }\norm{\cdot}_{\X^*} , \Psi(0)=0  \right\}
\end{align*}

The following corollary follows directly from the above lemma.
\begin{corollary}
For any $p \in [1,2]$,  $D_p \le C_p$.
\end{corollary}

The proof of Lemma \ref{lem:construct22} goes further and gives a specific uniformly
convex function $\Psi$ satisfying the desired requirement (i.e.~establishing $D_p \leq C_p$) under the
assumptions of the previous lemma:
{\small
\begin{align}\label{eq:construction2}
\Psi^*_q(\x) &:= \sup\left\{\frac{1}{C^p} \sup_{n} \E{\norm{\x + \sum_{i=1}^n \epsilon_i \tx_i(\epsilon)}^p_{\Hd}} - \sum_{i\ge 1} \E{\norm{\tx_i(\epsilon)}_{\X}^p} \right\} ~~ , ~~
\Psi_q  :=  (\Psi_q^*)^* \ .
\end{align}}
where the supremum above is over $\Bd$-valued tree $\tx$ of infinite depth  and $p = \frac{q}{q-1}$.

%
%
%

\section{Main Result : Optimality of Online Mirror Descent}\label{sec:mdopt}

In previous chapter we argued that if we can find an appropriate
uniformly convex function to use with the mirror descent algorithm, one
can guarantee diminishing regret. However the pending question there
was when one can find such a function and what is the rate one can
gaurantee. In Section \ref{sec:martingale} we introduced the extended
notion of Martingale type of a pair $(\Hd,\X)$ and
how it related to the value of the game. In Section
\ref{sec:typeconvex}, we saw how the concept of M-type related to
existence of certain uniformly convex functions.  We can now combine
these results to show that the mirror descent algorithm is a universal
online learning algorithm for convex learning problems.  Specifically
we show that whenever a problem is online learnable, the mirror
descent algorithm can guarantee near optimal rates:

\begin{theorem}\label{thm:mdoptlip}
  If for some constant $V > 0$ and some $q \in [2,\infty)$, $
  \Val_n(\H,\X) \le V n^{-\frac{1}{q}}$ for all $n$, then for any $n >
  e^{q-1}$, there exists regularizer function $\Psi$ and step-size
  $\eta$, such that the regret of the mirror descent algorithm using
  $\Psi$ against any $\z_1,\ldots,\z_n \in \Z_{\mrm{Lip}}$ chosen by the adversary is bounded as:
\begin{align}\label{eq:mdval}
\Reg_n(\Algomd,\z_{1},\ldots,\z_n) \le\, 6002\, V\, \log^2(n)\ n^{- \frac{1}{q}}
\end{align}
\end{theorem}
\begin{proof}
Combining Mirror descent guarantee in Lemma \ref{lem:md}, Lemma \ref{lem:construct22} and the lower bound in Lemma \ref{lem:mn} with $s = \frac{q}{q-1} - \frac{1}{\log(n)}$ we get the above statement.
\end{proof}

The above Theorem tells us that, with appropriate $\Psi$ and learning
rate $\eta$, mirror descent will obtain regret at most a factor of
$6002 \log^2(n)$ from the best possible worst-case upper bound.  We would
like to point out that the constant $V$ in the value of the game
appears linearly and there is no other problem or space related hidden
constants in the bound.


The following figure summarizes the relationship between the various constants. The arrow mark from $C_{p'}$ to $C_p$ indicates that for any  $n$, all the quantities are within $\log^2 n$ factor of each other.
\begin{figure}[h]
\begin{center}
\begin{tikzpicture}[node distance=1.5cm, auto,>=latex',
cond/.style={draw, thin, rounded corners, inner sep=1ex, text centered},
cond1/.style={}]
\node[text width=1.6cm, style=cond] (lower) {\small $p' < p,\ C_{p'}$};
\node[text width=0.5cm, style=cond1,right of=lower] (le1) {\huge $ \le$};
\node[text width=1.6cm, style=cond, right of=le1] (value)
{\small $V_p$};
\node[text width=0.5cm, style=cond1,right of=value] (le2) {\huge $\ \le$};\node[text width =1.6cm, style=cond, right of=le2] (MD) {\small
$\MD_p$};
\node[text width=0.5cm, style=cond1,right of=MD] (le3) {\huge $\ \le$};
\node[text width=1.6cm, style=cond, right of=le3] (D) {\small $D_p$};
\node[text width=0.5cm, style=cond1,right of=D] (le4) {\huge $\ \le$};
\node[text width=1.6cm, style=cond, right of=le4] (C) {\small $C_p$};
\node[text width=1.6cm, style=cond1] at (1.8,-0.55) (le5) {\tiny Lemma \ref{lem:mn} };
\node[text width=3cm, style=cond1] at (1.8,-0.78) (le5) {\tiny (extending Pisier's result \cite{Pisier75})  };
\node[text width=1.6cm, style=cond1] at (4.6,-0.55) (le6) {\tiny Definition of $V_p$ };
\node[text width=3cm, style=cond1] at (7.7,-0.78) (le6) {\tiny (Generalized MD guarantee)  };
\node[text width=3cm, style=cond1] at (8.5,-0.55) (le6) {\tiny Lemma \ref{lem:md} };
\node[text width=3cm, style=cond1] at (10.55,-0.55) (le6) {\tiny Construction of $\Psi$, Lemma \ref{lem:construct2} };
\node[text width=3cm, style=cond1] at (10.7,-0.78) (le5) {\tiny (extending Pisier's result \cite{Pisier75})  };
\path[<-, draw, double distance=1pt,sloped] (C) -- +(0,0.75) -| (lower);
\end{tikzpicture}
\end{center}
\caption{Relationship between the various constants}\label{fig:main}
\end{figure}
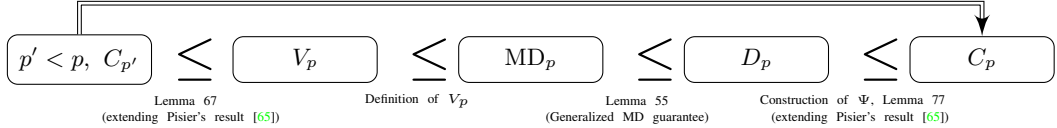

We now provide some general guidelines that will help us in picking out appropriate function $\Psi$ for mirror descent. First we note that though the function $\Psi_q$ in the construction \eqref{eq:construction2} need not be such that $(q \Psi_q(\h))^{1/q}$ is a norm, with a simple modification as noted in \cite{Pisier11} we can make it a norm. This basically tells us that the pair $(\H,\X)$ is online learnable, if and only if we can sandwich a $q$-uniformly convex norm in-between $\Xd$ and a scaled version of $\H$ (for some $q < \infty$). Also note that by definition of uniform convexity, if any function $\Psi$ is $q$-uniformly convex w.r.t. some norm $\norm{\cdot}$ and we have that $\norm{\cdot} \ge c \norm{\cdot}_\X$, then $\tfrac{\Psi(\cdot)}{c^q}$ is $q$-uniformly convex w.r.t. norm $\norm{\cdot}_\X$. These two observations together suggest that given pair $(\H,\X)$ what we need to do is find a norm $\norm{\cdot}$ in between $\norm{\cdot}_\Xd$ and $C \norm{\cdot}_\H$ ($C < \infty$, smaller the $C$ better the bound ) such that $\norm{\cdot}^q$ is $q$-uniformly convex w.r.t $\norm{\cdot}$.

\subsection{Smooth Loss Case}

\begin{lemma}\label{lem:smtlow}
Given any $\overline{\L^*} \in (0,\frac{3}{4}]$ and any $n \in \mathbb{N}$, for any online convex learning algorithm $\Algo$, there exists instances $\z_1,\ldots,\z_n \in \Z_{\mrm{smt(1)}}$ such that, $\inf_{\h \in \cH} \frac{1}{n} \sum_{t=1}^n \ell(\h,\z_t) \le \overline{\L^*}$ and
\begin{align*}
\Reg_n(\Algo,\z_{1:n}) \ge \frac{1}{n} \sup_{\tx} \Es{\epsilon}{\norm{\sum_{t=1}^{n \overline{\L^*}} \epsilon_t \tx_t(\epsilon)}_{\Hd}} \ge \frac{\overline{\L^*}}{2} \Val_{n\L^*}(\H,\Z_{\mrm{lin}}(\X))
\end{align*}
\end{lemma}

The following theorem shows the almost optimality of mirror descent for non-negative smooth convex objectives. The upper bound on the regret of the mirror descent algorithm also shows how having small $\overline{\L^*}$ helps reduce regret.

\begin{theorem}\label{thm:smtlow}
If for some $V > 0$ and $q \in [2,\infty)$ we have that
$$
\Val_n(\H,\Z_{\mrm{smt}(1)}) \le \frac{V}{n^{1/q}}
$$
then for any $\overline{\L^*} >0$ and $n > e^{q-1}$ there exists regularizer function $\Psi$ and step-size $\eta$ such that regret of the mirror descent algorithm with these $\Psi$ and $\eta$ against any $\z_1,\ldots,\z_n \in \Z_{\mrm{smt}(H)}$ s.t. $\inf_{\h \in \cH} \frac{1}{n} \sum_{t=1}^n \ell(\h,\z_t) \le \overline{\L^*}$ is bounded as :
\begin{align*}
\Reg_n(\Algomd,\z_1,\ldots,\z_n) \le  \frac{48016 V \sqrt{H \overline{\L^*}} \log^2 n}{n^{\frac{1}{q}}} +  \frac{ (37960 V)^2  H \log^4 n}{n^{\frac{2}{q}}}
\end{align*}
\end{theorem}
\begin{proof}
Using Lemma \ref{lem:smtlow} with $\overline{\L^*} = \frac{3}{4}$ with the premise of the theorem we have that for any $n$,
$$
\Val_{3n/4}(\H,\Z_{\mrm{lin}}(\X)) \le \frac{8 V}{3 n^{1/q}} \le \frac{8 V}{3 (3n/4)^{1/q}} 
 $$
and so we can conclude that for any $m$, $\Val_{m}(\H,\Z_{\mrm{lin}}(\X)) \le \frac{8 V }{3 m^{1/q}}$. Combining Mirror descent guarantee for smooth loss in Lemma \ref{lem:mdsmooth}, Lemma \ref{lem:construct22} and the lower bound in Lemma \ref{lem:mn} with $s = \frac{q}{q-1} - \frac{1}{\log(n)}$ we get the above statement.
\end{proof}

The above theorem shows that the mirror descent can achieve a near optimal rate in terms of dependence on $n$. However a closer look at the above Theorem and Lemma also reveals that we can in fact also capture a tighter dependence on $\overline{\L^*}$. In fact when $q = 2$, (that is when we get $1/\sqrt{n}$ rates), we can show that the dependence on $\overline{\L^*}$ of mirror descent is tight. 

To see this assume that algorithm $\Algo$ is a minimax optimal algorithm and  is such that for any $\overline{\L^*} \in (0,3/4]$ and $n$ the regret of the algorithm against any $\z_1,\ldots,\z_n \in \Z_{\mrm{smt}(1)}$ s.t. $\inf_{\h \in \cH} \frac{1}{n} \sum_{t=1}^n \ell(\h,\z_t) \le \overline{\L^*}$ is bounded as :
$$
\Reg(\Algo,\z_1,\ldots,\z_n) \le \frac{V \overline{\L^*}^a}{n^{1/q}} + \frac{V}{n^{1/b}}
$$
where $V, a, b > 0$ and $q \in [2,\infty)$ with $b < q$. 

Using the guarantee of the algorithm and Lemma \ref{lem:smtlow} we see that
$$
\Val_{n\L^*}(\H,\Z_{\mrm{lin}}(\X)) \le \frac{2 V \overline{\L^*}^{a-1}}{n^{1/q}} + \frac{2 V}{\overline{\L^*} n^{1/b}}
$$
Hence using notation $m = n \overline{\L^*}$ we conclude that for any $m$,
\begin{align}\label{eq:linvalgau}
\Val_{m}(\H,\Z_{\mrm{lin}}(\X)) & \le \frac{2 V \overline{\L^*}^{a-1 + \frac{1}{q}}}{m^{1/q}} + \frac{2 V}{\overline{\L^*}^{1 - \frac{1}{b}} m^{1/b}} 
\end{align}
From the above inequality we can conclude that $a \ge 1 - \frac{1}{q} = \frac{1}{p}$. This is because if $a < \frac{1}{p}$, then since the above inequality holds for all $\overline{\L^*}$, by picking $\overline{\L^*}$ to minimize the above inequality we can conclude that $\Val_{m}(\H,\Z_{\mrm{lin}}(\X)) = o\left(\frac{1}{m^{1/q}}\right)$. Hence Theorem \ref{thm:smtlow} we can conclude that regret of the mirror descent algorithm 
for any $\z_1,\ldots,\z_n \in \Z_{\mrm{smt}(1)}(\X)$ is bounded by $o\left(\frac{1}{m^{1/q}}\right)$ which is a contradiction since $\Algo$ is minimax optimal but can only guarantee bound of order $1/n^{1/q}$. Hence we can conclude that $a \le 1/p$.

Now note that by the guarantee of algorithm $\Algo$ we can conclude that  
$$
\Val_n(\H,\Z_{\mrm{smt}(1)}(\X)) \le \frac{2 V}{n^{1/q}}
$$
and so by Theorem \ref{thm:smtlow} we can conclude that there exists regularizer function $\Psi$ and step-size $\eta$ such that regret of the mirror descent algorithm with these $\Psi$ and $\eta$ against any $\z_1,\ldots,\z_n \in \Z_{\mrm{smt}(H)}$ s.t. $\inf_{\h \in \cH} \frac{1}{n} \sum_{t=1}^n \ell(\h,\z_t) \le \overline{\L^*}$ is bounded as :
\begin{align*}
\Reg_n(\Algomd,\z_1,\ldots,\z_n) \le  \frac{96032 V \sqrt{H \overline{\L^*}} \log^2 n}{n^{\frac{1}{q}}} +  \frac{ (75920 V)^2  H \log^4 n}{n^{\frac{2}{q}}}
\end{align*}

Thus we can conclude that while the best guarantee of any algorithm $\Algo$ is, 
$$
\Reg(\Algo,\z_1,\ldots,\z_n) \le \frac{V \overline{\L^*}^{1/p}}{n^{1/q}} + \frac{V}{n^{1/b}}~~,
$$
the regret of mirror descent can be bounded as 
$$
\Reg_n(\Algomd,\z_1,\ldots,\z_n) \le  \frac{96032 V \sqrt{H \overline{\L^*}} \log^2 n}{n^{\frac{1}{q}}} +  \frac{ (75920 V)^2  H \log^4 n}{n^{\frac{2}{q}}}
$$
Notice that when $q = 2$, for any $\overline{\L^*} > \frac{V}{n^{1/q}}$, guarantee of mirror descent is not only optimal in terms of $n$ but also in terms of $\overline{\L^*}$.

\subsection{Uniformly Convex Loss Case}

\begin{lemma}\label{lem:ucvxlow}
For any $\sigma > 0$ and $q' \in (2,\infty)$ we have that,
$$
\Val_n(\H,\Z_{\mrm{ucvx}(\sigma,q')}(\X)) \ge \frac{1}{2 \sigma^{p'-1} D_{p'}^{p'}}\left(\sup_{\tu} \Es{\epsilon}{\norm{\frac{1}{n} \sum_{t=1}^n \epsilon_t \tu_t(\epsilon)}_{\Hd}}\right)^{p'} \ge \frac{1}{2 \sigma^{p'-1} D_{p'}^{p'}} \left(\Val_n(\H,\Z_{\mrm{lin}}(\X))\right)^{p'}
$$
where the supremum is over all $\X$-valued trees $\tu$ of depth $n$.
\end{lemma}

For simplicity we are going to assume now that $\sigma = 1$.  The following theorem shows that for certain $q'$, mirror descent achieves optimal rate. 
\begin{theorem}\label{thm:mdoptucvx}
Given a $q'$, if there exists $V > 0$ and $a > 0$ such that : 
$$
\Val_n(\H,\Z_{\mrm{ucvx}(1,q')})  \le \frac{V}{n^{a}}
$$
then  for any $n > e^{q-1}$ there exists regularizer function $\Psi$ and step-size $\eta$ such that regret of the mirror descent algorithm with these $\Psi$ and $\eta$  against any $\z_1,\ldots,\z_n \in \Z_{\mrm{ucvx}(1,q')}$ or the case when $q' > 2$ is bounded as :
$$
\Reg_n(\Algomd,\z_1,\ldots,\z_n)  \le \min\left\{\frac{12008 (2 V)^{\frac{1}{p'}} D_{p'} \log^2 n}{n^{\frac{a}{p'}}} , \frac{2}{(2 - p') n^{p' - 1}} \right\}  + \frac{\sup_{\h \in \H} R(\h)}{n} ~.
$$
and for $q' = 2$ is bounded as :
$
~~~\Reg_n(\Algomd,\z_1,\ldots,\z_n)  \le \frac{2\log n}{n} + \frac{\sup_{\h \in \H} R(\h)}{n} ~.
$
\end{theorem}

Notice that when $q' = \infty$, $D_{p'} = D_1 = 1$ and so by above theorem, 
$$
\Reg_n(\Algomd,\z_1,\ldots,\z_n)  \le \frac{24016  V  \log^2 n}{n^{a}} + \frac{\sup_{\h \in \H} R(\h)}{n} ~.
$$
Hence for this case mirror descent is near optimal. 

Next notice that by lower bound in lemma \ref{lem:ucvxlow} and assumption of above theorem, we have 
$$
\Val_n(\H,\Z_{\mrm{lin}}(\X)) \le \frac{(2 V)^{\frac{1}{p'}} D_{p'}}{n^{\frac{a}{p'}}}
$$
Now define  $\popt = \sup\{p : V_p <\infty \}$. By above inequality and definition of $\popt$ we see that if for any $p'$, $\frac{a}{p'} > \frac{1}{\qopt}$, then it should be true that $(2 V)^{\frac{1}{p'}} D_{p'} = \infty$.
In other words, it has to be true that for any $p'$, $a \le \frac{p'}{\qopt}$. Hence we can conclude that as $p' \rightarrow \popt$, any algorithm will have a regret of order at least $\frac{1}{n^{\frac{\popt}{\qopt}}} = \frac{1}{n^{\popt - 1}}$. However mirror descent guarantee in Theorem \ref{thm:mdoptucvx} shows that mirror descent algorithm achieves regret at most 
$$
\Reg_n(\Algomd,\z_1,\ldots,\z_n) \le \frac{2}{(2 - \popt) n^{\popt - 1}}   + \frac{\sup_{\h \in \H} R(\h)}{n}
$$
 when $\popt < 2$ and 
 $$
\Reg_n(\Algomd,\z_1,\ldots,\z_n) \le \frac{2 \log n}{n}  + \frac{\sup_{\h \in \H} R(\h)}{n}
$$
when $\popt = 2$. Hence we can conclude that for the case when $p' \rightarrow \popt$ again mirror descent is optimal. (The reason we are taking limit of $p' \rightarrow \popt$ is because one may not even be able to find $\qopt$ uniformly convex functions on the given space that is bounded). This specifically shows that mirror descent is near optimal for strongly convex objectives.

\section{Examples}\label{sec:eg}
We demonstrate our results on several online learning problems,
specified by $\H$ and $\X$.

\subsection{Example : $\ell_{p}$ non-dual pairs} \index{$\ell_p$ norm}
It is usual in the literature to consider the case when $\H$ is the unit ball of the $\ell_p$ norm in some finite dimension $d$ while $\X$ is taken to be the unit ball of the dual norm $\ell_q$ where $p,q$ are H\"older conjugate exponents. Using the machinery developed in this chapter, it becomes
effortless to consider the non-dual case when $\H$ is the unit ball $B_{p_1}$ of some $\ell_{p_1}$ norm while $\X$ is the unit ball $B_{p_2}$ for {\em arbitrary} $p_1,p_2$ in $[1,\infty]$. We shall use $q_1$ and $q_2$ to represent Holder conjugates of $p_1$ and $p_2$.  Before we proceed we first note that for any $r \in (1,2]$, $\psi_r(\h) := \tfrac{1}{2(r-1)} \|\h\|_r^2$ is $2$-uniformly w.r.t. norm $\norm{\cdot}_r$ (see for instance \cite{Shalev07}).  On the other hand by Clarkson's inequality, we have that for $r \in (2,\infty)$, $\psi_r(\h) := \frac{2^r}{r} \|\h\|_r^r$ is $r$-uniformly convex w.r.t. $\norm{\cdot}_r$.  Putting it together we see that for any $r \in (1,\infty)$, the function $\psi_r$ defined above, is $Q$-uniformly convex w.r.t $\norm{\cdot}_r$ for $Q = \max\{r,2\}$.
The basic technique idea is to be to select $\psi_r$ based on the guidelines in the end of the previous section. 
Finally we show that using $\tilde{\psi}_r := d^{Q \max\{\frac{1}{q_2} - \frac{1}{r},0\}} \psi_r$ in Mirror descent Lemma \ref{lem:md} yields the bound that for any $\z_1,\ldots,\z_n \in \Z_{\mrm{Lip}}$:
$$
\Reg_n(\Algomd,\z_{1},\ldots,\z_{n}) \le \frac{2 \max\{2 , \frac{1}{\sqrt{2(r-1)}}\} d^{\max\{\frac{1}{q_2} - \frac{1}{r},0\} + \max\{\frac{1}{r} - \frac{1}{p_1} ,0\}}}{n^{1/\max\{r,2\}}}
$$

The following table summarizes the scenarios where a value of $r=2$,
i.e.~a rate of $D_2/\sqrt{n}$, is possible, and lists the
corresponding values of $D_2$ (up to numeric constant of at most $16$): 
{\footnotesize
\begin{center}
\begin{tabular}{||l||l||l|l||}
\hline 
$p_1$ Range & $p_2$ Range &$~~~~~~~~~~~~D_2$ & Treating as dual in $\|\cdot\|_{p_1}$\\
\hline
$1 \le p_1 \le 2$ & $ p_2 < 2$ &  $1$ & $1/\sqrt{p_1-1}$\\ 
$1 \le p_1 \le 2$ & $ 2 \le p_2 \le \frac{p_1}{p_1- 1}$ & $\sqrt{p_2 - 1}$& $1/\sqrt{p_1-1}$\\
$1 \le p_1 \le 2$ & $\frac{p_1}{p_1-1} < p_2$ & $ d^{\frac{p_2 -1}{p_2} - \frac{1}{p_1}}/\sqrt{p_1 - 1}$ & $ d^{\frac{p_2 -1}{p_2} - \frac{1}{p_1}}/\sqrt{p_1 - 1}$ \\
$p_1 > 2$ & $p_2 < 2$ &  $d^{\frac{1}{2} - \frac{1}{p_1} } $ & $d^{\frac{1}{2} - \frac{1}{p_1} } $\\
 $p_1 > 2$ & $p_2 \ge 2$  & $d^{\frac{p_2 -1}{p_2} - \frac{1}{p_1}} $ & $d^{\frac{p_2 -1}{p_2} - \frac{1}{p_1}} $\\
 $1 \le p_1 \le 2$ & $p_2 = \infty$ & $\sqrt{\log(d)}$ &  $\sqrt{\log(d)}$
\\
 \hline
\end{tabular}
\end{center}
} Note that the first two rows are dimension free, and so apply also
in infinite-dimensional settings, whereas in the other scenarios,
$D_2$ is finite only when the dimension is finite.  An interesting
phenomena occurs when $d$ is $\infty$, $p_1 > 2$ and $q_2 \ge p_1$. In
this case $D_2 = \infty$ and so one cant expect a rate of
$O(\frac{1}{\sqrt{n}})$. However we have $D_{p_2} < 16$ and so can
still get a rate of $n^{-\frac{1}{q_2}}$.  

Ball et al \cite{BallCaLi94} tightly calculate the constants of strong
convexity of squared $\ell_p$ norms, establishing the tightness of
$D_2$ when $p_1=p_2$.  By extending their constructions it is also
possible to show tightness (up to a factor of 16) for all other values
in the table.  Also, Agarwal et al \cite{AgaBarRavWai11} recently
showed lower bounds on the sample complexity of stochastic
optimization when $p_1 = \infty$ and $p_2$ is arbitrary---their lower
bounds match the last two rows in the table.

\subsection{Example :  Non-dual Schatten norm pairs in finite dimensions} \index{Schatten norm}
Exactly the same analysis as above can be carried out for Schatten
$p$-norms, i.e.~when $\H = B_{S(p_1)}$, $\X = B_{S(p_2)}$ are the unit
balls of Schatten $p$-norm (the $p$-norm of the singular values) for matrix of dimensions $d_1 \times d_2$.  We get the same results as in the table
above (as upper bounds on $D_2$), with $d =\min\{d_1,d_2\}$. These results again follow using similar arguments as $\ell_p$ case and tight constants for strong convexity parameters of the Schatten norm from \cite{BallCaLi94}.

\subsection{Example : Non-dual group norm pairs in finite dimensions}\index{group norm}\index{mixed norm}
In applications such as multitask learning, groups norms such as
$\|\h\|_{q,1}$ are often used on matrices $\h \in \reals^{k \times d}$
where $(q,1)$ norm means taking the $\ell_1$-norm of the
$\ell_q$-norms of the columns of $\h$. Popular choices include $q =
2,\infty$. Here, it may be quite unnatural to use the dual norm
$(p,\infty)$ to define the space $\X$ where the data lives. For
instance, we might want to consider $\H = B_{(q,1)}$ and $\X =
B_{(\infty,\infty)} = B_{\infty}$.  In such a case we can calculate
that $D_2(\H,\X)=\Theta(k^{1 - \frac{1}{q}}\sqrt{\log(d)}) $ using $\Psi(\h) = \frac{1}{q + r-2}\norm{\h}^2_{q,r}$ where $r = \frac{\log \ d}{\log \ d - 1}$.


\subsection{Example : Max Norm}\index{max norm}
Max-norm has been proposed as a convex matrix regularizer for
application such as matrix completion \cite{SreRenJaa05}.  In the
online version of the matrix completion problem at each time step one
element of the matrix is revealed, corresponding to $\X$ being the set
of all matrices with a single element being $1$ and the rest $0$.
Since we need $\X$ to be convex we can take the absolute convex hull
of this set and use $\X$ to be the unit element-wise $\ell_1$ ball.
Its dual is $\norm{W}_{\Xd} = \max_{i,j} |W_{i,j}|$.  On the other hand
given a matrix $W$, its max-norm is given by $\norm{W}_{\mathrm{max}}
= \min_{U,V : W = UV^\top} \left(\max_{i} \norm{U_i}_2\right) \left(
  \max_j \norm{V_j}_2 \right)$. The set $\H$ is the unit ball under
the max norm.  As noted in \cite{SreShr05} the max-norm ball is
equivalent, up to a factor two, to the convex hull of all rank one
sign matrices.  Let us now make a more general observation. 

\begin{proposition}
Let $\H = \mathrm{abscvx}(\{\h_1,\ldots,\h_K\})$. The Minkowski norm for this $\H$ is given by 
$$ 
\norm{\h}_\H := \inf_{\alpha_1 ,\ldots, \alpha_K : \h = \sum_{i=1}^K \alpha_i  \h_i}\sum_{i=1}^K |\alpha_i| 
$$  
In this case, for any $q \in (1,2]$, if we define the norm :
\begin{align}\label{eq:maxreg}
\norm{\h}_{\H,q} = \inf_{\alpha_1 ,\ldots, \alpha_K : \h  = \sum_{i=1}^K \alpha_i \h_i}\left( \sum_{i=1}^K |\alpha_i|^q
\right)^{1/q} 
\end{align}
then the function  $\Psi(\h) = \frac{1}{2 (q-1)}\norm{\h}_{\H,q}^2$ is $2$-uniformly convex w.r.t. $\norm{\cdot}_{\H,q}$. Further if we use $q =
\frac{\log K}{ \log K - 1}$, then $\sup_{\h \in \H} \sqrt{\Psi(\h)} =
O(\sqrt{\log K})$.
\end{proposition}
    
Proof of the above proposition is similar to proof of strong convexity of $\ell_q$ norms.  For the max norm case as noted before the norm is equivalent to the norm got by the taking the absolute convex hull of 
the set of all rank one sign matrices. Cardinality of this set is of course $2^{N+M}$. Hence using the above proposition and noting that $\Xd$ is the unit ball of $|\cdot|_{\infty}$ we see that $\Psi$ is obviously $2$-uniformly convex w.r.t. $\norm{\cdot}_{\Xd}$ and so we get a regret bound $O\left(\sqrt{\frac{M+N}{n}} \right)$.  This matches the stochastic
(PAC) learning guarantee \cite{SreShr05}, and is the first guarantee we are
aware of for the max norm matrix completion problem in the online setting.  \index{cut norm}

\subsection{Example : Interpolation Norms } \index{interpolation norm}
Another interesting setting is when the set $\H$ is got by interpolating between unit balls of two other norms $\norm{\cdot}_{\H_1}$ and $\norm{\cdot}_{\H_2}$. Specifically one can consider $\H$ to be the unit ball of two such interpolated norms, the first type of interpolation norm is given by,
\begin{align}\label{eq:norminter1}
\norm{\h}_{\H} = \norm{\h}_{\H_1} + \norm{\h}_{\H_2}
\end{align}
The second type of interpolation norm one can consider is given by 
\begin{align}\label{eq:norminter2}
\norm{\h}_{\H} = \inf_{\h_1 + \h_2 = \h}\left(\norm{\h_1}_{\H_1} + \norm{\h_2}_{\H_2}\right)
\end{align}
In learning problems such interpolation norms are often used to induce certain structures or properties into the regularization. For instance one might want sparsity along with grouping effect in the linear predictors for which elastic-net type regularization introduced by Zou and Hastie \cite{ZouHastie05} (this is captured by interpolation of the first type between $\ell_1$ and $\ell_2$ norms). Another example is in matrix completion problems when we would like the predictor matrix to be decomposable into sum of sparse and low rank matrices as done by Chanrdasekaran et. al \cite{ChaSanParWil09} (here one can use the interpolation norm of second type to interpolate between trace norm and element wise $\ell_1$ norm). Another example where interpolation norms of type two are useful are in 
multi-task learning problems (with linear predictors) as done by Jalali et. al \cite{JalRavSanRua10}.  The basic idea is that the matrix of linear predictors can is decomposed into sum of two matrices one with for instance low entry-wise $\ell_1$ norm and other with low $B_{(2,\infty)}$ group norm (group sparsity).

\noindent While in these applications the set $\H$ used is obtained through interpolation norms, it is typically not natural for the set $\X$ to be the dual ball of $\H$ but rather something more suited to the problem at hand. For instance, for the elastic net regularization case, the set $\X$ usually considered are either the vectors with bounded $\ell_\infty$ norm or bounded $\ell_2$. Similarly for the \cite{JalRavSanRua10} case $\X$ could be either matrices with bounded entries or some other natural assumption that suits the problem.

\noindent  It can be shown that in general for any interpolation norm of first type specified in Equation \ref{eq:norminter1}, 
\begin{align}\label{eq:bndinter1}
D_2(\H,\X) \le 2 \min\{D_2(\H_1,\X), D_2(\H_2,\X) \}
\end{align}
Similarly for the interpolation norm of type two one can in general show that,
\begin{align}\label{eq:bndinter2}
D_2(\H,\X) \le \frac{1}{2} \max\{D_2(\H_1,\X), D_2(\H_2,\X) \}
\end{align}
Using the above bounds one can get regret bounds for mirror descent algorithm with appropriate $\Psi$ and step size $\eta$ for specific examples like the ones mentioned.\\

\noindent The bounds given in Equations \eqref{eq:bndinter1} and \eqref{eq:bndinter2} are only upper bounds and it would be interesting to analyze these cases in more detail and also to analyze interpolation between several norms instead of just two.

\section{Detailed Proofs}\label{sec:onproof}

\begin{proof}[Proof of Lemma \ref{lem:1storacle}]
Let $\Algo : \bigcup_{n \in \mbb{N}} \X^n \mapsto \cH$ be any proper learning algorithm for the linear learning problem specified by instance set $\Z_{\mrm{lin}}(\X)$.  One can use this learning algorithm to construct a new gradient-based learning algorithm  $\Algo^{\Oracle^{\mrm{1st}}} : \bigcup_{n \in \mbb{N}} \X^{n} \mapsto \cH$ as follows, for any $t \in \mathbb{N}$ and any $\x_1,\ldots,\x_{t} \in \X$,  
$$
\Algo^{\Oracle^{\mrm{1st}}}\left(\z_1,\ldots,\z_{t}\right) = \Algo(\nabla \ell(\h_1,\z_1),\ldots,\nabla \ell(\h_t,\z_t))~.
$$
That is, at the end of each round the learning algorithm for linear problem is fed with a sub-gradient of the loss at the hypothesis $\h_t$ played on that round $t$ and thus $\h_{t+1}$ is selected using this algorithm. Note that for any $\hopt \in \cH$, by convexity of the loss function,
\begin{align*}
& \frac{1}{n} \sum_{t=1}^n \ell\left(\Algo^{\Oracle^{\mrm{1st}}}\left(\nabla \ell(\h_1,\z_1), \ldots, \nabla \ell(\h_{t-1},\z_{t-1})\right),\z_t\right) -  \frac{1}{n} \sum_{t=1}^n \ell\left(\hopt,\z_t\right) \\
& ~~~~~~~~~~  \le  \frac{1}{n} \sum_{t=1}^n \ip{\nabla \ell\left(\Algo^{\Oracle^{\mrm{1st}}}\left(\nabla \ell(\h_1,\z_1), \ldots, \nabla \ell(\h_{t-1},\z_{t-1})\right),\z_t \right) }{\Algo^{\Oracle^{\mrm{1st}}}\left(\nabla \ell(\h_1,\z_1), \ldots, \nabla \ell(\h_{t-1},\z_{t-1})\right) - \hopt}\\
& ~~~~~~~~~~  =  \frac{1}{n} \sum_{t=1}^n \ip{\x^*_t}{\Algo\left(\x^*_1, \ldots, \x^*_{t-1}\right) - \hopt}\\
& ~~~~~~~~~~  \le  \frac{1}{n} \sum_{t=1}^n \ip{\x^*_t }{\Algo\left(\x^*_1, \ldots, \x^*_{t-1}\right)}  - \inf_{\h \in \cH}  \frac{1}{n} \sum_{t=1}^n \ip{\x^*_t }{\h}\\
& ~~~~~~~~~~  \le  \sup_{\x_1,\ldots,\x_n \in \X} \left\{ \frac{1}{n} \sum_{t=1}^n \ip{\x_t }{\Algo\left(\x_1, \ldots, \x_{t-1}\right)}  - \inf_{\h \in \cH}  \frac{1}{n} \sum_{t=1}^n \ip{\x_t }{\h} \right\}
\end{align*}
where in the above we used the notation $\x^*_t = \nabla \ell(\h_t,\z_t)$
Thus we can conclude that the new first-order oracle-based learning algorithm enjoys the same regret guarantee as the algorithm $\Algo$ enjoys on linear learning problems.  From this we conclude the lemma statement.
\end{proof}

\begin{proof}[Proof of Corollary \ref{cor:val}]
First by Lemma \ref{lem:1storacle} we have that for any algorithm $\Algo$ for linear learning problem, there exists an oracle based learning algorithm $\Algo^{\Oracle^\mrm{1st}}$ such that
$$
\sup_{\z_{1},\ldots,\z_n \in \Z} \Reg_n(\Algo^{\Oracle^\mrm{1st}},\z_{1:n}) \le \sup_{\z^*_{1},\ldots,\z_n \in \Z_{\mrm{lin}}} \Reg_n(\Algo,\z^*_{1:n}) 
$$
and so 
$$
\Val_n(\cH,\Z) \le \Val_n(\cH,\Z_{\mathrm{lin}}(\X))
$$
As for the set of equalities for specific classes, $\Z_{\mrm{sup}}$, $\Z_{\mrm{Lip}}$ and $\Z_{\mrm{lin}}$, note that since each of these classes are such that sub-gradients of losses belong to $\X$, we have that
$$
\Val_n(\cH,\Z_{\mrm{sup}}(\X)) \le \Val_n(\cH,\Z_{\mrm{Lip}}(\X)) \le \Val_n(\cH,\Z_{\mrm{lin}}(\X))
$$
On the other hand, $\Z_{\mrm{lin}}(\X) \subset \Z_{\mrm{Lip}}(\X)$ and so 
$
\Val_n(\cH,\Z_{\mrm{Lip}}(\X)) \ge \Val_n(\cH,\Z_{\mrm{lin}}(\X))
$
and hence we can conclude that 
$$
\Val_n(\cH,\Z_{\mrm{Lip}}(\X)) = \Val_n(\cH,\Z_{\mrm{lin}}(\X))~.
$$
Similarly for the supervised learning problem specified by instance set $\Z_{\mrm{sup}}$, note that if adversary always pick targets $y_t = -b = - \sup_{\h \in \cH, \x \in \X} \ip{\h}{\x}$ then for any $\x_1,\ldots,\x_n \in \X$,
\begin{align*}
\frac{1}{n} \sum_{t=1}^n \left|\ip{\x_t}{\h_t} - y_t\right| - \inf_{\h \in \cH} \frac{1}{n} \sum_{t=1}^n \left|\ip{\x_t}{\h} - y_t\right| & = \frac{1}{n} \sum_{t=1}^n \ip{\x_t}{\h_t} +b - \inf_{\h \in \cH} \frac{1}{n}\left( \sum_{t=1}^n \ip{\x_t}{\h} +b \right)\\
& = \frac{1}{n} \sum_{t=1}^n \ip{\x_t}{\h_t} - \inf_{\h \in \cH} \frac{1}{n} \sum_{t=1}^n \ip{\x_t}{\h}
\end{align*}
Thus we can conclude that 
$
~\Val_n(\cH,\Z_{\mrm{sup}}(\X)) \ge \Val_n(\cH,\Z_{\mrm{lin}}(\X))
$ and so, $\Val_n(\cH,\Z_{\mrm{sup}}(\X)) = \Val_n(\cH,\Z_{\mrm{lin}}(\X))$. Thus we conclude the corollary.
\end{proof}

\begin{lemma}\label{lem:construct2}
Let $1 < p \le 2$ and $C > 0$ be fixed constants, the following statements are equivalent : 
\begin{enumerate}
\item For all $\Bd$-valued tree $\tx$ of infinite depth and any $\x_0 \in \Bd$:
\begin{align*}
\sup_{n} \E{\norm{\x_0 + \sum_{i=1}^n \epsilon_i \tx_i(\epsilon)}_{\Hd}^p} \le C^p \left(\norm{\x_0}_{\X}^p + \sum_{n \ge 1} \E{\norm{\tx_n(\epsilon)}_{\X}^p} \right)
\end{align*}
\item There exist a non-negative convex function $\Psi$ on $\B$ with $\Psi(0) = 0$, that is $q$-uniformly convex w.r.t. norm $\norm{\cdot}_{\Xd}$ and for any $\h \in \B$, $ \frac{1}{q} \norm{\h}_{\Xd}^q \le \Psi(\h) \le \frac{C^q}{q} \norm{\h}_{\H}^q $.
\end{enumerate}
\end{lemma}
\begin{proof}
For any $\x \in \Bd$ define $\Psi^* : \Bd \mapsto \mbb{R}$ as
\begin{align*}
\Psi^*(\x ) := \sup\left\{\left(\frac{1}{C^p} \sup_{n} \E{\norm{\x + \sum_{i=1}^n \epsilon_i \tx_i(\epsilon)}_{\Hd}^p} - \sum_{i \ge 1} \E{\norm{\tx_i(\epsilon)}_\X^p}\right)\right\}
\end{align*}
where the supremum is over $\Bd$-valued tree $\tx$ of infinite depth  such that, $\underset{n}{\sup}\ \E{\norm{\x + \sum_{i=1}^n \x_i}_{\Hd}^p} < \infty$. Since supremum of convex functions is a convex function, it is easily verified that $\Psi^*(\cdot)$ is convex. Note that by the definition of M-type in Equation \ref{eq:mtype}, we have that for any $\x_0 \in \Bd$, $\Psi^*(\x_0) \le \norm{\x_0}_\X^p$. On the other hand, note that by considering the sequence of constant mappings, $\x_i = 0$ for all $i \ge 1$, we get that for any $\x_0 \in \Bd$,
\begin{align*}
\Psi^*(\x_0) &= \sup\left\{\left(\frac{1}{C^p} \sup_{n} \E{\norm{\x_0 + \sum_{i=1}^n \epsilon_i \tx_i(\epsilon)}_{\Hd}^p} - \sum_{i \ge 1} \E{\norm{\tx_i(\epsilon)}_\X^p}\right)\right\} \ge \frac{1}{C^p} \norm{\x_0}_{\Hd}^p
\end{align*}
Thus we can conclude that for any $\x \in \Bd$, $\frac{1}{C^p} \norm{\x}_{\Hd}^p \le \Psi^*(\x) \le \norm{\x}_{\X}^p $. \\

\noindent For any $\x_0 , \y_0 \in \Bd$, by definition of $\Psi^*(\x_0)$ and $\Psi^*(\y_0)$, for any $\gamma > 0$, there exist $\Bd$-valued trees $\tx$ and $\ty$ of infinite depth s.t. :
\begin{align*}
\Psi^*(\x_0) &\le \left(\frac{1}{C^p} \sup_{n} \E{\norm{\x_0 + \sum_{i =1}^n \epsilon_i \tx_{i}(\epsilon)}_{\Hd}^p} - \sum_{i \ge 1} \E{\norm{\tx_i(\epsilon)}_\X^p}\right) + \gamma
\end{align*}
and 
\begin{align*}
\Psi^*(\y_0^{(j)}) &\le \left(\frac{1}{C^p} \sup_{n} \E{\norm{\y_0 + \sum_{i=1}^n \epsilon_i \ty_i(\epsilon) }_{\Hd}^p} - \sum_{i \ge 1} \E{\norm{\ty_i(\epsilon)}_\X^p}\right) + \gamma
\end{align*}
In fact in the above two inequalities if the supremum over $n$ were achieved at some finite $n_0$, by replacing the original sequence by one which is identical up to $n_0$ and for any $i > n_0$ using $\tx_i(\epsilon) = 0$ (and similarly $\ty_i(\epsilon) = 0$), we can in fact conclude that using these $\x$'s and $\y$'s instead,
\begin{align}\label{eq:gx}
\Psi^*(\x_0) &\le \left(\frac{1}{C^p} \E{\norm{\x_0 + \sum_{i \ge 1} \epsilon_i \tx_{i}(\epsilon)}_{\Hd}^p} - \sum_{i \ge 1} \E{\norm{\tx_i(\epsilon)}_\X^p}\right) + \gamma
\end{align}
and 
\begin{align}\label{eq:gy}
\Psi^*(\y_0^{(j)}) &\le \left(\frac{1}{C^p} \E{\norm{\y_0 + \sum_{i\ge 1} \epsilon_i \ty_i(\epsilon) }_{\Hd}^p} - \sum_{i \ge 1} \E{\norm{\ty_i(\epsilon)}_\X^p}\right) + \gamma
\end{align}

Now consider a sequence formed by taking $\z_0 = \frac{\x_0 + \y_0}{2}$ and further let 
\begin{align*}
\z_1 = \left(\frac{1 + \epsilon_0}{2}\right) \frac{\x_0 - \y_0}{2}  + \left(\frac{1 - \epsilon_0}{2}\right) \frac{\y_0 - \x_0}{2} = \epsilon_0 (\x_0 - \y_0)
\end{align*}
and for any $i \ge 2$, define 
$$
\tz_i = \left(\frac{1 + \epsilon_0}{2}\right) \epsilon_{i-1} \tx_{i-1}(\epsilon) + \left(\frac{1 - \epsilon_0}{2}\right) \epsilon_{i-1} \ty_{i-1}(\epsilon)
$$
where $\epsilon_0 \in \{\pm 1\}$ is drawn uniformly at random. That is essentially at time $i=1$ we flip a coin and decide to go with tree $\tx$ with probability $1/2$ and $\ty$ with probability $1/2$. Clearly using the tree $\tz$, we have that,
\begin{align*}
& \Psi^*\left(\frac{\x_0 + \y_0}{2}\right)  = \sup_{\tz}\left\{\left(\frac{1}{C^p} \sup_{n} \E{\norm{\frac{\x_0 + \y_0}{2} + \sum_{i=1}^n  \tz_i(\epsilon_0,\epsilon)}_{\Hd}^p} - \sum_{i \ge 1} \E{\norm{\tz_i(\epsilon_0,\epsilon)}_\X^p}\right)^{1/p} \right\}^p \\
& ~~~~~~~~  \ge \frac{1}{C^p} \E{\norm{\z_0 + \sum_{i\ge 1} \tz_i(\epsilon_0,\epsilon) }_{\Hd}^p} - \sum_{i \ge 1} \E{\norm{\tz_i(\epsilon_0,\epsilon)}_\X^p}\\
& ~~~~~~~~  = \frac{1}{C^p} \frac{\E{\norm{\x_0 + \sum_{i\ge 1} \epsilon_i \tx_i(\epsilon) }_{\Hd}^p} + \E{\norm{\y_0 + \sum_{i\ge 1} \epsilon_i \ty_i(\epsilon) }_{\Hd}^p}}{2} - \sum_{i \ge 1} \E{\norm{\tz_i(\epsilon_0,\epsilon)}_\X^p}\\
& ~~~~~~~~  = \frac{1}{C^p} \frac{ \E{\norm{\x_0 + \sum_{i\ge 1} \epsilon_i \tx_i(\epsilon) }_{\Hd}^p} +  \E{\norm{\y_0 + \sum_{i\ge1} \epsilon_i \ty_i(\epsilon) }_{\Hd}^p}}{2} - \sum_{i \ge 1} \E{\norm{\tz_i(\epsilon_0,\epsilon)}_\X^p}\\
& ~~~~~~~~  = \frac{1}{C^p} \frac{ \E{\norm{\x_0 + \sum_{i\ge1} \epsilon_i \tx_i(\epsilon) }_{\Hd}^p} +  \E{\norm{\y_0 + \sum_{i\ge1} \epsilon_i \ty_i(\epsilon) }_{\Hd}^p}}{2} -\norm{\frac{\x_0 - \y_0}{2}}_{\X}^p - \sum_{i \ge 2} \E{\norm{\tz_i(\epsilon_0,\epsilon)}_\X^p}\\
& ~~~~~~~~  = \frac{1}{C^p} \frac{ \E{\norm{\x_0 + \sum_{i\ge1} \epsilon_i \tx_i(\epsilon) }_{\Hd}^p} + \E{\norm{\y_0 + \sum_{i\ge1} \epsilon_i \ty_i(\epsilon) }_{\Hd}^p}}{2} -\norm{\frac{\x_0 - \y_0}{2}}_{\X}^p \\
& ~~~~~~~~  ~~~~~~~~~~ - \sum_{i \ge 1} \frac{\E{\norm{\tx_i(\epsilon)}_\X^p} + \E{\norm{\ty_i(\epsilon)}_\X^p}}{2}\\ 
&  ~~~~~~~~  =  \frac{\frac{1}{C^p} \En{\norm{\x_0 + \sum_{i\ge1} \epsilon_i \tx_i(\epsilon) }_{\Hd}^p} - \sum_{i \ge 1} \En{\norm{\tx_i(\epsilon)}}_\X^p + \frac{1}{C^p}  \En{\norm{\y_0 + \sum_{i\ge1} \epsilon_i \ty_i(\epsilon) }_{\Hd}^p} - \sum_{i \ge 1} \En{\norm{\ty_i(\epsilon)}_\X^p} }{2} \\
& ~~~~~~~~  ~~~~~~~~~~ -\norm{\frac{\x_0 - \y_0}{2}}_{\X}^p \\
& ~~~~~~~~  \ge  \frac{\Psi^*(\x_0) + \Psi^*(\y_0)}{2}  -\norm{\frac{\x_0 - \y_0}{2}}_{\X}^p - \gamma
\end{align*}
where the last step is obtained by using Equations \ref{eq:gx} and \ref{eq:gy}. Since $\gamma$ was arbitrary taking limit we conclude that for any $\x_0$ and $\y_0$, 
$$
\frac{\Psi^*(\x_0) + \Psi^*(\y_0)}{2}  \le \Psi^*\left(\frac{\x_0 + \y_0}{2}\right)  + \norm{\frac{\x_0 - \y_0}{2}}_{\X}^p
$$
Hence we have shown the existence of a convex function $\Psi^*$ that is $p$-uniformly smooth w.r.t. norm $\norm{\cdot}_\X$ such that $\frac{1}{C^p} \norm{\cdot}_{\Hd}^p \le \Psi^*(\cdot) \le \norm{\cdot}_{\X}^p$. Using convex duality we can conclude that the convex conjugate $\Psi$ of function $\Psi^*$, is $q$-uniformly convex w.r.t. norm $\|\cdot\|_{\Xd}$ and is such that $\norm{\cdot}_{\X}^q \le \Psi(\cdot) \le C^q \norm{\cdot}_{\H}^q$. 
That $2$ implies $1$ can be easily verified using the smoothness property of $\Psi^*$.

\end{proof}

\noindent The following sequence of four lemma's give us the essentials towards proving Lemma \ref{lem:mn}. They use similar techniques as in \cite{Pisier75}.

\begin{lemma}
Let $1 < r \le 2$. If there exists a constant $D > 0$ such that any $\x_0 \in \Bd$ and any $\Bd$-valued tree $\tx$ of infinite depth satisfies :
\begin{align*}
\forall n \in \mathbb{N}, ~~~~~  \E{\norm{\x_0 + \sum_{i=1}^n \epsilon_i \tx_i(\epsilon)}_{\Hd}} \le D (n + 1)^{1/r} \sup_{0 \le i \le n} \sup_{\epsilon} \norm{\tx_i(\epsilon)}_{\X}
\end{align*}
then for all $p < r$ and $\alpha_p = \frac{20 D}{r - p} $ we can conclude that any $\x_0 \in \Bd$ and any $\Bd$-valued tree $\tx$ of infinite depth will satisfy :
\begin{align*}
\sup_n \E{\norm{\x_0 + \sum_{i=1}^n \epsilon_i \tx_i(\epsilon)}_{\Hd}} \le \alpha_p  \sup_{\epsilon} \left(\sum_{i\ge 0}\norm{\tx_i(\epsilon)}_{\X}^p\right)^{1/p} 
\end{align*}  
\end{lemma}
\begin{proof}
To begin with note that in the definition of type, if the supremum over $n$ were achieved at some finite $n_0$, then by replacing the original sequence by one which is identical up to $n_0$ and then on for any $i > n_0$ using $\tx_i(\epsilon) = 0$ would only tighten the inequality. Hence it suffices to only consider such sequences. Further to prove the statement we only need to consider finite such sequences (ie. sequences such that there exists some $n$ so that for any $i > n$, $\tx_i = 0$) and show that the inequality holds for every such $n$ (every such sequence). 

Restricting ourselves to such finite sequences, we now use the shorthand,\\
$
S = \sup_{\epsilon} \left(\sum_{i=0}^n \norm{\tx_i(\epsilon)}_{\X}^p \right)^{1/p}
$.
Now define
\begin{align*}
& I_k(\epsilon) = \left\{i \ge 0 \middle| \tfrac{S}{2^{(k+1)/p}} < \|\tx_i(\epsilon)\|_\X  \le \tfrac{S}{2^{k/p}}\right\}~~,\\
& T_0^{(k)}(\epsilon) = \inf\{i \in I_k(\epsilon)\}~~\textrm{and}\\
& \forall m \in \mathbb{N},\ T_m^{(k)}(\epsilon) = \inf\{i > T_{m-1}^{(k)}(\epsilon), i \in I_k(\epsilon)\}
\end{align*}
Note that for any $\epsilon \in \{\pm1\}^{\mathbb{N}}$,
$$
S^p \ge \sum_{i \in I_k(\epsilon)} \|\tx_i(\epsilon)\|_\X^p > \tfrac{S^p\ |I_k(\epsilon)|}{2^{(k+1)}} 
$$
and so we get that $\sup_{\epsilon} |I_k(\epsilon)| < 2^{k+1}$. From this we conclude that

\begin{align*}
\E{\norm{\x_0 + \sum_{i=1}^n \epsilon_i \tx_i(\epsilon)}_{\Hd}} & \le \sum_{k \ge 0}\E{\norm{\sum_{i\in I_k(\epsilon)} \epsilon_i  \tx_i(\epsilon)}_{\Hd}} \\
& = \sum_{k \ge 0}\E{\norm{\sum_{i \ge 0} \epsilon_{T_i^{(k)}(\epsilon)} \tx_{T_i^{(k)}(\epsilon)}}_{\Hd}} \\
& \le \sum_{k \ge 0} \left( D\ \sup_{\epsilon}\{ |I_k(\epsilon)|^{1/r}\} \sup_{\epsilon}\{\sup_{i \in I_k(\epsilon)} \norm{\tx_i(\epsilon)}_{\X}\} \right) \\
& \le \sum_{k \ge 0} \left( D\ 2^{(k+1)/r} \sup_{\epsilon} \sup_{i \in I_k(\epsilon)} \norm{\tx_i(\epsilon)}_{\X,\infty} \right) \\
& \le \sum_{k \ge 0} \left( D\ 2^{(k+1)/r}\  2^{-k/p} S \right) \\
& = D\ 2^{1/r} \ \sum_{k \ge 0}2^{k (\frac{1}{r} - \frac{1}{p})}\ S\\
& \le  \frac{2 D}{1 - 2^{(\frac{1}{r} - \frac{1}{p})}} S\\
& \le \frac{2 D}{1 - 2^{-(r-p)/4}} S\\
& \le \frac{12 D}{r - p} S\\
& = \alpha_p \sup_{\epsilon}\left(\sum_{i=0}^n \norm{\tx_i(\epsilon)}_{\X}^p \right)^{1/p}
\end{align*}
\end{proof}

\begin{lemma}\label{lem:sub1}
Let $1 < r \le 2$. If there exists a constant $D > 0$ such that any $\x_0 \in \Bd$ and any $\Bd$-valued tree $\tx$ of infinite depth satisfies :
\begin{align*}
\forall n \in \mathbb{N}, ~~~~~  \E{\norm{\x_0 + \sum_{i=1}^n \epsilon_i \tx_i(\epsilon)}_{\Hd}} \le D (n + 1)^{1/r} \sup_{0 \le i \le n} \sup_{\epsilon} \norm{\tx_i(\epsilon)}_{\X}
\end{align*}
then for any $p < r$, any $\x_0 \in \Bd$ and any $\Bd$-valued tree $\tx$ of infinite depth :
\begin{align*}
\mathbb{P}\left(\sup_{n} \norm{\x_0 + \sum_{i=1}^n \epsilon_i \tx_i(\epsilon)}_{\Hd} > c \right)  \le  2 \left( \frac{\alpha_p}{c} \right)^{p/(p+1)} \left(\norm{\x_0}_{\X}^p + \sum_{i \ge 1} \E{\norm{\tx_i(\epsilon)}_{\X}^p} \right)^{1/(p+1)}
\end{align*}
\end{lemma}
\begin{proof}
For any $\x_0 \in \Bd$ and $\Bd$-valued tree $\tx$ of infinite depth define
$$
V_n(\epsilon) = \sum_{i=0}^n \norm{\tx_i(\epsilon)}_{\X}^p
$$
For appropriate choice of $a > 0$ to be fixed later, define stopping time 
$$
\tau(\epsilon) = \inf\left\{ n \ge 0 \middle| V_{n+1} > a^p \right\}
$$
Now for any $c > 0$  we have,
{\small
\begin{align}\label{eq:part}
\mathbb{P}\left(\sup_n \norm{\x_0 + \sum_{i=1}^n \epsilon_i \tx_i(\epsilon)}_{\Hd} \hspace{-0.3cm}> c\right)  & \le \mathbb{P}(\tau(\epsilon) < \infty) + \mathbb{P}\left( \tau(\epsilon) = \infty , \sup_{n} \norm{\sum_{i=0}^n \epsilon_i \tx_i(\epsilon)}_{\Hd} \hspace{-0.3cm}> c \right) \notag\\
& \hspace{-0.25in}\le \mathbb{P}(\tau(\epsilon) < \infty) + \mathbb{P}\left( \tau(\epsilon) >0,\ \sup_{n} \norm{ \x_{0} + \sum_{i=1}^{n\wedge \tau(\epsilon)} \epsilon_i \tx_i(\epsilon)}_{\Hd} \hspace{-0.3cm}> c \right)
\end{align}
}
As for the first term in the above equation note that 
\begin{align}\label{eq:p1}
\mathbb{P}(\tau(\epsilon) < \infty) = \mathbb{P}(\sup_{n} V_n > a^p) \le  \frac{\norm{\x_0}_{\X}^p + \sum_{i \ge 1} \E{\norm{\tx_i(\epsilon)}_{\X}^p}}{a^p}
\end{align}
To consider the second term of Equation \ref{eq:part} we note that $\left(\ind{\tau(\epsilon) > 0} (\x_0 + \sum_{i=1}^{n \wedge \tau(\epsilon)} \epsilon_i \tx_i(\epsilon) ) \right)_{n \ge 0}$ is a valid martingale (stopped process) and hence, $\left(\norm{\ind{\tau(\epsilon) > 0}(\x_0 + \sum_{i=1}^{n \wedge \tau(\epsilon)} \epsilon_i \tx_i(\epsilon))}_{\Hd}\right)_{n \ge 0}$ is a sub-matingale. Hence by Doob's inequality we conclude that,
\begin{align*}
\mathbb{P}\left(T > 0, \ \sup_n \norm{\x_0 + \sum_{i=1}^{n \wedge \tau(\epsilon)} \epsilon_i \tx_i(\epsilon)}_{\Hd} > c\right) & \le \frac{1}{c} \sup_n \E{\norm{\ind{\tau(\epsilon) > 0}\left( \x_0 + \sum_{i=1}^{n \wedge \tau(\epsilon)} \epsilon_i \tx_i(\epsilon)\right)}_{\Hd}}
\end{align*}
Applying conclusion of the previous lemma we get that
\begin{align*}
\mathbb{P}\left(T > 0, \ \sup_n \norm{\x_0 + \sum_{i=1}^{n \wedge \tau(\epsilon)} \epsilon_i \tx_i(\epsilon)}_{\Hd} > c\right) & \le \frac{\alpha_p}{c} \sup_{\epsilon}\left(\ind{\tau(\epsilon) > 0} \left(\norm{\x_0}^p_\X + \sum_{i=1}^{\tau(\epsilon)} \norm{\tx_i(\epsilon)}^p_\X \right) \right)^{1/p}\\
& \le \frac{\alpha_p}{c} (a^p)^{1/p} = \frac{\alpha_p\ a}{c} 
\end{align*}
Plugging the above and Equation \ref{eq:p1} into Equation \ref{eq:part} we conclude that:
\begin{align*}
\mathbb{P}\left(\sup_n \norm{ \x_0 + \sum_{i=1}^n \epsilon_i \tx_i(\epsilon)}_{\Hd} > c\right) & \le  \frac{\norm{\x_0}_{\X}^p + \sum_{i \ge 1} \E{\norm{\tx_i(\epsilon)}_{\X}^p}}{a^p} + \frac{\alpha_p \ a}{c}
\end{align*}
Using $a = \left(\frac{c}{\alpha_p} \left(\norm{\x_0}_{\X}^p + \sum_{i \ge 1} \E{\norm{\tx_i(\epsilon)}_{\X}^p}\right) \right)^{1/(p+1)}$ we conclude that
\begin{align*}
\mathbb{P}\left(\sup_n \norm{\x_0 + \sum_{i=1}^n \epsilon_i \tx_i(\epsilon)}_{\Hd} > c\right) & \le  2 \left( \frac{\alpha_p}{c} \right)^{p/(p+1)} \left(\norm{\x_0}_{\X}^p + \sum_{i \ge 1} \E{\norm{\tx_i(\epsilon)}_{\X}^p} \right)^{1/(p+1)}
\end{align*}
This conclude the proof.
\end{proof}

\begin{lemma}\label{lem:sub2}
Let $1 < r \le 2$. If there exists a constant $D > 0$ such that any $\x_0 \in \Bd$ and any $\Bd$-valued tree $\tx$ of infinite depth satisfies :
\begin{align*}
\forall n \in \mathbb{N}, ~~~~~  \E{\norm{\x_0 + \sum_{i=1}^n \epsilon_i \tx_i(\epsilon)}_{\Hd}} \le D (n + 1)^{1/r} \sup_{0 \le i \le n} \sup_{\epsilon} \norm{\tx_i(\epsilon)}_{\X}
\end{align*}
then for any $p < r$, any $\x_0 \in \Bd$ and $\Bd$-valued tree $\tx$ of infinite depth will satisfy :
\begin{align*}
& \sup_{\lambda > 0} \lambda^p\ \mathbb{P}\left( \sup_{n} \norm{\x_0 + \sum_{i=1}^n \epsilon_i \tx_i(\epsilon)}_{\Hd} > \lambda \right) \\
& ~~~~~~~~~~ \le \max\left\{4^{\frac{p+1}{p}}  \alpha_p \left(\norm{\x_0}_{\X}^p + \sum_{i \ge 1} \E{\norm{\tx_i(\epsilon)}_{\X}^p} \right)^{\frac{1}{p}}, 2^{2p + 3} \log(2)\   \alpha_p^p\ \left(\norm{\x_0}_{\X}^p + \sum_{i \ge 1} \E{\norm{\tx_i(\epsilon)}_{\X}^p} \right) \right\}
\end{align*}
\end{lemma}
\begin{proof}
We shall use Proposition 8.53 of Pisier's notes which is restated below to prove this lemma. To this end consider any $\x_0 \in \Bd$ and any $\Bd$-valued tree $\tx$ of infinite depth. Given an $\epsilon \in \{\pm1\}^{\mathbb{N}}$, for any $j \in [M]$ and $i \in \mathbb{N}$ let $\epsilon^{(j)}_i = \epsilon_{(i-1) M + j}$. Let $\z_0 = \x_0\ M^{-1/p}$ and define the sequence $(\z_i)_{i \ge 1}$ as follows, for any $k \in \mathbb{N}$ given by $k = j + (i-1)M$ where $j \in [M]$ and $i \in \mathbb{N}$,
\begin{align*}
\z_k(\epsilon) =  \tx_{i}(\epsilon^{(j)})\  M^{-1/p}
\end{align*}
Clearly, 
\begin{align*}
\norm{\z_0}_\X^p + \sum_{k \ge 1} \E{\norm{\z_k(\epsilon)}_\X^p} &= \norm{\x_0}_\X^p + \frac{1}{M} \sum_{j=1}^M \sum_{k \ge 1} \E{\norm{\x_k(\epsilon^{(j)})}_\X^p} \\
&  = \norm{\x_0}_\X^p + \sum_{i \ge 1} \E{\norm{\tx_i(\epsilon)}_\X^p} 
\end{align*}
By previous lemma we get that for any $c > 0$,
\begin{align*}
\mathbb{P}\left(\sup_n \norm{\z_0 + \sum_{i=1}^n \epsilon_i \z_i(\epsilon)}_{\Hd} > c\right) & \le  2 \left( \frac{\alpha_p}{c} \right)^{p/(p+1)} \left(\norm{\z_0}_{\X}^p + \sum_{i \ge 1} \E{\norm{\z_i(\epsilon)}_{\X}^p} \right)^{1/(p+1)}\\
& =  2 \left( \frac{\alpha_p}{c} \right)^{p/(p+1)} \left(\norm{\x_0}_{\X}^p + \sum_{i \ge 1} \E{\norm{\tx_i(\epsilon)}_{\X}^p} \right)^{1/(p+1)}
\end{align*}
Note that
\begin{align*}
\sup_n \norm{\z_0 + \sum_{i=1}^n \epsilon_i \z_i(\epsilon)}_{\Hd} & = M^{-1/p} \sup_{j \in [M]} \sup_{n} \norm{\x_0 + \sum_{i=1}^n \epsilon^{(j)}_i \tx_i(\epsilon^{(j)})}_{\Hd}
\end{align*}
Hence we conclude that
\begin{align*}
\mathbb{P}\left( \sup_{j \in [M]} M^{-1/p} \sup_{n} \norm{\x_0 + \sum_{i=1}^n \epsilon^{(j)}_i \tx_i(\epsilon^{(j)})}_{\Hd} > c\right) & \le   2 \left( \frac{\alpha_p}{c} \right)^{\frac{p}{(p+1)}} \left(\norm{\x_0}_{\X}^p + \sum_{i \ge 1} \E{\norm{\tx_i(\epsilon)}_{\X}^p} \right)^{\frac{1}{(p+1)}}
\end{align*}
For any $j \in [M]$, defining $Z^{(j)} = \sup_{n} \norm{\x_0 + \sum_{i=1}^n \epsilon^{(j)}_i \tx_i(\epsilon^{(j)})}_{\Hd}$ and using Proposition \ref{prop:revholder} we conclude that for any $c > 0$,
\begin{align*}
&\sup_{\lambda > 0} \lambda^p\ \mathbb{P}\left( \sup_{n} \norm{\x_0 + \sum_{i=1}^n \epsilon_i \tx_i(\epsilon)}_{\Hd} > \lambda \right)  \\
& ~~~~~~~~~~ \le \max\left\{c , 2 c^p \log\left(\frac{1}{ 1 -  2 \left( \frac{\alpha_p}{c} \right)^{\frac{p}{(p+1)}} \left(\norm{\x_0}_{\X}^p + \sum_{i \ge 1} \E{\norm{\tx_i(\epsilon)}_{\X}^p} \right)^{\frac{1}{(p+1)}}} \right)  \right\}
\end{align*}
Picking 
$$
c =   4^{\frac{p+1}{p}}  \alpha_p \left(\norm{\x_0}_{\X}^p + \sum_{i \ge 1} \E{\norm{\tx_i(\epsilon)}_{\X}^p} \right)^{1/p}
$$
we conclude that 
\begin{align*}
& \sup_{\lambda > 0} \lambda^p\ \mathbb{P}\left( \sup_{n} \norm{\x_0 + \sum_{i=1}^n \epsilon_i \tx_i(\epsilon)}_{\Hd} > \lambda \right) \\
& ~~~~~~~~~~ \le \max\left\{4^{\frac{p+1}{p}}  \alpha_p \left(\norm{\x_0}_{\X}^p + \sum_{i \ge 1} \E{\norm{\tx_i(\epsilon)}_{\X}^p} \right)^{\frac{1}{p}}, 2^{2p + 3} \log(2)\   \alpha_p^p\ \left(\norm{\x_0}_{\X}^p + \sum_{i \ge 1} \E{\norm{\tx_i(\epsilon)}_{\X}^p} \right) \right\}
\end{align*}
\end{proof}

\begin{lemma}\label{lem:sub3}
Let $1 < r \le 2$. If there exists a constant $D > 0$ such that any $\x_0 \in \Bd$ and any $\Bd$-valued tree $\tx$ of infinite depth  satisfies :
\begin{align*}
\forall n \in \mathbb{N}, ~~~~~  \E{\norm{\x_0 + \sum_{i=1}^n \epsilon_i \tx_i(\epsilon)}_{\Hd}} \le D (n + 1)^{1/r} \sup_{0 \le i \le n} \sup_{\epsilon} \norm{\tx_i(\epsilon)}_{\X}
\end{align*}
then for all $p < r$, we can conclude that any $\x_0 \in \Bd$ and any $\Bd$-valued tree $\tx$ of infinite depth will satisfy :
\begin{align*}
\sup_n \E{\norm{\x_0 + \sum_{i=1}^n \epsilon_i \tx_i(\epsilon)}^p_{\Hd}} \le \left( \frac{1104\ D}{(r - p)^2}\right)^p \left(\norm{\x_0}_{\X}^p + \sum_{i \ge 1} \E{\norm{\tx_i(\epsilon)}_{\X}^p} \right)
\end{align*} 
That is the pair $(\H,\X)$ is of martingale type $p$.
\end{lemma}
\begin{proof}
Given any $p < r$ pick $r > p' >p$, due to the homogeneity of the statement we need to prove, w.l.o.g. we can assume that
\begin{align*}
\norm{\x_0}_{\X}^{p'} + \sum_{i \ge 1} \E{\norm{\tx_i(\epsilon)}_{\X}^{p'}} = 1
\end{align*}
Hence by previous lemma, we can conclude that
\begin{align}
\sup_{\lambda > 0} \lambda^{p'}\  \mathbb{P}\left( \sup_{n} \norm{\x_0 + \sum_{i=1}^n \epsilon_i \tx_i(\epsilon)}_{\Hd} > \lambda \right) & \le {p'} 2^{2p' + 3} \log(2)\   \alpha_{p'}^{p'}  \le (32\ \alpha_{p'})^{p'} \label{eq:finsublem}
\end{align}
Hence,
\begin{align*}
\E{\sup_{n} \norm{\x_0 + \sum_{i=1}^n \epsilon_i \tx_i(\epsilon)}_{\Hd}^{p}} & \le \inf_{a > 0}\left\{ a^{p'} +  p \int_{a}^{\infty} \lambda^{p - 1} \mathbb{P}\left( \sup_{n} \norm{\x_0 + \sum_{i=1}^n \epsilon_i \tx_i(\epsilon)}_{\Hd} > \lambda \right) d \lambda \right\}\\
& \le \inf_{a > 0}\left\{ a^{p} + p  (32\ \alpha_{p'})^{p'} \int_{a}^{\infty} \lambda^{p - 1 - p'}   d \lambda \right\} \\
& \le \inf_{a > 0}\left\{ a^{p} + p  (32\ \alpha_{p'})^{p'}  \left[\frac{\lambda^{p  - p'}}{p - p'}\right]_{a}^{\infty} \right\}\\
& \le \inf_{a > 0}\left\{ a^{p} +  (46\ \alpha_{p'})^{p'}   \frac{a^{p  - p'}}{p' - p} \right\}\\
& = 2 \frac{(46\ \alpha_{p'})^{p}}{(p' - p)^{p/p'}}  \le 2 \frac{(46\ \alpha_{p})^{p}}{(p' - p)^{p/p'}} 
\end{align*}
Since $\norm{\x_0}_{\X}^{p'} + \sum_{i \ge 1} \E{\norm{\tx_i(\epsilon)}_{\X}^{p'}} = 1$ and $p' > p$, we can conclude that $\norm{\x_0}_{\X}^{p} + \sum_{i \ge 1} \E{\norm{\tx_i(\epsilon)}_{\X}^{p}} \ge 1$ and so
\begin{align*}
\E{\sup_{n} \norm{\x_0 + \sum_{i=1}^n \epsilon_i \tx_i(\epsilon)}_{\Hd}^{p}} & \le 2 \frac{(46\ \alpha_{p})^{p}}{(p' - p)^{p/p'}} \left( \norm{\x_0}_{\X}^{p} + \sum_{i \ge 1} \E{\norm{\tx_i(\epsilon)}_{\X}^{p}} \right)\\
& \le 2 \frac{(46\ \alpha_{p})^{p}}{(p' - p)} \left( \norm{\x_0}_{\X}^{p} + \sum_{i \ge 1} \E{\norm{\tx_i(\epsilon)}_{\X}^{p}} \right)
\end{align*}
Since $p'$ can be chosen arbitrarily close to $r$, taking the limit we can conclude that
\begin{align*}
\E{\sup_{n} \norm{\x_0 + \sum_{i=1}^n \epsilon_i \tx_i(\epsilon)}_{\Hd}^{p}} & \le 2 \frac{(46\ \alpha_{p})^{p}}{(r - p)} \left( \norm{\x_0}_{\X}^{p} + \sum_{i \ge 1} \E{\norm{\tx_i(\epsilon)}_{\X}^{p}} \right)
\end{align*}
Recalling that $\alpha_p = \frac{12 D}{r-p}$ we conclude that
\begin{align*}
\E{\sup_{n} \norm{\x_0 + \sum_{i=1}^n \epsilon_i \tx_i(\epsilon)}_{\Hd}^{p}} & \le \left(\frac{1104\ D}{(r - p)^{(p+1)/p}}\right)^p \left( \norm{\x_0}_{\X}^{p} + \sum_{i \ge 1} \E{\norm{\tx_i(\epsilon)}_{\X}^{p}} \right)\\
& \le \left(\frac{1104\ D}{(r - p)^2}\right)^p \left( \norm{\x_0}_{\X}^{p} + \sum_{i \ge 1} \E{\norm{\tx_i(\epsilon)}_{\X}^{p}} \right)
\end{align*}
This concludes the proof.
\end{proof}

We restate below a proposition from Pisier's note (in \cite{Pisier11}) 
\begin{proposition}[Proposition 8.53 of \cite{Pisier11}] \label{prop:revholder}
Consider a random variable $Z \ge 0$ and a sequence $Z^{(1)},Z^{(2)},\ldots$ drawn iid from some distribution. For some $0 < p < \infty$, $0 < \delta <1$ and $R > 0$, 
\begin{align*}
\sup_{M \ge 1} \mathbb{P}\left(\sup_{m \le M} M^{-1/p} Z^{(m)} > R \right) \le \delta 
~~~\Longrightarrow ~~~
\sup_{\lambda > 0} \lambda^p\ \mathbb{P}\left(Z > \lambda \right) \le \max\left\{R, 2 R^p \log\left(\frac{1}{1 - \delta} \right) \right\}
\end{align*}
\end{proposition}
\vspace{0.2in}

\begin{proof}[{\bf Proof of Lemma \ref{lem:mn}}]
By Theorem \ref{thm:cite} and our assumption that $\Val_n(\H,\X) \le D n^{-(1 - 1/r)}$, we have that for any $\Bd$-valued tree $\tx$ of infinite depth and any $n \ge 1$,
$$
\E{\frac{1}{n} \norm{\sum_{i=1}^n \epsilon_i \tx_i(\epsilon)}_{\Hd}}   \le D n^{-\left(1 - \frac{1}{r}\right)} 
$$
Hence we can conclude for any $\Bd$-valued tree $\tx$ of infinite depth  and any $n \ge 1$,
$$
\E{\norm{\sum_{i=1}^n \epsilon_i \tx_i(\epsilon)}_{\Hd}} \le D n^{\frac{1}{r}}  \sup_{1 \le i \le n} \sup_{\epsilon} \norm{\tx_i(\epsilon)}_{\X}
$$
Hence for any $\x_0 \in \Bd$, we have that
\begin{align*}
\E{\norm{\x_0 + \sum_{i=1}^n \epsilon_i \tx_i(\epsilon)}_{\Hd}} & \le \E{\norm{\sum_{i=1}^n \epsilon_i \tx_i(\epsilon)}_{\Hd}} + \norm{\x_0}_{\Hd}\\
& \le \E{\norm{\sum_{i=1}^n \epsilon_i \tx_i(\epsilon)}_{\Hd}} + D \norm{\x_0}_{\X}\\
& \le D n^{\frac{1}{r}} \sup_{1 \le i \le n} \sup_{\epsilon} \norm{\tx_i(\epsilon)}_{\X} + D \norm{\x_0}_{\X}\\
& \le 2 D (n+1)^{\frac{1}{r}} \sup_{0 \le i \le n} \sup_{\epsilon} \norm{\tx_i(\epsilon)}_{\X} 
\end{align*}
Now applying Lemma \ref{lem:sub3} with $s = p$ completes the proof.
\end{proof}

\begin{proof}[Proof of Lemma \ref{lem:smtlow}]
We now show that the bound got by the Mirror Descent algorithm is tight. 
First assume without loss of generality that $\sup_{\x \in \X , \h \in \H} \ip{\h}{\x} = 1$. Consider the $1$-smooth convex loss function. 
$$
\phi(z,y) = \left\{\begin{array}{ll}
|z - y| - \frac{1}{4} & \textrm{if } |z - y| > \frac{1}{2} \\
(z - y)^2 & \textrm{otherwise}
\end{array}
\right.
$$
This is basically a smoothed version of the absolute loss. Consider the smooth convex objective 
$$
\ell(\h,(\x,y)) = \phi(\ip{\h}{\x} ,y)
$$
given by instances $\x \in \X$ and $y \in [-1,1]$. Now before we proceed we recall from \cite{RakSriTew10} that the value of the online learning game is equal to  :
\begin{align}\label{eq:vallowbnd}
& \Val_n(\H,\X) \notag \\
&~~~ = \sup_{p_1}\underset{(\x_1,y_1) \sim p_1}{\En}\hspace{-0.07in} \ldots \sup_{p_n} \underset{(\x_n,y_n) \sim p_n}{\En}\left[ \frac{1}{n} \sum_{t=1}^n \inf_{\h_t \in \bcH}\underset{(\x_t,y_t) \sim p_t}{\En}\left[\phi(\ip{\h_t}{\x_t},y_t)\right] - \inf_{\h \in \H} \frac{1}{n} \sum_{t=1}^n \phi(\ip{\h}{\x_t},y_t)\right] \notag \\
&~~~ \ge \underset{(\x_1,y_1) \sim p^*_1}{\En} \hspace{-0.07in}\ldots  \hspace{-0.07in} \underset{(\x_n,y_n) \sim p^*_n}{\En}\left[ \frac{1}{n} \sum_{t=1}^n \inf_{\h_t \in \bcH}\underset{(\x_t,y_t) \sim p^*_t}{\En}\left[\phi(\ip{\h_t}{\x_t},y_t)\right] - \inf_{\h \in \H} \frac{1}{n} \sum_{t=1}^n \phi(\ip{\h}{\x_t},y_t) \right]
\end{align}
Where $p^*_1,\ldots,p^*_n$ is the distribution on $\X \times [-1,1]$ described as follows :
\begin{itemize}
\item For the first $n - m$ rounds, $p^*_t$ deterministically sets $\x_t = 0$ and $y_t = 0$. 
\item For the remaining $m$ rounds, that is for $t > n-m$, we first consider a $\Bd$-valued tree $\tu$ of infinite depth. The distribution $p^*_t$ picks $y_t = \epsilon_{t - n + m}$ where each $\epsilon_i \sim \mrm{Unif}\{\pm1\}$ are Rademacher random variables. $\x_t$ is chosen by setting $\x_{t} = \tu_{t - n - m}(\epsilon_{1},\ldots,\epsilon_{t - n - m -1})$.
\end{itemize}
Now notice that on the first $n-m$ round any algorithm suffers no regret. Using this particular distribution in the lower bound in Equation \ref{eq:vallowbnd} we get that :
\begin{align*}
\Val_n(\H,\X) \ge \frac{1}{n} \En_{(\x_1,y_1) \sim p^*_1} \ldots  \Es{(\x_n,y_n) \sim p^*_n}{  \sum_{t=n-m+1}^n \inf_{\h_t \in \H}\Es{(\x_t,y_t) \sim p^*_t}{\phi(\ip{\h_t}{\x_t},y_t)} - \inf_{\h \in \H}\sum_{t=n-m+1}^n \phi(\ip{\h}{\x_t},y_t)}
\end{align*}
Now we shall select the tree $\tu$ such that for any $n$ and any $\epsilon \in \{\pm 1\}^{n-1}$, $\|\tu_n(\epsilon)\|_{\X} \le \frac{1}{2}$. Hence we see that for any choice of $\h \in \H$, and any $t > n-m$, $|\ip{\h}{\x_t}| \le \frac{1}{2}$. Hence we can conclude that for any $t < n-m$, and any $\h \in \H$,
$$
\phi(\ip{\h}{\x_t},y_t) = |\ip{\h}{\x_t} - y_t| - \frac{1}{4} = \frac{3}{4} - y_t \ip{\h}{\x_t}
$$
Further notice that for any $t > n-m$, since $y_t \sim \mrm{Unif}\{\pm 1\}$ we have that,
$$
\Es{(\x_t,y_t)}{ \phi(\ip{\h}{\x_t}, y_t)} = \frac{3}{4}
$$
Plugging this in the lower bound lower bound to the value, we get that
\begin{align*}
\Val_n(\H,\X) & \ge \frac{1}{n} \En_{(\x_1,y_1) \sim p^*_1} \ldots  \Es{(\x_n,y_n) \sim p^*_n}{  \sup_{\h \in \H} \sum_{t=n-m+1}^n \left(\frac{3 }{4} -  \phi(\ip{\h}{\x_t},y_t)\right)}\\
& = \frac{1}{n} \En_{(\x_1,y_1) \sim p^*_1} \ldots  \Es{(\x_n,y_n) \sim p^*_n}{  \sup_{\h \in \H} \sum_{t=n-m+1}^n \left(\frac{3 }{4} -  \left(\frac{3 }{4} - y_t \ip{\h}{\x_t} \right)\right)}\\
& = \frac{1}{n} \En_{(\x_1,y_1) \sim p^*_1} \ldots  \Es{(\x_n,y_n) \sim p^*_n}{  \sup_{\h \in \H} \sum_{t=n-m+1}^n \left(\epsilon_{t - n + m} \ip{\h}{\x_t} \right)}\\
& = \frac{1}{n} \Es{\epsilon}{\sup_{\h \in \H} \ip{\h}{ \sum_{i=1}^{m} \epsilon_i \tu_i(\epsilon)}}\\
& = \frac{1}{n} \Es{\epsilon}{\norm{\sum_{i=1}^{m} \epsilon_i \tu_i(\epsilon)}_{\H}}
\end{align*}
Specifically picking $m = \overline{\L^*} n$ concludes the proof.
\end{proof}

\begin{proof}[Proof of Lemma \ref{lem:ucvxlow}]
Consider the functions of form $\ell(\h,\z_t) = \ip{\x_t}{\h} + \sigma \psi(\h)$. Recall from \cite{RakSriTew10} that the value of the online learning game is equal to  :
\begin{align*}
& \Val(\H,\Z_{\mrm{ucvx}(\sigma,q')}(\X))  = \sup_{p_1}\En_{\x_1 \sim p_1} \ldots \sup_{p_n} \Es{\x_n \sim p_n}{ \frac{1}{n} \sum_{t=1}^n \inf_{\h_t \in \H}\Es{\x_t \sim p_t}{\ell(\h_t,\z_t)} - \inf_{\h \in \H} \frac{1}{n} \sum_{t=1}^n f_t(\h)} \notag \\
&~~ \ge \En_{\x_1 \sim p^*_1} \ldots  \Es{\x_n \sim p^*_n}{ \frac{1}{n} \sum_{t=1}^n \inf_{\h_t \in \H}\Es{\x_t \sim p^*_t}{\ell(\h_t,\z_t)} - \inf_{\h \in \H} \frac{1}{n} \sum_{t=1}^n \ell(\h,\z_t)} \notag \\
&~~ = \En_{\x_1 \sim p^*_1} \ldots  \Es{\x_n \sim p^*_n}{ \frac{1}{n} \sum_{t=1}^n \inf_{\h_t \in \H}\Es{\x_t \sim p^*_t}{\ip{\h_t}{\x_t}} + \sigma \psi(\h_t) - \inf_{\h \in \H} \left\{\frac{1}{n} \sum_{t=1}^n \ip{\h}{\x_t} + \sigma \psi(\h) \right\}}
\end{align*}
Where $p^*_1,\ldots,p^*_n$ is a distribution on $\X$ specified as follows :  
We consider a $\X$-valued tree $\tu$ of infinite depth and the distribution $p^*_t$ picks
$\x_t = \epsilon_t \tu_t(\epsilon)$ where each $\epsilon_t \sim \mrm{Unif}\{\pm1\}$ are Rademacher random variables. Hence we conclude that
\begin{align*}
\Val(\H,\Z_{\mrm{ucvx}(\sigma,q')}(\X)) & \ge \Es{\epsilon}{ \frac{1}{n} \sum_{t=1}^n \inf_{\h_t \in \H}\Es{\epsilon_t}{\epsilon_t \ip{\h_t}{\epsilon_t \tu_t(\epsilon)}} + \sigma \psi(\h_t) - \inf_{\h \in \H} \left\{\frac{1}{n} \sum_{t=1}^n \ip{\h}{\epsilon_t \tu_t(\epsilon)} + \sigma \psi(\h) \right\}}\\
& \ge \Es{\epsilon}{\sup_{\h \in \H} \left\{ \ip{\h}{- \frac{1}{n} \sum_{t=1}^n \epsilon_t \tu_t(\epsilon)} - \sigma \psi(\h) \right\}}\\
& \ge \Es{\epsilon}{\sup_{\h \in \H} \left\{ \ip{\h}{- \frac{1}{n} \sum_{t=1}^n \epsilon_t \tu_t(\epsilon)} - \frac{D_{p'}^{q'} \sigma}{q'} \norm{\h}^{q'}_{\H} \right\}}\\
& = \frac{1}{p'\ \sigma^{p'-1} D_{p'}^{p'}}\Es{\epsilon}{\norm{\frac{1}{n} \sum_{t=1}^n \epsilon_t \tu_t(\epsilon)}^{p'}_{\Hd}}\\
& \ge \frac{1}{p'\ \sigma^{p'-1} D_{p'}^{p'}}\left(\Es{\epsilon}{\norm{\frac{1}{n} \sum_{t=1}^n \epsilon_t \tu_t(\epsilon)}_{\Hd}}\right)^{p'}
\end{align*}
Since choice of $\X$-valued the tree $\tu$ is arbitrary we can take supremum over all such trees which concludes the proof.
\end{proof}

\section{Discussion}\label{sec:ondis}
We first note that using mirror descent with uniformly convex function (as opposed to strongly convex) is not new and has been used in optimization setting in \cite{NemirovskiYu78}. The key result of this chapter is to establish universality and near optimality of mirror descent for online learning problems. As we showed it is even optimal for smooth learning problems and some uniformly convex learning problems. While the classic definition of martingale type and the associated results are for dual pairs, in this chapter we extended results by \cite{Pisier75} to handle non-dual scenario. Also the proofs have been slightly modified to obtain right dependence on the constants since these constants could be dimension dependent and hence its important to keep track of them.


%% file: optstat.tex
\chapter{Optimality of Mirror Descent for Statistical Convex Learning Problems}\label{chp:optstat}

In the previous chapter we considered convex learning problems in the online learning framework and showed near optimality of mirror descent algorithm for learning rates for online convex learning problems. In this chapter we consider these convex learning problems in statistical learning framework.  We show that in fact for most commonly occurring convex learning problems, stochastic mirror descent is in fact near optimal in terms of both sample complexity and efficiency. 

To mirror the notation and analysis in previous chapter, given any target hypothesis $\cH$ and any instance class $\Z$ let us define
$$
\Valstat_n(\H,\Z) := \inf_{\Algo} \sup_{\D \in \Delta(\Z)} \Es{S \sim \D^n}{\L_\D(\Algo(S)) - \inf_{\h \in \H} \L_\D(\h)}~~.
$$
However unlike the online learning setting for the infimum over the learning algorithm above we shall only consider proper learning algorithms (and as mentioned earlier it suffices to only consider deterministic algorithms), that is $\bcH = \cH$. Notice that $\Valstat_n(\H,\Z)$ is essentially the same as term $\epscon(n)$ we introduced in chapter \ref{chp:stat}. We introduce this notation to mirror the results in previous chapter and because in this chapter the sets $\H$ and $\Z$ we are referring to is rather important and so we prefer to explicitly show this. Also as in the previous chapter the linear learning problem is central to obtaining many of the lower bounds and throughout this chapter we shall use the notation 
$$
\F_{\mrm{lin}}(\H,\X) = \{\x \mapsto \ip{\h}{\x} : \h \in \H\}
$$
to represent the linear loss class $\F_{\mrm{lin}}(\H,\X) \subset \reals^\X$.

In Section \ref{sec:statratelow} we provide lower bounds on learning rates for various convex learning problems (including for smooth losses) using the statistical Rademacher complexity of the linear class $\F_{\mrm{lin}}(\H,\X)$.  While all the lower bounds are provided based on statistical Rademacher complexity of the linear function class, in Section \ref{sec:radtype}, analogous to previous chapter we show that the concept of Rademacher type and the Rademacher Complexity of the linear function class are closely related. In chapter \ref{chp:md} we described the stochastic mirror descent algorithm and provided guarantees for it for various statistical convex learning problems. As we saw in the previous chapter these rates are characterized by Martingale type. However the lower bounds in Section \ref{sec:radtype} are in terms of statistical Rademacher complexity or equivalently characterized by Rademacher type. In general, it turns out that Rademacher type and Martingale types do not match and so in general the upper bound on performance of stochastic mirror descent do not match lower bound. However in Section \ref{sec:noncrazy} we show that for most reasonable cases, the concepts of Rademacher type and Martingale types do in fact match and so we use this result to argue that stochastic mirror descent is near optimal even for statistical learning problems. In the same section we also show that stochastic mirror descent algorithm is optimal also in terms of number of gradient access and for certain supervised learning problems, optimal in terms of computational complexity. In the succeeding section \ref{sec:egstat} we review the examples we saw in the previous chapter to see how mirror descent is indeed optimal. Section \ref{sec:proofsstat} provides detailed proof of the results in this chapter and we finally conclude the chapter with some discussions in Section \ref{sec:statdis}.

\section{Lower Bounds for Statistical Learning Rates}\label{sec:statratelow}

\noindent We would like to point out here that linear instance class $\Z_{\mrm{lin}}$ is a subclass of the Lipschitz class $\Z_{\mrm{Lip}}$. Hence any lower bound on learning rates for linear class is also a lower bound for the convex Lipschitz class. Similarly, if we consider the supervised learning class with label $y$ always being $b = - \sup_{\h \in \H, \x \in \X} \ip{\h}{\x}$ then we see that $\ell(\h;(\x,y)) = \ip{\h}{\x}- b$. Hence we see that in excess risk for any distribution $\D$ on $\x$'s with $y$ being deterministically set to $b$ is same as excess risk of linear class with distribution $\D$ on $\x$'s. Hence we see that lower bound on linear class is also a lower bound on supervised learning class. The following proposition formalizes this.

\begin{proposition}\label{prop:linstatislow}
For any hypothesis set $\cH \subset \B$, any $\X \subset \Bd$ and $\epsilon >0$
$$
\Valstat_n(\H,\Z_{\mrm{Lip}}(\X)) \ge \Valstat_n(\H,\Z_{\mrm{lin}}(\X)) ~~~~\textrm{ and }~~~~\Val_n(\H,\Z_{\mrm{sup}}(\X)) \ge \Val_n(\H,\Z_{\mrm{lin}}(\X))~~.
$$
\end{proposition}

Owing to the fact that optimal learning rates for the linear class provide lower bounds for convex lipschitz and supervised convex learning problems, we define for each $p \in [1,2]$ the constant $V^{\mrm{iid}}_p$ analogous to the definition of $V_p$ in the previous chapter.

\begin{align}\label{eq:vp}
V^{\mrm{iid}}_p := \inf\left\{V\ \middle|\ \forall n \in \mathbb{N}, \Valstat_n(\H,\Z_{\mrm{lin}}(\X))\le V n^{-\left(1 - \frac{1}{p}\right)}\right\}
\end{align}


In the previous chapter we saw that the value of the linear game was closely related to the sequential Rademacher complexity of the linear class, $\Radon_n(\F_{\mrm{lin}}(\H,\X))$. We will see an analogous relationship in the statistical setting with the worst case statistical Rademacher complexity.  Recall that for the linear class specified by sets $\H \subset \B$ and $\X \subset \Bd$, the worst case statistical Rademacher complexity is given by 
$$
\Radstat_n(\F_{\mrm{lin}}(\H,\X)) = \sup_{\x_1,\ldots,\x_n \in \X} \Es{\epsilon}{\sup_{\h \in \H} \frac{1}{n} \sum_{i=1}^n \epsilon_i \ip{\h}{\x_i}} = \sup_{\x_1 , \ldots,\x_n \in \X} \Es{\epsilon}{\norm{\frac{1}{n} \sum_{i=1}^n \epsilon_i \x_i}_{\Hd}}
$$
where $\epsilon \in \{\pm1\}^n$ are iid Rademacher random variables. In the following lemma we show that the learning rate of the linear class can in turn be bounded by the worst case statistical Rademacher complexity of the linear class.

\begin{lemma}\label{lem:linlowstat}
For any $\cH \subset \B$ and $\X \subset \Bd$, 
$$
\Valstat_n(\H,\Z_{\mrm{lin}}(\X)) \ge \Radstat_{2 n}(\F_{\mrm{lin}}(\H,\X)) - \frac{1}{2}\Radstat_{n}(\F_{\mrm{lin}}(\H,\X))
$$
where $n = |S|$. 
\end{lemma}

While the above lemma along Proposition \ref{prop:linstatislow} lower bounds learning rates for convex lipschitz, supervised and linear problems by Rademacher complexity of linear class, for linear and supervised learning problems we can also show that Rademacher complexity of linear class can be used to upper bound the optimal learning rates.  
The following proposition which is a direct consequence of results from \cite{BarMed03} show that the learning rates for linear and supervised convex learning problems are upper bounded by the statistical Rademacher complexity.

\begin{proposition}\cite{BarMed03} \label{prop:radcontract}
For any set $\X \in \Bd$ if $\Z(\X)$ is one of either $\Z_{\mrm{lin}}(\X)$ or $\Z_{\mrm{sup}}(\X)$, then for any $n \in \mathbb{N}$, 
$$
\sup_{\D \in \Delta(\Z(\X))} \Es{S \sim \D^n}{\L_\D(\Algo_{\mrm{ERM}}(S)) - \inf_{\h \in \H} \L_\D(\h) } \le 2 \Radstat_n(\F_{\mrm{lin}}(\X))
$$
and hence :  \hspace{1.4in}$\Valstat_n(\H,\Z(\X))\le 2 \Radstat_n(\F_{\mrm{lin}}(\H,\X))~~.$
\end{proposition}
\begin{proof}
The inequality of linear class is a direct consequence of symmetrization (see \cite{BarMed03}). The inequality for the supervised learning class $\Z_{\mrm{sup}}(\X)$ additionally uses the Lipschitz contraction property (Theorem 10 (4) of \cite{BarMed03} ) since  the absolute loss is $1$-Lipschitz.
\end{proof}
While the above proposition bounds the learning rates of linear and supervised convex learning classes using Rademacher complexity of the linear class we don't know if such a result is true for the convex Lipschitz class (unlike the online case).

While Lemma \ref{lem:linlowstat} provides lower bound on learning rates of linear learning problems in terms of statistical Rademacher complexity of linear class, the RHS there involves the difference of two Rademacher complexity terms. The following theorem shows that learning rates for these convex learning problems provide direct upper bounds on statistical Rademacher complexity of the linear class that captures all polynomial dependences right.

\begin{theorem}\label{thm:statlinbnd}
Given any target hypothesis set $\H \subset \B$ and instance space $\X \subset \Bd$, let $\Z(\X)$ be one of $\Z_{\mrm{Lip}}(\X)$, $\Z_{\mrm{sup}}(\X)$ or $\Z_{\mrm{lin}}(\X)$. If for some $q \in [2,\infty)$ and $V > 0$,
$$
\Valstat_n(\H,\Z(\X)) \le \frac{V}{n^{1/q}}
$$
then,
$$
\forall n \in \mathbb{N},~~\Radstat_{2n}(\F_{\mrm{lin}}(\H,\X)) \le \frac{5 V}{(2n)^{\frac{1}{q}}}~~.
$$
\end{theorem}


\subsection{Lower Bounds for Smooth Losses}

The following lemma lower bounds learning rate for non-negative smooth convex learning problems and also captures dependence on expected loss of the optimal hypothesis in target class $\H$. This result is an analog to Lemma \ref{lem:smtlow} of previous chapter.

\begin{lemma}\label{lem:statsmtlow}
Given $\cH \subset \B$, $\X \subset \Bd$ and $\L^* \in (0,3/4]$ for any learning algorithm $\Algo$, there exists a distribution $\D$ over instances in $\Z_{\mrm{smt(1)}}(\X)$ s.t. $\inf_{\h \in \cH} \L_\D(\h) \le \L^*$ and 
$$
\Es{S}{\L_\D(\Algo(S)) - \inf_{\h \in \cH} \L_\D(\h)} \ge \frac{\L^*}{4} \left( \Radstat_{n\L^*}(\F_{\mrm{lin}}(\H,\X)) - \frac{1}{2} \Radstat_{n \L^*}(\F_{\mrm{lin}}(\H,\X)) \right)  
$$
\end{lemma}


\section{Optimal Rates and Rademacher Type}\label{sec:radtype}
In this section we extend the classic notion of Rademacher 
type of a Banach space (see for instance \cite{Maurey03}) to one that
accounts for the pair $(\Hd,\X)$. The results of this section are analogous to the results in Section \ref{sec:martype}.

\begin{definition}\label{def:type} \index{type}
A pair $(\Hd,\X)$ of subsets of a vector space $\Bd$ is said to be of Rademacher type $p$ if there exists a constant $C \ge 1$ such that for any $n \in \mathbb{N}$ and any $\x_1,\x_2,\ldots, \x_n \in \Bd$ : 
\begin{align} \label{eq:rtype}
\E{\norm{\sum_{i=1}^n \epsilon_i \x_i}_{\Hd}^p} \le C^p \left(\sum_{i = 1}^n \norm{\x_i}_{\X}^p \right)
\end{align}
\end{definition}
It can be shown that rate of convergence of i.i.d. random variables in Banach spaces is governed by the above notion of Rademacher type of the associated Banach space. We point the reader to \cite{Pisier11} for more details.
Further, for any $p \in [1,2]$ we also define constant $C^{\mrm{iid}}_p$, analogous to the definition of $C_p$ in previous chapter.
{\small
\begin{align*}
C^{\mrm{iid}}_p := \inf\left\{C\ \middle|\ \forall n \in \mathbb{N}, \forall \x_1, \ldots, \x_{n} \in \Bd,\  \E{\norm{\sum_{i=1}^n \epsilon_i \x_i}_{\Hd}^p} \le C^p \left(\sum_{i= 1}^n \norm{\x_n}_{\X}^p \right) \right\}
\end{align*}}
$C^{\mrm{iid}}_p$ is useful in determining if the pair $(\Hd,\X)$ has Rademacher type $p$.

The following lemma is an analog to Lemma \ref{lem:mn} of previous section.
\begin{lemma}\label{lem:radtype}
If for some $r \in (1,2]$ there exists a constant $D > 0$ such that for any $n$,
\begin{align*}
\Radstat_n(\F_{\mrm{lin}}(\H,\X)) = \sup_{\x_1,\ldots, \x_n \in \X} \E{\norm{\frac{1}{n}\sum_{i=1}^n \epsilon_i \x_i}_{\Hd}} \le D n^{-(1 - \frac{1}{r})}
\end{align*}
then for all $p < r$, we can conclude that for any $\x_{1}, \ldots, \x_{n} \in \Bd$ :
\begin{align*}
\E{\norm{\sum_{i=1}^n \epsilon_i \x_i}^p_{\Hd}} \le \left( \frac{12 \sqrt{2}\ D}{(r - p)}\right)^p \left(\sum_{i = 1}^n \norm{\x_i}_{\X}^p \right)
\end{align*} 
That is, the pair $(\Hd,\X)$ is of Rademacher type $p$.
\end{lemma}

\begin{corollary}\label{cor:radtype}
Given any pair $(\Hd,\X)$ and any $p$ and $p' < p$, we have that 
$$
\tfrac{(p - p') }{60 \sqrt{2}} C^{\mrm{iid}}_{p'} \le V^{\mrm{iid}}_p \le 2 C^{\mrm{iid}}_p
$$
\end{corollary}
\begin{proof}
Owing to Proposition \ref{prop:radcontract} and definition of $V^{\mrm{iid}}_p$ using Jensen's inequality, we get that $V^{\mrm{iid}}_p \le 2 C^{\mrm{iid}}_p$. The second inequality is a consequence of using the above Lemma \ref{lem:radtype} with Theorem \ref{thm:statlinbnd} and the definitions of $C^{\mrm{iid}}_p$ and $V^{\mrm{iid}}_p$.
\end{proof}

The following figure summarizes the relationship between $V^{\mrm{iid}}_p$ and $C^{\mrm{iid}}_p$. The arrow mark from $C^{\mrm{iid}}_{p'}$ to $C^{\mrm{iid}}_p$ indicates that for any  $n$, the quantities are within $\log n$ factor of each other.
\begin{figure}[h]
\begin{center}
\begin{tikzpicture}[node distance=1.5cm, auto,>=latex',
cond/.style={draw, thin, rounded corners, inner sep=1ex, text centered},
cond1/.style={}]
\node[text width=1.7cm, style=cond] (lower) {\small $p' < p,\ C^{\mrm{iid}}_{p'}$};
\node[text width=0.5cm, style=cond1,right of=lower] (le1) {\huge $ \le$};
\node[text width=1.7cm, style=cond, right of=le1] (value)
{\small $V^{\mrm{iid}}_p$};
\node[text width=0.5cm, style=cond1,right of=value] (le2) {\huge $\ \le$};
\node[text width=1.7cm, style=cond, right of=le2] (C) {\small $C^{\mrm{iid}}_p$};
\node[text width=1.6cm, style=cond1] at (1.9,-0.65) (le5) {\tiny Lemma \ref{lem:radtype} };
\node[text width=1.6cm, style=cond1] at (4.9,-0.55) (le6) {\tiny Proposition \ref{prop:radcontract} };
\node[text width=1.8cm, style=cond1] at (4.9,-0.78) (le6) {\tiny Definition of $V^{\mrm{iid}}_p$ };
\path[<-, draw, double distance=1pt,sloped] (C) -- +(0,0.75) -| (lower);
\end{tikzpicture}
\end{center}
\vspace{-0.32in} \caption{Relationship between the various constants in statistical setting}\label{fig:mainstat}
\end{figure}
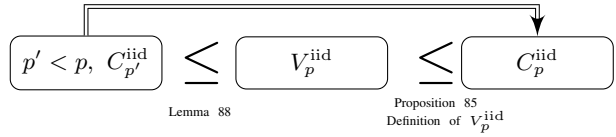


\section{Main Result : Optimality of Stochastic Mirror Descent}\label{sec:noncrazy}
In Section \ref{sec:statratelow} we provided lower bounds on learning rates  of convex learning problems in the statistical learning framework. In previous chapter we provided upper bounds on learning rate of stochastic mirror descent algorithm for statistical convex learning problems. A natural question that arises is whether the stochastic mirror descent algorithm is optimal in terms of learning rates or efficiency or both and whether they match the lower bounds obtained. 

It turns out that in  general even for convex learning problems, this is not true. A problem could be statistically learnable but not online learnable. However in this section we will see that for large class of convex learning problems (in fact most of the commonly occurring ones) it is true that online methods are near optimal for statistical convex learning problems in terms of both rates and efficiency. In the previous we saw that mirror descent algorithm was near optimal and the key to showing this result is depicted in Figure \ref{fig:main}. In this chapter again in Section \ref{sec:radtype} we obtained an analogous result depicted in Figure \ref{fig:mainstat}. The attractive feature in the online learning scenario was that optimal learning rates and hence concept of martingale type was closely connected to existence of appropriate uniformly convex function which then we could use with mirror descent to get near optimal guarantees. In the statistical case we only have the connection between optimal rates (or at least lower bounds) and the concept of Rademacher type. By the definitions of $C_p$ and $C^{\mrm{iid}}_p$, it is easy to see that $C^{\mrm{iid}}_p \le C_p$ and due to online to batch conversion we can also infer that $V^{\mrm{iid}}_p \le V_p$. However assume that we could get a reverse bound, this would then imply that online learning methods are optimal at least up to logarithmic factors. Such optimality of learning rate is essentially captured in the following theorem.

\begin{theorem}\label{thm:stattight1}
Let $\Z(\X)$ stand for one of $\Z_{\mrm{lin}}(\X)$, $\Z_{\mrm{sup}}(\X)$ or $\Z_{\mrm{Lip}}(\X)$. If there exists some constant $G \ge 1$ such that for any $1< p' < r \le2$ 
$$
C_{p'} \le G\ C^{\mrm{iid}}_{r}
$$
and it is true that for some $V > 0$ and $p \in (1,\infty)$,
$$
\forall n \in \mathbb{N},~  \Valstat_n(\H,\Z(\X)) \le \frac{V}{n^{1/q}} 
$$
then there exists function $\Psi$ and step size $\eta$ using which the stochastic mirror descent algorithm enjoys the learning guarantee
$$
\forall n \in \mathbb{N},\ \sup_{\D \in \Delta(\Z(\X))} \Es{S \sim \D^n}{\L_\D(\Algosmd(S)) - \inf_{\h \in \H} \L_\D(\h)} \le \frac{30010 G V \log^3 n}{n^{1/q}}
$$
\end{theorem}
\begin{proof}
Using Theorem \ref{thm:statlinbnd} we first get a bound on $\Radstat_n(\F_{\mrm{lin}})$. Next using Corollary \ref{cor:radtype}  we see that for any $p' < p$, $C^{\mrm{iid}}_{p'} \le \frac{60 \sqrt{2} V}{p - p'}$. Using the assumption that $C_{p} \le G\ C^{\mrm{iid}}_{p}$ we get bound on $C_{p'}$ and from here on using exactly the same proof as that of theorem \ref{thm:mdoptlip} we get the bound on the regret. To convert it to bound on excess risk we can use Proposition \ref{prop:o2b}. 
\end{proof}

Even for the smooth loss case one can show that if for any $1< p' < r \le2$ 
$C_{p'} \le G\ C^{\mrm{iid}}_{r}$ then online mirror descent is near optimal as the following theorem shows.

\begin{theorem}\label{thm:stattight3}
If there exists some constant $G \ge 1$ such that for any $1< p' < r \le2$ 
$$
C_{p'} \le G\ C^{\mrm{iid}}_{r}
$$
and it is true that for some $V > 0$ and $p \in (1,\infty)$,
$$
\forall n \in \mathbb{N},~  \Valstat_n(\H,\Z_{\mrm{smt}(1)}(\X)) \le \frac{V}{n^{1/q}}~~,
$$
then there exists function $\Psi$ and step size $\eta$ using which the mirror descent algorithm followed by online to batch conversion technique, that is the algorithm $\Algosmd$, enjoys the guarantee that for any $\L^* \in (0,3/4]$ and any distribution $\D \in \Delta(\Z_{\mrm{smt}(H)}(\X))$ s.t. $\inf_{\h \in \H} \L_\D(\h) \le \L^*$, 
$$
\Es{S \sim \D^n}{\L_\D(\Algosmd(S)) - \inf_{\h \in \H} \L_\D(\h)} \le \frac{240080\ G\ V\ \sqrt{H \L^*} \log^3 n}{n^{1/q}} + \frac{(189800\ G\ V)^2H \log^6 n}{n^{2/q}}
$$
\end{theorem}
\begin{proof}
Again the proof is similar to previous two proofs. The main difference is that at the very first step we use Lemma \ref{lem:statsmtlow} with $\L^* = 3/4$ to bound $\Radstat_n(\F_{\mrm{lin}})$. Also in the step before last (ie. just before the online to batch conversion), instead of using Theorem \ref{thm:mdoptlip} we instead use Theorem \ref{thm:smtlow}. Also after the online to batch step note that 
$$
\Es{S \sim \D^n}{\inf_{\h \in \H} \frac{1}{n} \sum_{t=1}^n \ell(\h,\z_t) } \le \inf_{\h \in \H}\L_\D(\h) \le \L^*
$$
\end{proof}
Also note that the discussion after Theorem \ref{thm:smtlow} in  the previous chapter for the online setting also applies here and specifically when $q=2$ we also get tightness w.r.t. dependence in $\L^*$

Owing to the above two theorems the main condition we are now looking to satisfy is that there exists some constant $G \ge 1$ such that for any $1< p' < r \le2$ 
$$
C_{p'} \le G\ C^{\mrm{iid}}_{r}
$$
\begin{remark}\label{rem:fudge}
In some cases we might only be able to prove that 
there exists some constant $G \ge 1$ such that for any $1< p' < r \le2$ 
$$
C_{p'} \le \frac{G}{r - p'} C^{\mrm{iid}}_{r}
$$
such an inequality is also fine and basically all the above three theorems  still hold with the only modification that the logarithmic factors in the theorems increase by one more power.
\end{remark}

In the following subsections we will show that for a large class of spaces such $G$ exists (or owing to above remark with additional $(r-p')^{-1}$ factor). From this we can infer that mirror descent is near optimal for convex learning optimization problems in these spaces.

\subsection{Banach Lattices} \index{Banach Lattice}
In this subsection we show that if the Banach space specified by norm $\norm{\cdot}_{\Hd}$ is a Banach lattice, then one can relate martingale type and Rademacher type constants $C_p$ and $C^{\mrm{iid}}_{p}$ to within constant factor of each other. 

\begin{definition}[Banach Lattice \cite{LinTza79}]
A partially ordered Banach space is called a Banach lattice provided :
\begin{enumerate}
\item For any $\x,\x' , \tilde{\x} \in \Bd$, $\x \preceq \x'$ implies that $\x + \tilde{\x} \preceq  \x' + \tilde{\x}$
\item $a \x \succeq \mbf{0}$ for every $\x \succeq \mbf{0}$ and non-negative scalar $a$.
\item For all $\x,\x' \in \Bd$, there exists a least upper bound (l.u.b.) represented by $\x \vee \x'$.
\item For any $\x,\x' \in \Bd$ such that $|\x| \preceq |\x'|$, we have that  $\norm{\x} \le \norm{\x'}$ (where the absolute value $|\x|$ is defines as $|\x| = \x \vee (-\x)$).
\end{enumerate}
\end{definition}

Before we proceed we notice that all $\ell_p$ spaces with partial order given by, $\x \preceq \y$ if anf only if on each co-ordinate $i$, $\x_i \le y_i$ is a Banach lattice. In fact all the examples we saw in previous chapters were Banach lattices under appropriate partial order. In fact one could safely say that most Banach spaces one would encounter in machine learning applications would be Banach lattices. Hence the results in this section are very general from a practical view-point. 

Now we introduce the notation that for any $p \in [1,\infty)$, we will use $\left(\sum_{i=1}^n |\x_i|^p \right)^{1/p}$ to represent the vector,
$$
\left(\sum_{i=1}^n |\x_i|^p \right)^{1/p}  := \mrm{l.u.b.}\left\{\sum_{i=1}^n a_i \x_i \middle|  (a_1,\ldots,a_n) \in \reals,  \sum_{i=1}^n |a_i|^q \le 1\right\}
$$
where $q = \frac{p}{p-1}$. Notice that if $\x_i$'s were reals this would we simply the $\ell_p$ norm. Here it is a vector though. The main technology behind proving results about Banach lattices arises from the so called functional calculus over banach lattices introduced by Krivine \cite{Krivine} (See \cite{LinTza79}). The basic idea is a theorem (theorem 1.d.1 in \cite{LinTza79}) which roughly states that if we prove any inequality involving continuous degree $1$, homogenous equations involving finite number of real valued variables, then the same inequality is true with of course appropriate changes like $\le$ replaced by $\preceq$ and absolute value replaced by the lattice version and so on. 

At first glance the statements we would like to prove would involve expectation over Rademacher random variables and trees (of depth $n$). However simple observation that expectation over Rademacher are finite averages and that the tree of depth $n$ (involving the $n$ mappings) can be expanded to the $2^n-1$ variables involved we see automatically that this general technology can be used to prove results that involve trees and expectations w.r.t. $\epsilon$'s too. Detailed proof and associated definitions are provided in Section \ref{sec:proofsstat}. We delineate the main results below.

Before we proceed we first define below the notion of co-type of a Banach space (again extended to the non-dual case, see \ref{Maurey03} for details about classical definition). 

\begin{definition}\label{def:cotype} \index{cotype}
A pair $(\Hd,\X)$ of subsets of a vector space $\Bd$ is said to be of Rademacher co-type $q$ if there exists a constant $C \ge 1$ such that for any $\x_1,\x_2,\ldots, \x_n \in \Bd$ : 
\begin{align} \label{eq:rtype}
 \left(\sum_{i = 1}^n \norm{\x_i}_{\X}^q \right) \le C^q \E{\norm{\sum_{i=1}^n \epsilon_i \x_i}_{\Hd}^q} 
\end{align}
\end{definition}

The following lemma shows that if $\Hd$ formed a Banach lattice and were of some finite co-type $r$ then for any $p$, $C_p$ can be bounded by a constant factor times $C^{\mrm{iid}}_p$.
\begin{lemma}\label{lem:latticeG}
If $(\Bd,\norm{\cdot}_\Hd)$ is a Banach lattice and the pair $(\Hd,\Hd)$
is of co-type $r$ for some $r \in [2,\infty)$ with some constant $\tilde{C}_r$, then for any $p \in (1,2]$ we have that 
$$
C_p \le \frac{72^3 r^2 \tilde{C}_r}{\sqrt{p-1}} C^{\mrm{iid}}_p
$$
\end{lemma}

The above Lemma shows that as long as $\Hd$ has a lattice structure and is of finite co-type, $G$ is bounded by $\frac{72^3 r^2 \tilde{C}_r}{\sqrt{p-1}}$ and so for any such case owing to the theorems in  the beginning of this section we automatically can conclude optimality of mirror descent.

\subsubsection{Dual Learning Problem : }
In Lemma \ref{lem:latticeG} we needed that $\Hd$ had finite co-type. Though this is a fairly weak condition we still needed to verify that this was satisfied before we used the result. However if we consider the dual learning problem, that is the case when $\X$ is the dual ball of $\H$ then it turns out that this condition is automatically satisfied as long as the problem is statistically learnable (if it is not we anyway are not in a position to give any meaningful bounds). This is due to the celebrated result of Maurey and Pisier \cite{MauPis76} which (in the dual case) assures that any space with non-trivial type also has finite co-type. The following corollary which uses a result by Konig and Tzafriri  \cite{KonTza}
shows the exact relationship between $C_p$ and $C^{\mrm{iid}}_p$ for Banach lattices.

\begin{corollary}
For the dual learning problem with $\X = \Hd$, if $(\Bd,\norm{\cdot}_\Hd)$ is a Banach lattice, then for any $p \in (1,2]$ we have that 
$$
C_p \le \frac{72^3}{\sqrt{p-1}} (3 C^{\mrm{iid}}_p)^{2 q + 1} 
$$
where $q = \frac{p}{p-1}$.
\end{corollary}
\begin{proof}
Note that for the dual problem $\X = \Hd$ and by definition, $(\Hd,\Hd)$ has Rademacher type $p$ with constant $C_p^{\mrm{iid}}$. However by 
Theorem 3 of \cite{KonTza} we have that the pair $(\H,\H)$ is of co-type $2 + (2 C_p^{\mrm{iid}})^q$ with constant $2$ (for co-type constant of $2$ refer proof of the theorem). Hence using this in Lemma \ref{lem:latticeG} with $r = 2 + (2 C_p^{\mrm{iid}})^q$ and $\tilde{C}_r = 2$ and simplifying yields the required statement.
\end{proof}

Thus for dual learning problems if the Banach space is a Banach Lattice then online methods (specifically mirror descent) is always near optimal in terms of dependence on $n$ for learning rate and in terms of dependence on $\epsilon$ for oracle complexity. More specifically, for a dual learning problem on a Banach lattice, if for instance optimal rate is $V/\sqrt{n}$ then mirror descent will guarantee a rate of order $V^5 \log^2(n) /\sqrt{n}$.


\subsection{Decoupling Inequalities}
Another way to guarantee that $C_p \le G C^{\mrm{iid}}_p$ for some finite $G$ is by using the so called decoupling inequalities (see \cite{Garling86,CoxVer11} for more details).

\begin{definition}\index{decoupling}
We say that a Banach space  with norm $\norm{\cdot}_\Hd$ satisfies $p$-decoupling inequality with constant $B > 0$ if :
\begin{align}
\Es{\epsilon}{\norm{\sum_{i=1}^n \epsilon_i \tx_i(\epsilon)}^p_{\Hd}} \le B^p\ \Es{\epsilon,\epsilon'}{\norm{\sum_{i=1}^n \epsilon_i \epsilon'_i \tx_i(\epsilon)}^p_{\Hd}}
\end{align}
where $\epsilon'_1,\ldots,\epsilon'_n$ are Rademacher random variables (drawn independent of $\epsilon_1,\ldots,\epsilon_n$).
\end{definition}

We would like to point out that the above definition is not the same as the decoupling inequalities in \cite{Garling86,CoxVer11} where the above needs to be true for all martingale difference sequences where as above we only consider Walsh-Paley martingales. However since above condition is weaker, all positive results for the stronger definition also hold for the above definition. In any case, the below lemma shows that if a Banach space satisfies decoupling inequality, then there exists $G < \infty$ through which we can upper bound $C_p$ using $C^{\mrm{iid}}_{p'}$.

\begin{lemma}\label{lem:decoupling}
If the Banach space equipped with norm $\norm{\cdot}_{\Hd}$ satisfies $1$-decoupling inequality with some constant $B > 0$, then it is true that for all $1 < p' < r \le 2$,
\begin{align*}
C_{p'} \le \frac{92 B}{(r-p')} C^{\mrm{iid}}_{r}
\end{align*}
\end{lemma}

First we would like to point out that if a Banach space satisfies $p$-decoupling inequality for some $p \in [1,\infty)$ with constant $D$, then it satisfies $1$-decoupling inequality with constant $K_p D$ where $K_p$ only depends on $p$ (see Theorem 4.1 of \cite{CoxVer11}). Hence we can talk of space satisfying decoupling inequality without refering to exponent (of course the exponent plays a role in the constant) .Several spaces we commonly encounter satisfy the decoupling inequality. It can be shown that called Unconditional Martingale difference (UMD) Banach spaces always satisfy decoupling inequality. These spaces include all $L_p$, $\ell_p$ Schatten $p$ norms for $p \in (1,\infty)$. However while every UMD space satisfies decoupling for instance the $\ell_1$ space while satisfying the decoupling inequality is not UMD.  Using some of the results in \cite{CoxVer11,CoxVer07,HitMon96,Hit94,Hit88} we can conclude that for many interesting spaces we commonly encounter, there in fact even exists a universal $1$-decoupling constant, call it $B_{\reals}$.  (This constant $B_{\reals}$ is the one referred to as $D_{\reals}$ in \cite{CoxVer11}). The following proposition proved in \cite{CoxVer11} (see also \cite{CoxVer07}) is particularly useful especially to provide decoupling inequalities for group norms and interpolation norms. 

\begin{proposition}[Corollaries 4.6 \cite{CoxVer11}] \label{prop:decouplelp}
If $Y$ is a Banach space that satisfies $p$-decoupling inequality with constant $B$, then for any $\sigma$-finite measure space $(S,\Sigma,\mu)$ and any $p \in [1,\infty)$, the space $X = L^p(S;Y)$ satisfies $p$-decoupling inequality with same constant $B$.
\end{proposition}

From the above we immediately see that $\ell_p$ spaces satisfy $p$-decoupling inequality with universal constant $B_{\reals}$ (the space $Y$ in this case is the reals). We can also use the above result to estimate decoupling constants for group norm by taking the the above Proposition space $Y$ to be the $\ell_p$ space corresponding to the inner index of the group norm (mixed norm). In Corollary 4.11 of \cite{CoxVer11} it has been shown that for $p \ge \log_2(d)$, the $p$-decoupling constant of $\ell_\infty^d$ spaces if bounded by $2 B_\reals$. In \cite{HitMon96} it has been shown that a large class of Orlicz and Rearrangement invariant function spaces satisfy $1$-decoupling inequality with universal constant $B_\reals$ combining these result with the above proposition one can obtain decoupling inequalities for even more examples.

Hence overall, the following figure captures the scenario for the statistical convex learning problems.
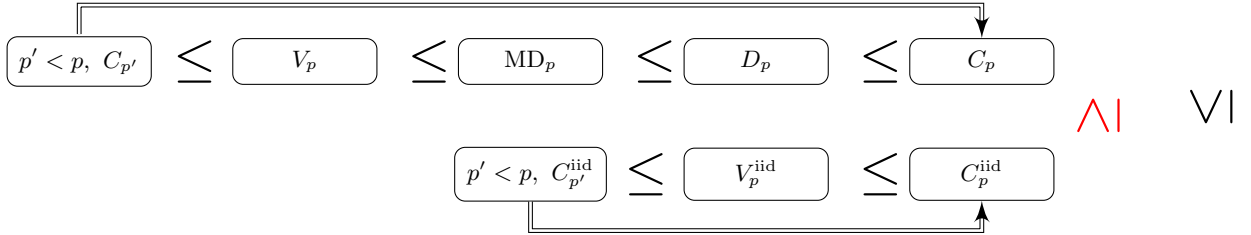
\begin{figure}[h]
\begin{center}
\begin{tikzpicture}[node distance=1.5cm, auto,>=latex',
cond/.style={draw, thin, rounded corners, inner sep=1ex, text centered},
cond1/.style={}]
\node[text width=1.6cm, style=cond] (lower1) {\small $p' < p,\ C_{p'}$};
\node[text width=0.5cm, style=cond1,right of=lower] (le3) {\huge $ \le$};
\node[text width=1.6cm, style=cond, right of=le3] (value)
{\small $V_p$};
\node[text width=0.5cm, style=cond1,right of=value] (le4) {\huge $\ \le$};\node[text width =1.6cm, style=cond, right of=le4] (MD) {\small
$\MD_p$};
\node[text width=0.5cm, style=cond1,right of=MD] (le5) {\huge $\ \le$};
\node[text width=1.6cm, style=cond, right of=le5] (D) {\small $D_p$};
\node[text width=0.5cm, style=cond1,right of=D] (le6) {\huge $\ \le$};
\node[text width=1.6cm, style=cond, right of=le6] (C) {\small $C_p$};
\node[rotate= 90.00,color=red,text width=2cm, style=cond1,below of=C] (le7) {\vspace{-0.5in}\huge $\ge$};
\node[rotate= 90.00,text width=2cm, style=cond1,below of=le7] (le8) {\vspace{-1.45in} \huge $\le$};
\node[text width=1.6cm, style=cond, below of=C] (Ciid) {\small $C^{\mrm{iid}}_p$};
\node[text width=0.5cm, style=cond1,left of=Ciid] (le2) {\huge $\ \le$};
\node[text width=1.6cm, style=cond, left of=le2] (Viid)
{\small $V^{\mrm{iid}}_p$};
\node[text width=0.5cm, style=cond1,left of=Viid] (le1) {\huge $ \le$};
\node[text width=1.7cm, style=cond,left of=le1] (lowC) {\small $p' < p,\ C^{\mrm{iid}}_{p'}$};
\path[<-, draw, double distance=1pt,sloped] (Ciid) -- +(0,-0.75) -| (lowC);
\path[<-, draw, double distance=1pt,sloped] (C) -- +(0,0.75) -| (lower);
\end{tikzpicture}
\end{center}
\vspace{-0.22in} \caption{Relationship between the various constants in both online and statistical settings. ~~~~~~~~~~~~~~~~~~~~~~~~~~~~~~ Red inequality is not always true but true for most commonly encountered problems. }\label{fig:mainstaton}
\end{figure}

\subsection{Optimality of Mirror Descent in Terms of Efficiency}
In this section we argue that mirror descent is not only optimal in terms of learning rate but also in terms of efficiency. Specifically we are interested in achieving excess risk less than some target sub-optimality $\epsilon > 0$, that is, we are interested in coming up with learning algorithm $\Algo$ such that 
$$
\sup_{\D \in \Delta(\Z)} \Es{S \sim \D^n}{ \L_\D(\Algo(S))  - \inf_{\h \in \cH} \L_\D(\h)} \le \epsilon~~.
$$
Now there are two main question one can ask, the first is the minimum sample size required by any algorithm to be able to ensure the above guarantee on sub-optimality. This question is essentially the question about optimal learning rates and we already showed optimality of stochastic mirror descent. The second question one can ask is the computational time required by any algorithm given any amount of requested samples to ensure the above sub-optimality guarantee. Note that in general giving exact computational complexity is cumbersome and even impossible at full generality. A good proxy for computational complexity is the number of gradient calculations or calculation or some kind of other local information like hessians or higher order derivative. Note that mirror descent is a sub-gradient based method which in each iteration only calculates one gradient. Since we already established that mirror descent is optimal for many statistical convex learning problem, we can easily conclude that mirror descent is also optimal in terms of number of gradient calculations needed to achieve a target sub-optimality. 

As mentioned, in general getting a handle on exact computational complexity is hard. However for a few important class of supervised learning problems with absolute loss (can also be extended to hinge loss, logistic loss etc.) one can show that mirror descent is optimal even in terms of exact computational complexity. To see this consider a supervised learning problems where instances given are $S = (\x_1,y_1),\ldots,(\x_n,y_n)$. Let $d$ be the intrinsic dimensionality of the sample $S$ (in vector space $\B$). Then the total time required just to read the sample is $d n$. However note that for many cases like when $\H/\X$ are various  unit balls of $\ell_{p}$ norms or $L_p$ norms or group norms then as we will see in the next section, mirror descent update step has tipple complexity of order $d$ and so since mirror descent is a single pass algorithm that goes over samples one by one, its  time complexity when run once over the sample is again of order $n d$. However we already argued that the learning rate of mirror descent is near optimal and so sample size $n$ required for getting sub-optimality of at most $\epsilon$ is also near optimal. Hence we can conclude that mirror descent algorithm for these cases have time complexity of same order as time complexity needed just to read the minimum required sample. Hence mirror descent for these cases is near optimal even in terms of exact computational time.

\section{Examples}\label{sec:egstat}
First of all, we start by noticing that all the examples we considered in the previous chapter were Banach Lattices with finite co-type and so for all these problems by Theorems \ref{thm:stattight1}, \ref{thm:stattight2} and \ref{thm:stattight3} stochastic mirror descent technique is near optimal in terms of both rates and number of gradient access, for convex Lipschitz problems, supervised learning problems and non-negative smooth problems. We will further see below that for these examples we don't just get near optimality but that in fact $C_p$ and $C^{\mrm{iid}}_p$ are within fixed numeric constant factor of each other. Hence we see that beyond the logarithmic factor there aren't any space dependent hidden factors even.

\subsection{Example : $\ell_{p}$ non-dual pairs} \index{$\ell_p$ norm}
In the previous chapter we gave characterization of constant $D_2$ (corresponding to $1/\sqrt{n}$ rates in online setting) for the $\ell_p$ pairs. Recall the setting, that is $\H$ is the unit ball of $\ell^d_{p_1}$ ball and $\X$ is the dual ball of $\ell^d_{p_2}$ norm. It turns out that for $\ell_p$ norms, when $p \in [1,2]$ we have a universal constant for $1$-decoupling by Proposition \ref{prop:decouplelp}. On the other hand, when $p \in (2,\infty)$, the type constant of $\ell_p$ spaces is of order $\sqrt{p}$ and so we can conclude that the table is essentially tight even for the statistical learning setting because martingale type constant ($C_p$) and type constant ($C^{\mrm{iid}}_p$) are within fixed numeric constant factor of each other.

\subsection{Example :  Non-dual Schatten norm pairs in finite dimensions} \index{Schatten norm}
As we showed in the previous chapter, the constants $D_p$ ($C_p$, $V_p$ etc.) for Schatten norms were same as those for $\ell_p$ norms in  the online learning setting. However as we saw in  the previous section, constants for $\ell_p$ norms match in statistical and online frameworks. But since lower bounds for $\ell_p$ case can be converted to atleast same lower bound on Schatten norm case (by diagonalizing). We can again conclude that rates of Mirror descent in statistical case are near optimal and moreover that $C_p$ and $C^{\mrm{iid}}_p$ are within fixed numeric constant factor of each other.

\subsection{Example : Non-dual group norm pairs in finite dimensions}\index{group norm}\index{mixed norm}
For the group norm case using Decoupling inequalities for $\ell_p$ along with Proposition \ref{lem:decoupling} we can show again that $C_p$ and $C^{\mrm{iid}}_p$ are within constant factor of each other which guarantees tightness of mirror descent for for these problems in the statistical setting.


\subsection{Computational Efficiency Issues}
Up to now we used number of gradient calculations to argue that mirror descent is optimal in terms of efficiency. Notice that we showed optimality using mirror descent algorithm which has a simple update step that at every round uses only previous hypothesis and gradient of loss at the hypothesis. Hence the time complexity of the update at each round is of the order of effective dimension (as an example in the $\ell_p$ case it is linear in $d$). Thus once can translate in these cases oracle complexity to time complexity of the algorithm. An for $\ell_p$ cases one can even show that time complexity is near optimal. However this is not always the case, it might be that the complexity of update step of mirror descent (which depends on $\Psi$) is large that time complexity blows up. This can be seen for instance in the max norm example. In the previous chapter when we considered the max norm example, the function $\Psi$ we considered in Equation \ref{eq:maxreg} had a summation over $2^{M+N}$ elements for matrix of size $M \times N$. This is of course prohibitive to use in practice. However notice that while oracle complexity is still tight, for matrix completion problem in the statistical learning framework, one can use a constrained minimization approach (minimize average loss subject to constraint on max norm or equivalently max norm regularization) and using an SDP approach this can be done in time polynomial in the matrix size. This shows a disparity between oracle complexity and actual time complexity. It is an interesting open question whether there is a polynomial time algorithm for max norm based learning problem in the online framework. In \cite{CesSha11}, in the transductive online setting for max norm it is shown that once can again use the SDP to obtain a poly-time algorithm. An even stronger flavor of questioning is whether one can provide a mirror descent (or variant) algorithm for learning with max-norm that works in time polynomial in matrix size.

\section{Detailed Proofs}\label{sec:proofsstat}
\begin{proof}[Proof of Lemma \ref{lem:linlowstat}]
First we pick elements $\x_1,\ldots,\x_{2n} \in \X$ and then draw Rademacher random variables $\epsilon \in \{\pm1\}^{2n}$ uniformly at random.  We shall now use this to construct a distribution $\D_\epsilon$ over instances which is the one we shall use for the lower bound. Specifically, given a draw $\epsilon \in \{\pm1\}^{2n}$ consider the distribution $\D_{\epsilon}$ to be the uniform distribution over the set $\{\epsilon_1 \x_1, \ldots, \epsilon_{2n} \x_{2n}\}$.
Now consider a $S$ of size $n$ drawn iid from the distribution $\D_\epsilon$. Now note that for any learning algorithm $\Algo$,
\begin{align*}
\sup_{\D} \Es{S}{\L_\D(A(S)) - \inf_{\h \in \cH} \L_\D(\h)}  \ge \sup_{\x_1,\ldots,\x_{2n}} \En_{\epsilon} \Es{S \sim \D_{\epsilon}^n }{ \L_{\D_{\epsilon}}(A(S)) - \inf_{\h \in \H} \L_{\D_{\epsilon}}(\h)}
\end{align*}
Since $\D_{\epsilon}$ is the uniform distribution over set $\{\epsilon_1 \x_1,\ldots,\epsilon_{2n} \x_{2n}\}$ one can rewrite sampling $S$ from the distribution as follows :  First we sample $n$ numbers, $t_1,\ldots,t_n$ uniformly at random from the set $[2n]$. Next, the sample $S$ is given by $S = \epsilon_{t_1} \x_{t_1},\ldots, \epsilon_{t_n} \x_{t_n}$. Hence we can rewrite the above inequality as
\begin{align*}
\sup_{\D} \Es{S}{\L_\D(A(S)) - \inf_{\h \in \cH} \L_\D(\h)}  \ge \sup_{\x_1,\ldots,\x_{2n}} \En_{\epsilon} \Es{t_{1},\ldots,t_n \sim \mrm{unif}([2n])}{ \L_{\D_{\epsilon}}(A(S)) - \inf_{\h \in \cH} \L_{\D_{\epsilon}}(\h) }
\end{align*}
Now let the set $J \subset [2n]$ be the set 
$$
J = \{i \in [2n] :  i \in \{t_1,\ldots,t_n\}\}
$$
that is the set of indices $i$ such that $\x_i$ appeared at least once in the sample $S$ provided to the learner. Also let $J^c \in [2n]$ stand for the complement of the set $J$. For linear class, for the distribution $\D_\epsilon$ note that for any $\h \in \cH$,
$$
\L_{\D_\epsilon}(\h) = \frac{1}{2n} \sum_{i=1}^{2n} \ip{\h}{\epsilon_i \x_i}
$$
Hence we see that,
\begin{align*}
& \sup_{\D} \Es{S}{\L_\D(A(S)) - \inf_{\h \in \cH} \L_\D(\h)}  \ge \sup_{\x_1,\ldots,\x_{2n} \in \X} \En_{\epsilon} \Es{t_{1},\ldots,t_n \sim \mrm{unif}([2n])}{ \L_{\D_\epsilon}(A(S)) - \inf_{\h \in \cH} \L_{\D_\epsilon}(\h) }\\
&~~~~~ = \sup_{\x_1,\ldots,\x_{2n} \in \X} \En_{\epsilon}\Es{t_{1},\ldots,t_n \sim \mrm{unif}([2n])}{ \sup_{\h \in \cH} \frac{1}{2n} \sum_{i=1}^{2n} \ip{\h}{- \epsilon_i \x_i}  -\frac{1}{2n} \sum_{i=1}^{2n} \ip{A(S)}{- \epsilon_i \x_i}  }\\
& ~~~~~ =  \frac{1}{2n}\ \sup_{\x_1,\ldots,\x_{2n} \in \X} \En_{\epsilon} \Es{t_{1},\ldots,t_n \sim \mrm{unif}([2n])}{\norm{\sum_{i = 1}^{2n} \epsilon_i \x_i}_{\Hd} - \sum_{i\in J} \ip{A(S)}{- \epsilon_i \x_i} - \sum_{i\in J^c} \ip{A(S)}{- \epsilon_i \x_i} }\\
& ~~~~~ =  \frac{1}{2n}\  \sup_{\x_1,\ldots,\x_{2n} \in \X} \underset{t_{1},\ldots,t_n \sim \mrm{unif}([2n])}{\En} \left[  \Es{\epsilon}{\norm{\sum_{i = 1}^{2n} \epsilon_i \x_i}_{\Hd}}  - \Es{\epsilon}{\sum_{i\in J} \ip{A(S)}{- \epsilon_i \x_i}} - \Es{\epsilon}{\sum_{i\in J^c} \ip{A(S)}{- \epsilon_i \x_i}} \right] 
\end{align*}
Now we notice that given $t_1,\ldots,t_n \in [2n]$, the sets $J$ and $J^c$ are fixed and $S$ only consists of $\epsilon_i \x_{i}$ s.t. $i \in J$. Thus, $S$ is only a function of $\epsilon_i$ where $i \in J$. Hence, 
$$
\Es{\epsilon}{\ip{A(S)}{- \sum_{i \in J^c} \epsilon_i \x_i} } = \Es{\epsilon_{J}}{ \Es{\epsilon_{J^c}}{\ip{A(S)}{- \sum_{i \in J^c} \epsilon_i \x_i}}} = \Es{\epsilon_{J}}{ \ip{A(S)}{- \Es{\epsilon_{J^c}}{ \sum_{i \in J^c} \epsilon_i \x_i}}}  = 0
$$
where we use the shorthand $\epsilon_J$ to refer to Rademacher random variables $\epsilon_i$ where $i \in J$ and $\epsilon_{J^c}$ refers to  the remaining $\epsilon_i$'s not in $J$. 
Hence we see that
\begin{align*}
& \hspace{-0.3in}\sup_{\D} \Es{S}{\L_\D(A(S)) - \inf_{\h \in \cH} \L_\D(\h)}  \\
& \ge  \frac{1}{2n}\  \sup_{\x_1,\ldots,\x_{2n} \in \X}\ \underset{t_{1},\ldots,t_n \sim \mrm{unif}([2n])}{\En}\left[\Es{\epsilon}{\norm{\sum_{i = 1}^{2n} \epsilon_i \x_i}_{\Hd}} - \Es{\epsilon}{\sum_{i\in J} \ip{A(S)}{- \epsilon_i \x_i}} \right]\\
& \ge   \frac{1}{2n}\ \sup_{\x_1,\ldots,\x_n \in \X}\  \underset{t_{1},\ldots,t_n \sim \mrm{unif}([2n])}{\En} \left[\Es{\epsilon}{\norm{\sum_{i = 1}^{2n} \epsilon_i \x_i}_{\Hd}}  - \Es{\epsilon}{\norm{\sum_{i\in J} \epsilon_i \x_i}_{\Hd}} \right]\\
& =  \sup_{\x_1,\ldots,\x_{2n} \in \X} \left\{ \Es{\epsilon}{\norm{\frac{1}{2n} \sum_{i = 1}^{2n} \epsilon_i \x_i}_{\Hd}}  - \frac{1}{2n}\ \underset{t_{1},\ldots,t_n \sim \mrm{unif}([2n])}{\En} \ \Es{\epsilon}{ \norm{ \sum_{i\in J} \epsilon_i \x_i}_{\Hd} } \right\}\\
& \ge   \sup_{\x_1 , \ldots,\x_{2n} \in \X} \Es{\epsilon}{\norm{\sum_{i = 1}^{2n}\epsilon_i \x_i}_{\Hd}}  - \frac{1}{2n}  \ \underset{t_{1},\ldots,t_n \sim \mrm{unif}([2n])}{\En} \  \sup_{\x_1,\ldots,\x_{|J|} \in \X} \Es{\epsilon}{ \norm{ \sum_{i=1}^{|J|} \epsilon_i \x_i}_{\Hd} }  
\end{align*}
Note that since sample size $|S| = n$, we have that $|J| \le n$. Therefore, 
$$
\sup_{\x_1,\ldots,\x_{|J|} \in \X} \Es{\epsilon}{ \norm{ \sum_{i\in J} \epsilon_i \x_i}_{\Hd}} \le \sup_{\x_1,\ldots,\x_{n} \in \X} \Es{\epsilon}{ \norm{ \sum_{i=1}^{n} \epsilon_i \x_i}_{\Hd}}
$$
(the above is obvious because worst case we can always put $\x_{|J|+1} , \ldots, \x_{n} = 0$ so the two are equal). Thus we conclude that :
\begin{align*}
\sup_{\D} \Es{S}{\L_\D(A(S)) - \inf_{\h \in \cH} \L_\D(\h)}  & \ge   \sup_{\x_1 , \ldots,\x_{2n} \in \X} \Es{\epsilon}{\norm{\sum_{i = 1}^{2n}\epsilon_i \x_i}_{\Hd}}  - \frac{1}{2n}  \ \underset{t_{1},\ldots,t_n \sim \mrm{unif}([2n])}{\En} \  \sup_{\x_1,\ldots,\x_{|J|} \in \X} \Es{\epsilon}{ \norm{ \sum_{i=1}^{|J|} \epsilon_i \x_i}_{\Hd} }  \\
 & \ge   \sup_{\x_1 , \ldots,\x_{2n} \in \X} \Es{\epsilon}{\norm{\sum_{i = 1}^{2n}\epsilon_i \x_i}_{\Hd}}  - \frac{1}{2n}  \   \sup_{\x_1,\ldots,\x_{n} \in \X} \Es{\epsilon}{ \norm{ \sum_{i=1}^{n} \epsilon_i \x_i}_{\Hd} }  \\
 & = \Radstat_{2n}(\F_{\mrm{lin}}(\H,\X)) - \frac{1}{2} \Radstat_{n}(\F_{\mrm{lin}}(\H,\X))
\end{align*}
This conclude the lemma.
\end{proof}

\begin{proof}[Proof of Theorem \ref{thm:statlinbnd}]
Applying the upper bound guaranteed by the assumption of the theorem and using the lower bound from lemma \ref{lem:linlowstat} we get that  for any $n \in \mathbb{N}$:
\begin{align*}
\Radstat_{2n}(\F_{\mrm{lin}}(\H,\X)) \le \frac{1}{2} \Radstat_{n}(\F_{\mrm{lin}}(\H,\X)) + \frac{V}{n^{1/q}}
\end{align*}
Expanding the recursive inequality on the right above we get that
$$
\Radstat_{2n}(\F_{\mrm{lin}}(\H,\X)) \le \frac{V}{n^{1/q}}\left(1 + \frac{1}{2^{1/p}} + \frac{1}{2^{2/p}} + \ldots \right) \le \frac{2^{1/p}}{2^{1/p} - 1} \frac{V}{n^{1/q}}
$$
 Thus we conclude the first inequality in the theorem statement. As for the upper bound on the fat shattering dimension, we have that whenever $\beta \le \Radstat_n(\F_{\mrm{lin}}(\H,\X))$,
$$
\fatstat_{\beta}(\F_{\mrm{lin}}(\H,\X)) \le \frac{n \left(\Radstat_{n}(\F_{\mrm{lin}}(\H,\X))\right)^2}{\beta^2} \le n 
$$
Hence we conclude that for all $n$ such that $\Radstat_n(\F_{\mrm{lin}}(\H,\X)) \ge \beta$,
$\fatstat_{\beta}(\F_{\mrm{lin}}(\H,\X)) \le n$. In other words,
$$
\fatstat_\beta(\F_{\mrm{lin}}(\H,\X)) \le \inf\{n :  \Radstat_n(\F_{\mrm{lin}}(\H,\X)) \ge \beta\} \inf\{2n :  \Radstat_{2n}(\F_{\mrm{lin}}(\H,\X)) \ge \beta\}
$$
However since we already proved that for all $n \in \mathbb{N}$, $\Radstat_{2n}(\F_{\mrm{lin}}(\H,\X)) \le \frac{5 V}{n^{1/q}}$ we can conclude that for any $\beta > 0$,  
$$
\fatstat_\beta(\F_{\mrm{lin}}(\H,\X)) \le \left(\frac{5 V}{\beta}\right)^q~~.
$$
This concludes the theorem.
\end{proof}

\begin{proof}[Proof of Lemma \ref{lem:statsmtlow}]
Consider instance space $\Z = \X \times [-1,1]$ and the non-negative function $\phi: \reals \times [-1,1] \mapsto \reals^+$ that is $1$-smooth (in its first argument) given by :
$$
\phi(z,y) = \left\{ \begin{array}{ll}
(z - y)^2 & \textrm{if } |z - y| \le \frac{1}{2}\\
|z - y| - \frac{1}{4} & \textrm{otherwise}
\end{array}\right.
$$
Now the loss function we consider is given by $\ell(\h,(\x,y)) = \phi(\ip{\h}{\x},y)$. Since $\phi$ is $1$-smooth in its first argument, we see that
\begin{align*}
\norm{\nabla \ell(\h,(\x,y)) - \nabla \ell(\h',(\x,y))}_{\X} & = |\partial \phi(\ip{\h}{\x},y) - \partial \phi(\ip{\h'}{\x},y) | \norm{\x}_{\X}  \le  |\ip{\h - \h'}{\x}| \norm{\x}_{\X}\\
& \le \norm{\h - \h'}_{\Xd} \norm{\x}^2_{\X} \le \norm{\h - \h'}_{\Xd}
\end{align*}
Thus we conclude that the loss is $1$-smooth and so belongs to class $\Z_{\mrm{smooth}}$. The distribution we pick for showing the lower bound is a slight modification of the one used in Lemma \ref{lem:linlowstat}. We start by picking elements $\x_1,\ldots,\x_{2n\L^*} \in \X$ and then draw Rademacher random variables $\epsilon \in \{\pm1\}^{2n \L^*}$ uniformly at random.  We shall now use this to construct a distribution $\D'_\epsilon$ over instances which is the one we shall use for the lower bound. Specifically, given a draw $\epsilon \in \{\pm1\}^{2n \L^*}$ consider the distribution $\D_{\epsilon}$ on $\X$ to be the uniform distribution over the set $\{\epsilon_1 \x_1, \ldots, \epsilon_{2n \L^*} \x_{2n \L^*}\}$. Now the distribution $\D'_\epsilon$ is the distribution on $\Z = \X \times [-1,1]$ to be the one that picks $(\mbf{0},0)$ with probability $1 - \L^*$ and with probability $\L^*$, draws input instance $\x \in \X$ i.i.d. from distribution $\D_\epsilon$ and deterministically picks label $y = -1$ (or whatever $\sup_{\h \in \H, \x \in \X} \ip{\h}{\x}$ is with appropriate scaling changes but for simplicity let us assume its bounded by 1) . Note that for any $\h \in \H$,
$$
\L_{\D'_{\epsilon}}(\h)  = \L^* \Es{\D_{\epsilon}}{|\ip{\h}{\x} - y| - \frac{1}{4}} =  \frac{1}{2n} \sum_{i=1}^{2n \L^*} \left(\ip{\h}{\epsilon_i \x_i} - 1 - \frac{1}{4}\right) ~~.
$$
Hence we see that for any $\epsilon \in \{\pm1\}^{2n}$,
$$
\L_{\D'_\epsilon}(A(S)) - \inf_{\h \in \H} \L_{\D'_\epsilon}(\h) =  \frac{1}{2n} \sum_{i=1}^{2n \L^*} \ip{A(S)}{\epsilon_i \x_i}  - \inf_{\h \in \H} \frac{1}{2n} \sum_{i=1}^{2n \L^*} \ip{\h}{\epsilon_i \x_i} 
$$
Now consider sample $S$ of size $n$ drawn iid from the distribution $\D'_\epsilon$. Note that for any learning algorithm $\Algo$,
\begin{align*}
\sup_{\D} \Es{S}{\L_\D(A(S)) - \inf_{\h \in \cH} \L_\D(\h)} &\ge \sup_{\x_1,\ldots,\x_{2n \L^*}} \En_{\epsilon} \Es{S \sim {\D'_{\epsilon}}^n }{ \L_{\D'_{\epsilon}}(A(S)) - \inf_{\h \in \H} \L_{\D'_{\epsilon}}(\h)} \notag \\
& =  \sup_{\x_1,\ldots,\x_{2n \L^*}} \En_{\epsilon} \Es{S \sim \D'_{\epsilon}}{\frac{1}{2n} \sum_{i=1}^{2n \L^*} \ip{A(S)}{\epsilon_i \x_i}  - \inf_{\h \in \H} \frac{1}{2n} \sum_{i=1}^{2n \L^*} \ip{\h}{\epsilon_i \x_i} }\\
& =  \sup_{\x_1,\ldots,\x_{2n \L^*}} \En_{\epsilon} \Es{S \sim \D'_{\epsilon}}{\frac{1}{2n} \sum_{i=1}^{2n \L^*} \ip{A(S)}{\epsilon_i \x_i}  + \norm{\frac{1}{2n} \sum_{i=1}^{2n \L^*} \epsilon_i \x_i}_{\Hd} }
\end{align*}
Proceeding in similar fashion as in the proof of Lemma \ref{lem:linlowstat}, we see that one can rewrite sampling $S$ from the distribution as follows :  First we sample $m$ from binomial distribution $\mrm{Binomial}(\L^*,n)$ (represents the samples for which $y \ne 0$). Next, numbers $t_1,\ldots,t_m$ is drawn uniformly at random from the set $[2n \L^*]$. Hence we get the inequality :
\begin{align*}
& \sup_{\D} \Es{S}{\L_\D(A(S)) - \inf_{\h \in \cH} \L_\D(\h)} \\
& ~~~~~ \ge \sup_{\x_1,\ldots,\x_{2n \L^*}} \En_{\epsilon} \underset{m \sim \mrm{Binomial(n,\L^*)}}{\En} \Es{t_1,\ldots,t_m \sim \mrm{Unif}([2n \L^*])}{\frac{1}{2n} \sum_{i=1}^{2n \L^*} \ip{A(S)}{\epsilon_i \x_i}  + \norm{\frac{1}{2n} \sum_{i=1}^{2n \L^*} \epsilon_i \x_i}_{\Hd} }\\
& ~~~~~ = \sup_{\x_1,\ldots,\x_{2n \L^*}} \underset{m \sim \mrm{Binomial(n,\L^*)}}{\En} \Es{t_1,\ldots,t_m \sim \mrm{Unif}([2n \L^*])}{\Es{\epsilon}{\ip{A(S)}{\frac{1}{2n} \sum_{i=1}^{2n \L^*}\epsilon_i \x_i}} + \Es{\epsilon}{\norm{\frac{1}{2n} \sum_{i=1}^{2n \L^*} \epsilon_i \x_i}_{\Hd}}  }
\end{align*}
Now let the set $J \subset [2n \L^*]$ be the set  $J = \{i \in [2n \L^*] :  i \in \{t_1,\ldots,t_m\}\}$
that is the set of indices $i$ such that $(\x_i,-1)$ appeared at least once in the sample $S$ provided to the learner. Also let $J^c \in [2n \L^*]$ stand for the complement of the set $J$. Following the same line of proof as in Lemma \ref{lem:linlowstat} noting that $S$ only depends on the $\epsilon$'s that occur in the sample $S$, we can see that
$$
\Es{\epsilon}{\ip{A(S)}{\frac{1}{2n} \sum_{i=1}^{2n \L^*}\epsilon_i \x_i}} \ge - \Es{\epsilon}{\norm{\frac{1}{2n} \sum_{i \in J}\epsilon_i \x_i}_{\Hd}}
$$
and so we conclude that : 
\begin{align*}
& \sup_{\D} \Es{S}{\L_\D(A(S)) - \inf_{\h \in \cH} \L_\D(\h)} \\
& ~~~~~ \ge  \sup_{\x_1,\ldots,\x_{2n \L^*}} \underset{m \sim \mrm{Binomial(n,\L^*)}}{\En} \Es{t_1,\ldots,t_m \sim \mrm{Unif}([2n \L^*])}{\Es{\epsilon}{\norm{\frac{1}{2n} \sum_{i=1}^{2n \L^*} \epsilon_i \x_i}_{\Hd}} - \Es{\epsilon}{\norm{\frac{1}{2n} \sum_{i \in J}\epsilon_i \x_i}_{\Hd}}  }\\
& ~~~~~ \ge  \sup_{\x_1,\ldots,\x_{2n \L^*}} \Es{m \sim \mrm{Binomial(n,\L^*)}}{ \Es{\epsilon}{\norm{\frac{1}{2n} \sum_{i=1}^{2n \L^*} \epsilon_i \x_i}_{\Hd}} - \sup_{\x_1,\ldots, \x_{m} \in \X} \Es{\epsilon}{\norm{\frac{1}{2n} \sum_{i = 1}^{\min\{m , 2 \L^* n \}} \epsilon_i \x_i}_{\Hd}}  }
\end{align*}
However note that with probability $1/2$, $m \le n \L^*$ and so we have that 
\begin{align*}
 \sup_{\D} \Es{S}{\L_\D(A(S)) - \inf_{\h \in \cH} \L_\D(\h)} & \ge  \frac{1}{2} \sup_{\x_1,\ldots,\x_{2n \L^*}}  \Es{\epsilon}{\norm{\frac{1}{2n} \sum_{i=1}^{2n \L^*} \epsilon_i \x_i}_{\Hd}} - \sup_{\x_1,\ldots, \x_{n \L^*} \in \X} \Es{\epsilon}{\norm{\frac{1}{2n} \sum_{i = 1}^{\L^* n} \epsilon_i \x_i}_{\Hd}}  \\
& =  \frac{\L^*}{2} \left( \Radstat_{2 \L^* n}(\F_{\mrm{lin}}(\H,\X)) - \frac{1}{2} \Radstat_{\L^* n}(\F_{\mrm{lin}}(\H,\X)) \right)  
\end{align*}
This conclude the lemma.
\end{proof}

\begin{proof}[Proof of Lemma \ref{lem:radtype}]
First, since both sides below are homogenous, the premise of the lemma can be rewritten as, for all $n \in \mathbb{N}$ and all $\x_1,\ldots,\x_n \in \Bd$, 
\begin{align}\label{eq:premise}
\Es{\epsilon}{\norm{\sum_{i=1}^n \epsilon_i \x_i}_{\Hd}} \le D\ n^{\frac{1}{r}} \max_{i \in [n]} \norm{\x_i}_{\X}
\end{align}
Let  $S =  \left(\sum_{i=1}^n \norm{\x_i}_{\X}^p \right)^{1/p}$, define
\begin{align*}
& I_k := \left\{i \ge 1 \middle| \tfrac{S}{2^{(k+1)/p}} < \|\x_i\|_\X  \le \tfrac{S}{2^{k/p}}\right\}~~,\\
& T_0^{(k)} := \inf\{i \in I_k\}~~\textrm{and}\\
& \forall m \in \mathbb{N},\ T_m^{(k)} := \inf\{i > T_{m-1}^{(k)}, i \in I_k\}
\end{align*}
Note that, $S^p \ge \sum_{i \in I_k} \|\x_i\|_\X^p > \tfrac{S^p\ |I_k|}{2^{(k+1)}}$ and so we get that $|I_k| < 2^{k+1}$. From this, using the premise in Equation \ref{eq:premise} we conclude that
\begin{align*}
\Es{\epsilon}{\norm{\sum_{i=1}^n \epsilon_i \x_i}_{\Hd}} & \le \sum_{k \ge 0}\Es{\epsilon}{\norm{\sum_{i\in I_k} \epsilon_i \x_i}_{\Hd}} = \sum_{k \ge 0}\Es{\epsilon}{\norm{\sum_{i \ge 0}  \epsilon_{T_i^{(k)}} \x_{T_i^{(k)}}}_{\Hd}}\\
&  \le \sum_{k \ge 0} \left( D\ \{ |I_k|^{1/r}\} \{\sup_{i \in I_k } \norm{\x_i }_{\X}\} \right) \\
& \le \sum_{k \ge 0} \left( D\ 2^{(k+1)/r}  \sup_{i \in I_k } \norm{\x_i }_{\X,\infty} \right) \\
& \le \sum_{k \ge 0} \left( D\ 2^{(k+1)/r}\  2^{-k/p} S \right) \\
& = D\ 2^{1/r} \ \sum_{k \ge 0}2^{k (\frac{1}{r} - \frac{1}{p})}\ S  \le  \frac{2 D}{1 - 2^{(\frac{1}{r} - \frac{1}{p})}} S  \le \frac{12 D}{r - p} S\\
& = \frac{12 D}{r - p}  \left(\sum_{i=0}^n \norm{\x_i }_{\X}^p \right)^{1/p}
\end{align*}
We conclude the proof by using Kahane Inequality (see \cite{Kahane}) which asserts that for any $p \in [1,2]$,
$$
\left(\Es{\epsilon}{\norm{\sum_{i=1}^n \epsilon_i \x_i}^p_{\Hd}} \right)^{1/p} \le \sqrt{2}\ \Es{\epsilon}{\norm{\sum_{i=1}^n \epsilon_i \x_i}_{\Hd}}
$$
\end{proof}

\begin{proof}[Proof of Proposition \ref{prop:o2b}]
Since we are dealing with convex loss, using Jensen's inequality we have that
$$
\L_\D\left(\frac{1}{n}\sum_{t = 1}^n \Algo(\z_{1:t})\right) \le \frac{1}{n}\sum_{t = 1}^n \L_\D\left(\Algo(\z_{1:t})\right)~~.
$$
Hence we see that 
\begin{align*}
\Es{S \sim \D^n}{\left(\frac{1}{n}\sum_{t = 1}^n \Algo(\z_{1:{t-1}})\right) - \inf_{\h \in \cH} \L_\D(\h)} & \le \frac{1}{n} \sum_{t=1}^n \Es{S}{\L_\D(\Algo(\z_{1:t-1}))} - \inf_{\h \in \cH} \frac{1}{n} \sum_{t=1}^n \L_\D(\h)\\
& = \frac{1}{n} \sum_{t=1}^n \Es{S  \sim \D^n}{\ell(\Algo(\z_{1:t-1}),\z_t)} - \inf_{\h \in \cH} \frac{1}{n} \sum_{t=1}^n \Es{S  \sim \D^n}{\ell(\h,\z_t)}\\
& = \Es{S  \sim \D^n }{ \frac{1}{n} \sum_{t=1}^n \ell(\Algo(\z_{1:t-1}),\z_t)} - \inf_{\h \in \cH} \Es{S}{\frac{1}{n} \sum_{t=1}^n \ell(\h,\z_t)}\\
& \le \Es{S  \sim \D^n}{ \frac{1}{n} \sum_{t=1}^n \ell(\Algo(\z_{1:t-1}),\z_t)} -  \Es{S}{\inf_{\h \in \cH}\frac{1}{n} \sum_{t=1}^n \ell(\h,\z_t)}\\
& = \Es{S  \sim \D^n}{  \Reg_n(\Algo,\z_{1},\ldots,\z_n) } \\
& \le \sup_{\z_1,\ldots,\z_n \in \Z}  \Reg_n(\Algo,\z_{1},\ldots,\z_n) 
\end{align*}
where the second step is because $\Algo(\z_1,\ldots,\z_{t-1})$ only depends on $\z_1,\ldots,\z_{t-1}$ and so $\Es{S}{\ell(\Algo(\z_1,\ldots,\z_{t-1}),\z_t)} = \Es{S}{\L_\D(\Algo(\z_1,\ldots,\z_{t-1}))}$. 
The second part of the proposition is from the fact that the above holds for any online learning algorithm.
\end{proof}

\begin{proposition}\label{prop:typecotypedual}
If $(\Hd,\X)$ has type $p$ with some constant $C^{\mrm{iid}}_p$ then for $q = \frac{p}{p-1}$, the pair $(\H,\Xd)$ has co-type $q$ with constant $\frac{1}{C^{\mrm{iid}}_p}$
\end{proposition}
\begin{proof}
Given $\h_1,\ldots,\h_n \in \Bd$ for any $\epsilon > 0$ we can pick $\x_1,\ldots,\x_n$ such that 
\begin{align}\label{eq:epslowtype}
\sum_{t=1}^n \ip{\h_t}{\x_t} \ge (1 - \epsilon) \left(\sum_{t=1}^n \norm{\h_t}^q_{\Xd} \right)^{1/q} \left(\sum_{t=1}^n \norm{\x_t}^p_{\X} \right)^{1/p}
\end{align}
On the other hand we have that,
\begin{align*}
\sum_{t=1}^n \ip{\h_i}{\x_i} &= \Es{\epsilon}{ \ip{ \sum_{t=1}^n \epsilon_t \h_t}{\sum_{t=1}^n \epsilon_t \x_t} }\\
& \le  \left(\Es{\epsilon}{ \norm{\sum_{t=1}^n \epsilon_t \h_t}^q_{\H}}\right)^{1/q} \left( \Es{\epsilon}{\norm{\sum_{t=1}^n \epsilon_t \x_t}^p_{\Hd}}\right)^{1/p}\\
& \le  C^{\mrm{iid}}_p \left(\Es{\epsilon}{ \norm{\sum_{t=1}^n \epsilon_t \h_t}^q_{\H}}\right)^{1/q} \left( \sum_{t=1}^n \norm{\x_t}^p_{\X}\right)^{1/p}
\end{align*}
where the last step is by the type inequality for $(\H,\X)$. Combining the above and Equation \ref{eq:epslowtype} and taking limit of $\epsilon \rightarrow 0$ proves the statement.
\end{proof}

\begin{definition}\label{def:pconvex} \index{p-convex}
For $p \in [1,2]$ the pair $(\Hd,\X)$ is said to be $p$-convex with constant $K_p$ if for any $\x_1,\ldots,\x_N \in \X$,
$$
\norm{\left(\sum_{t=1}^N |\x_t|^p \right)^{1/p}}_{\Hd} \le K_p \left( \sum_{t=1}^N \norm{\x_t}^p_{\X} \right)^{1/p}
$$
\end{definition}

\begin{definition}\label{def:qconcave}\index{q-concave}
For $q \in [2,\infty)$ the pair $(\H,\Xd)$ is said to be $q$-concave with constant $\overline{K}_q$ if for any $K \le \overline{K}_q(\H,\Xd)$ and any $\x_1,\ldots,\x_N \in \X$,
$$
\norm{\left(\sum_{t=1}^N |\x_t|^q \right)^{1/q}}_{\Hd} \ge \frac{1}{\tilde{K}_q} \left( \sum_{t=1}^N \norm{\x_t}^q_{\X} \right)^{1/q}
$$
\end{definition}

\begin{lemma}\label{lem:pconvextype}
If the pair $(\Hd,\X)$ is of type $p$ with constant $C^{\mrm{iid}}_p$ then, for any $\x_1,\ldots,\x_n \in \X$,
\begin{align*}
\sqrt{2} C^{\mrm{iid}}_p  \left( \sum_{t=1}^n \norm{\x_t}^p_{\Xd} \right)^{1/p} & \ge  \norm{\left(\sum_{t=1}^n \left|\x_t\right|^p\right)^{\frac{1}{p}}}_{\Hd}
\end{align*}
That is it is $p$-convex with constant $\sqrt{2} C^{\mrm{iid}}_p$. 
\end{lemma}
\begin{proof}
By type $p$ with constant $C^{\mrm{iid}}_p$,
\begin{align*}
 C^{\mrm{iid}}_p \left( \sum_{t=1}^n \norm{\x_t}^p_{\Xd} \right)^{1/p} & \ge \left(\Es{\epsilon}{\norm{\sum_{t=1}^n \epsilon_t \x_t}^p_{\Hd}}\right)^{\frac{1}{p}}  \ge  \Es{\epsilon}{\norm{\sum_{t=1}^n \epsilon_t \x_t}_{\Hd}}  \ge  \norm{\En_{\epsilon}\left|\sum_{t=1}^n \epsilon_t \x_t\right|}_{\Hd}  \\
 &  \ge \frac{1}{\sqrt{2}}  \norm{\left(\sum_{t=1}^n \left|\x_t\right|^p\right)^{\frac{1}{p}}}_{\Hd}
\end{align*} 
where the last inequality is got by using scalar version of Kintchine's inequality with functional calculus for Banach lattice (specifically Theorem 1.d.1 of \cite{LinTza79} by noting that expectation w.r.t. $\epsilon$ can be written as finite average). This concludes the proof.
\end{proof}

\begin{proposition}\label{prop:dualpcvx}
If $(\Hd,\X)$ is $p$-convex with some constant $K_p$ then for $q = \frac{p}{p-1}$, the pair $(\H,\Xd)$ is $q$-concave with constant $K_p$. That is for any $\h_1,\ldots,\h_n \in \H$
$$
\norm{\left(\sum_{t=1}^n |\h_t|^q \right)^{1/q}}_{\H} \ge  \frac{1}{K_p} \left(\sum_{t=1}^n \norm{\h_t}^q_{\Xd} \right)^{1/q}
$$
\end{proposition}
\begin{proof}
Given $\h_1,\ldots,\h_n \in \Bd$ for any $\epsilon > 0$ we can pick $\x_1,\ldots,\x_n$ such that 
\begin{align}\label{eq:epslowsvx}
\sum_{t=1}^n \ip{\h_t}{\x_t} \ge (1 - \epsilon) \left(\sum_{t=1}^n \norm{\h_t}^q_{\Xd} \right)^{1/q} \left(\sum_{t=1}^n \norm{\x_t}^p_{\X} \right)^{1/p}
\end{align}
On the other hand, using 1.d.2 (iii) \cite{LinTza79} we can conclude that
\begin{align*}
\sum_{t=1}^n \ip{\h_t}{\x_t} &\le \ip{ \left(\sum_{t=1}^n |\h_t|^q \right)^{1/q} }{ \left(\sum_{t=1}^n |\x_t|^p \right)^{1/p}} \\
&\le  \norm{\left(\sum_{t=1}^n |\h_t|^q \right)^{1/q}}_{\H} \norm{ \left(\sum_{t=1}^n |\x_t|^p \right)^{1/p} }_{\Hd}\\
&\le K_p \norm{\left(\sum_{t=1}^n |\h_t|^q \right)^{1/q}}_{\H} \left( \sum_{t=1}^n \norm{\x_t}^p_{\X}\right)^{1/p}
\end{align*}
where the last step is by the $p$-convex inequality of $(\Hd,\X)$. Combining the above and Equation \ref{eq:epslowcvx} and taking limit of $\epsilon \rightarrow 0$ proves the statement.
\end{proof}

\begin{proposition}\label{prop:burkres}
For any $n \in \mathbb{N}$, any sequence real valued tree $\ta$ of depth $n$ and for any $1 \le p \le 2 \le  r < \infty$,
$$
\frac{18^3\ \sqrt{2}\ r^3 p^{\frac{3}{2}}}{\sqrt{p-1}\ (r-1)} \left( \Es{\epsilon}{\left|\sum_{t=1}^n \epsilon_t \ta_t(\epsilon)\right|^p} \right)^{1/p} \ge \left( \Es{\epsilon}{\left|\sum_{t=1}^n \epsilon_t \ta_t(\epsilon)\right|^r} \right)^{1/r}
$$
\end{proposition}
\begin{proof}
\begin{align*}
\left(\Es{\epsilon}{\sum_{t=1}^n  |\ta_t(\epsilon)|^p }\right)^{\frac{1}{p}} & \ge \left(\Es{\epsilon}{\left(\sum_{t=1}^n  |\ta_t(\epsilon)|^2 \right)^{\frac{p}{2}}}\right)^{\frac{1}{p}} & \forall p\le 2, \ \ \norm{\cdot}_p^p \ge \norm{\cdot}_2^p \\
& \ge \frac{\sqrt{p-1}}{18\ p^{\frac{3}{2}}}\left(\Es{\epsilon, \epsilon'}{\left|\sum_{t=1}^n \epsilon'_t \epsilon_t \ta_t(\epsilon)\right|^p}\right)^{1/p} & \textrm{Burkholder's Inequality \cite{Burkholder66}}\\
& \ge \frac{\sqrt{p-1}}{18\  \sqrt{2 (r-1)} p^{\frac{3}{2}}}\left(\Es{\epsilon, \epsilon'}{\left|\sum_{t=1}^n \epsilon'_t \epsilon_t \ta_t(\epsilon)\right|^r}\right)^{1/r} \\
& \ge \frac{\sqrt{p-1}\ (r-1)}{18^3\ \sqrt{2}\ r^3 p^{\frac{3}{2}}}\left(\Es{\epsilon}{\left|\sum_{t=1}^n \epsilon_t \ta_t(\epsilon)\right|^r}\right)^{1/r} 
\end{align*}
This concludes the proof.
\end{proof}

\begin{proof}[Proof of Lemma \ref{lem:latticeG}]
Consider any $\X$ valued tree $\tx$ of infinite depth. We start by noting that type $p$ with constant $C^{\mrm{iid}}_p$ implies $p$-convexity with constant $\sqrt{2} C^{\mrm{iid}}_p$ and so 
\begin{align}\label{eq:typetomarttype}
\frac{18^3 \sqrt{8} p^{\frac{3}{2}} r^3 C^{\mrm{iid}}_p \tilde{C}_r}{\sqrt{p-1} (r-1)}  \left(\Es{\epsilon}{\sum_{t=1}^n \norm{ \tx_t(\epsilon)}^p_{\X} }\right)^{\frac{1}{p}} & =  \frac{18^3 \sqrt{8} p^{\frac{3}{2}} r^3 C^{\mrm{iid}}_p \tilde{C}_r}{\sqrt{p-1} (r-1)} \left(\frac{1}{2^n}\sum_{\epsilon \in \{\pm1\}^{n}}\sum_{t=1}^n \norm{ \tx_t(\epsilon)}^p_{\X} \right)^{\frac{1}{p}}\notag \\
& \ge  \frac{18^3 \sqrt{4} p^{\frac{3}{2}} r^3 \tilde{C}_r}{\sqrt{p-1} (r-1)} \norm{ \left(\frac{1}{2^{n}} \sum_{\epsilon \in \{\pm1\}^{n}}\sum_{t=1}^n  | \tx_t(\epsilon)|^p \right)^{\frac{1}{p}} }_{\Hd} \notag \\
& =  \frac{18^3 \sqrt{4} p^{\frac{3}{2}} r^3  \tilde{C}_r}{\sqrt{p-1} (r-1)} \norm{ \left(\Es{\epsilon}{\sum_{t=1}^n  | \tx_t(\epsilon)|^p }\right)^{\frac{1}{p}} }_{\Hd} 
\end{align}
Now note that for the real-valued case by Proposition \ref{prop:burkres} we have that, for any real valued tree $\ta$ of depth $n$ and for any $1 \le p \le 2 \le  r < \infty$,
\begin{align*}
\frac{18^3\ \sqrt{2}\ r^3 p^{\frac{3}{2}}}{\sqrt{p-1}\ (r-1)} \left( \Es{\epsilon}{\left|\sum_{t=1}^n \epsilon_t \ta_t(\epsilon)\right|^p} \right)^{1/p} \ge \left( \Es{\epsilon}{\left|\sum_{t=1}^n \epsilon_t \ta_t(\epsilon)\right|^r} \right)^{1/r}
\end{align*}
Since both sides the expressions are degree one homogenous, applying Theorem 1.d.1 \cite{LinTza79} (by expanding  out the tree of depth $n$ to its $2^{n} - 1$ elements and noting that expectation w.r.t. $n$ Rademacher variables is in fact a finite average of $2^n$ signs) we conclude that
\begin{align*}
\frac{18^3 \sqrt{2} r^3 p^{\frac{3}{2}}}{\sqrt{p-1} (r-1)} \left(\Es{\epsilon}{\sum_{t=1}^n  | \tx_t(\epsilon)|^p }\right)^{\frac{1}{p}} & \succeq \left(\Es{\epsilon}{\left|\sum_{t=1}^n \epsilon_t  \tx_t(\epsilon)\right|^r}\right)^{1/r} 
\end{align*}
Plugging this back in Equation \ref{eq:typetomarttype} and noting that co-type $r$ with constant $\tilde{C}_r$ of the pair $(\Hd,\Hd)$ implies its $r$-concavity with constant $\sqrt{2} \tilde{C}_r$, we conclude that : 
\begin{align*}
\frac{18^3 \sqrt{8} p^{\frac{3}{2}} r^3 C^{\mrm{iid}}_p \tilde{C}_r}{\sqrt{p-1} (r-1)}  \left(\Es{\epsilon}{\sum_{t=1}^n \norm{ \tx_t(\epsilon)}^p_{\X} }\right)^{\frac{1}{p}} 
& \ge \sqrt{2}\ \tilde{C}_r \norm{ \left( \Es{\epsilon}{\left|\sum_{t=1}^n  \epsilon_t  \tx_t(\epsilon)  \right|^r} \right)^{\frac{1}{r}}  }_{\Hd}\\
& \ge   \left( \Es{\epsilon}{ \norm{ \sum_{t=1}^n  \epsilon_t  \tx_t(\epsilon) }^r_{\Hd}} \right)^{\frac{1}{r}}\\
& \ge \left(\Es{\epsilon}{ \norm{ \sum_{t=1}^n  \epsilon_t  \tx_t(\epsilon) }^p_{\Hd}} \right)^{1/p}
\end{align*}
Noting that $r^3/(r-1) \le 2 r^2$ and that $p^{3/2} \le \sqrt{8}$ and by definition of $C_p$ we  conclude the proof.
\end{proof}

\begin{proof}[Proof of Lemma \ref{lem:decoupling}]
Consider any $\X$ valued infinite depth tree $\tx$. We have,
\begin{align*}
\Es{\epsilon}{\norm{\sum_{i=1}^n \epsilon_i \tx_i(\epsilon)}_{\Hd}}   & \le D \Es{\epsilon,\epsilon'}{\norm{\sum_{i=1}^n \epsilon_i \epsilon'_i \tx_i(\epsilon)}_{\Hd}}  \\
& = D \Es{\epsilon}{\Es{\epsilon'}{\norm{\sum_{i=1}^n \epsilon_i \epsilon'_i \tx_i(\epsilon)}_{\Hd}}}\\
& \le D C^{\mrm{iid}}_r  \Es{\epsilon}{ \left( \sum_{i=1}^n \norm{ \epsilon_i \tx_i(\epsilon)}^r_{\X}\right)^{1/r}}\\
& \le D C^{\mrm{iid}}_r    \sup_{\epsilon \in \{\pm1\}^n }\left(\sum_{i=1}^n  \norm{ \tx_i(\epsilon)}^r_{\X}\right)^{1/r}
\end{align*}
Now this is effectively what we had in the proof of Lemma \ref{lem:mn} and exactly as we did there applying Lemma's \ref{lem:sub1}, \ref{lem:sub2} and \ref{lem:sub3} and repeating the steps in proof of Lemma \ref{lem:mn} we can conclude that 
\begin{align*}
\Es{\epsilon}{\norm{\sum_{i=1}^n \epsilon_i \tx_i(\epsilon)}^{p'}_{\Hd}}   & \le\ \frac{92 D C^{\mrm{iid}}_{r}}{(r - p')}\ \sum_{i=1}^n \Es{\epsilon}{ \norm{ \tx_i(\epsilon)}^{p'}_{\X}}
\end{align*}
This concludes the proof.
\end{proof}


\section{Discussion} \label{sec:statdis}
The highlight of this chapter was that we showed that for most commonly occurring learning problems constant $C_p$ can be almost bounded by $C_p^{\mrm{iid}}$ which we could the use to conclude that mirror descent was near optimal in terms of both learning rate and oracle complexity for statistical convex problems too. In fact for the common examples we saw that the constant factor relating $C_p$ and $C_p^{\mrm{iid}}$ can be a fixed universal constant. We also saw that while many a times oracle complexity could be directly associated with time complexity of the algorithms this is not always true and for the very practical problem of matrix completion with max norm this issue arose leading to the open question of whether we can efficiently solve in the online setting, matrix completion with max-norm with optimal rates.

We also saw strong connection between oracle complexity of offline optimization of convex lipschitz class and statistical complexity of linear class and using this we also informally argued that for large dimensional problems in ``reasonable space", mirror descent algorithm is near optimal even for offline optimization. Overall the aim of this section is to show that mirror descent is near optimal ubiquitously for most practical convex statistical learning problems and for high dimensional offline optimization problems not just in the online learning setting as was shown in the previous chapter.

%% file: optoff.tex
\chapter{Optimality of Mirror Descent for Offline Convex Optimization}\label{chp:optoff}

In this chapter we consider the problem of offline convex optimization. To address the issue of efficiency of optimization methods for the convex optimization problems, we use the notion of oracle complexity introduced by Nemirovski and Yudin in \cite{NemirovskiYu78}. We show interesting connections between convex optimization and statistical convex learning. Furthermore for several commonly encountered high dimensional convex optimization problems we also show that mirror descent algorithm is near optimal, in terms of oracle complexity, even for offline convex optimization. Based on the results we also show that for several statistical convex learning problems, mirror descent is optimal even when one has access to more powerful oracles that can account for parallel computational methods for optimization. In fact we show that for these problems, parallelization does not help (in improving efficiency of the learning procedure) and simply using single pass mirror descent algorithm is optimal both in terms of rates and efficiency.

Section \ref{sec:orcmdl} introduces the oracle based offline optimization model of Nemirovski and Yudin \cite{NemirovskiYu78}. Section \ref{sec:offline} provides lower bound on oracle complexity of offline convex optimization problems in terms of fat shattering dimension of associated linear class and shows connections between convex optimization and statistical learning and specifically also shows that whenever a the convex optimization problem over function class associated with $\Z_{\mrm{Lip}}(\X)$ is efficiently optimizable, then the supervised learning problem is also efficiently learnable. Section \ref{sec:optoff} shows that for several high dimensional problems, mirror descent is also optimal for offline convex optimization. Section \ref{sec:statoptdis} deals with statistical learning with distributed oracles and shows that for several cases, mirror descent is near optimal even when compared to learning algorithms that have access to distributed oracles and thus show that in these cases parallelization does not help. Section \ref{sec:offproof} provides the detailed proofs of this chapter and we conclude with some discussions in Section \ref{sec:offdis}.

\section{Oracle-based Offline Convex Optimization}\label{sec:orcmdl}
A typical convex optimization algorithm initially picks some point in the feasible set $\H$ and iteratively updates these points based on some local information about the function it calculates around this point. Examples of these type of procedures are gradient descent methods that uses first order gradient information, newton's method that uses second over hessian information, interior point methods and so on. In fact most procedures one can think of for optimization are based on iteratively updating based on some local information about the function at the point. In general computing the exact computational complexity of these methods is cumbersome and may not even be possible in full generality. To capture efficiency of optimization procedures in a general way, we consider the oracle based optimization problem and associated oracle complexity of the problem. To this end we first formally define an Oracle. Such models have been introduced and analyzed in \cite{NemirovskiYu78}.  In general, given an instance set $\Z$ indexing convex functions on hypothesis set $\bcH$, we use the term oracle to refer to any mapping $\Oracle:  \bcH \times \Z \mapsto \Info$ from hypothesis instance pairs to an answer $I \in \Info$, where set $\Info$ is some arbitrary information set $\Info$. However at this generality note that the information set $\Info$ could be all of $\Z$ and Oracle $\Oracle$ could be the identity mapping. This would defeat the purpose of introducing oracle models and associated oracle complexity. To address this issue we use the notion of  local oracle defined by Nemirovski and Yudin \cite{NemirovskiYu78} and whenever we use the term oracle we mean local oracle. Before we formally define a (local) oracle we first define the neighborhood set of a point $\h \in \bcH$. Given $\delta > 0$ and a point $\h \in \H$, the $\delta$-neighborhood of the point $\h$ is the set
$$
B_{\delta}(\h) = \left\{\h' \in \bcH : \norm{\h - \h'}_{\cH} \le \delta \right\}~.
$$

\begin{definition}\label{def:oracle}\index{oracle}
A (local) oracle $\Oracle: \bcH \times \Z \mapsto \Info$ is a mapping, which, given any point $\h \in \bcH$ and instances $\z , \z' \in \Z$ such that 
$$
\lim_{\delta \rightarrow 0}\ \sup_{\h' , \h'' \in B_{\delta}(\h) } |\ell(\h',\z)  - \ell(\h'' , \z')| = 0~,
$$
outputs answers that satisfy the equality $\Oracle(\z,\h) = \Oracle(\z',\h)$.
\end{definition}

The definition basically says if two instances $\z ,\z' \in \Z$ correspond to convex losses $\ell(\cdot,\z)$ and $\ell(\cdot,\z')$ respectively that are indistinguishable about some neighborhood of a point $\h \in \bcH$, then querying an oracle about the two instances at this point $\h$ leads to the same answer.  For a query with instance $\h \in \bcH$ at instance $\z \in \Z$, the oracle only provides local information about the function $\ell(\cdot,\z)$ at the point of query $\h$. To make the concept clearer we provide the following examples of (local) oracles commonly used in practice.

\begin{example}[Zero'th-order Oracle]
This is perhaps the simplest oracle that simply evaluates the given function at the query point and returns the value. That is
$
\Oracle(\h,\z) = \ell(\h,\z)
$.
Clearly $\Info \subseteq \reals$.
\end{example}

The zero'th order oracle captures bandit learning problems.

\begin{example}[First-order Oracle]
This is an Oracle that provides the sub-gradient of the function at query point. That is
$
\Oracle(\h,\z) = \nabla_{\h} \ell(\h,\z)
$.
Clearly $\Info \subset \bnsp^*$ the dual space of $\bnsp$ (the banach space containing $\bcH$).
\end{example}

\begin{example}[Second-order Oracle]
Consider the example where $\bcH \subseteq \reals^d$ and each $\z \in \Z$ corresponds to twice differentiable convex loss. In this case a second-order oracle is one that returns the hessian of the function at the query point. That is
$
\Oracle(\h,\z) = \nabla^2_{\h} \ell(\h,\z)
$.
\end{example}

As mentioned earlier since we are restricting ourselves to convex problems, we only need to consider deterministic algorithms. We now generically define Oracle Based Learning/Optimization Algorithm for a given oracle $\Oracle$.

\begin{definition}\index{oracle based learning algorithm}
For a given Oracle $\Oracle$, an ``Oracle Based Optimization/Learning Algorithm", $\Algo^\Oracle : \bigcup_{n \in \mbb{N}} \Info^n \rightarrow \bcH$ is a mapping from a sequences of oracle answers in $\Info$ to an element of hypothesis set $\bcH$.
\end{definition}

We now describe the oracle-based offline convex optimization protocol. Given $z \in \Z$ corresponding to convex function $\ell(\cdot,\z)$ (unknown to the learner) the optimization procedure is as follows :
\begin{center}\index{learning protocol!oracle-based!offline}
\fbox{
\begin{minipage}[t]{0.52\textwidth} 
{\bf Oracle-based Offline Optimization Protocol :  } \vspace{0.05in}

{\bf for } $t=1$ {\bf to }m \vspace{-0.12in}
\begin{itemize}
\item[] Pick hypothesis $\h_t \in \bcH$ for query \vspace{-0.12in}
\item[] Oracle provides answer $\Ans_t = \Oracle(\h_t,\z)$ \vspace{-0.12in}
\end{itemize}
{\bf end for}   
\end{minipage}
}
\end{center}

The goal is to solve the optimization problem 
$$
\argmin{\h \in \bcH} \ell(\h,\z)
$$
and at the end of $m$ steps the sub-optimality of the \Learner is given by 
$$
\ell(\h_m,\z) - \inf_{\h \in \bcH} \ell(\h,\z)~.
$$

Given an Oracle-based learning/optimization algorithm, $\Algo^\Oracle$ we shall use the short-hand,  $\Ans_1 = \Oracle\left(\h_1,\z\right)$ and further iteratively, use the notation $\Ans_t  = \Oracle(\h_t,\z)$ where of course each hypothesis $\h_t = \Algo^\Oracle(\Ans_{1},\ldots,\Ans_{t-1})$ is picked using the Oracle-based learning/optimization algorithm. We now define what it means to be offline optimizable using an oracle $\Oracle$.

\begin{definition}\index{oracle complexity!offline}
For a given class of convex functions corresponding to instance set $\Z$ and a given Oracle $\Oracle$, the offline oracle complexity of a given "Oracle based optimization/learning algorithm", $\Algo^\Oracle$, is defined as
$$
\moff(\epsilon,\Algo^\Oracle,\Z) = \inf\left\{m \in \mathbb{N} \ \middle| \ \sup_{\z \in \Z}\left\{ \ell(\Algo^\Oracle(\Ans_1,\ldots,\Ans_m),\z) - \inf_{\h \in \cH} \ell(\h,\z)\right\} \le \epsilon \right\}
$$
Further, for the given oracle $\Oracle$, the $\Oracle$-offline oracle complexity of the given offline optimization problem is defined as 
$\moff(\epsilon,\Z,\Oracle) = \inf_{\Algo^\Oracle} \moff(\epsilon,\Algo^\Oracle,\Z)$. 
\end{definition}

Roughly speaking, given $\epsilon > 0$, the oracle complexity of an algorithm is the minimum number of oracle answers needed by the algorithm to guarantee sub-optimality smaller than $\epsilon$ against any instance $\z \in \Z$. Further, the above definition basically implies that for any $\epsilon$, there exists an Oracle Based Optimization Algorithm, $\Algo^\Oracle$ that provides an $\epsilon$ sub-optimality for any $\z \in \Z$ within $m(\epsilon,\Z,\Oracle)$ steps. While the above definition of oracle complexity tells us what is the best (oracle) efficiency achievable for a problem for a given problem using oracle $\Oracle$, one might still wonder if using some other oracle one can improve efficiency. To address this issue we now define oracle complexity of a problem (independent of any particular oracle). 

\begin{definition}\index{oracle-based offline optimizable}
For a given class of convex functions corresponding to instance set $\Z$, the offline oracle complexity of the given offline optimization problem is defined as 
$$
\moff(\epsilon,\Z) = \inf_{\Oracle} \moff(\epsilon,\Z,\Oracle)~.
$$ 
Further we say that a given problem is oracle-based offline optimizable if there exists some Oracle $\Oracle$ such that $\forall \epsilon > 0$, $\moff(\epsilon,\Z,\Oracle) < \infty$.
\end{definition}

\section{Lower Bounding Oracle Complexity: Connections to Statistical Convex Learning} \label{sec:offline}
In this section we establish some connections between offline optimization of convex Lipschitz instance class $\Z_{\mrm{Lip}}(\X)$ and fat-shattering dimension and Rademacher complexities of the linear function class. We also establish interesting connection between convex optimization and statistical convex learning.

The following lemma which lower bounds oracle complexity by fat-shattering dimension of the corresponding linear function class is based on the proof technique for lower bounds on oracle complexity for offline optimization of convex Lipschitz function classes in \cite{NemirovskiYu78}.

\begin{lemma}\label{lem:offlow}
For any $\H \subset \B$ and $\X \subset \Bd$ that are convex and centrally symmetric,,
$$
\moff(\beta,\Z_{\mrm{Lip}}(\X)) \ge \fatstat_{2 \beta}(\F_{\mrm{lin}}(\H,\X))~~.
$$
\end{lemma}

Based on the above lemma and relationship between fat-shattering dimension and statistical Rademacher complexity, we prove the following theorem which bounds Rademacher complexity of the linear function class in terms using any polynomial upper bound on oracle complexity of offline optimization of the convex Lipschitz function class.

\begin{theorem}\label{thm:moff}
If there exists some $V > 0$ and $q \in (0,\infty)$ such that for any $\beta > 0$, 
$$
\moff(\beta,\Z_{\mrm{Lip}}(\X)) \le \left(\frac{V}{\beta}\right)^q
$$
then using the shorthand $r = \max\{2 , q\}$, we have that :
$$
\forall n \in \mathbb{N}, ~~\Radstat_{n}(\F_{\mrm{lin}}(\H,\X)) \le \frac{9 V \log^{1 + \frac{1}{r}}(n)}{n^{1/r}}   ~~.
$$
\end{theorem}

Notice that in the above theorem if $q \ge 2$ then $r = q$ and so like in Theorem \ref{thm:orcrad} we get a tight relationship between Rademacher complexity and oracle complexity of offline optimization. However this is not always true. In $d$ dimensional space where $d$ is small, using the centroid methods once can guarantee oracle complexity upper bound of $d \log(1/\epsilon)$. However it turns out that at least for the dual learning problems (when $\H = \Xd$) if the dimension is large enough, by the celebrated Dvoretzky-Roger's theorem (see for instance \cite{JohLin01}), we can infer that for all $\beta <\epsilon$, $\fatstat_{\beta}(\F_{\mrm{lin}}(\H,\X))$ is larger than order $1/\beta^2$. Later on in Section \ref{sec:noncrazy} as a side interest we can use this to conclude that for dual learning problems in high dimensional problems in appropriate spaces (most common ones), mirror descent is almost optimal even for offline optimization for the convex Lipschitz class!  

The above theorem establishes connections between offline convex optimization and statistical convex learning especially for supervised learning problems. The following corollary implies that for any $\H$ and $\X$, if we can find an efficient optimization algorithm for optimizing functions in class $\Z_{\mrm{Lip}}(\X)$, then we can find an efficient algorithm for statistical learning of supervised convex learning problem.

\begin{corollary}
If there exists some $V > 0$ and $q \in (0,\infty)$ such that for any $\beta > 0$, 
$$
\moff(\beta,\Z_{\mrm{Lip}}(\X)) \le \left(\frac{V}{\beta}\right)^q
$$
then using the shorthand $r = \max\{2 , q\}$, we have that there exists a learning algorithm $\Algo_{\mrm{aERM}}$ (that solves the empirical risk minimization problem approximately) which enjoys learning guarantee for the supervised convex learning problem in the statistical framework:
$$
\sup_{\D \in \Delta(\Z_{\mrm{sup}}(\X))} \Es{S \sim \D^n}{\L_\D(\Algo_{\mrm{aERM}}(S)) - \inf_{\h \in \H} \L_\D(\h) }   \le \frac{18 V \log^{1 + \frac{1}{r}}(n)}{n^{1/r}}   ~~.
$$
Furthermore, the number of oracle calls needed by this algorithm is bounded by $\frac{n^{1 + q/r}}{18^q \log^{q + \frac{q}{r}}(n)}$, and so number of oracle calls is at most order $n^2$, the sample size.
\end{corollary}
\begin{proof}
We provide a proof sketch of this simple corollary. We start by noting that by Proposition \ref{prop:radcontract}, learning rate for the empirical risk minimizer is bounded by Rademcher complexity of the associated linear class and by previous theorem we see that upper bound on oracle complexity for optimization of $\Z_{\mrm{Lip}}(\X)$ implies upper bound on Rademacher complexity. Thus combining these results we conclude the statement of the corollary. 
\end{proof}

First of all, while the above corollary we give for supervised convex learning problem with absolute loss, it can also be extended to logistic loss, hinge losses and basically any convex Lipschitz loss.

\section{Main Result : Optimality of Mirror Descent for Offline Convex Optimization}\label{sec:optoff}
The following theorem shows that for a large class of convex problems, mirror descent is near optimal even for offline convex optimization in terms of oracle complexity.

\begin{theorem}\label{thm:stattight2}
If there exists some constant $G \ge 1$ such that for any $1< p' < r \le2$ 
$$
C_{p'} \le G\ C^{\mrm{iid}}_{r}
$$
and it is true that for some $V > 0$ and $p \in (1,\infty)$,
$$
\moff(\epsilon,\Z_{\mrm{Lip}}(\X)) \le \left(\frac{V}{\epsilon}\right)^q
$$
where $q = p/(p-1)$. Then there exists function $\Psi$ and step size $\eta$ using which the stochastic mirror descent algorithm enjoys the guarantee
$$
\moff(\epsilon,\Algosmd,\Z(\X)) \le \left(\frac{54018\ G\ V\ \log^5 \left(\frac{1}{\epsilon}\right)}{\epsilon}\right)^{r}
$$
\end{theorem}
\begin{proof}
First step we use Theorem \ref{thm:moff}  to bound $\Radstat_n(\F_{\mrm{lin}}(\X))$. Next we use proceed as in the proof of \ref{thm:statlinbnd}. Next using Corollary \ref{cor:radtype}  we see that for any $p' < p$, $C^{\mrm{iid}}_{p'} \le \frac{60 \sqrt{2} V}{p - p'}$. Using the assumption that $C_{p} \le G\ C^{\mrm{iid}}_{p}$ we get bound on $C_{p'}$. However $C_p$ is an upper bound on $D_p$ and so using Corollary \ref{cor:offmd} we conclude the proof.
\end{proof}

The first thing we notice is that if $q \ge  2$ then $r = q$. Hence in this case the guarantee on the oracle complexity of offline optimization using mirror descent algorithm matches the upper bound on oracle complexity for the problem. Hence we see that for such a case mirror descent algorithm is near optimal even for offline convex optimization.  In particular, if for a  problem no algorithm can guarantee a better oracle complexity than order $1/\epsilon^2$ then from the above theorem we can conclude that mirror descent is near optimal (up to polynomials). 

Now let us consider the dual problem, that is the case when $\H = \Xd$. In this case, the celebrated Dvoretzky Roger's theorem (see \cite{JohLin01} for geometric interpretation we use here) implies that for any Banach space of dimension large enough (larger than $2^n$), there exists set of $n$ points, $\x_1,\ldots,\x_n \in \X$ such that for all $\epsilon_1,\ldots,\epsilon_n \in \{\pm1\}^n$,
$$
\frac{1}{\sqrt{n}} \le \norm{\frac{1}{n} \sum_{t=1}^n \epsilon_t \x_t}~.
$$
Hence we can conclude that for any $\beta > 0$, if $\mrm{dim}(\H) \rightarrow \infty$, then 
$$
\fat_\beta(\F_{\mrm{lin}}(\H,\Hd)) \ge \frac{1}{\beta^2}~.
$$
Hence by Lemma \ref{lem:offlow} we can conclude that oracle complexity of offline optimization for the convex Lipschitz class in the dual case when the dimensionality is large enough is lower bounded by $1/2 \beta^2$. Hence we can conclude that for large dimensional problems, mirror descent is near optimal even for offline optimization problem. Also as per the discussion in sub-section \ref{subsec:offline} we can also conclude that for statistical learning problem with any powerful oracle $\overline{\Oracle}$, mirror descent is near optimal for convex Lipschitz problems in the dual case when dimension is large enough.

The Dvoretzky Roger's theorem tells us very generally that for any large dimensional dual problem, mirror descent is always near optimal. However the dimension of the problem needs to be very large (of order exponential in $1/\epsilon$ or more) for this claim to be true. While at complete generality we can only ensure this much, for many commonly encountered problems one cash provide lower bounds for fat-shattering dimension of order $1/\epsilon^2$ or worse even when dimensionality is only as large as some polynomial in $1/\epsilon$. More specifically, for the case when the pair $\H/\X$ are either $\ell^d_2/\ell^d_2$ or $\ell^d_1/\ell^d_\infty$, fat-shattering dimension can be lower bound by $1/\epsilon^2$ as long as dimension is larger than order $1/\epsilon^2$. Hence we can conclude that for these cases, when we are dealing with high dimensional convex optimization problems when dimension is as large as order $1/\epsilon^2$ where $\epsilon$ is desired sub-optimality, then mirror descent is near optimal even for offline optimization problem. This result can also be extended to  $\ell_p/\ell_q$ pair case when $p \in [2,\infty)$ and Schatten norm counterparts of all the $\ell_p$ cases.  We would also like to point out that such high dimensional (relative to desired sub-optimality) problems are common in machine learning and high-dimensional statistics application for cases when dimension $d$ is large than the sample size $n$. In all these cases mirror descent is also near optimal for offline convex optimization.

\section{Statistical Learning With Distributed Oracles}\label{sec:statoptdis}
In the previous chapter, in section \ref{sec:opteff} we showed that for most reasonable statistical convex learning problems, stochastic mirror descent is near optimal not only in terms of learning rate but also in terms of number of gradient access (or even any other local oracle information) needed to guarantee excess risk smaller than some target value. Of course in the result we counted accessing gradient at a point of one point sample point as one oracle query. While we abstract off oracle as a black box, in reality the answer that the oracle computes, example gradient, is generally needs to be calculated by the learning algorithm itself. With the growth of cluster and grid computing and its influence in the design of machine learning and optimization algorithms one may wonder if having access to several machines can help speed up learning. For instance one could think of performing gradient descent or mirror descent on training sample but then calculate gradient of empirical average loss in parallel so that computation of in fact several gradients is done more or less in the time of computing one. In this section we try to formalize a more powerful oracle based statistical learning model that can capture such distributed computing scenarios. We show that for most reasonable cases in high dimension for statistical convex learning (of $\Z_{\mrm{Lip}}(\X)$) even having such more powerful oracles do not help and the single pass gradient based mirror descent is near optimal in terms of efficiency even when compared against learning algorithms that have access to these more powerful distributed oracles. In short we show that parallelization does not help for the case when instance set is $\Z_{\mrm{Lip}}(\X)$.


 More formally, in the oracle-based statistical learning model we consider, learner has access to a local distributed oracle $\overline{\Oracle} : \H \times \bigcup_{n \in \mathbb{N}} \Z^n \mapsto \Info$ that can provide answers when queried at a point about multiple instances simultaneously. Of course here local means the oracle is local w.r.t. the sequence of instances $\z_1,\ldots,\z_n$ it is queried on at any times.  The corresponding learning protocol is given below.

\begin{center}\index{learning protocol!oracle-based!distributed-statistical}
\fbox{
\begin{minipage}[t]{0.53\textwidth} 
{\bf Powerful Oracle-based Statistical Learning Protocol : \label{algo:obstat} } \vspace{0.03in}
\noindent \hspace{-0.1in} Sample $S = (\z_1,\ldots,\z_n) \sim \D^n$\\
{\bf for } $t=1$ {\bf to }m \vspace{-0.12in}
\begin{itemize}
\item[] Learner picks hypothesis $\h_t \in \bcH$\vspace{-0.12in}
\item[] Oracle provides answer $\Ans_t = \overline{\Oracle}(\h_t,S)$ \vspace{-0.12in}
\end{itemize}
{\bf end for}   
\end{minipage}
}
\end{center}

Now to study efficiency of statistical learning algorithms that have access to such distributed oracles, similar to $\moff$, the oracle complexity of offline optimization procedures, we introduce oracle complexity $\mstat$ of oracle-based statistical learning algorithms that have access to distributed oracles.

\begin{definition}\index{oracle complexity!distributed statistical}
For a given class of convex functions corresponding to instance set $\Z$ and a distributed Oracle $\overline{\Oracle}$, the distributed statistical oracle complexity of a given "Oracle based optimization/learning algorithm", $\Algo^{\overline{\Oracle}}$, is defined as
$$
\mstat(\epsilon,\Algo^{\overline{\Oracle}},\Z) = \inf\left\{m \in \mathbb{N} \ \middle| \ \sup_{\D \in \Delta(\Z)}\left\{ \L_\D(\Algo^{\overline{\Oracle}}(\Ans_1,\ldots,\Ans_m)) - \inf_{\h \in \cH} \L_\D(\h)\right\} \le \epsilon \right\}
$$
Further, for the given distributed oracle $\overline{\Oracle}$, the $\overline{\Oracle}$-distributed statistical oracle complexity of the given statistical convex learning problem is defined as 
$\mstat(\epsilon,\Z,\overline{\Oracle}) = \inf_{\Algo^{\overline{\Oracle}}} \mstat(\epsilon,\Algo^{\overline{\Oracle}},\Z)$. 
\end{definition}

Further we can define oracle complexity of a distributed oracle-based statistical convex learning problem (independent of any particular oracle) as follows : 

\begin{definition}
For a given class of convex functions corresponding to instance set $\Z$, the distributed statistical oracle complexity of the given distributed oracle-based statistical convex learning problem is defined as 
$$
\mstat(\epsilon,\Z) = \inf_{\overline{\Oracle}} \mstat(\epsilon,\Z,\overline{\Oracle})~.
$$ 
\end{definition}


Now to provide lower bounds on oracle complexity under this more powerful oracle based statistical learning scenario, we start by noticing that 
lower bounds on oracle complexity of offline optimization problems also provide lower bounds on oracle complexity for learning under the powerful oracle based statistical learning protocol. The reasoning for this is simple : If distribution $\D$ deterministically picked a single function then the problem is identical to offline optimization in the function class and querying oracle on a sequence or just the one function is exactly the same. The following proposition captures exactly this. 

\begin{proposition}
For any $\H$ and convex learning instance space $\Z$, we have that
$$
\mstat(\epsilon,\Z) \ge \moff(\epsilon,\Z)
$$
\end{proposition}

In the previous section specifically Theorem \ref{thm:stattight2} we showed that for several reasonable, high dimensional offline convex optimization problem over instance space $\Z_{\mrm{Lip}}(\X)$ mirror descent is in fact near optimal even when we consider distributed oracle-based statistical convex learning algorithm. That is we have the following corollary which is  trivial given previous proposition and Theorem \ref{thm:stattight2}. 

\begin{corollary}
If there exists some constant $G \ge 1$ such that for any $1< p' < r \le2$ 
$$
C_{p'} \le G\ C^{\mrm{iid}}_{r}
$$
and it is true that for some $V > 0$ and $p \in (1,\infty)$,
$$
\mstat(\epsilon,\Z_{\mrm{Lip}}(\X)) \le \left(\frac{V}{\epsilon}\right)^q
$$
where $q = p/(p-1)$. Then there exists function $\Psi$ and step size $\eta$ using which the stochastic mirror descent algorithm enjoys the guarantee
$$
\mstat(\epsilon,\Algosmd,\Z(\X)) \le \left(\frac{54018\ G\ V\ \log^5 \left(\frac{1}{\epsilon}\right)}{\epsilon}\right)^{r}
$$
\end{corollary}

However note that mirror descent is a gradient-based single pass algorithm that at each iteration only queries gradient at a single sample point. Hence we see that for these convex learning  problems over instance space $\Z_{\mrm{Lip}}(\X)$, stochastic mirror descent is near optimal and that having access to any kind of distributed local oracle does not help. That is in the worst case parallelization does not help.

\section{Detailed Proofs}\label{sec:offproof}

\begin{proof}[Proof of Lemma \ref{lem:offlow}]
The proof is essentially a vdery simple modification of the one provided by Nemirovski and Yudin in Section 4.4.2 of \cite{NemirovskiYu78}. We provide an abridged version here with the appropriate modifications needed to deal with the non-dual case with a few minor alterations to relate to fat-shattering dimension.  To prove the lower bound, we first start by picking $\x_1,\ldots,\x_m \in \X$ add $s_1,\ldots,s_m \in \reals$. Now the functions we shall consider are of form
$$
z_{\epsilon}(\h;(\x_1,s_1),\ldots,(\x_m,s_m)) = \max_{i \in [m]} \epsilon_i (\ip{\h}{- \x_i} + s_i) 
$$
where $\epsilon \in \{\pm1\}^m$. Notice that each $z_{\epsilon} \in \Z_{\mrm{Lip}}(\X)$. Note also that for any $\epsilon \in \{\pm 1\}^m$,
\begin{align}\label{eq:optimal}
- \inf_{\h \in \H} z_{\epsilon}(\h;(\x_1,s_1),\ldots,(\x_m,s_m)) = - \inf_{\h \in \H} \max_{i \in [m]} \epsilon_i (\ip{\h}{- \x_i} + s_i) =  \sup_{\h \in \H} \min_{i \in [m]} \epsilon_i( \ip{\h}{\x_i} - s_i)~~.
\end{align}

Remember that we want to show that for any $\Algo^\Oracle$ there exists a function that requires at least $m$ calls to some Oracle $\Oracle$ to ensure sub-optimality less than $\epsilon > 0$. The first thing we notice is that the family of functions we consider are piece wise linear and so any local oracle can give no more information that function value and gradient at point of query. Now given an Optimization algorithm $\Algo^\Oracle$ the exact function we shall use for the lower bound will be constructed in $m$ steps based on the algorithm and the choosen $\x_1,\ldots,\x_m \in \X$ add $s_1,\ldots,s_m \in \reals$. The procedure for constructing the function is given below :

\begin{center}
\fbox{
\begin{minipage}[t]{0.52\textwidth} 
{\bf Initialize }$I_1 = [m]$\\ \vspace{-0.1in} 
{\bf for } $t=1$ {\bf to }m 
\begin{itemize}
\item[] $\Algo^\Oracle$ picks $\h_t \in \cH$ for query \vspace{-0.1in}
\item[] $i(t) = \argmax{i \in I_t} \left\{\left|\ip{\h_t}{-\x_i} + s_i\right| \right\}$\vspace{-0.1in}
\item[]  $\epsilon_t = \left\{\begin{array}{ll}
+1 & \textrm{if }(\ip{\h_t}{-\x_{i(t)}} + s_{i(t)}) \ge 0\\
-1 & \textrm{otherwise}
\end{array}\right.
$ \vspace{-0.1in}
\item[] $I_{t+1} = I_t \setminus \{i(t)\}$\vspace{-0.1in} 
\item[] $\z^t(\h) = \max_{j \in [t]}\left\{\epsilon_j (\ip{\h}{-\x_{i(j)}} + s_{i(j)})  \right\}$\vspace{-0.1in} 
\item[] Return answer to query as $\Ans_t = \Oracle(\h_t,\z^t)$.\vspace{-0.1in} 
\end{itemize}
{\bf end for}   
\end{minipage}
}
\end{center}
The first thing we notice about $\z^m$ is that it is of the form :
\begin{align}\label{eq:form}
\z^m(\cdot) = \z_{\epsilon}(\cdot;(\x_{i(1)},s_{i(1)}),\ldots,(\x_{i(m)},s_{i(m)}))
\end{align}
where $\epsilon_1,\ldots,\epsilon_m$ are given by the procedure above.
Next, $\z^m$ is such that, for any $i \in [m]$ and any local oracle $\Oracle$,
$$
\Oracle(\h_i,\z^m) = \Oracle(\h_i,\z^i)
$$
hence the $\h_1,\ldots,\h_m$ returned by the algorithm when it is presented with function $f^m$ is the same as the corresponding ones in the above procedure. Finally, by the way the functions are constructed (specifically the way $\epsilon_t$ is picked), 
$$
\z^m(\h_m) \ge 0
$$
Hence we conclude that
\begin{align*}
\z^m(\h_m) - \inf_{\h \in \H} \z^m(\h) & \ge - \inf_{\h \in \H} \z^m(\h)\\
& = - \inf_{\h \in \H} z_{\epsilon}(\h;(\x_{i(1)},s_{i(1)}),\ldots,(\x_{i(m)},s_{i(m)})) & \textrm{(by Eq. \ref{eq:form}) }\\
& = \sup_{\h \in \H} \min_{i \in [m]} \epsilon_i( \ip{\h}{\x_i} - s_i) & \textrm{(by Eq. \ref{eq:optimal})} \\
& \ge \inf_{\epsilon \in \{\pm1\}^m} \sup_{\h \in \H} \min_{i \in [m]} \epsilon_i( \ip{\h}{\x_i} - s_i)~~.
\end{align*}
Furthermore note that the choice of $\x_1,\ldots,\x_m \in \X$ and $s_1,\ldots,s_m$ are arbitrary. Hence we can conclude that for any $\beta > 0$, if for any $m$ there exists $\x_1,\ldots,\x_m \in \X$ and $s_1,\ldots,s_m \in \reals$ such that 
$$
\inf_{\epsilon \in \{\pm1\}^m} \sup_{\h \in \H} \min_{i \in [m]} \epsilon_i( \ip{\h}{\x_i} - s_i) > \beta~~,
$$
then no oracle based algorithm can achieve sub-optimality smaller than $\beta$ in $m$ or less steps. However note that this is exactly the definition of fat-shattering dimension at scale $2 \beta$ (Definition \ref{def:statfat}) for the linear class given by 
$$
\F = \{\x \mapsto \ip{\h}{\x} : \h \in \H \}~~.
$$
Hence we conclude the lemma statement.
\end{proof}

\begin{proof}[Proof of Theorem \ref{thm:moff}]
The first inequality is a direct consequence of Lemma \ref{lem:offlow}. For the upper bound on the Rademacher complexity note that, for any $n \in \mathbb{N}$, by the refined Dudley integral bound we have that : 
\begin{align*}
\Radstat_{n}(\F_{\mrm{lin}}(\H,\X)) \le \inf_{\alpha > 0}\left\{4 \alpha + 10 \int_{\alpha}^{1} \sqrt{\frac{\fatstat_\beta(\F_{\mrm{lin}}(\H,\X))\ \log(n)}{n}} d \beta\right\}
\end{align*}
We now divide the analysis into two cases, first where $q \in [2,\infty)$ and next where $q \in (0,2)$. We start for the case when $q \in [2,\infty)$ and see that using the assumption of this theorem and Lemma \ref{lem:offlow} we see that for any $q \in [2, \infty)$: 
\begin{align*}
\Radstat_{n}(\F_{\mrm{lin}}(\H,\X)) & \le \inf_{\alpha > 0}\left\{4 \alpha + \sqrt{\frac{V^q \log(n)}{n}} 10 \int_{\alpha}^{1} \frac{1}{\beta^{\frac{q}{2}}}  d \beta\right\}\\
& \le \inf_{\alpha > 0}\left\{4 \alpha + \sqrt{\frac{V^q \log(n)}{n}} 10 \int_{\alpha}^{1} \frac{\left(\frac{q}{2} - 1\right) \log(1/\beta) + 1}{\beta^{\frac{q}{2}}}  d \beta\right\}\\
& \le \inf_{\alpha > 0}\left\{4 \alpha + 10 \sqrt{\frac{V^q \log(n)}{n}} \left[\frac{\log(\beta)}{\beta^{\frac{q}{2} - 1}} \right]_{\alpha}^{1}\right\}\\
& \le \inf_{\alpha > 0}\left\{4 \alpha + 10 \sqrt{\frac{V^q \log(n)}{n}} \frac{\log(1/\alpha)}{\alpha^{\frac{q}{2} - 1}} \right\}\\
& \le \frac{9 V \log^{1 + \frac{1}{q}}(n)}{n^{1/q}}   
\end{align*}
where in the last step above we used the value $\alpha =  \frac{V \log^{1/q}(n)}{n^{1/q}}$. 

\noindent Now we turn our attention to case when $q \in (0,2)$. As for this case we simply note that $\left(\frac{V}{\epsilon}\right)^q \le \left(\frac{V}{\epsilon}\right)^2$ and so using the case when $q = 2$ we conclude that 
$$
\Radstat_{n}(\F_{\mrm{lin}}(\H,\X))  \le \frac{9 V \log^{3/2}(n)}{n^{1/2}}   
$$
This concludes the proof.
\end{proof}

\section{Discussion}\label{sec:offdis}
The key result of this chapter is that for most reasonable cases, if the dimension of the vector space $\B$ is large enough, then stochastic mirror descent algorithm is near optimal even for offline convex optimization for instance space $\Z_{\mrm{Lip}}(\X)$. We further show that for statistical convex learning problems over instance space $\Z_{\mrm{Lip}}(\X)$ even when we consider learning algorithm that use distributed oracles, (ie. uses distributed computation of local oracle information like gradients etc.) mirror descent is still near optimal and parallelization does not help.

%% file: future.tex
\chapter[\Large Conclusion and Future Work]{Conclusion and Future Work} \label{chp:future}
In this section we delineate some important questions that are related to the work in this dissertation. We also discuss some further directions of research. Finally we summarize and give some concluding remarks.

\section{Open Problems}

\subsection{Online Optimization and Stability}
Another direction yet to be explored is the question of online learnability in the general learning setting. In the statistical paradigm we used the tool of stability and properties of asymptotic empirical minimizer of learning rule to determine learnability for the general setting. We would like to explore the problem of online learnability in the general setting on similar lines. 

\begin{question}
Are there properties analogous to stability and AERM property in the online paradigm that guarantee online learnability in the general setting for learning ?
\end{question}

\begin{question}
Can we provide a generic strategy for \Learner in the online learning framework that guarantees diminishing regret whenever the problem is learnable ?
\end{question}

\subsection{Upper Bounding Oracle Complexity in Terms of Fat-Shattering Dimension}
In the second part of the dissertation especially in Chapter \ref{chp:stat opt} we showed that the oracle complexity of offline convex optimization problem, $\moff(\epsilon,\Z_{\mrm{Lip}}(\X))$ is lower bounded by the fat shattering dimension of the associated linear class $\F_{\mrm{lin}}(\H,\X)$. Using this we showed that at least for supervised learning problems, if one can efficiently optimize convex function corresponding to the class $\Z_{\mrm{Lip}}(\X)$ then one can also statistically learn and efficiently. In the same chapter we also showed that for most reasonable cases, if dimension is large enough, then mirror descent is near optimal even for offline optimization. Using the results in the thesis one can also conclude that for these large dimensional cases, $\moff(\epsilon,\Z_{\mrm{Lip}}(\X))$ can also be upper bounded by $\tilde{O}(\fat_{\epsilon})$. Can this result be generalized and can we show the upper bound on oracle complexity in terms of fat shattering dimension always hold? We pose the  question.

\begin{question}
Is it always true that 
$$
\moff(\epsilon,\Z_{\mrm{Lip}}(\X)) \le \fat_{c \epsilon}(\F_{\mrm{lin}}(\H,\X))
$$
where $c$ is some universal constants? If it is true, then can one give an optimization algorithm that has oracle complexity bounded by fat-shattering dimension?
\end{question}

\section{Further Directions}
The thesis mainly covers the story of learning from the perspective of optimization and answers questions about learnability. However there were a few results that emerged out of results and techniques provided in this thesis and we delineate a few below. 

We mainly considered two extreme scenarios while considering statistical and online learning framework. In statistical framework, instances were sampled iid and in the online learning framework instances picked adversarially. It is interesting to consider the intermediate scenarios where learner is not faced with a completely worst case adversary but is also not faced with iid sampling of instances. Maybe adversary might have  some constraints on instances that can be chosen or choses instances in a stochastic way that is more complex than iid sampling. Such a scenario is analyzed in \cite{RakSriTew11b} based on techniques in Chapter \ref{chp:online}. Another orthogonal way in which the results in the chapter were extended was to games beyond online learning to include games like Blackwell's approachability, calibration etc. in \cite{RakSriTew11,FosRakSriTew11}. 

Results in chapters \ref{chp:cnvxonline} and \ref{chp:oraclecomplexitylearn} influenced and shaped  the work in paper \cite{CotShaSreSri11} where we showed how one can make appropriate changes to stochastic mirror descent 
and accelerated methods to include mini-batching (breaking sample into blocks and instead of updating in each step with single gradient update with average of the block of gradients). We showed that this helped in guaranteeing better time complexity with parallelization of these methods.

\section{Summary}
An important question in the field of theoretical machine learning is that of learnability and learning rates.  We have explored this question for various learning problems in both statistical and online learning frameworks. 
In the statistical learning framework we provide the first general characterization of learnability in the general setting using the notion of stability of learning algorithms. We also provided a generic algorithm for learning in the statistical learning framework. As for the problem of learnability in the online framework while we don't yet have a complete picture we introduced various complexity measures analogous to the ones in statistical learning framework. We also provide characterization of online learnability for real valued supervised learning problem.

An integral part of machine learning is optimization. While the question of learnability and learning rates are central to machine learning theory, from a practical point of view one would like to consider problems that are efficiently learnable. To address this issue in a general way, we considered convex learning and optimization problems in both statistical and online learning framework. We used the notion of oracle complexity to address issue of efficiency. For the online learning problems, we showed mirror descent is universal and near optimal. That is whenever a convex problem is online learnable, it is learnable with near optimal rates using mirror descent. Since mirror descent is a first order method (sub-gradient based) we could infer that for online learning scenario mirror descent is near optimal in terms of both rates and oracle complexity.  We also explored connections between learning in the various frameworks and oracle based optimization. For the statistical convex learning problem, unlike online setting, in general it is not true that mirror descent is universal. However we saw that for problems we would encounter in practical applications though, this was in fact the case. Mirror descent would indeed be near optimal. We also saw that for certain offline optimization problems in high enough dimensions, mirror descent can again shown to be near optimal.

We expect the work to provide a better understanding of learning algorithms especially from the perspective of optimization. While it is common that for machine learning practitioners optimization is often an after thought and is in a sense mainly a computational issue, through this work we would like to stress that learning can be seen as optimization and should in fact be seen as so. On the other hand, we also show some strong connections between optimization and showed how tools from learning theory can be used to prove results on optimization. Hence we would also like to stress overall the strong and inevitable connections between the two.

In this work we also used several concepts from the theory of Banach space geometry. It would certainly be interesting to see if more connections can be made and techniques from Banach space geometry be used to prove more results about learning and optimization.

%% file: appendix1.tex
\chapter[\Large Relating Various Complexity Measures : Statistical Learning]{Relating Various Complexity Measures : Statistical Learning}\label{app:complexity}

\section{The Refined Dudley Integral: Bounding Rademacher
  Complexity with $L_2$ Covering Numbers}\label{sec:radcov}

We shall find it simpler here to use the empirical Rademacher
complexity for a given sample $x_1,\ldots,x_n$ \cite{BartlettMe02}:
\begin{equation}
  \label{eq:empraddef}
\widehat{\mc{R}}_n(\H) =  \Es{\sigma \sim \text{Unif}(\{\pm 1\}^n)}{\sup_{h \in \H} \frac{1}{n} \left|\sum_{i=1}^n  h(x_i) \sigma_i \right|}
\end{equation}
and the $L_2$ covering number at scale $\epsilon > 0$ specific to a sample $x_1,\ldots,x_n$, denoted by $N_2\left(\epsilon,\mathcal{F},(x_1,\ldots,x_n)\right)$ as the size of a minimal cover $\mathcal{C}_\epsilon$ such that 
$$\forall f \in \mathcal{F}, \exists f_\epsilon \in
\mathcal{C}_\epsilon\ \textrm{s.t.}~\sqrt{\frac{1}{n}
  \sum_{i=1}^n (f(z_i) - f_\epsilon(z_i))^2} \le
\epsilon~.$$ We will also denote $\widehat{\En}[f^2] = \frac{1}{n} \sum_{i=1}^n f^2(x_i)$.

We state our bound in terms of the empirical Rademacher
complexity and covering numbers.  Taking a supremum over samples of
size $n$, we get the same relationship between the worst-case
Rademacher complexity and covering numbers, as is used in Section
\ref{sec:rad}.

\begin{lemma}\label{lem:cover}
For any function class $\mathcal{F}$ containing functions $f : \mathcal{X} \mapsto \mathbb{R}$, we have that
$$
\widehat{\Rad}_n(\mathcal{F}) \le \inf_{\alpha \ge 0}\left\{ 4 \alpha + 10 \int_{\alpha}^{\sup_{f \in \F} \sqrt{\widehat{\En}[f^2]}} \sqrt{\frac{\log \mathcal{N}\left(\epsilon,\mathcal{F},(x_1,\ldots,x_n)\right)}{n}} d \epsilon \right\} \ .
$$
\end{lemma}
\begin{proof}
Let $\beta_0 = \sup_{f \in \F}\sqrt{ \widehat{\En}[f^2]} $ and for any $j \in \mathbb{Z}_+$ let $\beta_j = 2^{-j} \sup_{f \in \F} \sqrt{\widehat{\En}[f^2]}$. The basic trick here is the idea of chaining. For each $j$ let $T_i$ be a (proper) $L_2$-cover  at scale $\beta_j$ of $\mathcal{F}$ for the given sample. For each $f \in \mathcal{F}$ and $j$, pick an $\hat{f}_i \in T_i$ such that $\hat{f}_i$ is an $\beta_i$ approximation of $f$. Now for any $N$, we express $f$ by chaining as
$$
f = f - \hat{f}_N + \sum_{i=1}^{N}\left(\hat{f}_{i} - \hat{f}_{i-1}\right)
$$
where $\hat{f}_0 = 0$. Hence for any $N$ we have that
\begin{align}
\widehat{\Rad}_n(\mathcal{F}) &= \frac{1}{n}\Es{\sigma}{\sup_{f \in \mathcal{F}} \sum_{i=1}^n  \sigma_i \left(f(\x_i) - \hat{f}_N(\x_i) + \sum_{j=1}^N \left( \hat{f}_j(\x_i) - \hat{f}_{j-1}(\x_i)\right) \right)} \notag \\
& \le \frac{1}{n}\Es{\sigma}{\sup_{f \in \mathcal{F}} \sum_{i=1}^n  \sigma_i \left(f(\x_i) - \hat{f}_N(\x_i) \right)} + \sum_{j=1}^N \frac{1}{n} \Es{\sigma}{\sup_{f \in \mathcal{F}} \sum_{i=1}^n \sigma_i \left( \hat{f}_j(\x_i) - \hat{f}_{j-1}(\x_i)\right)} \notag \\
& \le \frac{1}{n} \sqrt{\sum_{i=1}^n \sigma_i^2}\ \sup_{f \in \mathcal{F}} \sqrt{ \sum_{i=1}^n (f(x_i) - \hat{f}_N(x_i)^2} + \sum_{j=1}^N \frac{1}{n} \Es{\sigma}{\sup_{f \in \mathcal{F}} \sum_{i=1}^n \sigma_i \left( \hat{f}_j(\x_i) - \hat{f}_{j-1}(\x_i)\right)} \notag\\
& \le \beta_N +  \sum_{j=1}^N \frac{1}{n} \Es{\sigma}{\sup_{f \in \mathcal{F}} \sum_{i=1}^n \sigma_i \left( \hat{f}_j(\x_i) - \hat{f}_{j-1}(\x_i)\right)} \label{eq:chain}
\end{align}
where the step before last is due to Cauchy-Shwarz inequality and $\mathbf{\sigma} = \left[ \sigma_1, ...,\sigma_n \right]^\top$. Now note that 
\begin{align*}
\frac{1}{n} \sum_{i=1}^n (\hat{f}_j(x_i) - \hat{f}_{j-1}(x_i))^2 & = \frac{1}{n} \sum_{i=1}^n \left( (\hat{f}_j(x_i)) - f(x_i)) + (f(x_i) - \hat{f}_{j-1}(x_i)) \right)^2 \\
& \le \frac{2}{n} \sum_{i=1}^n \left( \hat{f}_j(x_i)) - f(x_i)\right)^2 +  \frac{2}{n} \sum_{i=1}^n \left(f(x_i) - \hat{f}_{j-1}(x_i) \right)^2 \\
& \le 2 \beta_j^2 + 2 \beta_{j-1}^2 = 6 \beta_j^2 \ .
\end{align*}
Now Massart's finite class lemma \cite{Massart00} states that if for any function class $\mathcal{G}$, $\sup_{g \in \mathcal{G}} \sqrt{\frac{1}{n} \sum_{i=1}^n g(x_i)^2 } \le R$, then 
$\widehat{\Rad}_{n}(\mathcal{G}) \le \sqrt{\frac{2 R^2 \log(|\mathcal{G}|)}{n}}$. Applying this to function classes $\{f - f' : f \in T_j ,\ f' \in T_{j-1}\}$ (for each $j$) we get from \eqref{eq:chain} that for any $N$,
\begin{align*}
\widehat{\Rad}_n(\mathcal{F}) & \le \beta_N +  \sum_{j=1}^N  \beta_j \sqrt{\frac{12 \log(|T_j|\ |T_{j-1}|)}{n}} \\
& \le \beta_N +   \sum_{j=1}^N  \beta_j \sqrt{\frac{24 \ \log\ |T_j|}{n}} \\
& \le \beta_N +  10 \sum_{j=1}^N  (\beta_j - \beta_{j+1}) \sqrt{\frac{\log\ |T_j|}{n}} \\
& \le \beta_N +  10 \sum_{j=1}^N  (\beta_j - \beta_{j+1}) \sqrt{\frac{\log\ \mathcal{N}\left(\beta_j,\mathcal{F},(x_1,\ldots,x_n)\right)}{n}} \\
& \le \beta_N +  10 \int_{\beta_{N+1}}^{\beta_{0}} \sqrt{\frac{\log\ \mathcal{N}\left(\epsilon,\mathcal{F},(x_1,\ldots,x_n)\right)}{n}} d \epsilon
\end{align*}
where the third step is because $2 (\beta_j - \beta_{j+1}) = \beta_j$ and we bounded $\sqrt{24}$ by $5$. Now for any $\alpha > 0$, pick $N = \sup\{j : \beta_j > 2 \alpha\}$. In this case we see that by our choice of $N$, $\beta_{N+1} \le 2 \alpha$ and so $\beta_N = 2 \beta_{N+1} \le 4 \epsilon$. Also note that since $\beta_{N} > 2 \alpha$, $\beta_{N+1} = \frac{\beta_N}{2}> \alpha$. Hence we conclude that
 \begin{align*}
\widehat{\Rad}_n(\mathcal{F}) & \le 4 \alpha +  10 \int_{\alpha}^{\sup_{f \in \F} \sqrt{\widehat{\En}[f^2]}} \sqrt{\frac{\log\ \mathcal{N}\left(\epsilon,\mathcal{F},(x_1,\ldots,x_n)\right)}{n}} d \epsilon \ .
\end{align*}
Since the choice of $\alpha$ was arbitrary we take an infimum over $\alpha$.
\end{proof}

\section{Bounding $L_\infty$ covering number by  Fat-shattering Dimension}\label{sec:covfat}
The following proposition and lemma are standard in statistical learning theory and their proof can be found for instance in \cite{AlonBeCeHa93}. We provide the statement and the proof of the Lemma for completeness and so that we can state it in the exact form it is used in, in this work.

\begin{proposition}\label{prop:VC_multiclass}
Let $\mathcal{H} \subseteq \{0,\ldots,k\}^{\mathcal{X}}$ be a class of functions with $\mathrm{fat}_2 = d$. Then, we have,
$$
\mathcal{N}_\infty(1/2,\mathcal{H},n) \le \sum_{i=0}^d {n \choose  i} k^i
$$
and specifically for $n \ge d$ this gives,
$$
\mathcal{N}_\infty(1/2,\mathcal{H},n) \le \left(\frac{e k n }{d}\right)^d \ .
$$
\end{proposition}

\begin{lemma}\label{lem:covfat}
For any function class $\mathcal{H}$ bounded by $B$ and any $\alpha >0$ such that $\mathrm{fat}_\alpha < n$, we have,
$$
\mathcal{N}_\infty(\alpha,\mathcal{H},n) \le \left(\frac{2 e B n}{\alpha\ \mathrm{fat}_\alpha(\mathcal{H})} \right)^{\mathrm{fat}_\alpha(\mathcal{H})} \ .
$$
\end{lemma}
\begin{proof}
For any $\alpha > 0$, define an $\alpha$-discretization of the $[-B,B]$ interval as  $B_\alpha = \{-B+\alpha /2 , -B + 3 \alpha /2 , \ldots, -B+(2k+1) \alpha /2, \ldots \}$ for $0\leq k$ and $(2k+1)\alpha \leq 4 B$. Also for any $a \in [-B,B]$, define $\lfloor a \rfloor_\alpha = \argmin{r \in B_\alpha} |r - a|$ with ties being broken by choosing the smaller discretization point. For a  function $h:\X\mapsto [-B,B]$ let the function $\lfloor h \rfloor_\alpha$ be defined pointwise as $\lfloor h(x) \rfloor_\alpha$, and let $\lfloor \mathcal{H} \rfloor_\alpha = \{\lfloor h \rfloor_\alpha : h\in\mathcal{H}\}$. First, we prove that $\mathcal{N}_\infty(\alpha, \mathcal{H}, \{x_i\}_{i=1}^n) \leq \mathcal{N}_\infty(\alpha/2, \lfloor \mathcal{H} \rfloor_\alpha, \{x_i\}_{i=1}^n)$. Indeed, suppose the set $V$ is a minimal $\alpha/2$-cover of $\lfloor \mathcal{H} \rfloor_\alpha$ on $\{x_i\}_{i=1}^n$. That is,
	$$
	\forall h_{\alpha} \in \lfloor \mathcal{H} \rfloor_\alpha,\ \exists \v \in V \  \mathrm{s.t.}  ~~~~ |v_i -  h_{\alpha}(x_i)| \leq \alpha/2\ .
	$$
	Pick any $h\in\mathcal{H}$ and let $h_\alpha = \lfloor h \rfloor_\alpha$. Then $\|h-h_\alpha \|_\infty \leq \alpha/2$ and for any $i \in [n]$
	$$\left|h(x_i)- v_i\right| \leq \left|h(x_i)- h_{\alpha}(x_i)\right| + \left|h_\alpha(x_i)- v_i\right| \leq \alpha,$$
	and so $V$ also provides an $L_\infty$ cover at scale $\alpha$.
	
	We conclude that $\mathcal{N}_\infty(\alpha, \mathcal{H}, \{x_i\}_{i=1}^n) \leq \mathcal{N}_\infty(\alpha/2, \lfloor \mathcal{H} \rfloor_\alpha, \{x_i\}_{i=1}^n) =  \mathcal{N}_\infty( 1/2, {\mathcal G}, \{x_i\}_{i=1}^n) $ where $\mathcal{G} = \frac{1}{\alpha} \lfloor \mathcal{H} \rfloor_\alpha$. The functions of ${\mathcal G}$ take on a discrete set of at most $\lfloor 2B/\alpha \rfloor + 1$ values. Obviously, by adding a constant to all the functions in ${\mathcal G}$, we can make the set of values to be $\{0, \ldots, \lfloor 2B/\alpha \rfloor \}$. We now apply Proposition~\ref{prop:VC_multiclass} with an upper bound $\sum_{i=0}^d {n\choose i} k^i \leq \left(\frac{ekn}{d}\right)^d$ which holds for any $n>d$. This yields  $\mathcal{N}_\infty(1/2, {\mathcal G}, \{x_i\}_{i=1}^n) \leq \left(\frac{2e B n}{\alpha \fat_{2}({\mathcal G})} \right)^{\fat_{2}({\mathcal G})}$. 
	
	It remains to prove $\fat_2({\mathcal G}) \leq \fat_\alpha (\mathcal{H})$, or, equivalently (by scaling) $\fat_{2\alpha} (\lfloor \mathcal{H} \rfloor_\alpha)  \leq \fat_\alpha (\mathcal{H})$. 
	To this end, suppose there exists a set $\{x_{i=1}^n\}$ of size $d=\fat_{2\alpha}(\lfloor \mathcal{H} \rfloor_\alpha)$ such that there is an witness  $s_1,\ldots,s_n$ with
	$$
	\forall \epsilon \in \{\pm1\}^d , \ \exists h_{\alpha} \in \lfloor \mathcal{H} \rfloor_\alpha \ \ \ \textrm{s.t. } \forall i \in [d], \  \epsilon_i (h_{\alpha}(x_i) - s_i) \ge \alpha\ .
	$$
	Using the fact that for any $h\in\mathcal{H}$ and $h_\alpha = \lfloor h \rfloor_\alpha$ we have $\|h-h_\alpha \|_\infty \leq \alpha/2$, it follows that
	$$
	\forall \epsilon \in \{\pm1\}^d , \ \exists h \in \mathcal{H} \ \ \ \textrm{s.t. } \forall i \in [d], \  \epsilon_i (h(x_i) - s_i) \ge \alpha/2\ .
	$$
	That is, $s_1,\ldots,s_n$ is a witness to $\alpha$-shattering by $\mathcal{H}$. Thus for any $\{x_i\}_{i=1}^n$, as long as $n > \mathrm{fat}_{\alpha}$
	$$\mathcal{N}_\infty(\alpha, \mathcal{H}, \{x_i\}_{i=1}^n\}) \leq \mathcal{N}_\infty( \alpha/2, \lfloor \mathcal{H} \rfloor_\alpha, \{x_i\}_{i=1}^n) \leq \left(\frac{2e B n}{\alpha \fat_{2 \alpha}(\lfloor \mathcal{H} \rfloor_\alpha)} \right)^{\fat_{2\alpha} (\lfloor \mathcal{H} \rfloor_\alpha) } \leq \left(\frac{2e B n}{\alpha \fat_\alpha}\right)^{\fat_{\alpha} (\mathcal{H}) } \ .$$
\end{proof}

\section{Relating Fat-shattering Dimension and Rademacher complexity}\label{sec:fatrad}

The following lemma upper bounds the fat-shattering dimension at scale
$\epsilon \ge \Rad_n(\H)$ in terms of the Rademacher complexity of
the function class. The proof closely follows the arguments of
Mendelson \cite[discussion after Definition 4.2]{Mendelson02}.

\begin{lemma}\label{lem:fatrad}
For any hypothesis class $\H$, any sample size $n$ and any $\epsilon > \Rad_n(\H)$ we have that 
$$
\mathrm{fat}_\epsilon(\H) \le \frac{4\ n\ \Rad_n(\H)^2}{\epsilon^2} \ .
$$
\end{lemma}
In particular, if $\Rad_n(\H) = \sqrt{R/n}$ (the typical case), then
$\mathrm{fat}_\epsilon(\H) \leq R/\epsilon^2$.
\begin{proof}
  Consider any $\epsilon \ge \Rad_n(\H)$. Let
  $x^*_1,\ldots,x^*_{\fat_\epsilon}$ be the set of $\fat_\epsilon$
  shattered points. This means that there exists
  $s_1,\ldots,s_{\fat_\epsilon}$ such that for any $J \subset
  [\fat_\epsilon]$ there exists $h_J \in \H$ such that $\forall i \in
  J, h_J(x_i) \ge s_i + \epsilon$ and $\forall i \not\in J, h_J(x_i)
  \le s_i - \epsilon$.  Now consider a sample $x_1,\ldots,x_{n'}$ of
  size $n' = \lceil \frac{n}{\fat_\epsilon}\rceil \fat_\epsilon$,
  obtained by taking each $x^*_i$ and repeating it $\lceil
  \frac{n}{\fat_\epsilon}\rceil$ times, i.e.~$x_i =
  x^*_{\lfloor \frac{i}{\fat_\epsilon} \rfloor}$. Now, following
  Mendelson's arguments:
\begin{align*}
\Rad_{n'}(\H) & \ge \Es{\sigma \sim \mathrm{Unif}\{\pm 1\}^{n'}}{\frac{1}{n'} \sup_{h \in \H}\left| \sum_{i=1}^{n'} \sigma_i h(x_i)\right|} \\
& \ge \frac{1}{2} \Es{\sigma \sim \mathrm{Unif}\{\pm 1\}^{n'}}{\frac{1}{n'} \sup_{h, h' \in \H}\left| \sum_{i=1}^{n'} \sigma_i (h(x_i) - h'(x_i))\right|} ~~~~~~~~~~~ \textrm{(triangle inequality)}\\
& = \frac{1}{2} \Es{\sigma \sim \mathrm{Unif}\{\pm 1\}^{n'}}{\frac{1}{n'} \sup_{h, h' \in \H}\left| \sum_{i=1}^{\fat_\epsilon} \left(\sum_{j=1}^{\lceil n/\fat_\epsilon \rceil} \sigma_{(i-1) \fat_\epsilon + j}\right) \left(h(x^*_i) - h'(x^*_i)\right)\right|}\\
& \ge \frac{1}{2} \Es{\sigma \sim \mathrm{Unif}\{\pm 1\}^{n'}}{\frac{1}{n'}\left| \sum_{i=1}^{\fat_\epsilon} \left(\sum_{j=1}^{\lceil n/\fat_\epsilon \rceil} \sigma_{(i-1) \fat_\epsilon + j}\right) \left(h_{R}(x^*_i) - h_{\overline{R}}(x^*_i)\right)\right|}
\intertext{
where for each $\sigma_1,\ldots,\sigma_{n'}$, $R \subseteq
[\fat_\epsilon]$ is given by $R = \left\{i \in [\fat_\epsilon]
  \middle|  \mathrm{sign}\left(\sum_{j=1}^{\lceil n/\fat_\epsilon
      \rceil} \sigma_{(i - 1) \lceil n/\fat_\epsilon\rceil + j}\right)
  \ge 0 \right\}$, $h_R$ is the function in $\H$ that
$\epsilon$-shatters the set $R$ and $h_{\overline{R}}$ be the function
that shatters the complement of set $R$.}
& \ge  \frac{1}{2} \Es{\sigma \sim \mathrm{Unif}\{\pm 1\}^{n'}}{\frac{1}{n'} \sum_{i=1}^{\fat_\epsilon} \left|\sum_{j=1}^{\lceil n/\fat_\epsilon \rceil} \sigma_{(i-1) \fat_\epsilon + j}\right| 2 \epsilon }\\
& \ge  \frac{\epsilon}{n'} \sum_{i=1}^{\fat_\epsilon} \Es{\sigma \sim \mathrm{Unif}\{\pm 1\}^{n'}}{ \left|\sum_{j=1}^{\lceil n/\fat_\epsilon \rceil} \sigma_{(i-1) \fat_\epsilon + j}\right|}\\
& \ge  \frac{\epsilon\ \fat_\epsilon}{n'}  \sqrt{\frac{\lceil n/\fat_\epsilon \rceil}{2}} \hspace{2.5in} \textrm{(Khintchine's inequality)}\\
& =  \sqrt{\frac{\epsilon^2\ \fat_\epsilon}{2\ n'}}.
\end{align*}
We can now conclude that:
\begin{align*}
\fat_\epsilon \le \frac{2 n' \Rad^2_{n'}(\H)}{\epsilon^2} \le \frac{4 n \Rad^2_{n}(\H)}{\epsilon^2}
\end{align*}
where last inequality is because Rademacher complexity decreases with
increase in number of samples and $n \leq n' \le 2n$ (because $\epsilon \ge
\Rad_n(\H)$ which implies that $\fat_\epsilon < n$). \qedhere
\end{proof}